%% file: main.tex
\documentclass[11pt]{article} 
\pdfoutput=1

\usepackage[margin=1in]{geometry}

\usepackage{amssymb,amsfonts,amsmath,amsthm,amscd,dsfont,mathrsfs,mathtools,microtype,nicefrac,pifont}
\usepackage{setspace}
\usepackage{subfigure}
\usepackage{mathrsfs}
\usepackage{graphicx}
\usepackage{float,mathtools}
\usepackage{makecell}
\usepackage{multirow}
\usepackage{url,todonotes}
\usepackage{enumitem,wrapfig}
\usepackage[linktocpage,colorlinks,linkcolor=blue,anchorcolor=blue,citecolor=blue,urlcolor=blue,pagebackref]{hyperref}
\usepackage[toc,page]{appendix}
\usepackage{tabularx}
\usepackage{algorithm,algorithmic}
\usepackage{color}
\allowdisplaybreaks

\usepackage{cleveref}
\RequirePackage[authoryear]{natbib}
\setcitestyle{authoryear,open={(},close={)}} 

\renewcommand*{\backref}[1]{\ifx#1\relax \else Page #1 \fi}
\renewcommand*{\backrefalt}[4]{%
  \ifcase #1 \footnotesize{(Not cited.)}%
  \or        \footnotesize{(Cited on page~#2.)}%
  \else      \footnotesize{(Cited on pages~#2.)}%
  \fi
}
\usepackage{authblk}

\DeclareMathOperator{\grad}{ grad\,}
\DeclareMathOperator{\Hess}{Hess\,}
\DeclareMathOperator{\Div}{div\,}

\DeclareMathOperator{\Tr}{Tr\,}
\DeclareMathOperator{\tr}{tr\,}

\DeclareMathOperator{\Sec}{Sec\,}
\DeclareMathOperator{\Cut}{Cut\,}
\DeclareMathOperator{\inj}{inj\,}

\DeclareMathOperator{\diam}{diam\,}

\DeclareMathOperator{\Poly}{Poly\,}

\DeclareMathOperator{\Exp}{Exp\,}
\DeclareMathOperator{\Log}{Log\,}
\DeclareMathOperator{\TV}{TV\,}

\DeclareMathOperator{\Vol}{Vol\,}

\DeclareMathOperator{\op}{op\,}

\DeclareMathOperator{\SPD}{SPD\,}
\DeclareMathOperator{\PD}{PD\,}

\newtheorem{lemmanostar}{Lemma}

\newtheorem{theorem}{Theorem}
\newtheorem{lemma}{Lemma}[section] 
\newtheorem{proposition}[lemma]{Proposition}
\newtheorem{corollary}[lemma]{Corollary}
\newtheorem{definition}{Definition}[section]

\newtheorem{remark}{Remark}

\renewenvironment{proof}{\noindent\textbf{Proof.}\hspace*{.3em}}{\qed \vspace{.1in}}

\newtheorem{assumption}{Assumption}
\theoremstyle{definition}

\makeatletter
\renewcommand{\paragraph}{%
  \@startsection{paragraph}{4}%
  {\z@}{1.25ex \@plus 1ex \@minus .2ex}{-1em}%
  {\normalfont\normalsize\bfseries}%
}
\makeatother

\title{Total Variation Rates for Riemannian Flow Matching }

\begin{document}

\author[1]{Yunrui Guan}
\author[2]{Krishnakumar Balasubramanian}
\author[1]{Shiqian Ma}
\affil[1]{Department of Computational Applied Mathematics and Operations Research, \newline Rice University.}
\affil[2]{Department of Statistics, University of California, Davis.}
\affil[1]{\texttt{\{yg83,sqma\}}@rice.edu}
\affil[2]{\texttt{\{kbala\}}@ucdavis.edu}
\date{}

\maketitle

\begin{abstract}
Riemannian flow matching (RFM) extends flow-based generative modeling to data supported on manifolds by learning a time-dependent tangent vector field whose flow-ODE transports a simple base distribution to the data law. We develop a nonasymptotic Total Variation (TV) convergence analysis for RFM samplers that use a learned vector field together with Euler discretization on manifolds. Our key technical ingredient is a differential inequality governing the evolution of TV between two manifold ODE flows, which expresses the time-derivative of TV through the divergence of the vector-field mismatch and the score of the reference flow; controlling these terms requires establishing new bounds that explicitly account for parallel transport and curvature. Under smoothness assumptions on the population flow-matching field and either uniform (compact manifolds) or mean-square (Hadamard manifolds) approximation guarantees for the learned field, we obtain explicit bounds of the form $\mathrm{TV}\le C_{\mathrm{Lip}}\,h + C_{\varepsilon}\,\varepsilon$ (with an additional higher-order $\varepsilon^2$ term on compact manifolds), cleanly separating numerical discretization and learning errors. Here, $h$ is the step-size and $\varepsilon$ is the target accuracy. Instantiations yield \emph{explicit} polynomial iteration complexities on the hypersphere $S^d$, and on the SPD$(n)$ manifolds under mild moment conditions.
\end{abstract}

\tableofcontents

\input{intro}

\input{Main_Text}

\clearpage
\bibliographystyle{abbrvnat}
\bibliography{ref}

\clearpage
\noindent\rule{\textwidth}{0.11pt}
\begin{center}
\vspace{7pt}
{\Large  Appendix}
\end{center}
\noindent\rule{\textwidth}{0.1pt}
\appendix

\input{Appendix_Manifold}

\input{Appendix_Proof_Theoren}

\input{Appendix_Hypersphere}

\input{Appendix_SPD}

\end{document}

%% file: intro.tex
\section{Introduction}

Flow-based generative modeling offers a conceptually clean sampling paradigm by representing a complex target distribution as the endpoint of a continuous-time transport from a simple reference law. In \emph{flow matching}, one fixes a family of interpolating measures $(\pi_t)_{t\in[0,1]}$ between an easy base distribution $\pi_0$ and the data distribution $\pi_1$, and learns a time-dependent vector field whose flow pushes $\pi_0$ to $\pi_1$. When the data lie on a Riemannian manifold $(M,g)$, particles evolve according to an intrinsic ODE $\dot X_t = v_t(X_t)$, and the associated densities satisfy the manifold continuity equation $\partial_t p_t + \operatorname{div}_g(p_t v_t)=0$, yielding a sampler that respects the geometry by construction. Crucially, because Riemannian Flow Matching (RFM) is purely deterministic and relies only on integrating intrinsic ODEs, it avoids discretizing manifold-valued Brownian motion as required by Riemannian diffusion models~\citep{de2022riemannian}, where both theoretical analysis and practical simulation are considerably more delicate.

Despite their ever-increasing practical usage~\citep{yue2025reqflow,sriram2024flowllm,mathieu2020riemannian,miller2024flowmm,collas2025riemannian,luo2025crystalflow}, a central theoretical question is unexplored: \emph{how fast} does a discretized RFM sampler with a learned velocity field converge to the target distribution, in  distributional metrics? Our main contribution is to provide the first \emph{non-asymptotic Total Variation (TV) bounds} for RFM samplers on both compact and Hadamard manifolds. We develop intrinsic stability estimates for solutions of the Riemannian continuity equation under perturbations of the driving vector field, yielding explicit control of $\mathrm{TV}(\hat\pi_T,\pi_T)$ for $T=1-\delta$, ($\delta>0$) in terms of (i) the approximation quality of the learned field $\hat v_t$ to an ideal field $v_t$ in geometry-aware norms, (ii) regularity of the dynamics (e.g., bounds on covariant derivatives and divergence), and (iii) geometric characteristics of $(M,g)$ that govern volume distortion. We then combine these stability estimates with a careful analysis of discretized flows to obtain end-to-end guarantees for practical samplers, isolating the statistical and computational contributions. The resulting theory clarifies which properties of the learned velocity field are truly required for strong sampling accuracy on manifolds, and how curvature and anisotropy enter the constants and rates, providing a principled foundation for understanding Riemannian flow matching algorithms.\vspace{0.1in}

%% file: Main_Text.tex
\textbf{Related works.} Generative modeling on Riemannian manifolds was introduced by \cite{de2022riemannian}, who generalized score-based generative modeling \citep{song2020score} from Euclidean space to Riemannian manifolds, and proved an error bound of sampling error. The error bound in \cite{de2022riemannian} fails to be polynomial, which is caused by technical difficulties in analyzing discretization error in manifold Brownian motion simulation. Other works including \cite{huang2022riemannian} and \cite{lou2023scaling} also studied generative modeling on mainfolds without theoretical guarantees.

This difficulty is related to sampling on Riemannian manifolds via diffusion processes. Early work analyzed KL divergence along Langevin diffusions, e.g., on hyperspheres \citep{li2023riemannian} and Hessian manifolds \citep{gatmiry2022convergence}. Later, \cite{cheng2022efficient} used a geometric approach with coupling to obtain $W_{1}$ bounds; see also \cite{kong2024convergence} for Lie groups and \cite{guan2025riemannian} for high-accuracy proximal sampling. These results highlight that Riemannian Brownian motion—the basic component of Riemannian diffusion models—is challenging both theoretically and computationally: discretization analyses are largely geometric and typically give Wasserstein error bounds (e.g., \cite{cheng2022efficient}), while exact simulation is generally unavailable since its transition density lacks a closed form except in special cases. This in turn complicates denoising score matching and sample generation for Riemannian diffusion models. 

More recently, \citet[Lemma 19]{xu2026polynomial} established a polynomial discretization error bound in total variation distance. Building on this, \cite{xu2026polynomial} further proved a polynomial iteration complexity for Riemannian score-based generative models. However, the established bounds are only qualitative and do not reveal the precise dependencies on the problem parameters.


Another line of research is the Riemannian Flow Matching method, first proposed in \cite{chen2023flow} and later explored by \cite{cheng2025stiefel} and \cite{wu2025riemannian}. 
We remark that, due to its deterministic formulation, RFM does not require access to the heat kernel or simulation of manifold Brownian motion, thereby avoiding the main bottleneck of Riemannian diffusion models.

From the case of Euclidean flow matching, works including \cite{li2024unified} and \cite{ huang2025convergence} established discretization error bounds in TV distance for probability flow ODEs with polynomial dependence on the Lipschitz constant. For flow matching with deterministic samplers via ODE discretization, \cite{benton2024error, bansal2024wasserstein, zhouerror, guan2025mirror} establish bounds in the Wasserstein distance, and~\cite{su2025flow} considered KL divergence. However, without a contraction/dissipativity-type structure for the score or flow-matching vector field, the bounds typically do not yield polynomial rates in the associated Lipschitz constants. To obtain a fully polynomial error bound, existing works~\citep{li2024sharp,li2024unified} analyze flow matching samplers in the TV distance. \cite{liu2025finiteode} provided a polynomial error bound for stochastic interpolant (which can be understood as a smoothed variant of flow matching) with deterministic sampling, by making use of the framework in \cite{li2024unified}. Very recently,~\cite{roy2026low} examined adaptivity of Euclidean FM to low-dimensional structures. To the best of our knowledge, no prior discretization error analysis is available for flow matching on general Riemannian manifolds. 




\section{Basics and Problem Formulation}

\input{manifold}

\subsection{Riemanian Flow Matching}


Let $\pi_{0}$ and $\pi_{1}$ be two probability distributions supported on a
Riemannian manifold $M$. The Riemannian Flow Matching (RFM) framework aims to learn a time-dependent
vector field $v(t,x)$ whose induced flow transports samples from $\pi_{0}$ to
$\pi_{1}$. Concretely, RFM seeks $v$ such that if $X_{0}\sim \pi_{0}$, then the
solution $\{X_t\}_{t\in[0,1]}$ to the ODE
\begin{align}
\label{Eq_RFM_ODE}
    dX_{t} = v(t, X_{t})\, dt, \qquad X_{0} \sim \pi_{0},
\end{align}
satisfies $X_{1}\sim \pi_{1}$. The ODE exhibits a maximal solution as long as $v$ is continuous in time and locally-Lipschitz in space; see, for example~\cite[Chapter IV]{lang2012differential}.  Moreover, if the manifold is complete and $v(t, x)$ satisfies a linear growth bound, the solution exists globally in time. We also remark here that \cite{wanelucidating} provided more refined conditions that hold more generally for the case of Euclidean spaces.

A key geometric constraint is that for every $(t,x)$, the velocity must lie in
the \emph{tangent space at the current point}:
\[
v(t,x)\in T_xM,\qquad t\in[0,1].
\]
That is, $v(t,x)$ represents a valid instantaneous direction of motion
\emph{along the manifold} at $x$. This is where the Riemannian setting differs fundamentally from the Euclidean
one. When $M=\mathbb{R}^d$, all tangent spaces are canonically identical, i.e., $T_x\mathbb{R}^d \cong \mathbb{R}^d$ for all $x$, so we can view the vector field simply as a global map $v:[0,1]\times \mathbb{R}^d\to \mathbb{R}^d$ with a single, fixed vector space
as codomain.

On a general manifold $M$, however, the tangent spaces $\{T_xM\}_{x\in M}$ form
a family of \emph{different vector spaces attached to different points}.
Although each $T_xM$ has the same dimension, there is no canonical
identification between $T_xM$ and $T_yM$ when $x\neq y$. As a consequence, the
``output space'' of $v(t,\cdot)$ depends on $x$: $v$ is a section of the tangent
bundle rather than a map into a single fixed $\mathbb{R}^d$. This location-dependence forms a main source of technical differences between Euclidean flow matching and RFM, and it will matter in the analysis later.

In flow matching, one typically parameterizes the time-dependent vector field
with a neural network and learns it by minimizing the \emph{conditional flow
matching} objective. In the Riemannian setting, choosing the coupling between
the endpoints to be the \emph{independent coupling}, i.e., sampling
$X_{0}\sim \pi_{0}$ and $X_{1}\sim \pi_{1}$ independently, yields the training
loss
\begin{align}
    \label{Eq_RFM_Loss}
    \min_{u}\;
    \mathbb{E}_{\,t,\;X_{0}\sim \pi_{0},\;X_{1}\sim \pi_{1}}
    \big[
        \|
            u(X_{t}, t)
            - P_{X_{0}}^{X_{t}} \Log_{X_{0}}(X_{1})
        \|^{2}
    \big],
\end{align}
where $\Log_{x}(\cdot)$ is the Riemannian logarithm map, and $P_{a}^{b}$ denotes
parallel transport along the interpolation curve from $a$ to $b$. In Euclidean
space, the standard choice is the straight-line interpolation
$X_{t} := tX_{1} + (1-t)X_{0}$ between $X_{0}$ and $X_{1}$; on a Riemannian
manifold, this naturally generalizes to the geodesic interpolation
$X_t$ in~\eqref{Eq_Geodesic_Interp}. We denote the population minimizer by $v(t,x)$ and the learned (neural)
approximation by $\hat v(t,x)$.

It is useful to note that the population minimizer admits a closed-form
expression as a conditional expectation (see, e.g., \cite{guan2025mirror}).
Specifically, for any fixed $t<1$ and $x\in M$,
\begin{align}
\begin{aligned}
    v(t,x)
    &= \mathbb{E}\!\left[ P_{X_{0}}^{x}\,\Log_{X_{0}}(X_{1}) \,\middle|\, X_{t}=x \right]
    = \frac{1}{1-t}\,\mathbb{E}\!\left[ \Log_{x}(X_{1}) \,\middle|\, X_{t}=x \right] \\
    &= \frac{1}{1-t}\int_{M} \Log_{x}(x_{1})\, p_{t}(x_{1}\mid x)\, dV_{g}(x_{1}),
\end{aligned}
\label{eq:truev}
\end{align}
where $p_{t}(x_{1}\mid x)$ denotes the conditional density of $X_{1}$ given
$X_{t}=x$ under the above independent coupling and the chosen interpolation, and
$dV_{g}$ is the Riemannian volume measure.

\begin{wrapfigure}{l}{0.55\textwidth}
\includegraphics[scale=0.2]{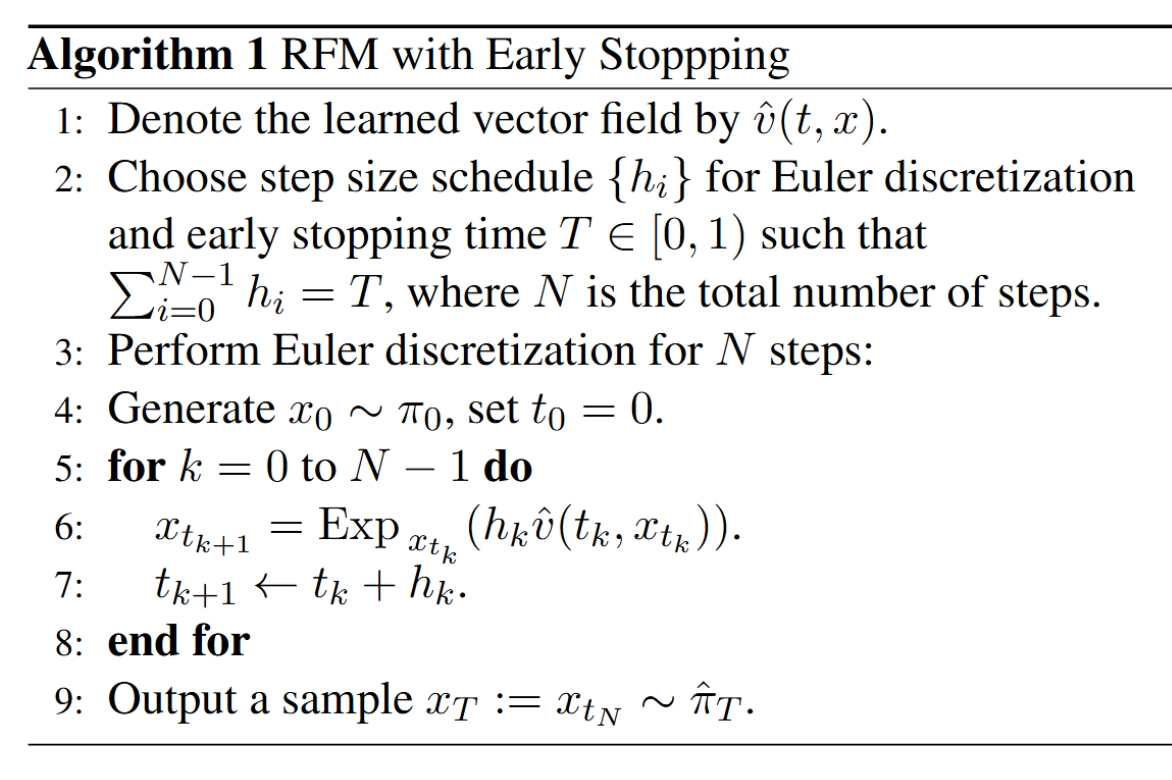}
\end{wrapfigure}

To generate samples, we numerically simulate the learned 
flow-matching ODE in
Algorithm \textcolor{blue}{1}. While the ideal (population) flow driven by the
true minimizer $v$ would transport $\pi_0$ to $\pi_1$ at time $t=1$, in practice
we only have access to a learned approximation $\hat v$ and a time-discretized
integrator. A key difficulty is the behavior of the conditional flow-matching
vector field near the terminal time. Note from~\eqref{eq:truev} that the population minimizer contains an explicit ${1}/{(1-t)}$ factor. Along the \emph{exact}
geodesic bridge used to define $X_t$, the term $\mathbb{E}[\Log_x(X_1)\mid X_t=x]$
typically scales like $(1-t)$ so that $v(t,x)$ remains finite; however, this
cancellation is fragile. Any mismatch between the learned field $\hat v$ and the
population field $v$, or any deviation of the simulated trajectory from the
training interpolation, can destroy the cancellation and make the effective
dynamics increasingly stiff as $t\to 1$ (large velocities and/or large
sensitivities with respect to $x$).

This stiffness directly motivates \emph{early stopping}. In
Algorithm \textcolor{blue}{1}, we use $\{h_{i}\}$ to denote the step size schedule. In the constant step size case, we simply have $h_{i} \equiv h, \forall i$. With a pre-designed step size schedule and early stopping time $T$, the algorithm generates $\{x_{t_{k}}\}$ where $t_{k} := \sum_{i = 0}^{k-1} h_{i}$ and consequently by definition $t_{N} = T$. The Euler discretization accumulates both numerical discretization error and statistical/approximation error from learning $\hat v$. Near $t=1$, the factor $\frac{1}{1-t}$ amplifies errors in estimating
the small conditional mean $\mathbb{E}[\Log_x(X_1)\mid X_t=x]$: a small absolute
error in this conditional expectation can translate into a much larger error in
the velocity, which then produces an $O(1)$ perturbation over the last few
integration steps and shifts the terminal law away from $\pi_1$. For this
reason, following existing probability-flow analyses in the Euclidean setting (e.g.,
\cite{li2024unified,li2024sharp}) we establish convergence rates to
an \emph{approximate} distribution at a terminal time $T<1$, corresponding to
stopping the integration before the most ill-conditioned regime. 

\section{Total Variation Rates: Compact Manifolds}

We now establish TV error bounds for compact manifolds, which covers application including $\textrm{SO}(3)$ \citep{yue2025reqflow},  Tori and hypersphere~\citep{mathieu2020riemannian,sriram2024flowllm}. We then specialize to the case of hypersphere to obtain explicit iteration complexity results. We start with the following assumption on curvature of manifold.



\begin{assumption}[Assumptions on Riemannian Manifold]\label{A_Curvature}
$M$ is a complete manifold without boundary.  There exists $L_{R} \ge 0$ s.t. $\|R(u, v)v\| \le L_{R} \|u\|\|v\|^{2}, \forall u, v, w \in T_{x}M$, and all sectional curvatures of $M$ are bounded below by $K_{\min}$ and bounded above by $K_{\max}$. 
\end{assumption}

This assumption enforces standard global regularity of the geometry of $M$: completeness guarantees geodesics and the exponential map are well-defined for all
times. The bound $\|R(u,v)v\|\le L_R\|u\|\|v\|^2$ controls the norm of the curvature operator, ensuring that curvature-induced deviations along geodesics are
uniformly bounded. Finally, the two-sided sectional curvature bound $K_{\min}\le K\le K_{\max}$
prevents the manifold from being too positively curved (excessive focusing of geodesics) or too negatively curved (overly rapid divergence), which is crucial for stability and comparison estimates. We impose the following assumptions on flow matching vector field $v$. 
\begin{assumption}\label{A_Regularity_V}
    (Regularity of Flow Matching Vector Field)
   The true vector field $v(t, x)$ is at least $C^{2}$ in $t, x$ and  satisfies the following for $\forall x \in M, t \in [0, 1)$:
    \begin{enumerate}[noitemsep,label={(\arabic*)}]
        \item $\|\nabla v(t, x) \|_{\op} \le L_{t}^{v, x}$.
        \item $\|v(t, x)\| \le L_{v}$ and $ \Bigl\|\frac{d}{dt}v(t,x)\Bigr\| \le L_{t}^{v, t}$. 
        \item $\|\grad_{x}\Div v(t,x)\|\le L_{t}^{\Div, x}$ and  $\Bigl|\frac{d}{dt}\Div v(t,x)\Bigr| \le L_{t}^{\Div, t}$.
    \end{enumerate}
\end{assumption}

The above conditions require the population vector field $v(t,x)$ to be smooth and uniformly well-behaved so that the induced flow depends stably on time and initial conditions. The bounds on $\|\nabla v(t,x)\|_{\op}$, $\|v(t,x)\|$, and $\bigl\|\tfrac{d}{dt}v(t,x)\bigr\|$ ensure the dynamics are Lipschitz in space and controlled in magnitude and time variation, preventing trajectories from separating too quickly or developing
instabilities as $t$ evolves. In Euclidean space $M=\mathbb{R}^d$ with the standard metric, $\nabla v(t,x)$
reduces to the Jacobian $D_x v(t,x)$, so $\|\nabla v(t,x)\|_{\op}$ is exactly
the usual spectral/operator norm $\|D_x v(t,x)\|_{2\to 2}$ (i.e., the Lipschitz
constant of $v(t,\cdot)$ at $x$). The bounds on $\|\grad_x \Div v(t,x)\|$ and $\bigl|\tfrac{d}{dt}\Div v(t,x)\bigr|$
control how local volume change induced by the flow varies across space and time, and is important for analyzing how densities evolve under the continuity equation.

\begin{assumption}\label{A_estimation_error}
    (Estimation Error)
Tthere exists some small $\varepsilon > 0$ such that the the followings hold, $\forall x \in M, t \in [0, 1) $: (1) $\|\hat{v}(t, x) - v(t, x)\| \le \varepsilon$,\quad and\quad (2) $\|\nabla \hat v(t, x)-\nabla v(t, x)\|_{\op}\le \varepsilon $.
\end{assumption}

We remark that such an assumption on estimation accuracy for derivative of $v(t, x)$ is standard in the literature for error analysis of ODE based generative models \citep{li2024sharp, li2024unified, liu2025finiteode}. The second item implies an error bound on uniform divergence estimation error, which together with the first item, helps bounding the discretization error. We also emphasize that uniform estimation error bounds obtainable on \emph{compact} manifolds~\citep{yarotsky2017error,mena2025statistical}.


Under our assumptions, we can prove the following bound on TV distance (see Definition~\ref{sec:defn}). 

\begin{theorem}\label{Theorem_Rate_Compact}
    Let Assumptions \ref{A_Curvature}, \ref{A_Regularity_V} and \ref{A_estimation_error} hold. Define 
    \begin{align}
    \mathsf{C}_{\textsf{Lip}} &\coloneqq 3\sqrt{L_{t}^{\text{score}}} L_{t}^{v, x}L^{v} 
    + \sqrt{L_{t}^{\text{score}}} L_{t}^{v, t} 
    + 3L^{v}L_{t}^{\Div, x} + L_{t}^{\Div, t} + L_{R} (L^{v})^{2} d \label{eq:constant1}\\
    \mathsf{C}_{\textsf{eps}} &\coloneqq \sqrt{ 2L_{t}^{\text{score}}} + 1 + 2\sqrt{L_{t}^{\text{score}}} L^{v} + \sqrt{L_{t}^{\text{score}}} L_{t}^{v, x} + L_{t}^{\Div, x} + 2L_{R}d L^{v} \label{eq:constant2}
    \end{align}        
Picking the constant step size $h$ to satisfy the requirements in Lemma \ref{Lemma_F_Invertible}, we have  $$        \TV(\pi_{T}, \hat{\pi}_{T}) 
\le  h \mathsf{C}_{\textsf{Lip}}
        + \varepsilon \mathsf{C}_{\textsf{eps}} + \varepsilon^2 (\sqrt{L_{t}^{\text{score}}} + L_{R}d).$$
\end{theorem}




We imposed an upper bound on the step size (as in Lemma~\ref{Lemma_F_Invertible}) to ensure that each Euler step remains within a controlled neighborhood---in particular, within the injectivity-radius regime where the exponential map does not ``fold'' and the map $F_{t_k,h}(x)=\Exp_x(h\hat v(t_k,x))$ stays invertible---so that the continuous-time interpolation and the associated vector field $\tilde v$ are well defined. For more details, see Section~\ref{sec:intuition} on the choice of $h$. 
At a high level, the constant $\mathsf{C}_{\textsf{Lip}}$ aggregates the \emph{Lipschitz regularity} moduli of the population flow (and the score) that control how local Euler truncation errors are amplified when transported through time; thus it governs the $\mathcal{O}(h)$ discretization contribution in the TV bound. 
In contrast, $\mathsf{C}_{\textsf{eps}}$ collects the same geometric and score-dependent stability factors but attached to the \emph{estimation perturbation} $\hat v-v$ (and its divergence), quantifying how an $\varepsilon$-accurate field estimate propagates through the continuity equation; this is why it multiplies the leading $\mathcal{O}(\varepsilon)$ term, with the remaining $\varepsilon^2(\sqrt{L_t^{\mathrm{score}}}+L_R d)$ capturing higher-order interactions between score regularity and field mismatch. 



\begin{remark}
    The early stopping error can be controlled in the Wasserstein distance (see Definition~\ref{sec:defn}) as: $W_{1}(\pi_{1}, \pi_{T}) \lesssim 1-T$. On a compact manifold, TV bound implies a $W_1$ bound, as $W_{1} (\pi_{T}, \hat{\pi}_{T}) \le \diam(M) \TV(\pi_{T}, \hat{\pi}_{T})$ \citep[Theorem 6.15]{villani2008optimal}. Thus we can convert our TV bound into a bound in $W_{1}$ distance, and obtain an error bound for $W_{1}(\hat{\pi}_{T}, \pi_{1}) $, consider both sampling error and early stopping error. 
\end{remark}

We now provide explicit rates on the $d$-dimensional hypersphere. Following~\cite{li2024sharp, li2024unified, liu2025finiteode}, we consider the case when the vector field is well-estimated (i.e., $\varepsilon \approx 0$) and report the number of sampling steps $N$ required to get $\varepsilon_{target}$ close to $\pi_{T}$.



\begin{proposition}
    \label{Cor_Iter_Complexity_Sphere}
    For some finite $d \ge 4$, let $M = \mathcal{S}^{d}$, and $T<1$ be the early stopping time. Assume $\pi_{1}$ has a smooth density and satisfies $0 < m_{1} \le p_{1}(x) \le M_{1}$. Pick $\pi_{0}(x)$ to be uniform on the hypersphere. Then Assumption \ref{A_Regularity_V} is satisfied and the constants in Assumption \ref{A_Regularity_V} are of polynomial order (see Lemma~\ref{Lemma_Constants_Hypersphere}). Now if Assumption \ref{A_estimation_error} is satisfied, to reach $\varepsilon_{\text{target}}$ accuracy in TV distance, 
    \begin{itemize}[noitemsep]
        \item For constant step size as in \eqref{Eq_Const_Stepsize_Sphere}, the complexity is $N = \mathcal{O}({ d^{2} }/{\varepsilon_{\text{target}} (1-T)^{2}})$.
        \item For a carefully designed step size schedule as in \eqref{Eq_Sphere_Step_Schedule}, it can be improved to $N = \mathcal{O}({d^{2}}/{\varepsilon_{\text{target}} (1-T)} )$.
    \end{itemize}
\end{proposition}

\section{Total Variation Rates: Hadamard Manifolds}

In this section, we establish error bounds in terms of total variation distance on Hadamard manifolds (which are non-compact), which covers applications including SPD manifolds~\citep{li2024spd,collas2025riemannian}. While retaining the same curvature condition in Assumption \ref{A_Curvature}, we modify Assumption \ref{A_Regularity_V} and \ref{A_estimation_error} as follows, to naturally handle the non-compactness of Hadamard manifolds.



\begin{assumption}[Regularity of the true Vector Field]\label{A_Regularity_Expectation_V}
    The true vector field $v$ satisfies $\forall t \in [0, 1)$:
    \begin{enumerate}[noitemsep,label={(\arabic*)}]
        \item $\mathbb{E}_{x_{t}}\|\nabla v(t,x_{t}) \|_{\op}^{2} \le (L_{t}^{v, x})^{2}$ and  $\mathbb{E}_{x_{t}} \bigl\|\frac{d}{dt}v(t,x_{t})\bigr\|^{2} \le (L_{t}^{v, t})^{2}$.
        \item $\mathbb{E}_{x_{t}}\Bigl|\frac{d}{dt}\Div v(t,x_{t})\Bigr| \le L_{t}^{\Div, t}$ and $\mathbb{E}_{x_{t}}\|\grad_{x}\Div v(t,x_{t})\|^{2}\le (L_{t}^{\Div, x})^{2}$. 
        \item $\mathbb{E}_{x_{t}}\|v(t,x_{t})\|^{2} \le (L_{t}^{v})^{2}$.
    \end{enumerate}
\end{assumption}

\begin{assumption}[Regularity of Learned Vector Field]\label{A_Regularity_Learned_Pointwise}
    Assume that the learned vector field $\hat{v}$ satisfies: (1) $\|\hat{v}(t, x)\| \le L_{t}^{\hat{v}} $, \quad (2) $\|\nabla \hat{v}(t, x) \| \le L_{t}^{\hat{v}, x}$ \quad  and \quad (3)  $\|\grad \Div \hat{v}(t, x) \| \le L_{t}^{\Div \hat{v}, x}$, where $L_{t}^{\hat{v}}, L_{t}^{\hat{v}, x}, L_{t}^{\Div \hat{v}, x}$ are of the same order as $L_{t}^{v}, L_{t}^{v, x}, L_{t}^{\Div, x}$.
\end{assumption}

\begin{assumption}[Estimation Error]
\label{A_estimation_error_Expectation}
 There exists some small $\varepsilon > 0$ s.t. the following hold $\forall  t \in \{t_{k}\}_{k = 0}^{N-1} $: (1) $\mathbb{E}_{x_{t}}[\|\hat{v}(t_{k}, X_{t_{k}}) - v(t_{k}, X_{t_{k}})\|^{2}] \le \varepsilon$, (2) $\mathbb{E}_{x_{t}}[\left|\Div \hat{v}(t_{k}, X_{t_{k}}) - \Div v(t_{k}, X_{t_{k}}) \right|] \le \varepsilon$.
\end{assumption}

The conditions in Assumption \ref{A_Regularity_Expectation_V}, \ref{A_Regularity_Learned_Pointwise} and \ref{A_estimation_error_Expectation} serve the same purpose as Assumption \ref{A_Regularity_V} and \ref{A_estimation_error}. We remark that on a non-compact manifold, similar to the Euclidean space, it is more natural to assume regularity holds in expectation instead of holding uniformly. Compared to the compact case, we need to additionally enforce Assumption \ref{A_Regularity_Learned_Pointwise}, which is of essential importance to guarantee the well-behavedness of the sampling ODE. The condition that the Lipschitz constants of the learned vector field are of the same order of the true field is made purely for convenience to avoid a more complicated looking bound. Note that for the compact case, since we can assume point-wise regularity, the condition made in Assumption \ref{A_Regularity_Learned_Pointwise} is a direct consequence of Assumption \ref{A_Regularity_V} and \ref{A_estimation_error}. 
We also remark that for existing works in Euclidean space, \cite{liu2025finiteode} and \cite{huang2025convergence} assumed uniformly bounded regularity for the learned vector field, and estimation error in expectation, which are similar in spirit to Assumptions \ref{A_Regularity_Learned_Pointwise} and \ref{A_estimation_error_Expectation}. Below, we present our rates.


\begin{theorem}[Sampling Error for Hadamard Manifold]\label{Theorem_Sampling_Error_Hadamard}
Let Assumptions \ref{A_Curvature}, \ref{A_Regularity_Expectation_V}, \ref{A_Regularity_Learned_Pointwise} and \ref{A_estimation_error_Expectation}, hold. Picking the step size $h$ to satisfy the requirements in Lemma \ref{Lemma_F_Invertible}, we have $$\TV(\pi_{T}, \hat{\pi}_{T}) \le  h \mathsf{C}_{\textsf{Lip}}+ \varepsilon \mathsf{C}_{\textsf{eps}, 1},$$
where $\mathsf{C}_{\textsf{Lip}}$ is as in~\eqref{eq:constant1} and $\mathsf{C}_{\textsf{eps},1}$ defined in \eqref{Eq_Const_Hadamard_Thm} represents the (vector field) estimation error.  
\end{theorem}



We now specialize our results to the case when  $M\equiv \SPD(n) := \{X \in \mathbb{R}^{n \times n}: X \succ 0, X^{T} = X\}$, the manifold of symmetric positive definite matrices endowed with the affine invariant metric $g_{X}(U, V) = \tr(X^{-1} U X^{-1} V)$. To do so, we compute the explicit order of regularity constants $\mathsf{C}_{\textsf{Lip}}$ and $\mathsf{C}_{\textsf{eps},1}$ (see Proposition \ref{Prop_Regularity_Explicit_SPD}), which in turn yields the result below.

\begin{proposition}
    \label{Cor_Iter_Complexity_SPD}
    Let $M = \SPD(n)$ for which the dimension $d=n(n+1)/2$. We adopt early stopping, i.e., simulate the ODE on the interval $[0, T] \subsetneq [0, 1]$. 
    Choose prior distribution as $\pi_{0}(x) \propto \exp( - \frac{n(n+1)d_{g}(x, I)^{2}}{2})$ being a Riemannian Gaussian distribution (the center of Riemannian Gaussian is arbitrary, here we set as $I \in \SPD(n)$ for notation simplicity). We further assume the data distribution $\pi_{1}$ satisfies 
    \begin{align}\label{Eq_Moment_Condition}
        \max\big\{\mathbb E [d(X_1,I)^2 e^{\lambda_1 d(X_1,I)}], \mathbb E [e^{\lambda_1 d(X_1,I)}]\big\} \le M_{\lambda_{1}}, \quad \text{where} \quad \lambda_{1} = 24 \max\{1, \kappa\},
    \end{align}  
    for some constant $M_{\lambda_{1}}$. Then Assumption \ref{A_Regularity_Expectation_V} is satisfied and the constants in Assumption \ref{A_Regularity_Expectation_V} are of polynomial order (see Proposition~\ref{Prop_Regularity_Explicit_SPD} for explicit bounds).
    Under Assumption \ref{A_Curvature}, \ref{A_Regularity_Learned_Pointwise} and Assumption \ref{A_estimation_error_Expectation}, (with constant step size) the iteration complexity to reach $\varepsilon_{\text{target}}$ accuracy in TV distance is 
    $ N = \mathcal{O} ( {d^{24} L_{R}^{3} M_{\lambda_{1}}^{\frac{3}{2}}}/{\varepsilon_{\text{target}}(1-T)^{3}} )$.
\end{proposition}
Here the moment condition \eqref{Eq_Moment_Condition} plays a role analogous to finite-moment conditions (e.g., $\mathbb{E}[\|X_{1}\|^{4}] < \infty$ and stronger conditions like bounded support) appeared in Euclidean discretization analysis \citep{liu2025finiteode, li2024unified, zhouerror}. The appearance of exponential moments is due to curvature distortion: the model function $s_{K_{\min}}(r)$ grows linearly in $r$ when $K_{\min} = 0$ but exponentially in $r$ when $K_{\min} < 0$ (see Section~\ref{sec:regexp}). Although the order of $d$ is polynomial, we expect it could be improved further by more refined computation and additional assumptions.

\section{Proof Sketch and Intermediate Results}\label{sec:intuition}

We now outline the main ideas behind the proof of our principal TV-rate bound.
At a high level, we view both the population flow-matching dynamics and its numerical approximation as \emph{deterministic transports} on the manifold driven by (possibly different) time-dependent vector fields.
Randomness enters only through the initialization, so the objects of interest are the time-marginal laws $(p_t)_{t\in[0,1]}$ and $(q_t)_{t\in[0,1]}$ induced by these transports.
Our argument has three conceptual steps:
\begin{enumerate}[noitemsep]
    \item (\emph{TV stability under transport}) derive a differential identity for $\partial_t \TV(p_t,q_t)$ in terms of the mismatch between the driving vector fields (Lemma~\ref{Lemma_TV_derivative});
    \item (\emph{continuous-time interpolation of Euler}) construct an interpolation of the Euler scheme that is itself an ODE on $M$, so that Lemma~\ref{Lemma_TV_derivative} applies;
    \item (\emph{term-by-term control}) bound the resulting RHS by establishing (i) score regularity for $p_t$, (ii) regularity of the relevant vector fields, and (iii) a curvature-dependent estimate for the divergence of the interpolated field.
\end{enumerate}
Integrating the differential inequality over time then yields a bound on the terminal sampling error $\TV(X_1,Y_1)$ (or, when needed, the limit $t\uparrow 1$).

We begin with the key TV-derivative lemma, which is the Riemannian analogue of \cite[Lemma 4.2]{li2024unified}.
The proof involves the following steps: (i) write the evolution of $p$ and $q$ through the continuity equation, and (ii) differentiate $\int (p-q)_+$, and integrate by parts---but all differential operators must be interpreted intrinsically (gradient, divergence, and the Riemannian volume form). The second line of \eqref{Eq_TV_derivative} is obtained by taking absolute values and applying Cauchy--Schwarz to the inner-product term.

\begin{lemma}
    \label{Lemma_TV_derivative}
    Let $X_{t}, Y_{t}$ be stochastic processes on $M$, satisfying the following ODE: 
    \begin{align}
        dX_t &= v(t,X_t)\,dt, & X_0 \sim p_0 \label{Eq_ode_true} \\
        dY_t &= \tilde v(t,Y_t)\,dt, & Y_0 \sim q_0 \label{Eq_ode_dis_interp}
\end{align}
    Let $p(t, x)$ be the law of $X_{t}$ and $q(t, x)$ be the law of $Y_{t}$.
    Define $\Omega_{t} = \{x \in M: p(t, x) - q(t, x) > 0 \}$.
    Then we have
    \begin{align}\label{Eq_TV_derivative}
        \frac{\partial \TV(X_{t}, Y_{t})}{\partial t}
        = &\int_{\Omega_{t}} p(t, x) \left(\Div (\tilde{v}(x, t) - v(x, t)) + \langle \grad \log p(t, x), \tilde{v}(x, t) - v(x, t) \rangle \right) dV_{g}(x) \notag \\
        \le & \mathbb{E}[ | \Div (\tilde{v}(x, t) - v(x, t)) | ] + \mathbb{E}[ \|\grad \log p(t, x)\| \| \tilde{v}(x, t) - v(x, t) \| ].
    \end{align}
\end{lemma}

\subsection{Continuous Time Interpolation}

We now explain how Lemma~\ref{Lemma_TV_derivative} is used to analyze the Euler discretization error.
We take \eqref{Eq_ode_true} to be the flow-matching ODE driven by the (population) vector field $v(t,\cdot)$, and we want \eqref{Eq_ode_dis_interp} to represent the numerical method.
A direct interpolation of the Euler scheme would be
\[
y_t \;=\; \Exp_{y_k}\!\big((t-t_k)\hat v(t_k,y_k)\big), \qquad t\in[t_k,t_{k+1}),
\]
i.e., the velocity is frozen at the left endpoint.
However, this curve is not immediately of the form $dy_t = \tilde v(t,y_t)\,dt$ with a \emph{vector field} $\tilde v(t,\cdot)$ on $M$, because the frozen vector $\hat v(t_k,y_k)$ lives in $T_{y_k}M$, whereas the ODE requires the instantaneous velocity to belong to $T_{y_t}M$.

The natural remedy is to \emph{transport the frozen velocity along the curve}.
Concretely, for $t\in[t_k,t_{k+1})$ we set
\[
\tilde v(t,y_t) := P_{y_k}^{y_t}\,\hat v(t_k,y_k),
\qquad\text{so that}\qquad
d y_t = \tilde v(t,y_t)\,dt,
\]
where $P_{y_k}^{y_t}$ denotes parallel transport along the (interpolated) trajectory.
This ensures $\tilde v(t,y_t)\in T_{y_t}M$ for all $t$, so the interpolation is an honest ODE on $M$.
To apply Lemma~\ref{Lemma_TV_derivative}, we further need $\tilde v$ to be defined as a function of $(t,x)$ (not implicitly through a particular path).
This requires that, given $(t,x)$, we can uniquely recover the footpoint $y_k$ that generated $x$ under the Euler interpolation.
Equivalently, we need the map $y_k \;\longmapsto\; \Exp_{y_k}\!\big((t-t_k)\hat v(t_k,y_k)\big)$ to be invertible on the relevant step size range.

For notational convenience, for $t\in[0,1)$ and $h\ge 0$ we define $F_{t,h}(x) := \Exp_{x}(h \hat{v}(t, x))$. In particular, for $t\in[t_k,t_{k+1})$, the Euler interpolation can be written as
$F_{t_k,t-t_k}(x_k)=\Exp_{x_k}((t-t_k)\hat v(t_k,x_k))$.
Lemma~\ref{Lemma_F_Invertible} below shows that for sufficiently small $h$ the map $F_{t_k,h}$ is invertible.
Assuming invertibility, we can express the interpolation vector field at an arbitrary point $x$ by pulling back to the unique preimage $x_k = F_{t_k,t-t_k}^{-1}(x)$ and then parallel-transporting the frozen velocity to $T_xM$:
\begin{align*}
    \tilde{v}(t, x) = P_{F_{t_{k}, t-t_{k}}^{-1}(x)}^{x} \hat{v}(t_{k}, F_{t_{k}, t-t_{k}}^{-1}(x)),
\end{align*}
where, for a trajectory $(x_t)$ of \eqref{Eq_ode_dis_interp}, we have $F_{t_{k}, t-t_{k}}^{-1}(x_{t}) = x_{t_{k}}$.

\begin{lemma}
\label{Lemma_F_Invertible}
    Let $M$ be simply connected Riemannian manifold that satisfies Assumption \ref{A_Curvature}.
    Let $b$ be any vector field on $M$, satisfying $\| b(x) \| \le B, \forall x \in M$.
    Assume $\|\nabla_{v}b(x)\| \le L_{\nabla} \|v\|$.
    Let $R = \inj(M)$. 
    To guarantee $F_{t_{k}, t-t_{k}}$ being invertible, we require 
    \begin{align*}
        h &< \min\{\frac{R}{B}, \frac{1}{4L_{\nabla}}, \sqrt{\frac{3}{4\|b(x)\|^{2} L_{R} (2 + 2 L_{\nabla}\max\{\frac{1}{\sqrt{K_{\min}}}, 1\})}}\}, &\text{if} \quad K_{\min} > 0, \\
        h &< \min\{\frac{R}{B}, \frac{1}{4L_{\nabla}}, \sqrt{\frac{3}{4\|b(x)\|^{2} L_{R} (2\frac{\sinh (\sqrt{-K_{\min}})}{\sqrt{-K_{\min}}} + 4\frac{\cosh (\sqrt{-K_{\min}}) - 1}{- K_{\min}} L_{\nabla})}}\},  &\text{if} \quad K_{\min} < 0, \\
        h &< \min\{\frac{R}{B}, \frac{1}{4L_{\nabla}}, \sqrt{\frac{3}{4\|b(x)\|^{2} L_{R} (2 + h L_{\nabla})}}\}, &\text{if} \quad K_{\min} = 0. 
    \end{align*}
\end{lemma}

\noindent
The invertibility requirement is precisely the condition that the exponential map used in one Euler step stays in a regime where it does not ``fold'' the manifold (e.g., by crossing conjugate points or exiting the injectivity radius).
If $(t-t_k)\hat v(t_k,Y_{t_k})$ is too large, the map $x\mapsto \Exp_x((t-t_k)\hat v(t_k,x))$ may fail to be injective, and then multiple preimages could map to the same $Y_t$.
Lemma~\ref{Lemma_F_Invertible} enforces a step-size restriction that prevents this pathology by combining (i) injectivity-radius control and (ii) bounds on the differential of $F_{t,h}$ via curvature comparison and covariant-derivative estimates.
When $\inj(M)=\infty$ (e.g., $M=\mathbb{R}^d$ or $M=\SPD(n)$), the injectivity-radius constraint becomes vacuous and the step-size condition simplifies accordingly.

\subsection{Discretization analysis}

Although \eqref{Eq_ode_true} and \eqref{Eq_ode_dis_interp} are deterministic ODEs, the random initialization makes $(X_t)$ and $(Y_t)$ stochastic processes through their induced laws.
To isolate \emph{discretization} error, we couple them by taking the \emph{same} initial distribution: $X_0 \sim \pi_0,
$ and $Y_0 \sim \pi_0.$ Then $\TV(X_t,Y_t)$ measures the discrepancy between the population flow-matching transport and the transport induced by the Euler interpolation.
In particular, at terminal time (or in the limit $t\uparrow 1$), $\TV(X_1,Y_1)$ is exactly the sampling error attributable to time discretization (and, if present, to the use of $\hat v$ instead of $v$).

Applying Lemma~\ref{Lemma_TV_derivative} with $p(t,\cdot)$ the law of $X_t$ and $q(t,\cdot)$ the law of $Y_t$, and integrating in time, yields 
\begin{align*}
\TV(X_1,Y_1)
\le&\TV(X_0,Y_0)
+
\int_0^1(\mathbb{E}[|\Div(\tilde v(t,X_t)-v(t,X_t))|]
+\\
&\qquad\qquad\qquad\qquad\qquad\qquad\qquad\qquad\mathbb{E}[\|\grad\log p_t(X_t)\|\|\tilde v(t,X_t)-v(t,X_t)\|]
)dt.
\end{align*}
Since $\TV(X_0,Y_0)=0$ under the shared initialization, the task reduces to controlling the two error terms on the RHS.
A convenient way to organize the analysis is to bound three quantities that repeatedly appear in the argument: $\mathbb{E}\bigl[\|\grad \log p_t(X_t)\|^{2}\bigr],$ $\mathbb{E}\bigl[\| \tilde{v}(t, X_t) - v(t, X_t) \|^{2}\bigr]
$, and $\mathbb{E}\bigl[| \Div (\tilde{v}(t, X_t) - v(t, X_t)) | \bigr]$. Specifically, the second expectation in \eqref{Eq_TV_derivative} can be decoupled via Cauchy--Schwarz: $$\mathbb{E}[\|\grad\log p_t(X_t)\|\|\tilde v(t,X_t)-v(t,X_t)\|]
\le
\sqrt{\mathbb{E}[\|\grad\log p_t(X_t)\|^2]} \sqrt{\mathbb{E}[\|\tilde v(t,X_t)-v(t,X_t)\|^2]}.$$ Thus, once we establish suitable bounds for these three terms (uniformly for $t$ away from $1$, and with explicit dependence on $t$ as $t\uparrow 1$), we can integrate over time to obtain a quantitative TV-rate.

\subsubsection{Score Regularity and Vector Field Regularity}\label{sec:regexp}

We now summarize the regularity inputs needed to control the score term and the vector-field mismatch.
The score bound is the main ``density regularity'' ingredient, while the vector-field bounds ensure that (i) the Euler interpolation is well defined and (ii) the mismatch $\tilde v-v$ can be quantified in geometry-aware norms.

\begin{itemize}[noitemsep, leftmargin=0.15in]
    \item \textbf{Score regularity.}
    We show that there exists a (possibly time-dependent) constant $L_t^{\mathrm{score}}$ such that, for all $t\in[0,1)$, $        \mathbb{E}\bigl[\|\grad \log p_t(X_t)\|^2\bigr] \le  L_t^{\mathrm{score}}.$   The bound is allowed to deteriorate as $t\uparrow 1$; this is unavoidable in general and matches the behavior of population conditional flow fields near the terminal time.
    We establish this in Proposition~\ref{Prop_Finite_Score_Compact} for compact manifolds and in Proposition~\ref{Prop_Score_Bound_Hadamard} for Hadamard manifolds (both stated below).
    A crucial point is that the existence of such an $L_t^{\mathrm{score}}$ relies on sufficient smoothness/positivity of $p_t$; without it, $\grad\log p_t$ may not be square-integrable and the bound can fail even for $t<1$.

    \item \textbf{Regularity on compact manifolds.}
    On compact manifolds, curvature is bounded and distances are uniformly controlled, but the presence of cut points can obstruct smoothness of conditional densities such as $p_{t}(x_1\mid x_t=x)$, which enter the explicit formulae for the population flow field.
    We verify Assumption~\ref{A_Regularity_V} on $\mathcal{S}^d$ in Appendix~\ref{Sect_Hypersphere} by exploiting explicit computations on the sphere.
    For more general compact geometries, controlling smoothness across the cut locus is substantially more delicate and is the main reason we phrase the required assumptions abstractly.

    \item \textbf{Regularity on Hadamard manifolds.}
    On Hadamard manifolds, the absence of cut points guarantees global smoothness of the exponential map and of the relevant conditional densities, which greatly simplifies differentiability issues.
    The trade-off is that negative curvature and unbounded distance can amplify derivative bounds along geodesics.
    In particular, pointwise bounds for derivatives of $v(t,x)$ typically involve factors that are $\Poly (d)$, $\Poly({1}/{(1-t)})$  and $\Poly(s_{K_{\min}}(d(x_0, x_1)))$, 
    where for $K_{\min}<0$ the model function $s_{K_{\min}}(r)= \tfrac{\sinh(r\sqrt{-K_{\min}})}{\sqrt{-K_{\min}}}$ grows exponentially in $r$.
    This motivates seeking \emph{regularity in expectation} rather than uniform-in-$x$ bounds.
    We verify Assumption~\ref{A_Regularity_Expectation_V} for $\SPD(n)$ in Appendix~\ref{Sect_SPD} by combining curvature comparison with a moment condition \eqref{Eq_Moment_Condition}, which plays the role of a Euclidean second-moment bound and compensates for curvature-induced growth.
\end{itemize}

\begin{proposition}\label{Prop_Finite_Score_Compact}
    Let $M$ be a compact manifold satisfying Assumption \ref{A_Curvature}. Under Assumption \ref{A_Regularity_V}, exists some finite number $L_t^{\mathrm{score}}$ that depends on $t$ s.t. $\mathbb{E}\bigl[\|\grad \log p_t(X_t)\|^2\bigr] \le  L_t^{\mathrm{score}}, \forall t \in [0, 1).$
\end{proposition}

\begin{proposition}\label{Prop_Score_Bound_Hadamard}
    Let $M$ be a Hadamard manifold with sectional curvature satisfying $-\kappa^{2} = K_{\min} \le \sec \le 0$.
    For prior being chosen as $X_{0} \sim e^{-d d(x_{0}, z)^{2}} $, we have  
    \begin{align*}
        \mathbb E\big[\|\grad \log p_t(X_t)\|^2\big] \lesssim M\frac{d^{2}}{(1-t)^{2}} L_{R}^{2} d^{2\lambda_{score}},
    \end{align*}
    where $\lambda_{score} = 6\max\{1, \kappa\}$ and $M = \mathbb{E}_{X_{1}}[ e^{\lambda d(x_1, z)}] $ depends on the data distribution.
\end{proposition}

\subsubsection{The Divergence Term}

Finally, we comment on the divergence contribution
$\mathbb{E}\big[|\Div(\tilde v(t,X_t)-v(t,X_t))|\big]$.
This is the most geometry-specific part of the discretization analysis.
In Euclidean space, the Euler interpolation simply freezes the velocity and no parallel transport is required; consequently, $\Div \tilde v$ can be handled with elementary calculus.
On a manifold, by contrast, our interpolated field $\tilde v$ is defined through (i) \emph{parallel transport} and (ii) an \emph{implicit inverse mapping} $F_{t_k,t-t_k}^{-1}$.
Both operations contribute nontrivially to $\nabla \tilde v$ and hence to $\Div \tilde v$.

Our approach is to explicitly differentiate the representation $$\tilde v(t,x)
=
P_{F_{t_k,t-t_k}^{-1}(x)}^{x}\,
\hat v\bigl(t_k, F_{t_k,t-t_k}^{-1}(x)\bigr),$$ for $t\in[t_k,t_{k+1}),$ along the interpolation geodesic under a parallel orthonormal frame, and decompose $\Div \tilde v(t,x)=\sum_{i=1}^{d}\big\langle \nabla_{E_i(t)} \tilde v(t,x),\, E_i(t)\big\rangle$ into two pieces:
\begin{enumerate}[noitemsep]
    \item Derivatives of $\hat v(t_k,\cdot)$ evaluated at the preimage point;
    \item Curvature distortion from differentiating parallel transport.
\end{enumerate}

Combining these estimates and using the curvature bounds in Assumption~\ref{A_Curvature}, we obtain a bound on $\Div \tilde v$ and hence on
$\Div(\tilde v-v)$, stated precisely in Lemma~\ref{Lemma_Divergence_tildev}.
This divergence control, together with the score/vector-field regularity results above, completes the term-by-term bounds needed to integrate \eqref{Eq_TV_derivative} and prove the stated TV-rate.

\subsection*{Acknowledgements}
KB is supported in part by National Science Foundation (NSF) grant DMS-2413426.

%% file: manifold.tex
\subsection{Riemannian Geometry Basics}\label{sec:basics}

Let $(M,g)$ be a $d$-dimensional Riemannian manifold, not necessarily realized as a submanifold of Euclidean space, where $g$ denotes the Riemannian metric. Unless otherwise specified, $\|\cdot\|$ denotes the norm induced by $g$. We use $T_{x}M$ to denote tangent space at location $x\in M$, and remark that $\|\cdot\|$ is well-defined on each
tangent space $T_xM$. We let $\Exp$ to denote exponential map, and also define the inverse of the exponential map, i.e., the logarithm map:
Given $X_{0}, X_{1} \in M$, $\Log_{X_{0}}(X_{1}) \in T_{X_{0}}M$ is such that $\Exp_{X_{0}}(\Log_{X_{0}}(X_{1})) = X_{1}$. We use $P_{x}^{y}$ to denote the parallel transport from $x$ to $y$, along the minimizing geodesic, unless otherwise specified. Based on this, we can define geodesic interpolation as follows. 
\begin{align}
    \label{Eq_Geodesic_Interp}
    X_{t} := \Exp_{X_{0}}(t\Log_{X_{0}}(X_{1}))
\end{align}
Note that $X_{t}' = P_{X_{0}}^{X_{t}} \Log_{X_{0}}(X_{1}) = \frac{1}{1-t}\Log_{X_{t}}(X_{1})$.

We use $\grad$ to denote the Riemannian gradient, $\Div$ the Riemannian
divergence, and $\nabla$ the Levi--Civita connection on $(M,g)$. For a smooth
function $f$, the gradient $\grad f$ is the unique vector field satisfying
$\langle \grad f,\, v\rangle = df(v)$ for all $v\in T_xM$. For a smooth vector
field $u$, the divergence is defined by $\Div u := \tr(\nabla u)$; equivalently,
for any local $g$-orthonormal frame $\{e_i\}_{i=1}^d$, $\Div u(x) \;=\; \sum_{i=1}^d \big\langle \nabla_{e_i}u,\, e_i\big\rangle_x$. For a vector field $u$ and a tangent vector $v\in T_xM$, the covariant
derivative $\nabla_v u(x)$ is defined as follows: choose any smooth local vector
field $V$ with $V(x)=v$, and set $\nabla_v u(x):=(\nabla_V u)(x)$; this is
well-defined (i.e., independent of the extension $V$) because $\nabla$ is
$C^\infty(M)$-linear in its first argument. The map $\nabla u(x)$ may be viewed
as a $(1,1)$-tensor (an endomorphism) $\nabla u(x):T_xM\to T_xM$ defined by $(\nabla u(x))(v) \;:=\; \nabla_v u(x)$. Accordingly, we define the operator norm induced by $g$ as
\[
\|\nabla u(x)\|_{\op}
\;:=\;
\sup_{v\neq 0}\frac{\|\nabla_v u(x)\|}{\|v\|}
\;=\;
\sup_{\|v\|=1}\|\nabla_v u(x)\|.
\]

We now provide some intuition on Levi--Civita connection and covariant derivatives. On a manifold, the tangent spaces $T_xM$ and $T_yM$ at different points are
distinct vector spaces, so one cannot subtract vectors at $x$ and $y$ directly.
The Levi--Civita connection $\nabla$ provides a canonical way to
\emph{differentiate} vector fields by comparing nearby tangent vectors in a
manner that is compatible with the metric ($\nabla g=0$) and has no torsion.
Along a curve $\gamma$, the induced covariant derivative
$D_t := \nabla_{\gamma'(t)}$ plays the role of a directional derivative: it
measures the intrinsic rate of change of a vector field along $\gamma$ after
accounting for the variation of tangent spaces, and it is the notion of
derivative that makes ``constant velocity'' along $\gamma$ coincide with the
geodesic equation.

Let $x\in M$. The injectivity radius at $x$, denoted $\inj(x)$, is defined as
\[
\inj(x)\;:=\;\sup\Bigl\{a>0:\ \Exp_x\big|_{B_a(0)}: B_a(0)\subset T_xM \to M
\text{ is a diffeomorphism onto its image}\Bigr\},
\]
where $B_a(0)$ is the open metric ball in $(T_xM,g_x)$. The injectivity radius
of $M$ is $\inj(M):=\inf_{x\in M}\inj(x)$.

Let $R$ denote the (Riemann) curvature tensor associated with $\nabla$, defined
for smooth vector fields $X,Y,Z$ by $R(X,Y)Z \;:=\; \nabla_X\nabla_Y Z \;-\; \nabla_Y\nabla_X Z \;-\; \nabla_{[X,Y]}Z$. For linearly independent $u,v\in T_xM$, the sectional curvature of the plane
$\sigma=\mathrm{span}\{u,v\}$ is
\[
K(\sigma)
\;:=\;
\frac{\langle R(u,v)v,\,u\rangle_x}{\|u\|^2\|v\|^2-\langle u,v\rangle_x^2},
\]
and in particular, if $u\perp v$ and $\|u\|=\|v\|=1$, then
$K(u,v)=\langle R(u,v)v,\,u\rangle_x$.

Let $\gamma:[a,b]\to M$ be a geodesic, and write $D_t:=\nabla_{\gamma'(t)}$ for
the covariant derivative along $\gamma$. A vector field $J$ along $\gamma$ is
called a \emph{Jacobi field} if it satisfies the Jacobi equation $
D_t^2 J \;+\; R(J,\gamma')\,\gamma' \;=\; 0$. Let $p=\gamma(a)$ and $q=\gamma(b)$. We say that $p$ and $q$ are \emph{conjugate
along $\gamma$} if there exists a nonzero Jacobi field $J$ along $\gamma$ such
that $J(a)=J(b)=0$.

For a unit-speed geodesic $\gamma:[0,\infty)\to M$ with $\gamma(0)=p$, define
its \emph{cut time} $t_{\textrm{cut}}(\gamma)
\;:=\;
\sup\{t>0:\ \gamma|_{[0,t]}\ \text{minimizes distance between its endpoints}\}.$ If $t_{\textrm{cut}}(\gamma)<\infty$, the \emph{cut point of $p$ along $\gamma$} is
$\gamma(t_{\textrm{cut}}(\gamma))$. The \emph{cut locus} of $p$, denoted $\Cut(p)$, is
the set of all cut points of $p$ over all unit-speed geodesics emanating from
$p$. It is a classical fact that along any geodesic starting at $p$, the cut
time occurs no later than the first conjugate time (if a \emph{conjugate point} occurs
at all); moreover, a cut point may occur strictly before the first conjugate
point when there are multiple minimizing geodesics to the same endpoint.

We now provide some intuition on the above concepts. The curvature tensor quantifies how covariant derivatives fail to commute; it
encodes the intrinsic bending of the manifold. Sectional curvature reduces this
information to a 2-dimensional direction $\sigma$: positive sectional curvature
tends to make nearby geodesics in $\sigma$ focus toward each other, while
negative curvature tends to make them spread apart. Jacobi fields describe the
infinitesimal separation between nearby geodesics (they arise from variations
of geodesics), so zeros of a Jacobi field correspond to a loss of local uniqueness/minimality of geodesics. Conjugate points describe a differential property, corresponding to the degeneracy of $d \exp$, hence naturally enter comparison theory used for our analysis later (see Section~\ref{sec:compthm}). On the other hand, cut points describe a geometric property, describing when geodesics cease to be minimizing, as well as when $\Log_{x}$ stops being single-valued.

%% file: Appendix_Manifold.tex
\section{Additional Background for Riemannian Geometry}
\subsection{Distances used}
\begin{definition}[Total variation distance and $W_{1}$ distance]\label{sec:defn}
Let $(M,g)$ be a Riemannian manifold, let $d$ denote the geodesic distance
induced by $g$, and let $\rho_{1},\rho_{2}$ be probability measures on the
Borel $\sigma$-algebra $\mathcal{B}(M)$.
\begin{itemize}
    \item The total variation distance between $\rho_{1}$ and $\rho_{2}$ is
    \[
        \TV(\rho_{1},\rho_{2})
        \;:=\;
        \sup_{A\in\mathcal{B}(M)} \bigl|\rho_{1}(A)-\rho_{2}(A)\bigr|.
    \]
    If $\rho_{1}$ and $\rho_{2}$ are absolutely continuous with respect to the
    Riemannian volume measure $dV_g$, with densities $p_{1}=\frac{d\rho_{1}}{dV_g}$
    and $p_{2}=\frac{d\rho_{2}}{dV_g}$, then
    \[
        \TV(\rho_{1},\rho_{2})
        \;=\;
        \frac12 \int_M |p_{1}(x)-p_{2}(x)|\, dV_g(x)
        \;=\;
        \int_{\{p_{1}>p_{2}\}} \bigl(p_{1}(x)-p_{2}(x)\bigr)\, dV_g(x).
    \]

    \item The $1$-Wasserstein distance (induced by $d$) is defined for measures
    with finite first moment (e.g., $\int_M d(x_0,x)\, d\rho_i(x)<\infty$ for some
    $x_0\in M$ and $i\in\{1,2\}$) by
    \[
        W_{1}(\rho_{1},\rho_{2})
        \;:=\;
        \inf_{\gamma \in \Pi(\rho_{1},\rho_{2})}
        \int_{M\times M} d(x,y)\, d\gamma(x,y),
    \]
    where $\Pi(\rho_{1},\rho_{2})$ denotes the set of couplings of
    $(\rho_{1},\rho_{2})$, i.e., probability measures $\gamma$ on $M\times M$
    whose first and second marginals are $\rho_{1}$ and $\rho_{2}$, respectively.
\end{itemize}
\end{definition}

\subsection{Comparison Theorem}\label{sec:compthm}
Comparison theorems~\citep{cheeger1975comparison} will be used frequently throughout the Appendix. For example, we used Rauch theorem in Appendix \ref{Sect_Aux_invertibility}, which is a key step in proving Lemma \ref{Lemma_F_Invertible}; we also use comparison theorems frequently in Appendix \ref{Sect_SPD} to verify Assumption \ref{A_Regularity_Expectation_V}. 

Here we provide a brief introduction to comparison theory, focusing on intuition. Let $k \in \mathbb{R}$ be some constant, and assume we have functions $f, g \ge 0$ defined on $[0, t]$. 
If both $f, g$ share the same initial value $f(0) = g(0)$, and we further assume $f'(0) \le g'(0)$. If we know, for second order ODEs, $f'' + \kappa f \le g'' + \kappa g$, then we can conclude, at least locally, $f \le g$. For more details, see, for example \citet[Lemma 4.8]{lezcano2020curvature}. Such an observation suggests a more general comparison principle, which is the key idea behind comparison theorems in Riemannian geometry. For example, we have the following theorem. 
\begin{theorem}[Proposition 4.9 in \cite{lezcano2020curvature}]
    Let $M$ be a Riemannian manifold with bounded sectional curvature satisfying $K_{\min} \le \Sec \le K_{\max}$. Let $\gamma: [0, r] \to M$ be a geodesic, and let $X, Y$ be vector fields along $\gamma$ with $X, Y \perp \gamma'$ satisfying the following ODE ($X'$is covariant derivative of $X$ along $\gamma$):
    \begin{align*}
        X'' + R(X, \gamma') \gamma' = Y, \qquad X(0) = 0, X'(0) = 0.
    \end{align*}
    Assume that there exists a continuous function $\eta$ that upper bounds $Y$: $\|Y\| \le \eta$ on $[0, r]$. Then $\rho$ defined as the solution of $\rho'' + K_{\min} \rho = \eta, \rho(0) = 0, \rho'(0) = 0$ upper bounds $X$: $\|X\| \le \rho$.
\end{theorem}

To summarize, using the comparison principle for ODE, we can make use of curvature information to establish bounds for certain vector fields defined on a Riemannian manifold, as long as the vector field satisfies some specific ODE. For a comprehensive discussion on comparison theorems, see for example \cite{lee2018introduction} and \cite{cheeger1975comparison}. 

\subsection{Riemannian Submanifolds}

We briefly introduce some results of Riemannian submanifolds following \cite{lee2018introduction}, and these results will be used in proving Lemma \ref{Lemma_gradDivLog_SPD_AIRM}.
Let $N \subseteq M$ be a Riemannian submanifold of $M$, with induced metric. To avoid ambiguity, we denote $\nabla^{M}$ to be the connection on $M$, and $\nabla^{N}$ to be the connection on $N$. 
Let $X, Y$ be vector fields on $N$. We can extend them to vector fields on $M$, and the covariant derivative $(\nabla^{M})_{X}Y$ can be decomposed as 
\begin{align*}
    (\nabla^{M})_{X}Y = ((\nabla^{M})_{X}Y)^{\perp} + ((\nabla^{M})_{X}Y)^{\parallel}.
\end{align*}
The normal component $((\nabla^{M})_{X}Y)^{\perp}$ defines the second fundamental form, which is a map from $\mathfrak{X}(N) \times \mathfrak{X}(N)$ to a section of the normal bundle of $N$, formally defined as 
\begin{align*}
    \mathrm{II}(X,Y) := ((\nabla^{M})_{X}Y)^{\perp}.
\end{align*}
Gauss formula \cite[Theorem 8.2]{lee2018introduction} states that, for $X, Y$ being vector fields on $N$ and extended arbitrarily to $M$, then 
\begin{align*}
    (\nabla^{M})_{X}Y = (\nabla^{N})_{X}Y + \mathrm{II}(X,Y).
\end{align*}
A closely related terminology is totally geodesic manifold. $N$ is said to be totally geodesic if every geodesic in $N$ is also a geodesic in $M$. Furthermore, by \citet[Proposition 8.12]{lee2018introduction}, we know the following are equivalent:
\begin{itemize}[noitemsep]
    \item $N$ is a totally geodesic submanifold of $M$.
    \item The second fundamental form of $N$ vanishes identically. 
    \item Every geodesic in $N$ is also a geodesic in $M$.
\end{itemize}

%% file: Appendix_Proof_Theoren.tex
\section{Proof of Main Theorem}

This Section is organized as follows. We first prove Lemma \ref{Lemma_TV_derivative} in Appendix \ref{Appendix_Proof_Lemma_TV}, which provides a way to control the propagation of TV distance along ODE simulation:
\begin{align*}
    \frac{\partial \TV(X_{t}, Y_{t})}{\partial t}
    = \int_{\Omega_{t}} p(t, x) \left(\Div (\tilde{v}(x, t) - v(x, t)) + \langle \grad \log p(t, x), \tilde{v}(x, t) - v(x, t) \rangle \right) dV_{g}(x).
\end{align*}

Note that the time derivative of TV distance consists of two parts:
\begin{itemize}
    \item We control the the ``velocity vector" term $p(t, x)\langle \grad \log p(t, x), \tilde{v}(x, t) - v(x, t) \rangle$ in Appendix \ref{Appendix_Vector_Term}.
    \item We control the ``divergence term" $p(t, x) \Div (\tilde{v}(x, t) - v(x, t))$ in Appendix \ref{Appendix_Div_Term}.
\end{itemize}

Moreover, recall that the well-definedness of $\tilde{v}$ depends on the invertibility of $F_{t, h}(x) := \Exp_{x}(h \hat{v}(t, x))$, which we justify in Appendix \ref{Appendix_Inver_F}. The proof of the main results (sampling error bound in Theorems~\ref{Theorem_Rate_Compact} and \ref{Theorem_Sampling_Error_Hadamard}) is presented in Appendix \ref{Appendix_Proof_MainThem}.
Finally, we present proofs for iteration complexity on the hypersphere and SPD manifold in Appendix \ref{Appendix_Hypersphere_Cor_Proof} and \ref{Appendix_SPD_Cor_Proof}, respectively.

Appendix~\ref{sec:auxmainthm} contains auxiliary results needed in the proofs of Theorems~\ref{Theorem_Rate_Compact} and \ref{Theorem_Sampling_Error_Hadamard}. Finally, Appendix~\ref{Sect_Hypersphere} and~\ref{Sect_SPD} proves the required regularity results needed for proving Propositions~\ref{Cor_Iter_Complexity_Sphere} and \ref{Cor_Iter_Complexity_SPD}. 

\subsection{Proof for Lemma \ref{Lemma_TV_derivative}}\label{Appendix_Proof_Lemma_TV}

We start with the following result. 
\begin{theorem}[Transport Theorem (\citeauthor{marsden2002manifolds}, \citeyear{marsden2002manifolds}, Theorem 8.1.12)] Let $(M, \mu)$ be a volume manifold and $X$ be a vector field on $M$ with flow $F_{t}$.
    For smooth function $f$ defined on $\mathcal{F}(M \times \mathbb{R})$, let $f_{t}(m) = f(m, t)$.
    We have that for any open set $U \subseteq M$, 
    \begin{align*}
        \frac{d}{dt} \int_{F_{t}(U)} f_{t} \mu = \int_{F_{t}(U)} \left( \frac{\partial f}{\partial t} + \Div_{\mu} (f_{t}X) \right) \mu.
    \end{align*}
\end{theorem}
\vspace{0.1in}
\begin{proof}[Proof of Lemma \ref{Lemma_TV_derivative}]
    We can write 
    \begin{align*}
        \TV(X_{t}, Y_{t}) = \int_{\Omega_{t}} p(t, x) - q(t, x) dV_{g}(x).
    \end{align*}
    By transport Theorem, 
    \begin{align*}
        \frac{\partial \TV(X_{t}, Y_{t})}{\partial t}
        &= \frac{d}{dt}\int_{\Omega_{t}} p(t, x) - q(t, x) dV_{g}(x)  \\
        &= \int_{\Omega_{t}} \left( \frac{\partial (p(t, x) - q(t, x))}{\partial t} + \Div ((p(t, x) - q(t, x))X) \right) dV_{g}(x)  \\
        &= \int_{\Omega_{t}} \frac{\partial (p(t, x) - q(t, x))}{\partial t} dV_{g}(x) 
        + \int_{\partial \Omega_{t}} (p(t, x) - q(t, x))\langle X, n \rangle dV_{\hat{g}}(x)  \\
        &= \int_{\Omega_{t}} \frac{\partial (p(t, x) - q(t, x))}{\partial t} dV_{g}(x) 
        + \int_{\partial \Omega_{t}} 0 \times \langle X, n \rangle dV_{\hat{g}}(x)  \\
        &= \int_{\Omega_{t}} \frac{\partial (p(t, x) - q(t, x))}{\partial t} dV_{g}(x) .
    \end{align*}

    The following continuity equation was proved in \citet[Theorem 2]{wu2025riemannian}.
    We provide a proof for completeness. For any test function $\varphi$ that is smooth (and of bounded support if the manifold is non-compact and does not have boundary), we have 
    \begin{align*}
        \int_{M} \varphi(x) \frac{\partial}{\partial t}p(t, x) dV_{g}(x)
        &= \frac{d}{dt} \int_{M} \varphi(x) p(t, x) dV_{g}(x)
        = \frac{d}{dt} \mathbb{E}[\varphi(x_{t})]
        = \mathbb{E}[\nabla \varphi(x_{t}) \circ \frac{d}{dt}x_{t}] \\
        &= \int_{M} \langle \grad \varphi(x), v(x, t) \rangle p(t, x) dV_{g}(x)\\
        &= - \int_{M} \varphi(x) \Div (v(x, t)p(t, x)) dV_{g}(x).
    \end{align*}
    Hence we conclude that 
    \begin{align*}
        \frac{\partial}{\partial t}p(t, x)
        = - \Div (v(x, t)p(t, x)).
    \end{align*}

    Therefore
    \begin{align*}
        \frac{\partial \TV(X_{t}, Y_{t})}{\partial t}
        &= \int_{\Omega_{t}} \frac{\partial (p(t, x) - q(t, x))}{\partial t} dV_{g}(x) \\
        &= \int_{\Omega_{t}} - \Div (v(x, t)p(t, x)) + \Div (\tilde{v}(x, t)q(t, x)) dV_{g}(x) \\
        &= \int_{\partial \Omega_{t}} - p(t, x) \langle v(x, t), n \rangle + q(t, x) \langle \tilde{v}(x, t), n \rangle dV_{\hat{g}}(x) \\
        &= \int_{\partial \Omega_{t}} p(t, x) \langle \tilde{v}(x, t) - v(x, t), n \rangle dV_{\hat{g}}(x) \\
        &= \int_{\Omega_{t}} \Div ((\tilde{v}(x, t) - v(x, t))p(t, x)) dV_{g}(x) \\
        &= \int_{\Omega_{t}} p(t, x)\Div (\tilde{v}(x, t) - v(x, t)) + \langle \grad p(t, x), \tilde{v}(x, t) - v(x, t) \rangle dV_{g}(x) \\
        &= \int_{\Omega_{t}} p(t, x) \left(\Div (\tilde{v}(x, t) - v(x, t)) + \langle \grad \log p(t, x), \tilde{v}(x, t) - v(x, t) \rangle \right) dV_{g}(x), 
    \end{align*}
    where observe that $p = q$ on $\partial \Omega_{t}$, and the last equality is due to the product rule of divergence. Note that $\grad p(t, x) = p(t, x) \grad \log p(t, x)$, thereby concluding the proof.
\end{proof}

\subsection{Velocity Vector Term}
\label{Appendix_Vector_Term}

In this section, we bound the velocity vector term. 
We remark that Lemma \ref{Lemma_Vec_Term_Error_Compact} and Lemma \ref{Lemma_Vec_Term_Error_nonCompact} are essentially the same, except that they are under different assumptions. 

\begin{lemma}
\label{Lemma_Vec_Term_Error_Compact}
    Under Assumption \ref{A_Regularity_V} and \ref{A_estimation_error}, we can bound 
    \begin{align*}
        &\mathbb{E}[\|\grad \log p(t, x)\| \cdot \|\tilde{v}(x, t) - v(x, t) \|] \\
        \le& \varepsilon \sqrt{ 2L_{t}^{\text{score}}} + \Bigl(L_{t}^{\text{score}} (2 (t-t_{k})^{2}(L_{t}^{\hat{v}, x}L^{\hat{v}} + L^{v} L_{t}^{\hat{v}, x} + L_{t}^{v, t} + L^{v} L_{t}^{v, x})^{2})\Bigr)^{\frac{1}{2}}.
    \end{align*}
\end{lemma}
\begin{proof}
    Using triangle inequality, we can write 
    \begin{align*}
        &\|\tilde{v}(t, X_{t}) - v(t, X_{t}) \| \\
        \le & \|P_{X_{t}}^{X_{t_{k}}}\tilde{v}(t, X_{t}) - \hat{v}(t_{k}, X_{t_{k}})\|
        + \|\hat{v}(t_{k}, X_{t_{k}}) - v(t_{k}, X_{t_{k}})\|
        + \|v(t_{k}, X_{t_{k}}) - P_{X_{t}}^{X_{t_{k}}}v(t, X_{t}) \|.
    \end{align*}
    Denote $X_{t \to t_{k}} = F_{t_{k}, t-t_{k}}^{-1}(X_{t})$.

For the first term above, we have 
        \begin{align*}
            &\|P_{X_{t}}^{X_{t_{k}}}\tilde{v}(t, X_{t}) - \hat{v}(t_{k}, X_{t_{k}})\| \\
            =& \|P_{X_{t}}^{X_{t_{k}}}P_{X_{t \to t_{k}}}^{X_{t}}\hat{v}(t_{k}, X_{t \to t_{k}}) - \hat{v}(t_{k}, X_{t_{k}})\|
            = \|P_{X_{t \to t_{k}}}^{X_{t}}\hat{v}(t_{k}, X_{t \to t_{k}}) - P_{X_{t_{k}}}^{X_{t}}\hat{v}(t_{k}, X_{t_{k}})\| \\
            \le & \|P_{X_{t \to t_{k}}}^{X_{t}}\hat{v}(t_{k}, X_{t \to t_{k}}) - \hat{v}(t_{k}, X_{t})\|
            + \|\hat{v}(t_{k}, X_{t}) - P_{X_{t_{k}}}^{X_{t}}\hat{v}(t_{k}, X_{t_{k}})\| \\
            \le & L_{t}^{\hat{v}, x} d(X_{t \to t_{k}}, X_{t}) + L_{t}^{\hat{v}, x} d(X_{t}, X_{t_{k}})
            = L_{t}^{\hat{v}, x}\left(d(F_{t_{k}, t-t_{k}}^{-1}(X_{t}), X_{t}) + d(X_{t}, X_{t_{k}}) \right),
        \end{align*}
        where we used the fact that parallel transport preserve norm.

        Note that $X_{t}$ is the trajectory of $X$ at time $t$. 
        Starting from $X_{t \to t_{k}}$, we have 
        \begin{align*}
            \Exp_{X_{t \to t_{k}}}((t-t_{k}) \hat{v}(t, X_{t \to t_{k}}) ) = X_{t}.
        \end{align*}
        Hence 
        \begin{align*}
            d(F_{t_{k}, t-t_{k}}^{-1}(X_{t}), X_{t}) = 
            (t-t_{k})\|\hat{v}(t, X_{t \to t_{k}})\| \le (t-t_{k}) L^{\hat{v}}.
        \end{align*}
        Also, we know $d(X_{t}, X_{t_{k}})$ is the distance of ODE trajectory, hence 
        \begin{align*}
            d(X_{t}, X_{t_{k}}) \le \int_{t_{k}}^{t} \|v(s, X_{s})\| ds 
            \le (t-t_{k}) L^{v}.
        \end{align*}
Now, for the second term, we have by assumption that $            \|\hat{v}(t_{k}, X_{t_{k}}) - v(t_{k}, X_{t_{k}})\| \le \varepsilon.$ For the third term,  
        \begin{align*}
            \|v(t_{k}, X_{t_{k}}) - P_{X_{t}}^{X_{t_{k}}}v(t, X_{t}) \|
            =& \|v(t_{k}, X_{t_{k}}) - v(t, X_{t_{k}}) \|
            + \|v(t, X_{t_{k}}) - P_{X_{t}}^{X_{t_{k}}}v(t, X_{t}) \| \\
            \le & (t-t_{k})L_{t}^{v, t} + d(X_{t}, X_{t_{k}}) L_{t}^{v, x}
            \le (t-t_{k})(L_{t}^{v, t} + L^{v} L_{t}^{v, x}).
        \end{align*}
Putting the above estimates  together, we obtain
    \begin{align*}
        &\|\tilde{v}(t, X_{t}) - v(t, X_{t}) \|^{2} \\
        \le & \Bigl((t-t_{k})L_{t}^{\hat{v}, x}(L^{\hat{v}} + L^{v} ) 
        + \varepsilon + (t-t_{k})(L_{t}^{v, t} + L^{v} L_{t}^{v, x})\Bigr)^{2} \\
        \le & 
        2 \varepsilon^{2} + 2 (t-t_{k})^{2}(L_{t}^{\hat{v}, x}L^{\hat{v}} + L^{v} L_{t}^{\hat{v}, x} + L_{t}^{v, t} + L^{v} L_{t}^{v, x})^{2}.
    \end{align*}
    Hence 
    \begin{align*}
        &\mathbb{E}[\|\grad \log p(t, x)\| \cdot \|\tilde{v}(x, t) - v(x, t) \|] \\
        \le& \mathbb{E}[\|\grad \log p(t, x)\|^{2}]^{\frac{1}{2}} \mathbb{E}[ \|\tilde{v}(x, t) - v(x, t) \|^{2}]^{\frac{1}{2}} \\
        \le& \Bigl(L_{t}^{\text{score}} (2 \varepsilon^{2} + 2 (t-t_{k})^{2}(L_{t}^{\hat{v}, x}L^{\hat{v}} + L^{v} L_{t}^{\hat{v}, x} + L_{t}^{v, t} + L^{v} L_{t}^{v, x})^{2})\Bigr)^{\frac{1}{2}} \\
        \le& \varepsilon \sqrt{ 2L_{t}^{\text{score}}} + \Bigl(L_{t}^{\text{score}} (2 (t-t_{k})^{2}(L_{t}^{\hat{v}, x}L^{\hat{v}} + L^{v} L_{t}^{\hat{v}, x} + L_{t}^{v, t} + L^{v} L_{t}^{v, x})^{2})\Bigr)^{\frac{1}{2}},
    \end{align*}
    where note that 
$\sqrt{\mathbb{E}[A+B]} \le \sqrt{\mathbb{E}[A]} + \sqrt{\mathbb{E}[B]}$.
\end{proof}

For Hadamard manifolds, when regularity conditions hold in expectation, we have the following result. Note that the proof strategy is essentially the same as the previous case.
\begin{lemma}\label{Lemma_Vec_Term_Error_nonCompact}
    Under Assumption \ref{A_Regularity_Expectation_V}, \ref{A_Regularity_Learned_Pointwise} and \ref{A_estimation_error_Expectation}, we can bound 
    \begin{align*}
        &\mathbb{E}[\|\grad \log p(t, x)\| \cdot \|\tilde{v}(x, t) - v(x, t) \|] \\
        \le& \varepsilon \sqrt{ 3L_{t}^{\text{score}}} 
        + \Bigl(L_{t}^{\text{score}} (6 (t-t_{k})^{2} \left( (L_{t}^{\hat{v}, x})^{2} (L_{t}^{\hat{v}})^{2} + (L_{t}^{\hat{v}, x})^{2} (L_{t}^{v})^{2} + (L_{t}^{v, t})^{2} + (L_{t}^{v})^{2} (L_{t}^{v, x})^{2}\right))\Bigr)^{\frac{1}{2}}.
    \end{align*}
\end{lemma}
\begin{proof}
    Using triangle inequality, we can write 
    \begin{align*}
        &\|\tilde{v}(t, X_{t}) - v(t, X_{t}) \| \\
        \le & \|P_{X_{t}}^{X_{t_{k}}}\tilde{v}(t, X_{t}) - \hat{v}(t_{k}, X_{t_{k}})\|
        + \|\hat{v}(t_{k}, X_{t_{k}}) - v(t_{k}, X_{t_{k}})\|
        + \|v(t_{k}, X_{t_{k}}) - P_{X_{t}}^{X_{t_{k}}}v(t, X_{t}) \|.
    \end{align*}
    Hence using Cauchy-Schwarz, 
    \begin{align*}
        &\mathbb{E}[\|\tilde{v}(t, X_{t}) - v(t, X_{t}) \|^{2}] \\
        \le & \mathbb{E}[3\|P_{X_{t}}^{X_{t_{k}}}\tilde{v}(t, X_{t}) - \hat{v}(t_{k}, X_{t_{k}})\|^{2}
        + 3\|\hat{v}(t_{k}, X_{t_{k}}) - v(t_{k}, X_{t_{k}})\|^{2}
        + 3\|v(t_{k}, X_{t_{k}}) - P_{X_{t}}^{X_{t_{k}}}v(t, X_{t}) \|^{2}].
    \end{align*}
    Denote $X_{t \to t_{k}} = F_{t_{k}, t-t_{k}}^{-1}(X_{t})$.

 For the first term, we have 
        \begin{align*}
            &\mathbb{E}\|P_{X_{t}}^{X_{t_{k}}}\tilde{v}(t, X_{t}) - \hat{v}(t_{k}, X_{t_{k}})\|^{2} \\
            =& \mathbb{E}\|P_{X_{t}}^{X_{t_{k}}}P_{X_{t \to t_{k}}}^{X_{t}}\hat{v}(t_{k}, X_{t \to t_{k}}) - \hat{v}(t_{k}, X_{t_{k}})\|^{2}
            = \mathbb{E}\|P_{X_{t \to t_{k}}}^{X_{t}}\hat{v}(t_{k}, X_{t \to t_{k}}) - P_{X_{t_{k}}}^{X_{t}}\hat{v}(t_{k}, X_{t_{k}})\|^{2} \\
            \le & 2\mathbb{E}[\|P_{X_{t \to t_{k}}}^{X_{t}}\hat{v}(t_{k}, X_{t \to t_{k}}) - \hat{v}(t_{k}, X_{t})\|^{2}
            + 2\|\hat{v}(t_{k}, X_{t}) - P_{X_{t_{k}}}^{X_{t}}\hat{v}(t_{k}, X_{t_{k}})\|^{2}] \\
            \le & 2(L_{t}^{\hat{v}, x})^{2} \mathbb{E}[d(X_{t \to t_{k}}, X_{t})^{2}] + 2(L_{t}^{\hat{v}, x})^{2} \mathbb{E}[d(X_{t}, X_{t_{k}})^{2}] \\
            = & 2(L_{t}^{\hat{v}, x})^{2}\mathbb{E}\left[d(F_{t_{k}, t-t_{k}}^{-1}(X_{t}), X_{t})^{2} + d(X_{t}, X_{t_{k}})^{2} \right],
        \end{align*}
        where we used the fact that parallel transport preserve norm.

        Note that $X_{t}$ is the trajectory of $X$ at time $t$. 
        Starting from $X_{t \to t_{k}}$, we have 
        \begin{align*}
            \Exp_{X_{t \to t_{k}}}((t-t_{k}) \hat{v}(t, X_{t \to t_{k}}) ) = X_{t}.
        \end{align*}
        Hence 
        \begin{align*}
            \mathbb{E}[d(F_{t_{k}, t-t_{k}}^{-1}(X_{t}), X_{t})^{2}] = 
            \mathbb{E}[(t-t_{k})^{2}\|\hat{v}(t, X_{t \to t_{k}})\|^{2}] \le (t-t_{k})^{2} (L_{t}^{\hat{v}})^{2}.
        \end{align*}
        Also, we know $d(X_{t}, X_{t_{k}})$ is the distance of ODE trajectory, hence 
        \begin{align*}
            \mathbb{E}[d(X_{t}, X_{t_{k}})^{2}] &\le \mathbb{E}[(\int_{t_{k}}^{t} \|v(s, X_{s})\| ds )^{2}]
            \le (t-t_{k})\mathbb{E}[\int_{t_{k}}^{t} \|v(s, X_{s})\|^{2} ds ] \\
            &= (t-t_{k})\int_{t_{k}}^{t}\mathbb{E}[ \|v(s, X_{s})\|^{2} ] ds 
            \le (t-t_{k})^{2}(L_{t}^{v})^{2} .
        \end{align*}
Now, for the second term, we have
        \begin{align*}
            \mathbb{E}\|\hat{v}(t_{k}, X_{t_{k}}) - v(t_{k}, X_{t_{k}})\|^{2} \le \varepsilon^{2}.
        \end{align*}
Finally, for the third term, by swapping the order of the derivative and parallel transport, we get 
    \begin{align*}
        \|v(t_k,X_{t_k})-P_{X_t}^{X_{t_k}}v(t,X_t)\|
    \le \int_{t_k}^t \Big\|\frac{D}{ds} P_{X_s}^{X_{t_k}} v(s,X_s) \Big\|\,ds
    = \int_{t_k}^t \Big\| \partial_s v(s,X_s)+\nabla_{\dot X_s}v(s,X_s)\Big\|\,ds.
    \end{align*}
    Therefore,
    \[
    \mathbb E\|v(t_k,X_{t_k})-P_{X_t}^{X_{t_k}}v(t,X_t)\|^2
    \le (t-t_k)\int_{t_k}^t \mathbb E\Big\| \partial_s v(s,X_s)+\nabla_{\dot X_s}v(s,X_s)\Big\|^2 ds.
    \]
    Notice that 
    \begin{align*}
        \Big\| \partial_s v(s,X_s)+\nabla_{\dot X_s}v(s,X_s)\Big\|^2 \lesssim \| \partial_s v(s,X_s)\|^2+ \| \nabla v(s,X_s)\|^2 \| v(s, X_{s}) \|^{2},
    \end{align*}
    so we have 
    \begin{align*}
        \mathbb E\|v(t_k,X_{t_k})-P_{X_t}^{X_{t_k}}v(t,X_t)\|^2
    \le & (t-t_k)\int_{t_k}^t \mathbb E\Big\| \partial_s v(s,X_s)+\nabla_{\dot X_s}v(s,X_s)\Big\|^2 ds \\
    \lesssim & (t-t_k)^2 \mathbb{E} [\| \partial_s v(s,X_s)\|^2+ \| \nabla v(s,X_s)\|^2 \| v(s, X_{s}) \|^{2}] \\
    \lesssim & (t-t_k)^2 ((L_{t}^{v, t})^{2} +  (L_{t}^{v, x} L_{t}^{v})^{2}).
    \end{align*}

    Putting the above estimates together, we obtain
    \begin{align*}
        &\mathbb{E}[\|\tilde{v}(t, X_{t}) - v(t, X_{t}) \|^{2}] \\
        \le & \mathbb{E}[3\|P_{X_{t}}^{X_{t_{k}}}\tilde{v}(t, X_{t}) - \hat{v}(t_{k}, X_{t_{k}})\|^{2}
        + 3\|\hat{v}(t_{k}, X_{t_{k}}) - v(t_{k}, X_{t_{k}})\|^{2}
        + 3\|v(t_{k}, X_{t_{k}}) - P_{X_{t}}^{X_{t_{k}}}v(t, X_{t}) \|^{2}] \\
        \le & 6(L_{t}^{\hat{v}, x})^{2} \left((t-t_{k})^{2} (L_{t}^{\hat{v}})^{2} + (t-t_{k})^{2}(L_{t}^{v})^{2} \right)
        + 3\varepsilon^{2} 
        + 6(t-t_{k})^{2}\left((L_{t}^{v, t})^{2} + (L_{t}^{v})^{2} (L_{t}^{v, x})^{2} \right) \\
        = & 6 (t-t_{k})^{2} \left( (L_{t}^{\hat{v}, x})^{2} (L_{t}^{\hat{v}})^{2} + (L_{t}^{\hat{v}, x})^{2} (L_{t}^{v})^{2} + (L_{t}^{v, t})^{2} + (L_{t}^{v})^{2} (L_{t}^{v, x})^{2}\right)
        + 3\varepsilon^{2}.
    \end{align*}
    Hence 
    \begin{align*}
        &\mathbb{E}[\|\grad \log p(t, x)\| \cdot \|\tilde{v}(x, t) - v(x, t) \|] \\
        \le& \mathbb{E}[\|\grad \log p(t, x)\|^{2}]^{\frac{1}{2}} \mathbb{E}[ \|\tilde{v}(x, t) - v(x, t) \|^{2}]^{\frac{1}{2}} \\
        \le& \Bigl(L_{t}^{\text{score}} 
        (6 (t-t_{k})^{2} \left( (L_{t}^{\hat{v}, x})^{2} (L_{t}^{\hat{v}})^{2} + (L_{t}^{\hat{v}, x})^{2} (L_{t}^{v})^{2} + (L_{t}^{v, t})^{2} + (L_{t}^{v})^{2} (L_{t}^{v, x})^{2}\right)
        + 3\varepsilon^{2} )\Bigr)^{\frac{1}{2}} \\
        \le& \varepsilon \sqrt{ 3L_{t}^{\text{score}}} 
        + \Bigl(L_{t}^{\text{score}} (6 (t-t_{k})^{2} \left( (L_{t}^{\hat{v}, x})^{2} (L_{t}^{\hat{v}})^{2} + (L_{t}^{\hat{v}, x})^{2} (L_{t}^{v})^{2} + (L_{t}^{v, t})^{2} + (L_{t}^{v})^{2} (L_{t}^{v, x})^{2}\right))\Bigr)^{\frac{1}{2}}.
    \end{align*}
\end{proof}

\subsection{Divergence Term}
\label{Appendix_Div_Term}

In this section, we bound the divergence term. 
We remark that Lemma \ref{Lemma_Div_Term_Compact} and Lemma \ref{Lemma_Div_Term_nonCompact} are essentially the same, except that they are under different assumptions. 

Denote $z = F_{t_{k}, t-t_{k}}^{-1}(x)$.
We know $\Exp_{z}((t-t_{k})\hat{v}(t_{k}, z)) = x$. 
Recall we denote 
\begin{align*}
    \tilde{v}(t, x) = P_{F_{t_{k}, t-t_{k}}^{-1}(x)}^{x} \hat{v}(t_{k}, F_{t_{k}, t-t_{k}}^{-1}(x)),
\end{align*}
by simply, 
\begin{align*}
    \tilde{v}(x) = P_{z}^{\Exp_{z}(h\hat{v}(z))} \hat{v}(z)
\end{align*}
for notational simplicity.

\begin{lemma}\label{Lemma_Div_Term_Compact}
    Under Assumption \ref{A_Regularity_V} and \ref{A_estimation_error}, we have 
    \begin{align*}
        &\mathbb{E}_{p(t, x)} [\Div (\tilde{v}(x, t) - v(t, X_{t}))] 
        \le 
        \varepsilon
        + (t-t_{k}) \Bigl(L_{t}^{\Div, x} (L^{\hat{v}} + 2L^{v}) + L_{t}^{\Div, t} + L_{R}(L^{\hat{v}})^{2} d \Bigr).
    \end{align*}
\end{lemma}
\begin{proof}
    By Lemma \ref{Lemma_Divergence_tildev}, we have 
    \begin{align*}
        &\mathbb{E}_{p(t, x)} [\Div (\tilde{v}(x, t) - v(t, X_{t}))] \\
        \le & 
        \mathbb{E}_{p(t, x)} [\left|\Div \hat{v}(t_{k}, z) - \Div v(t, X_{t}) \right|]
            + L_{R} (t-t_{k}) d \mathbb{E}_{p(t, x)} [\|\hat{v}(t_{k}, z)\|^{2}] \\
        \le &
        \mathbb{E}_{p(t, x)} [\left|\Div \hat{v}(t_{k}, z) - \Div v(t_{k}, z) \right|]
        + \mathbb{E}_{p(t, x)} [\left|\Div v(t_{k}, z) - \Div v(t_{k}, X_{t_{k}}) \right|] \\
        &+ \mathbb{E}_{p(t, x)} [\left|\Div v(t_{k}, X_{t_{k}}) - \Div v(t, X_{t}) \right|]
            + L_{R} (t-t_{k}) (L^{\hat{v}})^{2} d.
    \end{align*}
    Recall 
    \begin{align*}
        d(F_{t_{k}, t-t_{k}}^{-1}(X_{t}), X_{t}) &= 
        (t-t_{k})\|\hat{v}(t, X_{t \to t_{k}})\| \le (t-t_{k}) L^{\hat{v}}, \\
        d(X_{t}, X_{t_{k}}) &\le \int_{t_{k}}^{t} \|v(s, X_{s})\| ds 
        \le (t-t_{k}) L^{v}.
    \end{align*}
    For the first term, we simply have 
    \begin{align*}
        \mathbb{E}_{p(t, x)} [\left|\Div \hat{v}(t_{k}, z) - \Div v(t_{k}, z) \right|]
        \le \varepsilon.
    \end{align*}

For the second term, we have
    \begin{align*}
        &\mathbb{E}_{p(t, x)} [\left| \Div v(t_{k}, z) - \Div v(t_{k}, X_{t_{k}}) \right|] 
        \le L_{t}^{\Div, x} \mathbb{E}_{p(t, x)} [d(z, X_{t_{k}})]\\
        \le & L_{t}^{\Div, x}\left(\mathbb{E}[d(F_{t_{k}, t-t_{k}}^{-1}(X_{t}), X_{t}) + d(X_{t}, X_{t_{k}})] \right) \\
        \le & L_{t}^{\Div, x} (t-t_{k}) (L^{\hat{v}} + L^{v}) .
    \end{align*}

Similarly, for the third term, 
    \begin{align*}
        &\mathbb{E}_{p(t, x)} [\left|\Div v(t_{k}, X_{t_{k}}) - \Div v(t, X_{t}) \right|] \\
        \le &\mathbb{E}_{p(t, x)} [\left|\Div v(t_{k}, X_{t_{k}}) - \Div v(t, X_{t_{k}}) \right|] 
        + \mathbb{E}_{p(t, x)} [\left|\Div v(t, X_{t_{k}}) - \Div v(t, X_{t}) \right|] \\
        \le & L_{t}^{\Div, t} (t-t_{k})
        + L_{t}^{\Div, x} d(X_{t}, X_{t_{k}}) 
        \le (t-t_{k})(L_{t}^{\Div, t} + L_{t}^{\Div, x}L^{v}).
    \end{align*}

Putting the above estimates together,
we obtain 
\begin{align*}
        &\mathbb{E}_{p(t, x)} [\Div (\tilde{v}(x, t) - v(t, X_{t}))] \\
        \le &
        \mathbb{E}_{p(t, x)} [\left|\Div \hat{v}(t_{k}, z) - \Div v(t_{k}, z) \right|]
        + \mathbb{E}_{p(t, x)} [\left|\Div v(t_{k}, z) - \Div v(t_{k}, X_{t_{k}}) \right|] \\
        &+ \mathbb{E}_{p(t, x)} [\left|\Div v(t_{k}, X_{t_{k}}) - \Div v(t, X_{t}) \right|]
            + L_{R} (t-t_{k}) (L^{\hat{v}})^{2} d\\
        \le &
        \varepsilon
        + L_{t}^{\Div, x} (t-t_{k}) (L^{\hat{v}} + L^{v}) + (t-t_{k})(L_{t}^{\Div, t} + L_{t}^{\Div, x}L^{v})
            + L_{R} (t-t_{k}) (L^{\hat{v}})^{2} d\\
        \le &
        \varepsilon
        + (t-t_{k}) \Bigl(L_{t}^{\Div, x} (L^{\hat{v}} + 2L^{v}) + L_{t}^{\Div, t} + L_{R}(L^{\hat{v}})^{2} d \Bigr).
    \end{align*}
\end{proof}

\begin{lemma}\label{Lemma_Div_Term_nonCompact}
    Under Assumption \ref{A_Regularity_Expectation_V}, \ref{A_Regularity_Learned_Pointwise} and \ref{A_estimation_error_Expectation}, we have 
    \begin{align*}
        &\mathbb{E}_{p(t, x)} [\Div (\tilde{v}(x, t) - v(t, X_{t}))]\\ 
        \le 
        &3\varepsilon
        + (t-t_{k}) \Bigl(L_{t}^{\Div \hat{v}, x} (L_{t}^{\hat{v}} + L_{t}^{v}) + L_{t}^{\Div, x} L_{t}^{v} + L_{t}^{\Div, t} + L_{R}(L_{t}^{\hat{v}})^{2} d \Bigr).
    \end{align*}
\end{lemma}
\begin{proof}
    By Lemma \ref{Lemma_Divergence_tildev}, we have 
    \begin{align*}
        &\mathbb{E}_{p(t, x)} [\Div (\tilde{v}(x, t) - v(t, X_{t}))] \\
        \le & 
        \mathbb{E}_{p(t, x)} [\left|\Div \hat{v}(t_{k}, z) - \Div v(t, X_{t}) \right|]
            + L_{R} (t-t_{k}) d \mathbb{E}_{p(t, x)} [\|\hat{v}(t_{k}, z)\|^{2}] \\
        \le &
        \mathbb{E}_{p(t, x)} [\left|\Div \hat{v}(t_{k}, z) - \Div v(t_{k}, z) \right|]
        + \mathbb{E}_{p(t, x)} [\left|\Div v(t_{k}, z) - \Div v(t_{k}, X_{t_{k}}) \right|] \\
        &+ \mathbb{E}_{p(t, x)} [\left|\Div v(t_{k}, X_{t_{k}}) - \Div v(t, X_{t}) \right|]
            + L_{R} (t-t_{k}) (L^{\hat{v}})^{2} d.
    \end{align*}
    Recall 
    \begin{align*}
        \mathbb{E}[d(F_{t_{k}, t-t_{k}}^{-1}(X_{t}), X_{t})^{2}] 
        \le & (t-t_{k})^{2} (L_{t}^{\hat{v}})^{2}, \\
        \mathbb{E}[d(X_{t}, X_{t_{k}})^{2}] \le & (t-t_{k})^{2}(L_{t}^{v})^{2}.
    \end{align*}
For the first term, we simply have 
        \begin{align*}
            \mathbb{E}_{p(t, x)} [\left|\Div \hat{v}(t_{k}, z) - \Div v(t_{k}, z) \right|]
            \le \varepsilon.
        \end{align*}
For the second term, by triangle inequality, we have
        \begin{align*}
            \big|\Div v(t_k, z) - \Div v(t_k, X_{t_k})\big|
            \le\;& \big|\Div v(t_k, z) - \Div \hat{v}(t_k, z)\big| \\
            &+ \big|\Div \hat{v}(t_k, z) - \Div \hat{v}(t_k, X_{t_k})\big| \\
            &+ \big|\Div \hat{v}(t_k, X_{t_k}) - \Div v(t_k, X_{t_k})\big|.
        \end{align*}
        Taking expectation under $p(t,x)$ and using Assumption~\ref{A_estimation_error_Expectation},
        \begin{align*}
            \mathbb{E}_{p(t,x)}\big[\big|\Div v(t_k, z) - \Div v(t_k, X_{t_k})\big|\big]
            \le 2\varepsilon
            + \mathbb{E}_{p(t,x)}\big[\big|\Div \hat{v}(t_k, z) - \Div \hat{v}(t_k, X_{t_k})\big|\big].
        \end{align*}
        Moreover, by the pointwise spatial regularity of $\hat v$ (equivalently, a pointwise bound on
        $\|\grad_x \Div \hat v(t_k,\cdot)\|$), the function $\Div \hat v(t_k,\cdot)$ is Lipschitz:
        \begin{align*}
            \big|\Div \hat{v}(t_k, z) - \Div \hat{v}(t_k, X_{t_k})\big|
            \le L_t^{\Div \hat v, x}\, d(z, X_{t_k}).
        \end{align*}
        Finally, using $d(z,X_{t_k})\le d(z,X_t)+d(X_t,X_{t_k})$ and the bounds
        \[
        \mathbb{E}[d(z,X_t)]\le (t-t_k)L_t^{\hat v},
        \qquad
        \mathbb{E}[d(X_t,X_{t_k})]\le (t-t_k)L_t^{v},
        \]
        we obtain
        \begin{align*}
            \mathbb{E}_{p(t,x)}\big[\big|\Div v(t_k, z) - \Div v(t_k, X_{t_k})\big|\big]
            \le 2\varepsilon + (t-t_k)L_t^{\Div \hat v, x}(L_t^{\hat v}+L_t^{v}).
        \end{align*}
For the third term, by the fundamental theorem of calculus along the curve, we have
\begin{align*}
            \Div v(t,X_t) - \Div v(t_k,X_{t_k})
            = \int_{t_k}^{t} \left(\partial_s \Div v(s,X_s) + \langle \grad_x \Div v(s,X_s), \dot X_s\rangle \right)\, ds.
        \end{align*}
        Equivalently,
        \begin{align*}
            \big|\Div v(t, X_t) - \Div v(t_k, X_{t_k})\big|
            \le \int_{t_k}^{t} \left|\partial_s (\Div v)(s,X_s) + \langle \grad_x \Div v(s,X_s), v(s,X_s)\rangle\right|\, ds.
        \end{align*}
        Hence we obtain
        \begin{align*}
            &\mathbb{E}_{p(t,x)}\big[\big|\Div v(t, X_t) - \Div v(t_k, X_{t_k})\big|\big] \\
            \le & \mathbb{E}_{p(t,x)}\Big[ \int_{t_k}^{t} \left|\partial_s (\Div v)(s,X_s) + \langle \grad_x \Div v(s,X_s), v(s,X_s)\rangle\right|\, ds \Big] \\
            \le & \int_{t_k}^{t} \mathbb{E}_{p(t,x)}\Big[\left|\partial_s (\Div v)(s,X_s)| + |\langle \grad_x \Div v(s,X_s), v(s,X_s)\rangle\right|\, \Big] ds \\
            \lesssim & (t-t_{k}) (L_{t}^{\Div, t} + L_{t}^{\Div, x} L_{t}^{v}).
        \end{align*}
Putting together, the above estimates, we obtain
    \begin{align*}
        &\mathbb{E}_{p(t, x)} [\Div (\tilde{v}(x, t) - v(t, X_{t}))] \\
        \le &
        \mathbb{E}_{p(t, x)} [\left|\Div \hat{v}(t_{k}, z) - \Div v(t_{k}, z) \right|]
        + \mathbb{E}_{p(t, x)} [\left|\Div v(t_{k}, z) - \Div v(t_{k}, X_{t_{k}}) \right|] \\
        &+ \mathbb{E}_{p(t, x)} [\left|\Div v(t_{k}, X_{t_{k}}) - \Div v(t, X_{t}) \right|]
            + L_{R} (t-t_{k}) (L_{t}^{\hat{v}})^{2} d\\
        \le &
        3\varepsilon
        + L_{t}^{\Div \hat{v}, x} (t-t_{k}) (L_{t}^{\hat{v}} + L_{t}^{v}) + (t-t_{k})(L_{t}^{\Div, t} + L_{t}^{\Div, x}L_{t}^{v})
            + L_{R} (t-t_{k}) (L_{t}^{\hat{v}})^{2} d\\
        \le &
        3\varepsilon
        + (t-t_{k}) \Bigl(L_{t}^{\Div \hat{v}, x} (L_{t}^{\hat{v}} + L_{t}^{v}) + L_{t}^{\Div, x} L_{t}^{v} + L_{t}^{\Div, t} + L_{R}(L_{t}^{\hat{v}})^{2} d \Bigr). 
    \end{align*}
\end{proof}

\subsection{Results on Riemannian manifolds}
\label{Appendix_Inver_F}

Note that in Lemma \ref{Lemma_TV_derivative}, 
we do not require vector fields $v, \tilde{v}$ to be the flow matching vector field. 
To apply the Lemma for analyzing the discretization scheme, 
we will set $v$ to be the true vector field for flow matching, and $\tilde{v}$ to be corresponds to the learned vector field.

But however, there is a discrepency between continuous time ODE $dY_{t} = \tilde{v}(t, Y_{t}) dt$ and the Euler discretization scheme. 
Hence, to apply the Lemma to discretization (the actual method in Algorithm \textcolor{blue}{1}), 
we need to define a continuous time interpolation.

Define $F$ as $F_{t, h}(x) := \Exp_{x}(h \hat{v}(t, x))$.
Note that $$F_{t_{k}, t-t_{k}}(x_{k}) = \Exp_{x_{k}}((t-t_{k}) \hat{v}(t_{k}, x_{k}))$$ 
corresponds to the continuous time interpolation of Euler discretization.
Then we are able to define a interpolation vector field: 
we want to define $$dY_{t} = P_{Y_{t_{k}}}^{Y_{t}} \hat{v}(t_{k}, Y_{t_{k}}) dt \overset{?}{:=} \tilde{v}(t, Y_{t}) dt.$$
Here we use the question mark to emphasize that $\tilde{v}$ has not yet been proved to be well defined. We want to write the right hand side as a function of $(t, Y_{t}) = (t, F_{t_{k}, t-t_{k}}(Y_{t_{k}}))$.


\begin{lemmanostar}[Restated Lemma \ref{Lemma_F_Invertible}]
    Let $M$ be simply connected Riemannian manifold that satisfies Assumption \ref{A_Curvature}.
    Let $b$ be any vector field on $M$, satisfying $\| b(x) \| \le B, \forall x \in M$.
    Assume $\|\nabla_{v}b(x)\| \le L_{\nabla} \|v\|$.
    Let $R = \inj(M)$. 
    To guarantee $F_{t_{k}, t-t_{k}}$ being invertible:
    \begin{enumerate}
        \item If $K_{\min} > 0$, we require
        \begin{align*}
            h < \min\{\frac{R}{B}, \frac{1}{4L_{\nabla}}, \sqrt{\frac{3}{4\|b(x)\|^{2} L_{R} (2 + 2 L_{\nabla}\max\{\frac{1}{\sqrt{K_{\min}}}, 1\})}}\}.
        \end{align*}
        \item If $K_{\min} < 0$, we require 
        \begin{align*}
            h < \min\{\frac{R}{B}, \frac{1}{4L_{\nabla}}, \sqrt{\frac{3}{4\|b(x)\|^{2} L_{R} (2\frac{\sinh (\sqrt{-K_{\min}})}{\sqrt{-K_{\min}}} + 4\frac{\cosh (\sqrt{-K_{\min}}) - 1}{- K_{\min}} L_{\nabla})}}\}.
        \end{align*}
        \item If $K_{\min} = 0$, we require 
        \begin{align*}
            h < \min\{\frac{R}{B}, \frac{1}{4L_{\nabla}}, \sqrt{\frac{3}{4\|b(x)\|^{2} L_{R} (2 + h L_{\nabla})}}\}.
        \end{align*}
    \end{enumerate}
\end{lemmanostar}

\begin{proof}[Proof of Lemma \ref{Lemma_F_Invertible}]
    For ease of notation, we use $h$ to denote the time step and $b$ to denote the learned vector field, 
    for the map $F$. 
    That is, for $h \in \mathbb{R}$ and $b \in \mathfrak{X}(M)$, we write 
    \begin{align*}
        F_{h} : M \to M, 
        \qquad 
        F_{h}(x) := \Exp_x\big(h b(x)\big).
    \end{align*}

    We first compute the derivative of $F$. 
    Let $x \in M$ and $v \in T_{x}M$. 
    Note that $dF_{h}(x)$ can be viewed as an operator that maps $v \in T_{x}M$ to $dF_{h}(x)v \in T_{F_{h}(x)}M$.
    We compute $dF_{h}(x)v$.
    Let $c(s)$ be a smooth curve s.t. $c(0) = x$ and $c'(0) = v$.

    Define a variation through geodesics as 
    \begin{align*}
        \Lambda : (-\epsilon,\epsilon) \times [0,1] \to M,
        \qquad
        \Lambda(s,t) := \Exp_{c(s)}\big(t h b(c(s))\big).
    \end{align*}
    For each fixed $s$, the $t$ direction is the geodesic starting at $c(s)$
    with initial velocity $ h b(c(s))$.
    In particular, for every $s$ we have
    \begin{align*}
        F_{h}(b(c(s))) = \Exp_{c(s)} (h b(c(s))) = \Lambda(s,1).
    \end{align*}

    Now define a vector field $J_{v}$ along the central geodesic
    $\gamma(t) := \Lambda(0,t) = \Exp_{x}\big(t h b(x)\big)$ by
    \begin{align*}
        J_v(t) := \partial_{s} \Lambda(0,t).
    \end{align*}
    Then, by construction (view the left hand side as directional derivative along the curve induced by $v$),
    \begin{align*}
        dF_{h}(x)(v) = \frac{d}{ds} F_{h}(b(c(s)))\Big|_{s = 0} = \partial_{s} \Lambda(0,1) = J_v(1).
    \end{align*}
    This expresses the differential of $F_{h}$ at $x$ applied to $v$ as the value at $t=1$
    of the variation field $J_{v}$ along the geodesic $\gamma$.
    Note that $J_{v}$ is a Jacobi field, hence satisfies the Jacobi equation, see \citet[Theorem 10.1]{lee2018introduction}.
    We have initial conditions $J_{v}(0) = v$, and 
    $D_{t}J_{v}(0) = D_{t}\partial_{s} \Lambda(0,0) = D_{s}\partial_{t} \Lambda(0,0) = h \nabla_{v}b(x) =: \omega$.
    
    Now we analyze conditions that guarantee the invertibility of $F$. Define
    \begin{align*}
        Y(t) = P_{\gamma(t)}^{\gamma(0)} J_{v}(t).
    \end{align*}
    To show that $dF_{h}(x)$ is invertible, it suffices to show 
    $\inf_{v} \frac{\|dF_{h}(x)v\|}{\|v\|} > 0$, equivalently,
    \begin{align*}
        \|dF_{h}(x)v\| = \|P_{\gamma(1)}^{\gamma(0)} J_{v}(1)\|
        = \|Y(1)\| \ge C \|v\| > 0, \forall v \neq 0
    \end{align*}
    for some constant $C > 0$.

    Applying Lemma~\ref{Lemma_Para_Trans_Derivative_Interchange} with
    $c = \gamma$ and $Y = J_{v}$, we obtain
    \begin{align*}
        Y'(t)
        = \frac{d}{dt}\Big( P_{\gamma(t)}^{\gamma(0)} J_{v}(t) \Big)
        = P_{\gamma(t)}^{\gamma(0)} D_t J_{v}(t).
    \end{align*}
    Differentiating once more and using the same lemma with
    $Y = D_t J_{v}$, we get
    \begin{align*}
        Y''(t)
        = \frac{d}{dt}\Big( P_{\gamma(t)}^{\gamma(0)} D_t J_{v}(t) \Big)
        = P_{\gamma(t)}^{\gamma(0)} D_t^2 J_{v}(t).
    \end{align*}
    Apply the Jacobi equation $D_t^2 J_{v}(t) + R(J_{v}(t), \gamma'(t))\gamma'(t) = 0$ 
    we obtain 
    \begin{align*}
        Y''(t)
        = P_{\gamma(t)}^{\gamma(0)} D_t^2 J_{v}(t) 
        = - P_{\gamma(t)}^{\gamma(0)} \left( R(J_{v}(t), \gamma'(t))\gamma'(t) \right).
    \end{align*}
    Using $J_{v}(t) = P_{\gamma(0)}^{\gamma(t)} Y(t)$, we rewrite the
    curvature term and obtain
    \begin{align*}
        Y''(t) + P_{\gamma(t)}^{\gamma(0)}
        \Big( R\big( P_{\gamma(0)}^{\gamma(t)}Y(t), \gamma'(t) \big)
             \gamma'(t) \Big) = 0.
    \end{align*}
    Now we apply Taylor's theorem with integral remainder entry-wisely to $Y(t)$.
    \begin{align*}
        Y(t)
        &= Y(0) + t Y'(0) + \int_0^t (t-s)Y''(s) ds \\
        &= v + t h\,\nabla_v b(x)
        - \int_0^t (t-s)P_{\gamma(s)}^{\gamma(0)}
        \Big( R\big( P_{\gamma(0)}^{\gamma(s)}Y(s), \gamma'(s) \big)
             \gamma'(s) \Big)ds.
    \end{align*}
    In particular, for $t=1$,
    \begin{align*}
        Y(1) - v - h \nabla_v b(x)
        = 
        - \int_0^1 (1-s)P_{\gamma(s)}^{\gamma(0)}
        \Big( R\big( P_{\gamma(0)}^{\gamma(s)}Y(s), \gamma'(s) \big)
             \gamma'(s) \Big)ds.
    \end{align*}
    Since we assumed $\|R(u,v)w\| \le L_{R}\|u\|\|v\|\|w\|$,
    \begin{align*}
        \|R\big( P_{\gamma(0)}^{\gamma(t)}Y(t), \gamma'(t) \big)
             \gamma'(t)\|
        \le L_{R} \|\gamma'(t)\|^{2} \|Y(t)\|.
    \end{align*}
    Taking norms, we get
    \begin{align*}
        \|Y(1) - v - h \nabla_v b(x)\|
        &\le \int_0^1 (1-s)L_{R} \|\gamma'(s)\|^{2} \|Y(s)\|ds \\
        &\le h^{2} \|b(x)\|^{2} L_{R} \int_{0}^{1} \|Y(s)\| ds.
    \end{align*}
    where note $\|\gamma'(s)\| = h \|b(x)\|$.
    We remark that it suffices to upper bound $\int_{0}^{1} \|Y(s)\| ds$. 
    Then we can simply apply triangle inequality as 
    $\|Y(1)\| \ge \|v + h \nabla_v b(x)\| - \|Y(1) - v - h\nabla_v b(x)\|$.
    
    Now we bound $\int_{0}^{1} \|Y(s)\| ds$ through comparison theory.
    Since the computation that involves comparison theory is complicated, we summarize them into a separate Lemma, see Lemma \ref{Lemma_Upper_Bound_Y}.
    
    By Lemma \ref{Lemma_Upper_Bound_Y}, 
    when $K_{\min} > 0$, we have 
    \begin{align*}
        \int_{0}^{1}\|Y(t)\| dt 
        \le & \|v^{\parallel}\| + \|\omega^{\parallel}\| +
        \frac{\|\omega^{\perp}\|}{\sqrt{K_{\min}}} + \|v^{\perp}\|
        \le 2\|v\| + 2 h \|\nabla_{v}b(x)\| \max\{\frac{1}{\sqrt{K_{\min}}}, 1\} \\
        \le& \|v\| (2 + 2 h L_{\nabla}\max\{\frac{1}{\sqrt{K_{\min}}}, 1\}) 
        =: C_{+},
    \end{align*}
    where note that we did orthogonal decomposition on $v, \omega$.

    \begin{align*}
        \|Y(1) - v - h\,\nabla_v b(x)\| \le h^{2} \|b(x)\|^{2} L_{R} C_{+}.
    \end{align*}
    By triangle inequality, 
    \begin{align*}
        \|Y(1)\| \ge& \|v + h \nabla_v b(x)\| - \|Y(1) - v - h\nabla_v b(x)\| \\
        \ge & \|v\| - h \|\nabla_v b(x)\| - h^{2}\|b(x)\|^{2} L_{R} C_{+}\\
        \ge & \|v\|\left(1 - hL_{\nabla} - h^{2} \|b(x)\|^{2} L_{R} (2 + 2 h L_{\nabla}\max\{\frac{1}{\sqrt{K_{\min}}}, 1\})  \right).
    \end{align*}
    Clearly $h < 1$ is required, so we only need to find $h$ s.t. 
    \begin{align*}
        1 - hL_{\nabla}
        &> h^{2} \|b(x)\|^{2} L_{R} (2 + 2 L_{\nabla}\max\{\frac{1}{\sqrt{K_{\min}}}, 1\})\\
        &> h^{2} \|b(x)\|^{2} L_{R} (2 + 2 h L_{\nabla}\max\{\frac{1}{\sqrt{K_{\min}}}, 1\}).
    \end{align*}
    If $h \le \frac{1}{4L_{\nabla}}$, then $1 - hL_{\nabla} \ge \frac{3}{4}$. 
    The value $h \le \sqrt{\frac{3}{4\|b(x)\|^{2} L_{R} (2 + 2 L_{\nabla}\max\{\frac{1}{\sqrt{K_{\min}}}, 1\})}}$
    suffices. 

    When $K_{\min} < 0$, we have 
    \begin{align*}
        \int_{0}^{1}\|Y(t)\| dt 
        \le& \|v^{\parallel}\| + \|\omega^{\parallel}\| +
        \frac{\cosh (\sqrt{-K_{\min}}) - 1}{- K_{\min}} \|\omega^{\perp}\| 
        + \frac{\sinh (\sqrt{-K_{\min}})}{\sqrt{-K_{\min}}} \|v^{\perp}\| \\
        \le &\|v\|\left(2\frac{\sinh (\sqrt{-K_{\min}})}{\sqrt{-K_{\min}}} + 4\frac{\cosh (\sqrt{-K_{\min}}) - 1}{- K_{\min}}h L_{\nabla}\right) 
        =:C_{-}.
    \end{align*}
    Then we have 
    \begin{align*}
        \|Y(1) - v - h\,\nabla_v b(x)\| \le h^{2} \|b(x)\|^{2} L_{R} C_{-}.
    \end{align*}
    By triangle inequality, 
    \begin{align*}
        \|Y(1)\| & \ge \|v + h \nabla_v b(x)\| - \|Y(1) - v - h\nabla_v b(x)\|\\ 
        &\ge \|v\| - h \|\nabla_v b(x)\| - h^{2}\|b(x)\|^{2} L_{R} C_{-} \\
        &\ge \|v\|\left(1 - hL_{\nabla} - h^{2} \|b(x)\|^{2} L_{R} (2\frac{\sinh (\sqrt{-K_{\min}})}{\sqrt{-K_{\min}}} + 4\frac{\cosh (\sqrt{-K_{\min}}) - 1}{- K_{\min}}h L_{\nabla})  \right).
    \end{align*}
    Cleraly $h < 1$ is required, so we only need to find $h$ s.t. 
    \begin{align*}
        1 - hL_{\nabla}
        &> h^{2} \|b(x)\|^{2} L_{R} (2\frac{\sinh (\sqrt{-K_{\min}})}{\sqrt{-K_{\min}}} + 4\frac{\cosh (\sqrt{-K_{\min}}) - 1}{- K_{\min}} L_{\nabla}) \\
        &> h^{2} \|b(x)\|^{2} L_{R} (2\frac{\sinh (\sqrt{-K_{\min}})}{\sqrt{-K_{\min}}} + 4\frac{\cosh (\sqrt{-K_{\min}}) - 1}{- K_{\min}}h L_{\nabla}).
    \end{align*}
    If $h \le \frac{1}{4L_{\nabla}}$, then $1 - hL_{\nabla} \ge \frac{3}{4}$. 
    The value $h \le \sqrt{\frac{3}{4\|b(x)\|^{2} L_{R} (2\frac{\sinh (\sqrt{-K_{\min}})}{\sqrt{-K_{\min}}} + 4\frac{\cosh (\sqrt{-K_{\min}}) - 1}{- K_{\min}} L_{\nabla})}}$
    suffices. 

    When $K_{\min} = 0$, we have 
    \begin{align*}
        \int_{0}^{1}\|Y(t)\| dt 
        \le& \|v^{\parallel}\| + \frac{1}{2} \|\omega^{\parallel}\| + \frac{1}{2}\|\omega^{\perp}\| 
            + \|v^{\perp}\| \\
        \le &\|v\|\left(2 + h L_{\nabla}\right) 
        =:C_{=}.
    \end{align*}
    Then we have 
    \begin{align*}
        \|Y(1) - v - h\,\nabla_v b(x)\| \le h^{2} \|b(x)\|^{2} L_{R} C_{=}.
    \end{align*}
    By triangle inequality, 
    \begin{align*}
        \|Y(1)\|& \ge \|v + h \nabla_v b(x)\| - \|Y(1) - v - h\nabla_v b(x)\| \\
        &\ge \|v\| - h \|\nabla_v b(x)\| - h^{2}\|b(x)\|^{2} L_{R} C_{=}\\
        &\ge \|v\|\left(1 - hL_{\nabla} - h^{2} \|b(x)\|^{2} L_{R}  (2 + h L_{\nabla})  \right).
    \end{align*}
    Cleraly $h < 1$ is required, so we only need to find $h$ s.t. 
    \begin{align*}
        1 - hL_{\nabla}
        > h^{2} \|b(x)\|^{2} L_{R} (2 + L_{\nabla})
        > h^{2} \|b(x)\|^{2} L_{R} (2 + h L_{\nabla}).
    \end{align*}
    If $h \le \frac{1}{4L_{\nabla}}$, then $1 - hL_{\nabla} \ge \frac{3}{4}$. 
    The value $h \le \sqrt{\frac{3}{4\|b(x)\|^{2} L_{R} (2 + h L_{\nabla})}}$
    suffices. 
\end{proof}

\subsection{Main Theorem}\label{Appendix_Proof_MainThem}

Now we prove our main theorem. We remark that the proof of Theorem \ref{Theorem_Rate_Compact}
 and \ref{Theorem_Sampling_Error_Hadamard} are essentially the same. 
 
 \begin{proof}[Proof of Theorem \ref{Theorem_Rate_Compact}]
    By Lemma \ref{Lemma_TV_derivative}, we have 
\begin{align*}
    &\TV(\pi_{T}, \hat{\pi}_{T}) = \TV(\pi_{t_{N}}, \hat{\pi}_{t_{N}}) \\
    \le & \TV(\pi_{t_{0}}, \hat{\pi}_{t_{0}}) 
    + \int_{t_{0}}^{t_{N}} \mathbb{E}[\langle \grad \log p(t, x), \tilde{v}(x, t) - v(x, t) \rangle] dt
    + \int_{t_{0}}^{t_{N}} \mathbb{E}[|\Div (\tilde{v}(x, t) - v(x, t))|] dt \\
    \le & \TV(\pi_{t_{0}}, \hat{\pi}_{t_{0}}) 
    + \int_{t_{0}}^{t_{N}} \mathbb{E}[\|\grad \log p(t, x)\| \cdot \|\tilde{v}(x, t) - v(x, t) \|] dt
    + \int_{t_{0}}^{t_{N}} \mathbb{E}[|\Div (\tilde{v}(x, t) - v(x, t))|] dt ,
\end{align*}
where $\pi_{t}, \hat{\pi}_{t}$ denote the law for $X_{t}, Y_{t}$ respectively.

Notice that 
\begin{align}
    \int_{t_{k}}^{t_{k+1}} (t-t_{k}) dt 
    = (\frac{1}{2}t^{2} - t t_{k})|_{t_{k}}^{t_{k+1}}
    = \frac{1}{2} t_{k+1}^{2} - \frac{1}{2}t_{k}^{2} - t_{k}(t_{k+1} - t_{k})
    = \frac{1}{2}(t_{k+1} - t_{k})^{2}.
\end{align}
Hence using constant step size for discretization, we obtain 
\begin{align*}
    &\TV(\pi_{T}, \hat{\pi}_{T}) = \TV(X_{t_{N}}, Y_{t_{N}}) \\
    \le & \TV(X_{t_{0}}, Y_{t_{0}}) 
    + \int_{t_{0}}^{t_{N}} \mathbb{E}[\|\grad \log p(t, x)\| \cdot \|\tilde{v}(x, t) - v(x, t) \|] dt
    + \int_{t_{0}}^{t_{N}} \mathbb{E}[|\Div (\tilde{v}(x, t) - v(x, t))|] dt \\
    \le & \TV(X_{t_{0}}, Y_{t_{0}}) 
    + \frac{1}{2} N h^{2} \Bigl(\sqrt{2L_{t}^{\text{score}}} (L_{t}^{\hat{v}, x}L^{\hat{v}} + L^{v} L_{t}^{\hat{v}, x} + L_{t}^{v, t} + L^{v} L_{t}^{v, x})  \\
    &+ L_{t}^{\Div, x} (L^{\hat{v}} + 2L^{v}) + L_{t}^{\Div, t} + L_{R}(L^{\hat{v}})^{2} d \Bigr)
    + Nh (\varepsilon \sqrt{ 2L_{t}^{\text{score}}} + \varepsilon) \\
    \le & \TV(X_{t_{0}}, Y_{t_{0}}) 
    + h \Bigl(\sqrt{L_{t}^{\text{score}}} (L_{t}^{\hat{v}, x}L^{\hat{v}} + L^{v} L_{t}^{\hat{v}, x} + L_{t}^{v, t} + L^{v} L_{t}^{v, x})  \\
    &+ L_{t}^{\Div, x} (L^{\hat{v}} + 2L^{v}) + L_{t}^{\Div, t} + L_{R}(L^{\hat{v}})^{2} d \Bigr)
    + (\varepsilon \sqrt{ 2L_{t}^{\text{score}}} + \varepsilon). \\
    = &\TV(X_{t_{0}}, Y_{t_{0}}) 
    + h\mathsf{C}_{\textsf{Lip}} 
    + \varepsilon \mathsf{C}_{\textsf{eps}},
\end{align*}
where note that $Nh < 1$ and we compute \begin{align*}
    & \sqrt{L_{t}^{\text{score}}} ((L_{t}^{v, x} + \varepsilon)(L^{v} + \varepsilon) + L^{v} (L_{t}^{v, x} + \varepsilon) + L_{t}^{v, t} + L^{v} L_{t}^{v, x})  \\
    &+ L_{t}^{\Div, x} (3L^{v} + \varepsilon) + L_{t}^{\Div, t} + L_{R}(L^{v} + \varepsilon)^{2} d \\ 
    = &\varepsilon^{2}(\sqrt{L_{t}^{\text{score}}} + L_{R}d)
    + \varepsilon(2\sqrt{L_{t}^{\text{score}}} L^{v} + \sqrt{L_{t}^{\text{score}}} L_{t}^{v, x} + L_{t}^{\Div, x} + 2L_{R}d L^{v} ) \\
    &+ 3\sqrt{L_{t}^{\text{score}}} L_{t}^{v, x}L^{v} 
    + \sqrt{L_{t}^{\text{score}}} L_{t}^{v, t} 
    + 3L^{v}L_{t}^{\Div, x} + L_{t}^{\Div, t} + L_{R} (L^{v})^{2} d,
    \end{align*}
and denote \begin{align*}
    \mathsf{C}_{\textsf{Lip}} &\coloneqq 3\sqrt{L_{t}^{\text{score}}} L_{t}^{v, x}L^{v} 
    + \sqrt{L_{t}^{\text{score}}} L_{t}^{v, t} 
    + 3L^{v}L_{t}^{\Div, x} + L_{t}^{\Div, t} + L_{R} (L^{v})^{2} d ,\\
    \mathsf{C}_{\textsf{eps}} &\coloneqq \sqrt{ 2L_{t}^{\text{score}}} + 1 + 2\sqrt{L_{t}^{\text{score}}} L^{v} + \sqrt{L_{t}^{\text{score}}} L_{t}^{v, x} + L_{t}^{\Div, x} + 2L_{R}d L^{v} + \varepsilon(\sqrt{L_{t}^{\text{score}}} + L_{R}d).
    \end{align*}
\end{proof}

\begin{proof}[Proof of Theorem \ref{Theorem_Sampling_Error_Hadamard}]
    We have that 
    \begin{align*}
        &\mathbb{E}[\|\grad \log p(t, x)\| \cdot \|\tilde{v}(x, t) - v(x, t) \|] \\
        \le& \varepsilon \sqrt{ 3L_{t}^{\text{score}}} 
        + (t-t_{k}) \Bigl(6 L_{t}^{\text{score}} \left( (L_{t}^{\hat{v}, x})^{2} (L_{t}^{\hat{v}})^{2} + (L_{t}^{\hat{v}, x})^{2} (L_{t}^{v})^{2} + (L_{t}^{v, t})^{2} + (L_{t}^{v})^{2} (L_{t}^{v, x})^{2}\right)\Bigr)^{\frac{1}{2}},  
    \end{align*}
    and \begin{align*}
        &\mathbb{E}_{p(t, x)} [\Div (\tilde{v}(x, t) - v(t, X_{t}))] 
        \le 
        3\varepsilon
        + (t-t_{k}) \Bigl(L_{t}^{\Div \hat{v}, x} (L_{t}^{\hat{v}} + L_{t}^{v}) + L_{t}^{\Div, x} L_{t}^{v} + L_{t}^{\Div, t} + L_{R}(L_{t}^{\hat{v}})^{2} d \Bigr).
    \end{align*}
    Also recall 
    \begin{align*}
        &\TV(\pi_{T}, \hat{\pi}_{T}) = \TV(\pi_{t_{N}}, \hat{\pi}_{t_{N}}) \\
        \le & \TV(\pi_{t_{0}}, \hat{\pi}_{t_{0}}) 
        + \int_{t_{0}}^{t_{N}} \mathbb{E}[\langle \grad \log p(t, x), \tilde{v}(x, t) - v(x, t) \rangle] dt
        + \int_{t_{0}}^{t_{N}} \mathbb{E}[|\Div (\tilde{v}(x, t) - v(x, t))|] dt \\
        \le & \TV(\pi_{t_{0}}, \hat{\pi}_{t_{0}}) 
        + \int_{t_{0}}^{t_{N}} \mathbb{E}[\|\grad \log p(t, x)\| \cdot \|\tilde{v}(x, t) - v(x, t) \|] dt
        + \int_{t_{0}}^{t_{N}} \mathbb{E}[|\Div (\tilde{v}(x, t) - v(x, t))|] dt. 
    \end{align*}
    Notice that 
    \begin{align}
        \int_{t_{k}}^{t_{k+1}} (t-t_{k}) dt 
        = (\frac{1}{2}t^{2} - t t_{k})|_{t_{k}}^{t_{k+1}}
        = \frac{1}{2} t_{k+1}^{2} - \frac{1}{2}t_{k}^{2} - t_{k}(t_{k+1} - t_{k})
        = \frac{1}{2}(t_{k+1} - t_{k})^{2}.
    \end{align}
    Hence using constant step size for discretization, we obtain 
    \begin{align}\label{Eq_Const_Hadamard_Thm}
        &\TV(\pi_{T}, \hat{\pi}_{T}) \notag \\
        \le & \TV(X_{t_{0}}, Y_{t_{0}}) 
        + \int_{t_{0}}^{t_{N}} \mathbb{E}[\|\grad \log p(t, x)\| \cdot \|\tilde{v}(x, t) - v(x, t) \|] dt
        + \int_{t_{0}}^{t_{N}} \mathbb{E}[|\Div (\tilde{v}(x, t) - v(x, t))|] dt \notag \\
        \le & \TV(X_{t_{0}}, Y_{t_{0}}) 
        + \frac{1}{2} N h^{2} \Bigl(
            \sqrt{6 L_{t}^{\text{score}}} \left( (L_{t}^{\hat{v}, x})^{2} (L_{t}^{\hat{v}})^{2} + (L_{t}^{\hat{v}, x})^{2} (L_{t}^{v})^{2} + (L_{t}^{v, t})^{2} + (L_{t}^{v})^{2} (L_{t}^{v, x})^{2}\right) \notag \\
        &+ \Bigl(L_{t}^{\Div \hat{v}, x} (L_{t}^{\hat{v}} + L_{t}^{v}) + L_{t}^{\Div, x} L_{t}^{v} + L_{t}^{\Div, t} + L_{R}(L_{t}^{\hat{v}})^{2} d \Bigr)  \Bigr)
        + Nh (\varepsilon \sqrt{ 2L_{t}^{\text{score}}} + 3\varepsilon) \\
        \le & \TV(X_{t_{0}}, Y_{t_{0}}) 
        + h \mathsf{C}_{\textsf{Lip}}
        + \varepsilon \mathsf{C}_{\textsf{eps, 1}}, \notag 
    \end{align}
where note that $Nh < 1$. Take $t_{0} = 0$, $t_{N} = T < 1$, we obtain 
    \begin{align*}
        \TV(\pi_{T}, \hat{\pi}_{T}) \le h \mathsf{C}_{\textsf{Lip}}
        + \varepsilon \mathsf{C}_{\textsf{eps, 1}}.
    \end{align*}

\end{proof}

\subsection{Example: Hypersphere}
\label{Appendix_Hypersphere_Cor_Proof}

For compact manifolds, under uniform estimation error (Assumption \ref{A_estimation_error}, we can establish the regularity for $\hat{v}$.
In particular, under Assumption \ref{A_estimation_error} and Assumption \ref{A_Regularity_V}, 
since $\|\nabla \hat v(x)\|_{\op} = \|\nabla \hat v(x) - \nabla v(x) + \nabla v(x)\|_{\op}$, we have 
\begin{enumerate}
    \item $\|\hat{v}(t, x)\| \le L^{\hat{v}} = L^{v} + \varepsilon$.
    \item $\hat{v}(t, x)$ is Lipschitz in $x$ variable with $L_{t}^{\hat{v}, x} = L_{t}^{v, x} + \varepsilon$.
\end{enumerate}

\begin{lemma}\label{Lemma_Constants_Hypersphere}
    On $\mathcal{S}^{d}$, Assumption \ref{A_Regularity_V} holds with the following constants. 
    \begin{enumerate}
        \item $ L_{t}^{v, x} = \frac{12\pi M_{1}(d-1)}{m_{1}(1-t)}$ , $L_{t}^{\hat{v}, x} = \varepsilon + \frac{12\pi M_{1}(d-1)}{m_{1}(1-t)}$.
        \item $L_{t}^{v, t} := \frac{8\pi^{2}d}{1-t}\frac{M_{1}}{m_{1}}$.
        \item $L_{t}^{\Div, x} := \frac{128 \pi (d-1)^{2}}{(1-t)^{3}} \frac{M_{1}}{m_{1}}$.
        \item $L_{t}^{\Div, t} := \frac{128\pi^{2}(d-1)^{2}}{(1-t)^{3}}\frac{M_{1}}{m_{1}}$.
        \item $L_{t}^{\text{score}} := \frac{8(d-1)^{2}}{(1-t)^{2}}\frac{M_{1}}{m_{1}}$.
        \item $L^{v} = \pi$.
        \item $L_{R} = 1$.
    \end{enumerate}
\end{lemma}
As the proof is technical, we defer it to Appendix \ref{Sect_Hypersphere}.

Now we prove Proposition \ref{Cor_Iter_Complexity_Sphere}.
\begin{proof}[Proof of Proposition \ref{Cor_Iter_Complexity_Sphere}]
 We first recall the extra condition on $h$ imposed by Lemma \ref{Lemma_F_Invertible}:
 \begin{align*}
     h &< \min\{\frac{R}{\|b(x)\|}, \frac{1}{4L_{\nabla}}, \sqrt{\frac{3}{4\|b(x)\|^{2} L_{R} (2 + 2 L_{\nabla}\max\{\frac{1}{\sqrt{K_{\min}}}, 1\})}}\}.
 \end{align*}
 Since $\|\hat{v}(t, x)\| \le L^{\hat{v}} = L^{v} + \varepsilon$, we know $\|b(x)\|$ is of constant order, consequently $\frac{R}{\|b(x)\|}$ is of constant order. Thus such an condition is dominated by the term $\frac{1}{L_{\nabla}}$, which in our case is exactly $\frac{1}{L_{t}^{\hat{v}, x}}$. Plug in constants in Lemma \ref{Lemma_Constants_Hypersphere}, we obtain 
    \begin{align*}
        &\mathbb{E}[\|\grad \log p(t, x)\| \cdot \|\tilde{v}(x, t) - v(x, t) \|] \\
        \le& \varepsilon \sqrt{ 2\frac{8(d-1)^{2}}{(1-t)^{2}}\frac{M_{1}}{m_{1}}} 
        + (t-t_{k})\Bigl(\frac{8(d-1)^{2}}{(1-t)^{2}}\frac{M_{1}}{m_{1}} 
        (2 ((\varepsilon + \frac{12\pi M_{1}(d-1)}{m_{1}(1-t)})(\pi + \varepsilon) \\
        & \quad + \pi (\varepsilon + \frac{12\pi M_{1}(d-1)}{m_{1}(1-t)})
        + \frac{8\pi^{2}d}{1-t}\frac{M_{1}}{m_{1}} + \pi \frac{12\pi M_{1}(d-1)}{m_{1}(1-t)})^{2})\Bigr)^{\frac{1}{2}} \\
        \le& 4\varepsilon \frac{d-1}{1-t} \sqrt{ \frac{M_{1}}{m_{1}}} 
        + (t-t_{k}) 4 \frac{d-1}{1-t} \sqrt{ \frac{M_{1}}{m_{1}}} 
        \bigl((\varepsilon + \frac{12\pi M_{1}(d-1)}{m_{1}(1-t)})(\pi + \varepsilon) \\
        & \quad + \pi (\varepsilon + \frac{12\pi M_{1}(d-1)}{m_{1}(1-t)})
        + \frac{8\pi^{2}d}{1-t}\frac{M_{1}}{m_{1}} + \pi \frac{12\pi M_{1}(d-1)}{m_{1}(1-t)}\bigr) \\
        \lesssim & \varepsilon \frac{d-1}{1-t} \Bigl(\frac{M_{1}}{m_{1}}\Bigr)^{\frac{1}{2}}
        + (t-t_{k}) \frac{(d-1)^{2}}{(1-t)^{2}} \Bigl(\frac{M_{1}}{m_{1}}\Bigr)^{\frac{3}{2}}.
    \end{align*}
    And similarly, 
    \begin{align*}
        &\mathbb{E}_{p(t, x)} [\Div (\tilde{v}(x, t) - v(t, X_{t}))] 
        \le 
        \varepsilon
        + (t-t_{k}) \Bigl(L_{t}^{\Div, x} (L^{\hat{v}} + 2L^{v}) + L_{t}^{\Div, t} + L_{R}(L^{\hat{v}})^{2} d \Bigr) \\
        \le &
        \varepsilon
        + (t-t_{k}) \Bigl(\frac{128 \pi (d-1)^{2}}{(1-t)^{3}} \frac{M_{1}}{m_{1}} (3\pi + \varepsilon) + \frac{128\pi^{2}(d-1)^{2}}{(1-t)^{3}}\frac{M_{1}}{m_{1}} + (\pi + \varepsilon)^{2} d \Bigr) \\
        \lesssim &
        \varepsilon
        + (t-t_{k}) \Bigl(\frac{(d-1)^{2}}{(1-t)^{3}} \frac{M_{1}}{m_{1}}
        + d \Bigr).
    \end{align*}
    With early stopping, we terminate the sampling algorithm (Algorithm \textcolor{blue}{1}) at time $T < 1$ (i.e., $t_{0} = 0, t_{N} = T$)
    \begin{align*}
        & \TV(\pi_{T}, \hat{\pi}_{T}) = \TV(X_{t_{N}}, Y_{t_{N}}) \\
        \le & \TV(X_{t_{0}}, Y_{t_{0}}) 
        + \int_{t_{0}}^{t_{N}} \mathbb{E}[\|\grad \log p(t, x)\| \cdot \|\tilde{v}(x, t) - v(x, t) \|] dt
        + \int_{t_{0}}^{t_{N}} \mathbb{E}[|\Div (\tilde{v}(x, t) - v(x, t))|] dt \\
        \lesssim & \sum_{i = 0}^{N-1} \int_{t_{i}}^{t_{i+1}}  \varepsilon \frac{d-1}{1-t} \Bigl(\frac{M_{1}}{m_{1}}\Bigr)^{\frac{1}{2}}
        + (t-t_{i}) \frac{(d-1)^{2}}{(1-t)^{2}} \Bigl(\frac{M_{1}}{m_{1}}\Bigr)^{\frac{3}{2}} dt
        + \sum_{i = 0}^{N-1} \int_{t_{i}}^{t_{i+1}} \varepsilon
        + (t-t_{i}) \Bigl(\frac{(d-1)^{2}}{(1-t)^{3}} \frac{M_{1}}{m_{1}}
        + d \Bigr) dt \\
        \lesssim & \Bigl(\frac{M_{1}}{m_{1}}\Bigr)^{\frac{1}{2}}(d-1)\varepsilon \sum_{i = 0}^{N-1} \int_{t_{i}}^{t_{i+1}} \frac{1}{1-t} dt
        + \Bigl(\frac{M_{1}}{m_{1}}\Bigr)^{\frac{3}{2}}(d-1)^{2} 
        \sum_{i = 0}^{N-1} \int_{t_{i}}^{t_{i+1}} \frac{t-t_{i}}{(1-t)^{3}} dt.
    \end{align*}

    We first discuss the case of constant step size. Using 
        \begin{align*}
            \sum_{i = 0}^{N-1} \int_{t_{i}}^{t_{i+1}} \frac{1}{1-t} dt
            =  -\log(1-t_{N})
        \end{align*} 
        and (using $s := 1-t$ so that $ds = - dt$)
        \begin{align*}
            &\sum_{i = 0}^{N-1} \int_{t_{i}}^{t_{i+1}} \frac{t-t_{i}}{(1-t)^{3}} dt 
            \le h \sum_{i = 0}^{N-1} \int_{t_{i}}^{t_{i+1}} \frac{1}{(1-t)^{3}} dt 
            = h \sum_{i = 0}^{N-1} \int_{1 - t_{i+1}}^{1 - t_{i}} \frac{1}{s^{3}}  \\
            =& h \sum_{i = 0}^{N-1} - \frac{1}{(1 - t_{i})^{2}} + \frac{1}{(1 - t_{i+1})^{2}} 
            \lesssim h \frac{1}{(1-T)^{2}},
        \end{align*}
        we can finally bound the error as 
        \begin{align*}
            \varepsilon \Bigl(\frac{M_{1}}{m_{1}}\Bigr)^{\frac{1}{2}}(d-1)\log(\frac{1}{1-T})  
            + h\Bigl(\frac{M_{1}}{m_{1}}\Bigr)^{\frac{3}{2}}(d-1)^{2} \frac{1}{(1-T)^{2}}.
        \end{align*} 
        Now we have 
        \begin{align*}
            \TV(X_{T}, Y_{T}) \lesssim \varepsilon \Bigl(\frac{M_{1}}{m_{1}}\Bigr)^{\frac{1}{2}}(d-1)\log(\frac{1}{1-T})  
            + h\Bigl(\frac{M_{1}}{m_{1}}\Bigr)^{\frac{3}{2}}(d-1)^{2} \frac{1}{(1-T)^{2}}.
        \end{align*} 
        To obtain a sample up to $\varepsilon_{\text{target}}$ accuracy, we assume $\varepsilon$ is sufficiently small. 
        Then we need 
        \begin{align*}
            h\Bigl(\frac{M_{1}}{m_{1}}\Bigr)^{\frac{3}{2}}(d-1)^{2} \frac{1}{(1-T)^{2}} = \mathcal{O}(\varepsilon_{\text{target}}),
        \end{align*}
        which means 
        \begin{align}\label{Eq_Const_Stepsize_Sphere}
            h = \frac{\varepsilon_{\text{target}} (1-T)^{2} }{(d-1)^{2}}\Bigl(\frac{m_{1}}{M_{1}}\Bigr)^{\frac{3}{2}}.
        \end{align}
    Next, we discuss a specific step size schedule that can improve the dependency on $\frac{1}{1-T}$. Denote $h_{k} = t_{k+1} - t_{k}$.
        Then 
        \begin{align*}
            \int_{t_{i}}^{t_{i+1}} \frac{1-t_{i}}{(1-t)^{3}} dt
            = & \frac{1}{2}(1 - t_{i}) (\frac{1}{(1 - t_{i+1})^{2}} - \frac{1}{(1 - t_{i})^{2}} )
            = \frac{1}{2}(1 - t_{i}) (\frac{t_{i}^{2} - 2t_{i} - t_{i+1}^{2} + 2t_{i+1}}{(1 - t_{i+1})^{2}(1 - t_{i})^{2}} ) \\
            =& \frac{1}{2}(1 - t_{i})  \frac{(t_{i+1} - t_{i})(2- t_{i} - t_{i+1}) }{(1 - t_{i+1})^{2}(1 - t_{i})^{2}} 
            = \frac{1}{2}(1 - t_{i}) h_{i} \frac{2- t_{i} - t_{i+1}}{(1 - t_{i+1})^{2}(1 - t_{i})^{2}}, 
        \end{align*}
        \begin{align*}
            \int_{t_{i}}^{t_{i+1}} \frac{t-1}{(1-t)^{3}} dt
            = - \int_{t_{i}}^{t_{i+1}} \frac{1}{(1-t)^{2}} dt
            = \frac{1}{1-t_{i}} - \frac{1}{1-t_{i+1}}
            = - \frac{h_{i}}{(1 - t_{i+1})(1 - t_{i})}.
        \end{align*}
        Together,
        \begin{align*}
            \int_{t_{i}}^{t_{i+1}} \frac{t-t_{i}}{(1-t)^{3}} dt
            =&\frac{1}{2}(1 - t_{i}) h_{i} \frac{2- t_{i} - t_{i+1}}{(1 - t_{i+1})^{2}(1 - t_{i})^{2}} - \frac{h_{i}}{(1 - t_{i+1})(1 - t_{i})} \\
            =& \frac{h_{i}}{(1 - t_{i+1})(1 - t_{i})}(-1 + \frac{1}{2}\frac{2- t_{i} - t_{i+1}}{(1 - t_{i+1})}) \\
            = &\frac{h_{i}^{2}}{2(1 - t_{i+1})^{2}(1 - t_{i})} 
            \lesssim \frac{h_{i}^{2}}{(1 - t_{i+1})^{3}}. 
        \end{align*}
        Now we want to control 
        \begin{align*}
            \sum_{i = 0}^{N-1} \int_{t_{i}}^{t_{i+1}} \frac{t-t_{i}}{(1-t)^{3}} dt
            \lesssim \sum_{i = 0}^{N-1} \frac{h_{i}^{2}}{(1 - t_{i+1})^{3}}.
        \end{align*}
        With $t_{i} = 1 - \frac{1}{(1+\eta i)^{2}}$.
        Then $h_{i} = \frac{1}{(1+\eta i)^{2}} - \frac{1}{(1+\eta (i+1))^{2}} = \frac{\eta (2 + \eta (2i+1)) }{(1+\eta i)^{2}(1+\eta (i+1))^{2}}$.
        \begin{align*}
            \sum_{i = 0}^{N-1} \frac{h_{i}^{2}}{(1 - t_{i+1})^{3}} 
            &= \sum_{i = 0}^{N-1} \frac{\eta^{2} (2 + \eta (2i+1))^{2} }{(1+\eta i)^{4}(1+\eta (i+1))^{4}}(1+\eta (i+1))^{6} \\
            &= \eta^{2}\sum_{i = 0}^{N-1} \frac{ (2 + \eta (2i+1))^{2} }{(1+\eta i)^{4}}(1+\eta (i+1))^{2}  \\
            &\lesssim 4N \eta^{2}.
        \end{align*} 
        Note that by construction of early stopping, it must hold that $1 - \frac{1}{(1+\eta N)^{2}} = T$, which implies $\frac{\frac{1}{\sqrt{1-T}} - 1}{\eta} = N$.
        Hence 
        \begin{align*}
            \sum_{i = 0}^{N-1} \frac{h_{i}^{2}}{(1 - t_{i})^{3}} 
            \lesssim 4N \eta^{2}
            \lesssim \eta \frac{1}{\sqrt{1-T}}.
        \end{align*} 
        We want to reach
        \begin{align*}
            \Bigl(\frac{M_{1}}{m_{1}}\Bigr)^{\frac{3}{2}}(d-1)^{2} \eta \frac{1}{\sqrt{1-T}} = \mathcal{O}(\varepsilon_{\text{target}}).
        \end{align*} 
        So we need 
        \begin{align*}
            \eta \lesssim \Bigl(\frac{\varepsilon_{\text{target}} \sqrt{1-T}}{(d-1)^{2}  }\Bigr) .
        \end{align*} 
        and consequently 
        \begin{align*}
            N = \frac{\frac{1}{\sqrt{1-T}} - 1}{\eta} = \mathcal{O}(\frac{\frac{1}{\sqrt{1-T}} - 1}{\frac{\varepsilon_{\text{target}} \sqrt{1-T}}{(d-1)^{2}  }})
            = \mathcal{O}(\frac{d^{2}}{(1-T)\varepsilon_{\text{target}}} ).
        \end{align*}

        We remark that the step size schedule can be described by the discretized time points as follows:
        \begin{align}
            \label{Eq_Sphere_Step_Schedule}
            t_{i} = 1 - \frac{1}{(1 + \eta i)^{2}}, \qquad \text{where} \qquad \eta = \mathcal{O}\left(\frac{\varepsilon_{\text{target}} \sqrt{1-T}}{(d-1)^{2}  }\right).
        \end{align}
        
\end{proof}

\subsection{Example: SPD Manifold}\label{Appendix_SPD_Cor_Proof}

We first state the following result proved later in Appendix~\ref{Sect_SPD}.

\begin{proposition}\label{Prop_Regularity_Explicit_SPD}
Let $M$ be $\SPD(n)$. Assume Asumption \ref{A_Curvature}. We impose the following moment condition: there exists $M_{\lambda_{1}}$ be such that \begin{align*}
    \max\big\{\mathbb E [d(X_1,z)^2 e^{\lambda_1 d(X_1,z)}], \mathbb E [e^{\lambda_1 d(X_1,z)}]\big\} \le M_{\lambda_{1}}, \qquad \text{where} \qquad \lambda_{1} = 24 \max\{1, \kappa\}. 
\end{align*}  
We choose the prior distribution to be a Riemannian Gaussian distribution centered at some $z \in M$: $p_0(x) \propto \exp \big(-d\,d(x,z)^2\big)$. 
We then have the following regularity results.
\[
\begin{alignedat}{3}
    \mathbb{E}[\|v(t, x)\|^{2}] \lesssim & d, &\\
         \mathbb{E}[ \|\nabla v(t, x) \|] \lesssim & \frac{d^{2 + 6\lambda}}{1-t} L_{R} M_{\lambda_{1}}^{\frac{1}{2}}, \qquad 
        &\mathbb{E}[ \|\nabla v(t, x) \|^{2}] &\lesssim \frac{d^{3+12\lambda}}{(1-t)^{2}} L_{R}^{2} M_{\lambda_{1}}^{\frac{1}{2}}, \\
        \mathbb{E}[|\frac{d}{dt}v(t,x)|] 
        \lesssim & \frac{d^{2 + 6 \lambda}}{1-t}L_{R} M_{\lambda_{1}}^{\frac{1}{2}}, \qquad  &\mathbb{E}[|\frac{d}{dt}v(t,x)|^{2}] &\lesssim \frac{d^{3+12\lambda}}{(1-t)^{2}}L_{R}^{2} M_{\lambda_{1}}^{\frac{1}{2}},\\
        \mathbb{E}[\|\grad_{x}\Div v(t,x)\|] \lesssim &  \frac{d^{3+12\lambda}}{(1-t)^{2}} L_{R}^{3}, \qquad &\mathbb{E}[\|\grad_{x}\Div v(t,x)\|^{2}] &\lesssim \frac{d^{5+24\lambda}}{(1-t)^{4}} L_{R}^{6} M_{\lambda_{1}}, \\
        \mathbb{E}[|\frac{d}{dt}\Div v(t,x)|] 
        \lesssim &  \frac{d^{3 + 12\lambda}}{(1-t)^{2}} L_{R}^{3} M_{\lambda_{1}}^{\frac{1}{2}}, &\\  
\end{alignedat}
\]

where $\lambda = \max\{1, \kappa\}$.
\end{proposition}

\begin{proof}[Proof of Proposition \ref{Cor_Iter_Complexity_SPD}]
Using Proposition \ref{Prop_Regularity_Explicit_SPD}, and following Theorem \ref{Theorem_Sampling_Error_Hadamard}, we have 
\begin{align*}
        &\TV(\pi_{T}, \hat{\pi}_{T}) = \TV(X_{t_{N}}, Y_{t_{N}}) \\
        \le & \TV(X_{t_{0}}, Y_{t_{0}}) 
        + \frac{1}{2} N h^{2} \Bigl(
            \sqrt{6 L_{t}^{\text{score}}} \left( (L_{t}^{\hat{v}, x})^{2} (L_{t}^{\hat{v}})^{2} + (L_{t}^{\hat{v}, x})^{2} (L_{t}^{v})^{2} + (L_{t}^{v, t})^{2} + (L_{t}^{v})^{2} (L_{t}^{v, x})^{2}\right)  \\
        &+ \Bigl(L_{t}^{\Div \hat{v}, x} (L_{t}^{\hat{v}} + L_{t}^{v}) + L_{t}^{\Div, x} L_{t}^{v} + L_{t}^{\Div, t} + L_{R}(L_{t}^{\hat{v}})^{2} d \Bigr) \Bigr)
        + Nh (\varepsilon \sqrt{ 2L_{t}^{\text{score}}} + \varepsilon) \\
        \lesssim & N h^{2} \Bigl(
            \frac{d^{5 + 12\lambda}}{(1-t)^{2}} L_{R}^{2} M_{\lambda_{1}} \sqrt{ M\frac{d^{2}}{(1-t)^{2}} L_{R}^{2} d^{12\lambda}} 
           \Bigr)
        + Nh (\varepsilon \sqrt{ M\frac{d^{2}}{(1-t)^{2}} L_{R}^{2} d^{12\lambda}} + \varepsilon) \\
        \lesssim & h \Bigl(
            \frac{d^{6 + 18\lambda}}{(1-t)^{3}} L_{R}^{3} M_{\lambda_{1}}^{\frac{3}{2}} 
           \Bigr)
        + ( \frac{d^{6\lambda + 1}}{1-t} L_{R})M_{\lambda_{1}}^{\frac{1}{2}} \varepsilon .
    \end{align*}
    Thus to reach $\varepsilon_{\text{target}}$ accuracy, we need 
    \begin{align*}
        h = \mathcal{O}\left( \frac{\varepsilon_{\text{target}}(1-T)^{3}}{d^{6+18\lambda} L_{R}^{3} M_{\lambda_{1}}^{\frac{3}{2}}}\right).
    \end{align*}
    Consequently, the iteration complexity would be 
    \begin{align*}
        N = \mathcal{O} \left( \frac{d^{6+18\lambda} L_{R}^{3} M_{\lambda_{1}}^{\frac{3}{2}}}{\varepsilon_{\text{target}}(1-T)^{3}} \right).
    \end{align*}

    We remark that since $K_{\min} = - \frac{1}{2}$ \citep{criscitiello2023accelerated}, so $\lambda = \max\{1, \kappa\} = 1$.
    In the meanwhile, the upper bound on $h$ required by Lemma \ref{Lemma_F_Invertible} can be reduced to be of order
    \begin{align*}
        \min\{\frac{1}{L_{\nabla}}, \sqrt{\frac{1}{\|b(x)\|^{2} L_{R} L_{\nabla}}}\} = \frac{1-T}{d^{8} L_{R} M_{\lambda_{1}}^{\frac{1}{2}}},
    \end{align*}
    which is lower than our required order of $h$.

    Therefore, we conclude the iteration complexity as 
    \begin{align*}
        N = \mathcal{O} \left( \frac{d^{6+18\lambda} L_{R}^{3} M_{\lambda_{1}}^{\frac{3}{2}}}{\varepsilon_{\text{target}}(1-T)^{3}} \right)
        = \mathcal{O} \left( \frac{d^{24} L_{R}^{3} M_{\lambda_{1}}^{\frac{3}{2}}}{\varepsilon_{\text{target}}(1-T)^{3}} \right).
    \end{align*}

\end{proof}

\section{Auxiliary Results for Proof of Main Theorems }\label{sec:auxmainthm}

\subsection{Jacobi Equation}\label{Sect_Aux_invertibility}

\begin{lemma}[Solution for Jacobi equation]
    \label{Lemma_Sol_Jacobi} Let $M$ be of constant sectional curvature $c$.
    Let $J$ be a normal Jacobi field along $\gamma$, with initial condition $J(0) = v^{\perp}, J'(0) = 0$.
    Then we have 
    \begin{align*}
        J(t) = s_{c}^{(2)}(t) \|v^{\perp}\| E(t) :=
        \begin{cases}
            \|v^{\perp}\| E(t), & \text{ if } c = 0 ,\\
            \|v^{\perp}\| \cos(\sqrt{c}t) E(t), & \text{ if } c > 0 ,\\
            \|v^{\perp}\| \cosh (\sqrt{-c}t) E(t), & \text{ if } c < 0 .
        \end{cases}
    \end{align*}
    where $E$ is a parallel normal unit vector field with $E(0) = \frac{v^{\perp}}{\|v^{\perp}\|}$.
\end{lemma}
\begin{proof}[Proof of Lemma \ref{Lemma_Sol_Jacobi}] Similar to the proof of \citet[Proposition 10.12]{lee2018introduction}, the solution is of the form $J(t) = f(t) E(t)$ where $E(t)$ is a parallel unit normal vector field along $\gamma$. Since the curvature is constant and $J$ is a normal Jacobi field, the Jacobi equation reduces to $D_{t}^{2}J + c J = 0$, thus we only need to solve for $f''(t) + c f(t) = 0$, where $f(t) \in \mathbb{R}, \forall t$.
  \begin{itemize}[noitemsep]
      \item When $c = 0$, we obtain $f''(t) = 0$, hence (to satisfy initial condition)
    $f(t) = \|v^{\perp}\| $.
    \item When $c > 0$, we obtain $f(t) = \|v^{\perp}\| \cos(\sqrt{c}t)$.
    \item When $c < 0$, we obtain $f(t) = \|v^{\perp}\| \frac{e^{\sqrt{-c}t} + e^{-\sqrt{-c}t}}{2} = \|v^{\perp}\| \cosh (\sqrt{c}t)$.
  \end{itemize}  
Note that $E(0) = \frac{v^{\perp}}{\|v^{\perp}\|}$.
    Therefore, 
    \begin{align*}
        J(t) = \begin{cases}
            \|v^{\perp}\| E(t), & \text{ if } c = 0, \\
            \|v^{\perp}\| \cos(\sqrt{c}t) E(t), & \text{ if } c > 0, \\
            \|v^{\perp}\| \cosh (\sqrt{-c}t) E(t), & \text{ if } c < 0.
        \end{cases}
    \end{align*}
\end{proof}

\begin{lemma}
    \label{Lemma_Upper_Bound_Y}
    Let $J_{v}(t)$ be a Jacobi field along $\gamma(t) = \Exp_{x}(t h b(x))$, 
    with $J_{v}(0) = v$, $J_{v}'(0) = \omega$.
    Define $Y(t) = P_{\gamma(t)}^{\gamma(0)} J_{v}(t)$.
    Up to the first conjugate point of $x$, we have 
    \begin{align*}
        \|Y(t)\| 
        \le \|v^{\parallel}\| + \|\omega^{\parallel}\| t +
        s_{K_{\min}}(t)\|\omega^{\perp}\| + 
        \frac{s_{K_{\min}}(t)}{s_{K_{\min}}(t_0)} \|v^{\perp}\|.
    \end{align*}
    Take $h < \frac{R}{\|b(x)\|}$, 
    we have 
    \begin{align*}
        \int_{0}^{1}\|Y(t)\| dt \le
        \begin{cases}
            \|v^{\parallel}\| + \|\omega^{\parallel}\| +
        \frac{\|\omega^{\perp}\|}{\sqrt{K_{\min}}} + \|v^{\perp}\|, &  \text{ if } K_{\min} > 0 ,\\
            \|v^{\parallel}\| + \|\omega^{\parallel}\| +
        \frac{\cosh (\sqrt{-K_{\min}}) - 1}{- K_{\min}} \|\omega^{\perp}\| 
        + \frac{\sinh (\sqrt{-K_{\min}})}{\sqrt{-K_{\min}}} \|v^{\perp}\|, & \text{ if } K_{\min} < 0 ,\\
            \|v^{\parallel}\| + \frac{1}{2} \|\omega^{\parallel}\| + \frac{1}{2}\|\omega^{\perp}\| 
            + \|v^{\perp}\| , & \text{ if } K_{\min} = 0 .
        \end{cases}
    \end{align*}
\end{lemma}
\begin{proof}[Proof of Lemma \ref{Lemma_Upper_Bound_Y}]
    Using isometric property of parallel transport, $\|Y(t)\| = \|J_{v}(t)\|$.
    Decompose $v = v^{\parallel} + v^{\perp}$, where 
    $v^{\parallel}$ is the component in $\gamma'(0)$ direction, and 
    $\langle v^{\perp}, \gamma'(0) \rangle = 0$.
    Similarly, we decompose 
    $\omega = \omega^{\parallel} + \omega^{\perp}$.

    Define $J_{v} =: J^{(0)} + J^{(1)} + J^{(2)}$, where 
    \begin{align*}
        J^{(0)}(0) = v^{\parallel} , &\quad D_{t}J^{(0)}(0) = \omega^{\parallel}; \\
        J^{(1)}(0) = 0, &\quad D_{t}J^{(1)}(0) = \omega^{\perp} ;\\
        J^{(2)}(0) = v^{\perp}, &\quad D_{t}J^{(2)}(0) = 0 .
    \end{align*}
    Note that both $J^{(1)}(0)$, $D_{t}J^{(1)}(0)$ are orthogonla to $\gamma'(0)$, so $J^{(1)}(t)$ is a normal Jacobi field.

    We know $J^{(0)}$ is a tangential Jacobi field, so it has the form 
    \begin{align*}
        J^{(0)}(t) = (a + bt) \gamma'(t).
    \end{align*}
    Plug in the initial value, we get (note that $D_{t}\gamma'(0) = 0$)
    \begin{align*}
        a \gamma'(0) = v^{\parallel}, 
        \quad b\gamma'(0) = \omega^{\parallel}.
    \end{align*}
    Hence $|a| = \frac{\|v^{\parallel}\|}{\|\gamma'(0)\|}$, 
    and $|b| = \frac{\|\omega^{\parallel}\|}{\|\gamma'(0)\|}$.
    
    By definition, $\|\gamma'(0)\| = \|h b(x)\| = \|\gamma'(t)\|$.
    So we get 
    \begin{align*}
        \|J^{(0)}(t)\| \le \|v^{\parallel}\| + \|\omega^{\parallel}\| t.
    \end{align*}

    If all sectional curvatures of $M$ are bounded below by a constant $K_{\min}$,
    Jacobi field comparison theorem yield 
    \begin{align*}
        \|J^{(1)}(t)\| \le s_{K}(t) \|D_{t}J^{(1)}(0)\|
        = s_{K_{\min}}(t)\|\omega^{\perp}\|.
    \end{align*}

Now consider the term $J^{2}(0)$.  By \citet[Theorem 1.34]{cheeger1975comparison} applied with Jacobi field formula given in Lemma \ref{Lemma_Sol_Jacobi}, we obtain $\|J^{(2)}(t)\| \le \|\tilde{J}(t)\|$.
    We remark that the focal point free condition is saying $J^{(2)}(t) \neq 0$.
    As long as Cut locus is not reached, the geodesic is minimizing. 
    Hence it suffices to guarantee that the geodesic $\gamma$ satisfies $\|\gamma'(0)\| < \inj (M)$.
    
    Finally, recall $\|Y(t)\| = \|J_{v}(t)\|$.
    \begin{align*}
        \|Y(t)\| =& \|J_{v}(t)\| = \|J^{(0)}(t) + J^{(1)}(t) + J^{(2)}(t)\| \\
        \le & \|v^{\parallel}\| + \|\omega^{\parallel}\| t +
        |s_{K_{\min}}(t)| \|\omega^{\perp}\| + 
        |s_{K_{\min}}^{(2)}(t)| \|v^{\perp}\|.
    \end{align*}

    We split into cases. The first case is $K_{\min} > 0$.
        In this case we denote $R = \inj(M)$ and 
        \begin{align*}
            \int_{0}^{1} |s_{K_{\min}}(t)| dt 
            = \frac{1}{\sqrt{K_{\min}}} \int_{0}^{1} |\sin (t \sqrt{K_{\min}})| dt
            \le \frac{1}{\sqrt{K_{\min}}}.
        \end{align*}
        Also, 
        \begin{align*}
            \int_{0}^{1} |s_{K_{\min}}^{(2)}(t)| dt 
            = \int_{0}^{1} |\cos(\sqrt{K_{\min}}t)| dt
            \le 1.
        \end{align*}

        To summarize, when $h < \frac{R}{\|b(x)\|}$, 
        \begin{align*}
            \int_{0}^{1}\|Y(t)\| dt 
            \le & \|v^{\parallel}\| + \|\omega^{\parallel}\| +
            \frac{\|\omega^{\perp}\|}{\sqrt{K_{\min}}} + \|v^{\perp}\|.
        \end{align*}

    The second case is $K_{\min} < 0$.

        Then $s_{K_{\min}}(t) = \frac{1}{\sqrt{-K_{\min}}} \sinh(\sqrt{-K_{\min}} t)$.
        Consider $0 < t \le h < 1$. 
        \begin{align*}
            \int_{0}^{1} s_{K_{\min}}(t) dt 
            = \int_{0}^{1} \frac{1}{\sqrt{-K_{\min}}}\sinh t\sqrt{-K_{\min}} dt
            = \frac{1}{- K_{\min}} (\cosh (\sqrt{-K_{\min}}) - 1).
        \end{align*}
        For the second integral, 
        \begin{align*}
            \int_{0}^{1} \cosh (\sqrt{-K_{\min}}t) dt 
            = \frac{\sinh (\sqrt{-K_{\min}})}{\sqrt{-K_{\min}}}.
        \end{align*}

        To summarize, when $h < \frac{R}{\|b(x)\|}$, 
        \begin{align*}
            \int_{0}^{1}\|Y(t)\| dt 
            \le & \|v^{\parallel}\| + \|\omega^{\parallel}\| +
            \frac{\cosh (\sqrt{-K_{\min}}) - 1}{- K_{\min}} \|\omega^{\perp}\| 
            + \|v^{\perp}\| \frac{\sinh (\sqrt{-K_{\min}})}{\sqrt{-K_{\min}}}.
        \end{align*}

        The third case is $K_{\min} = 0$.
        In this case $s_{K_{\min}}(t) = t$.
        \begin{align*}
            \|Y(t)\| =& \|J_{v}(t)\| = \|J^{(0)}(t) + J^{(1)}(t) + J^{(2)}(t)\| \\
            \le & \|v^{\parallel}\| + \|\omega^{\parallel}\| t +
            \|\omega^{\perp}\| t + 
            \|v^{\perp}\|.
        \end{align*}

        Then we have 
        \begin{align*}
            \int_{0}^{1}\|Y(t)\| dt 
            \le & \|v^{\parallel}\| + \frac{1}{2} \|\omega^{\parallel}\| 
            + \frac{1}{2}\|\omega^{\perp}\| 
            + \|v^{\perp}\| .
        \end{align*}
        
\end{proof}

\subsection{Divergence Term}\label{Sect_Div_Term}

\begin{lemma}\label{Lemma_Cov_de_Parallel_Vec}
    Let $c$ denote the geodesic with $c(0) = z$, $c(h) = x$. 
    Define a variation of geodesics $\Lambda (s, t)$ as 
    $\Lambda(s, t) = \Exp_{c(t)}(sP_{c(0)}^{c(t)} E_{i}(t))$ 
    where $E_{i}(t)$ is basis vector field along $c$. 
    Let $V$ be a vector field. 
    For every $z$, we can obtain a new vector field by parallel transport, 
    denote as $P_{c(0)}^{c(t)} V(z)$.
    Then 
    \begin{align*}
        D_{s}P_{c(0)}^{c(h)} V(z)
        = P_{c(0)}^{c(h)}  D_{s}V(z) 
        - P_{c(0)}^{c(h)} \int_{0}^{h} P_{c(\tau)}^{c(t_{k})} 
        R(\partial_{s}\Lambda(0, \tau), \partial_{t}\Lambda(0, \tau))
        P_{c(t_{k})}^{c(\tau)} V(z) d\tau,
    \end{align*}
    where $D_{s}$ denote the covariant derivative along $E_{i}(t)$.
\end{lemma}

Note that the $s$ direction is actually arbitrary. The goal is to compute the divergence at a point, 
so we can enumerate over all possible $s$ direction, in all basis vectors. 
Each $s$ direction, roughly speaking, defines a $D_{s}$.

\begin{proof}[Proof of Lemma \ref{Lemma_Cov_de_Parallel_Vec}] By construction, it holds that  
    \begin{align*}
        \Lambda(0, 0) = z, \Lambda(0, h) = x .
    \end{align*}
    and the parallel transport $P_{z}^{x} = P_{\Lambda(0, 0)}^{\Lambda(0, h)}$ is along $c$.
    Define $W(s, t) = P_{\Lambda(s, 0)}^{\Lambda(s, t)} V(\Lambda(s, 0)) $, 
    where we perform parallel transport along curve $t \mapsto \Lambda(s, t)$ which might not be a geodesic.
    By definition of parallel transport, $D_{t}W(s, t) = 0, \forall s$.

    We know
    \begin{align*}
        - D_{t} D_{s}W(s, t) = D_{s} D_{t}W(s, t) - D_{t} D_{s}W(s, t) = R(\partial_{s}\Lambda(s, t), \partial_{t}\Lambda(s, t))W(s, t).
    \end{align*}
    Hence evaluating at $s = 0$, we obtain 
    \begin{align*}
        D_{t} D_{s}W(0, t) = - R(\partial_{s}\Lambda(0, t), \partial_{t}\Lambda(0, t))W(0, t).
    \end{align*}
    We perform parallel transport $P_{c(t)}^{c(0)}$ on both sides of the equation, 
    and by Lemma \ref{Lemma_Para_Trans_Derivative_Interchange} we have 
    \begin{align*}
        \frac{d}{dt} P_{c(t)}^{c(0)} D_{s}W(0, t)
        = P_{c(t)}^{c(0)} D_{t} D_{s}W(0, t) = - P_{c(t)}^{c(0)} R(\partial_{s}\Lambda(0, t), \partial_{t}\Lambda(0, t))W(0, t).
    \end{align*}

    Observe that both side of the equation is a time dependent vector field in $T_{c(0)}M = T_{z}M$.
    Hence we can perform integration
    \begin{align*}
        \int_{0}^{h} \frac{d}{d\tau} P_{c(\tau)}^{c(t_{k})} D_{s}W(0, \tau) d\tau
        = - \int_{0}^{h}  P_{c(\tau)}^{c(t_{k})} R(\partial_{s}\Lambda(0, \tau), \partial_{t}\Lambda(0, \tau))W(0, \tau) d\tau.
    \end{align*}
    Hence 
    \begin{align*}
        P_{c(h)}^{c(0)} D_{s}W(0, h)
        = D_{s}W(0, 0) - \int_{0}^{h}  P_{c(\tau)}^{c(t_{k})} R(\partial_{s}\Lambda(0, \tau), \partial_{t}\Lambda(0, \tau))W(0, \tau) d\tau.
    \end{align*}
    Note that $D_{s}W(0, 0) = D_{s}V(\Lambda(0, 0))$.
    Perform parallel transport $P_{c(0)}^{c(h)}$ on both sides of the equation, we have 
    \begin{align*}
        D_{s}W(0, h)
        = P_{c(0)}^{c(h)} D_{s}V(\Lambda(0, 0)) - P_{c(0)}^{c(h)} \int_{0}^{h} P_{c(\tau)}^{c(t_{k})} R(\partial_{s}\Lambda(0, \tau), \partial_{t}\Lambda(0, \tau))W(0, \tau) d\tau.
    \end{align*}

    Recall $W(s, t) = P_{\Lambda(s, 0)}^{\Lambda(s, t)} V(\Lambda(s, 0)) $. 
    \begin{align*}
        D_{s}P_{c(0)}^{c(h)} V(\Lambda(0, 0))
        = P_{c(0)}^{c(h)}  D_{s}V(\Lambda(0, 0)) 
        - P_{c(0)}^{c(h)} \int_{0}^{h} P_{c(\tau)}^{c(t_{k})} 
        R(\partial_{s}\Lambda(0, \tau), \partial_{t}\Lambda(0, \tau))
        P_{c(t_{k})}^{c(\tau)} V(z) d\tau.
    \end{align*}

\end{proof}

Notice that $\tilde{v}$ is defined through $\hat{v}$ and the inverse of $F$. We need to control $\Div \tilde{v}(x, t)$ in our analysis, by writing it as some expression involving 
$\Div \hat{v}$.

\begin{lemma}\label{Lemma_Divergence_tildev}
    Under Assumption \ref{A_Curvature}, we have 
    \begin{align*}
        | \Div (\tilde{v}(x, t) - v(x, t))| 
        \le \left|\Div \hat{v}(t_{k}, z) - \Div v(x, t) \right|
        + L_{R} (t-t_{k}) d \|\hat{v}(t_{k}, z)\|^{2}.
    \end{align*}
    where $\Exp_{z}((t-t_{k})\hat{v}(t_{k}, z)) = x$.
\end{lemma}

\begin{proof}[Proof of Lemma \ref{Lemma_Divergence_tildev}]
    We follow the setting in Lemma \ref{Lemma_Cov_de_Parallel_Vec} with $V$ replaced by $\hat{v}$.
    We use $\{E_{i}(t)\}_{i = 1}^{d}$ to denote an orthonormal basis vector field along geodesic $c$.
    Note that a slight difference is the time shift, where in Lemma \ref{Lemma_Cov_de_Parallel_Vec}
    we have the curve is from time $0$ to $h$, but here we are from $t_{k}$ to $t$ (so we have $h = t-t_{k}$).
    We denote our time variable as $\tau$. In most cases we mean the variable $\tau$ is in $[t_{k}, t]$.

    By definition, $\Exp_{z}((t-t_{k})\hat{v}(t_{k}, z)) = x$.
    So we know $\Lambda(0, \tau) = \Exp_{z}((\tau - t_{k})\hat{v}(t_{k}, z))$. By construction, 
    \begin{align*}
        \partial_{t}\Lambda(0, \tau)
        = P_{c(t_{k})}^{c(\tau)} \hat{v}(t_{k}, z), \quad \partial_{s}\Lambda(0, \tau) = E_{i}(\tau).
    \end{align*}
    By Lemma \ref{Lemma_Cov_de_Parallel_Vec} with $c(\tau) = \Exp_{z}((\tau - t_{k})\hat{v}(t_{k}, z))$, 
    \begin{align*}
        D_{s}P_{c(t_{k})}^{c(t)} \hat{v}(t_{k}, z)
        = P_{c(t_{k})}^{c(t)}  D_{s}\hat{v}(t_{k}, z)
        - P_{c(t_{k})}^{c(t)} \int_{t_{k}}^{t} P_{c(\tau)}^{c(t_{k})} 
        R(E_{i}(\tau), P_{c(t_{k})}^{c(\tau)} \hat{v}(t_{k}, z))
        P_{c(t_{k})}^{c(\tau)} \hat{v}(t_{k}, z) d\tau.
    \end{align*}

    By definition of Riemannian divergence, 
    \begin{align*}
        \Div \tilde{v}(t, x)
        &= \sum_{i = 1}^{d} \langle \nabla_{E_{i}(t)} \tilde{v}(t, x), E_{i}(t)\rangle
        = \sum_{i = 1}^{d} \langle \nabla_{E_{i}(t)} P_{F_{t_{k}, t-t_{k}}^{-1}(x)}^{x} \hat{v}(t_{k}, F_{t_{k}, t-t_{k}}^{-1}(x)), E_{i}(t)\rangle \\
        &= \sum_{i = 1}^{d} \langle D_{s}^{(i)}P_{c(t_{k})}^{c(t)} \hat{v}(t_{k}, z), E_{i}(t)\rangle \\
        &= \sum_{i = 1}^{d} \langle P_{c(t_{k})}^{c(t)}  D_{s}^{(i)}\hat{v}(t_{k}, z)
        - P_{c(t_{k})}^{c(t)} \int_{t_{k}}^{t} P_{c(\tau)}^{c(t_{k})} 
        R(E_{i}(\tau), P_{c(t_{k})}^{c(\tau)} \hat{v}(t_{k}, z))
        P_{c(t_{k})}^{c(\tau)} \hat{v}(t_{k}, z) d\tau , E_{i}(t)\rangle \\
        &= \sum_{i = 1}^{d} \langle P_{c(t_{k})}^{c(t)} \nabla_{E_{i}(t_{k})} \hat{v}(t_{k}, z)
        - P_{c(t_{k})}^{c(t)} \int_{t_{k}}^{t} P_{c(\tau)}^{c(t_{k})} 
        R(E_{i}(\tau), P_{c(t_{k})}^{c(\tau)} \hat{v}(t_{k}, z))
        P_{c(t_{k})}^{c(\tau)} \hat{v}(t_{k}, z) d\tau, E_{i}(t)\rangle \\
        &= \sum_{i = 1}^{d} \langle 
            P_{c(t_{k})}^{c(t)}  \nabla_{E_{i}(t_{k})} \hat{v}(t_{k}, z), E_{i}(t)\rangle \\
        & \quad - \sum_{i = 1}^{d} \langle P_{c(t_{k})}^{c(t)} \int_{t_{k}}^{t} P_{c(\tau)}^{c(t_{k})} 
        R(E_{i}(\tau), P_{c(t_{k})}^{c(\tau)} \hat{v}(t_{k}, z))
        P_{c(t_{k})}^{c(\tau)} \hat{v}(t_{k}, z) d\tau, E_{i}(t)\rangle \\
        &= \Div \hat{v}(t_{k}, z)
        - \sum_{i = 1}^{d} \langle P_{c(t_{k})}^{c(t)} \int_{t_{k}}^{t} P_{c(\tau)}^{c(t_{k})} 
        R(E_{i}(\tau), P_{c(t_{k})}^{c(\tau)} \hat{v}(t_{k}, z))
        P_{c(t_{k})}^{c(\tau)} \hat{v}(t_{k}, z) d\tau, E_{i}(t)\rangle,
    \end{align*}
    where $D_{s}^{(i)}$ represent the covariant derivative corresponds to $E_{i}(\tau)$.

    We have 
    \begin{align*}
        &\sum_{i = 1}^{d} \langle P_{c(t_{k})}^{c(t)} 
        \int_{t_{k}}^{t} P_{c(\tau)}^{c(t_{k})} 
        R(E_{i}(\tau), P_{c(t_{k})}^{c(\tau)} \hat{v}(t_{k}, z))
        P_{c(t_{k})}^{c(\tau)} \hat{v}(t_{k}, z) d\tau, E_{i}(t)\rangle \\
        =& \sum_{i = 1}^{d} \int_{t_{k}}^{t} \langle 
        P_{c(\tau)}^{c(t_{k})} 
        R(E_{i}(\tau), P_{c(t_{k})}^{c(\tau)} \hat{v}(t_{k}, z))
        P_{c(t_{k})}^{c(\tau)} \hat{v}(t_{k}, z), E_{i}(t_{k})\rangle d\tau \\
        =& \sum_{i = 1}^{d} \int_{t_{k}}^{t} \langle 
        R(E_{i}(\tau), P_{c(t_{k})}^{c(\tau)} \hat{v}(t_{k}, z))
        P_{c(t_{k})}^{c(\tau)} \hat{v}(t_{k}, z), E_{i}(\tau)\rangle d\tau \\
        \le & \sum_{i = 1}^{d} \int_{t_{k}}^{t}
        L_{R} \|\hat{v}(t_{k}, z)\|^{2}  d\tau 
        \le (t-t_{k}) d L_{R} \|\hat{v}(t_{k}, z)\|^{2}.
    \end{align*}
    
    It follows that 
    \begin{align*}
        &| \Div (\tilde{v}(x, t) - v(x, t))| \\
        =& |\Div \hat{v}(t_{k}, z)
        - \sum_{i = 1}^{d} \langle P_{c(t_{k})}^{c(t)} \int_{t_{k}}^{t} P_{c(\tau)}^{c(t_{k})} 
        R(E_{i}(\tau), P_{c(t_{k})}^{c(\tau)} \hat{v}(t_{k}, z))
        P_{c(t_{k})}^{c(\tau)} \hat{v}(t_{k}, z) d\tau, E_{i}(t)\rangle
        - \Div v(x, t)| \\
        \le &
        \left|\Div \hat{v}(t_{k}, z) - \Div v(x, t) \right|
        + L_{R} (t-t_{k}) d \|\hat{v}(t_{k}, z)\|^{2}.
    \end{align*}

\end{proof}

\begin{lemma}\label{Lemma_Para_Trans_Derivative_Interchange} We have that
    \begin{align*}
        \frac{d}{dt} P_{c(t)}^{c(0)} Y(t)
        = P_{c(t)}^{c(0)} D_{t} Y(t).
    \end{align*}
\end{lemma}

\begin{proof}[Proof of Lemma \ref{Lemma_Para_Trans_Derivative_Interchange}]
    Denote the geodesic along $t$ direction as $c(t)$.
    Let $v \in T_{c(0)}M$ be arbitrary, and define $Z(t) = P_{c(0)}^{c(t)} v$.
    \begin{align*}
        \frac{d}{dt} \langle P_{c(t)}^{c(0)} Y(t), v\rangle
        = \frac{d}{dt} \langle Y(t), P_{c(0)}^{c(t)} v\rangle
        = \langle D_{t} Y(t), P_{c(0)}^{c(t)} v\rangle + 0
        = \langle D_{t} Y(t), P_{c(0)}^{c(t)} v\rangle
        = \langle P_{c(t)}^{c(0)} D_{t} Y(t), v\rangle.
    \end{align*}
    On the other hand, 
    \begin{align*}
        \frac{d}{dt} \langle P_{c(t)}^{c(0)} Y(t), v\rangle
        = \langle \frac{d}{dt} P_{c(t)}^{c(0)} Y(t), v\rangle.
    \end{align*}
    Since the above holds for any $v$, we have 
    \begin{align*}
        \frac{d}{dt} P_{c(t)}^{c(0)} Y(t)
        = P_{c(t)}^{c(0)} D_{t} Y(t).
    \end{align*}
\end{proof}

%% file: Appendix_Hypersphere.tex

\section{Hypersphere Regularity Results}\label{Sect_Hypersphere}

Recall that \begin{align*}
    v(t,x) = \frac{1}{1-t}\int_{M} \Log_{x}(x_{1})\, p_{t}(x_{1}|x)\, dV_{g}(x_{1}).
\end{align*}
To establish regularity of $v$, it is natural to study the formula for conditional density, $p_{t}(x_{1}|x)$. 
In a Euclidean space, with Gaussian distribution as prior $p_{0}$, we have that $p_{t}(x_{1} \mid x_{t} = x) \propto p_{1}(x_{1})\exp(-\frac{\|t x_{1} - x\|^{2}}{2(1-t)^{2}})$. This is a standard result for flow matching, see for example \cite{guan2025mirror} and \cite{zhouerror}. But on a Riemannian manifold, the curvature would introduce an extra term due to the change of variable formula, and the existence of cut points would introduce an extra indicator function. 
We provide the formula for conditional density on the hypersphere $\mathcal{S}^d$ with uniform distribution as prior, and geodesic interpolation. See Lemma \ref{Lemma_Conditional_Density} below. 

\begin{lemma}[Conditional density on $S^{d}$ for geodesic interpolation]
    \label{Lemma_Conditional_Density}
Let $S^{d}$ be the unit sphere with round metric, $d\ge 2$.  
Let $X_{1}\sim p_{1}$ be the data distribution with smooth
densities $p_{1} > 0$ w.r.t.\ $dV_{g}$, and $X_{0}\sim p_{0}$ being uniform distribution independent of $X_{1}$.
Consider geodesic interpolation with minimizing geodesic. 
Fix $t\in [0,1) $ and $x,x_{1}\in S^{d}$. 
Write $r = d(x,x_{1})$.
Denote 
\begin{align*}
    J_{t}(x\mid x_{1}) &= \frac{1}{1-t}
       \Bigl(
           \frac{\sin\bigl(r/(1-t)\bigr)}{\sin r}
       \Bigr)^{d-1} \mathbf{1}_{\{d(x,x_{1})<(1-t)\pi\}}\\
    &= \begin{cases}
        \frac{1}{1-t}
       \Bigl(
           \frac{\sin\bigl(r/(1-t)\bigr)}{\sin r}
       \Bigr)^{d-1}, &\text{ if } r < (1-t)\pi, \\
        0, & \text{ if } r \ge (1-t)\pi.
    \end{cases}
\end{align*}
Then the conditional density of $X_{1}$ given $X_{t} = x$ is 
\begin{align*}
    p_{t}(x_{1}\mid x)
    &= \frac{
    p_{1}(x_{1})\,J_{t}(x\mid x_{1}) }{
    \displaystyle
    \int_{S^{d}} p_{1}(z)\, J_{t}(x\mid z) dV_{g}(z)}.
\end{align*}

\end{lemma}

\begin{proof} The proof strategy is as follows. 
We first write out the joint distribution of $X_{0}, X_{1}$, 
and then use the change of variable formula to obtain joint distribution of 
$X_{t}, X_{1}$. 

We apply the change of variable formula. We need a diffeomorphism between $(X_{0}, X_{1})$ and $(X_{t}, X_{1})$.
    Fix $t \in (0,1)$ and $x_{1}\in S^{d}$. 
    Use polar coordinates at $x_{1}$:
    every $y \in S^{d}\setminus\{-x_{1}\}$ can be written uniquely as
    $y = \Exp_{x_{1}}(r\omega)$ for some $r \in(0,\pi), \omega\in S^{d-1}$.
    And we know the Riemannian volume element is
    \begin{align*}
        dV_{g}(y) = (\sin r)^{d-1} dr\, d\omega .
    \end{align*}

    In particular, for any $x_{0}$, we set $r_{0} = d(x_{0},x_{1})\in(0,\pi)$ 
    and $\omega = \frac{\Log_{x_{1}}(x_{0})}{\|\Log_{x_{1}}(x_{0})\|}$.
    Then we can write 
    \begin{align*}
        x_{0} = \Exp_{x_{1}}(r_{0}\omega),
        \qquad
        x_{1} = \Exp_{x_{1}}(0\cdot\omega).
    \end{align*}
    and 
    \begin{align*}
        X_{t}
        = \Exp_{x_{0}}\bigl(t\Log_{x_{0}}(x_{1})\bigr)
        = \Exp_{x_{1}}\bigl((1-t)r_{0}\,\omega\bigr).
    \end{align*}

    Thus we can define the desired diffeomorphism as $F_{t}: (r_{0},\omega)\mapsto(r,\omega)$ satisfying 
    \begin{align*}
        F_{t, x_{1}}(r_{0}, \omega) = ((1-t)r_{0}, \omega), 
        \qquad 
        F_{t, x_{1}}^{-1}(r, \omega) = (\frac{r}{1-t}, \omega).
    \end{align*}
    But however, note that we have to restrict $r < (1-t)\pi$, otherwise $\frac{r}{1-t} \notin (0,\pi)$, 
    consequently
    $F_{t, x_{1}}^{-1}(r, \omega)$ is no longer under the polar coordinate.

    Now recall the change of variable formula. 
    We should have 
    \begin{align*}
        \int_{S^{d}} p_{0}(F_{t, x_{1}}^{-1}(x)) |\det d F_{t, x_{1}}^{-1}| dV_{g}(x)
        = \int_{S^{d}} p_{0}(x_{0}) dV_{g}(x_{0}).
    \end{align*} 

    Written in polar coordinates: 
    the volume element at $x_{0}$ and $x$ are 
    \begin{align*}
        dV_{g}(x_{0}) = (\sin r_{0})^{d-1} dr_{0}\, d\omega,
        \qquad dV_{g}(x) = (\sin r)^{d-1} dr\, d\omega.
    \end{align*}
    Define a function $J_{t}(x\mid x_{1})$ to satisfy 
    \begin{align*}
        dV_{g}(x_{0}) = J_{t}(x\mid x_{1})\, dV_{g}(x).
    \end{align*}

    Using $r_{0} = r/(1-t)$ (hence $dr_{0} = dr/(1-t)$), we get
    \begin{align*}
        (\sin r_{0})^{d-1} dr_{0}\, d\omega
        &= J_{t}(x\mid x_{1})
        (\sin r)^{d-1} dr\, d\omega \\
        (\sin(r/(1-t)))^{d-1}\frac{dr}{1-t} d\omega
        &= J_{t}(x\mid x_{1})
        (\sin r)^{d-1} dr\, d\omega,
    \end{align*}

Now we derive the density. Note that by our construction of geodesic interpolation, given any $x_{0}$ and $t$, 
    the resulting $x$ must satisfy $d(x_{1}, x) = (1-t)d(x_{1}, x_{0}) \le (1-t)\pi$.
    Hence 
    \begin{align*}
        J_{t}(x\mid x_{1})
        = \frac{1}{1-t}
        \Bigl(
            \frac{\sin\bigl(r/(1-t)\bigr)}{\sin r}
        \Bigr)^{d-1},
        \qquad r = d(x,x_{1}),
    \end{align*}
    for all $r<(1-t)\pi$, and $J_{t}(x\mid x_{1})=0$ otherwise. Thus, we can equivalently write 
    \begin{align*}
        J_{t}(x\mid x_{1}) = \frac{1}{1-t}
       \Bigl(
           \frac{\sin\bigl(r/(1-t)\bigr)}{\sin r}
       \Bigr)^{d-1} \mathbf{1}_{\{d(x,x_{1})<(1-t)\pi\}}.
    \end{align*}

    The joint density of $(X_{0}, X_{1})$ is
    \begin{align*}
        p(x_{0},x_{1})
        = p_{0}(x_{0}) p_{1}(x_{1})
    \end{align*}
    with respect to $dV_{g}(x_{0})\,dV_{g}(x_{1})$.
    By the change of variable formula, 
    \begin{align*}
        p_{t}(x,x_{1})
        &= p_{0}(F_{t, x_{1}}^{-1}(x))
        p_{1}(x_{1})
        J_{t}(x\mid x_{1}).
    \end{align*}

    Finally, since $p_{t}(x,x_{1}) = p_{t}(x_{1}\mid x) p_{t}(x) = p_{t}(x_{1}\mid x) \int_{S^{d}} p_{t}(x,z)\, dV_{g}(z)$, 
    the conditional density of $X_{1}$ given $X_{t}=x$ is
    \begin{align*}
        p_{t}(x_{1}\mid x)
        = \frac{p_{t}(x,x_{1})}{
            \int_{S^{d}} p_{t}(x,z)\, dV_{g}(z)
        },
    \end{align*}
    which gives the formula stated in the lemma.

\end{proof}

Following Lemma \ref{Lemma_Conditional_Density}, we can derive the following formula for interpolated density.
\begin{lemma}
    \label{Lemma_interpolated_Density}
    We can write the interpolated density as 
    \begin{align*}
        p_{t}(x)
        =
        \int_{S^{d}} p_{t}(x,x_{1})\,dV_{g}(x_{1})
        =
        \frac{1}{\Vol(S^{d})}\int_{S^{d}} p_{1}(x_{1})J_{t}(x\mid x_{1}) dV_{g}(x_{1}).
    \end{align*}
\end{lemma}
\begin{proof}
    Following Lemma \ref{Lemma_Conditional_Density}, we have 
    \begin{align*}
        p_{t}(x,x_{1})
        &= p_{0}(F_{t, x_{1}}^{-1}(x))
        p_{1}(x_{1})
        J_{t}(x\mid x_{1}).
    \end{align*}
    Consequently, we have (note that $p_{0}$ is uniform distribution)
    \begin{align*}
        p_{t}(x)
        =
        \int_{S^{d}} p_{t}(x,x_{1})\,dV_{g}(x_{1})
        =
        \frac{1}{\Vol(S^{d})}\int_{S^{d}} p_{1}(x_{1})J_{t}(x\mid x_{1})dV_{g}(x_{1}).
    \end{align*}
\end{proof}

Before justifying the regularity of $v$, we first need to show that its smoothness is not destroyed by the indicator function, which corresponds to the cut point. In the following Lemma, we show that on a Hypersphere, the $\sin$ function (that appeared in the conditional density function) would smooth out the indicator function, resulting in a smooth vector field. 

\begin{lemma}
\label{Lemma_Sphere_Jt_smooth_tx}
Let $\mathcal{S}^{d}$ be the unit hypersphere with round metric, $d\ge 2$.
Let $x_{1}\in \mathcal{S}^{d}$ and define, for $t\in [0,1)$ and $x\in\mathcal{S}^{d}$. Denote $r(x) = d(x,x_{1})$ as the radial distance function.
Then for integer $m$ with $0\le m\le d-2$, the function $J(t,x\mid x_{1})$ viewed as a function on $[0,1)\times \mathcal{S}^{d}$ is $C^{m}$.

Consequently, $v(t, x)$ as a conditional expectation is $C^{2}$ in $(t, x)$ for $t \in [0, 1)$. 
As a result, the solution for flow matching ODE exists and unique.
\end{lemma}

\begin{proof}
Denote $\mathcal U:= \{(t,x)\in [0,1) \times \mathcal{S}^{d}:\ r(x)< (1-t)\pi \}$.
We first show the smoothness on $\mathcal U$.
Fix $(t_{0},x_{0})\in\mathcal U$ where $x_{0} \neq x_{1}$, and write $r_{0}:=r(x_{0}) > 0$.
Since $r_{0}< (1-t_{0})\pi<\pi$, we have $x_{0}\neq -x_{1}$.
Moreover, by continuity of $r(\cdot)$ there exists a neighborhood
$\mathcal V$ of $(t_{0},x_{0})$ such that for all $(t,x)\in\mathcal V$, 
$\mathbf 1_{\{r(x)< (1-t)\pi\}}\equiv 1$, and consequently
\begin{align*}
    J(t,x\mid x_{1})
    &= \frac{1}{1-t}
    \Bigl(\frac{\sin \bigl(r(x)/(1-t)\bigr)}{\sin r(x)}\Bigr)^{d-1}.
\end{align*}
Since $r = d(x, x_{1})$ is smooth when $x \notin \{x_{1}, -x_{1}\}$, 
$J$ is smooth as a composition of smooth functions.

It remains to check the smoothness of $J(t,x\mid x_{1})$ at point $x=x_{1}$ ($r = 0$).
Introduce normal coordinates at $x_{1}$:
for $x$ near $x_{1}$ write
\begin{align*}
    v:=\Log_{x_{1}}(x)\in T_{x_{1}}\mathcal{S}^{d},
    \qquad
    r(x)=d(x,x_{1})=\|v\|.
\end{align*}
It suffices to show that under normal coordinates, $J$ viewed as a function of $v$, is smooth. 
Notice that $\|v\|$ is not differentiable at $v=0$, hence functions such as
$v\mapsto \sin\|v\|$ are not $C^{1}$ at $0$. 
But notice that in $J$, the ``non-smooth'' part on $\|v\|$ cancels.
By checking smoothness of 
\begin{align*}
    \frac{\sin(\|v\|/(1-t))}{\sin\|v\|},
\end{align*}
we conclude the smoothness of $J$ on $U$.

We extend the smoothness result by checking the boundary $r= (1-t)\pi$.
To prove $C^{m}$, fix $(t_{0},x_{0})$ with $r(x_{0})=(1-t_{0})\pi$.
Since $r(x_{0}) \in (0, \pi)$, we have $x\mapsto r(x)$ is smooth near $x_{0}$.
Set $s(t, x)$ as follows, describing how far will $x$ reach the boundary:
\begin{align*}
    s(t,x):=(1-t)\pi-r(x).
\end{align*}
On the side $s>0$ we have $u:=r/(1-t)\in(0,\pi)$ and
\begin{align*}
    \sin \Bigl(\frac{r}{1-t}\Bigr)
    =
    \sin \Bigl(\pi-\frac{s}{1-t}\Bigr)
    =
    \sin \Bigl(\frac{s}{1-t}\Bigr).
\end{align*}
Therefore, for $s>0$,
\begin{align*}
    J(t,x\mid x_{1})
    &=
    \frac{1}{1-t}
    \Bigl(
        \frac{\sin \bigl(s(t, x)/(1-t)\bigr)}
             {\sin r(x)}
    \Bigr)^{d-1} =
    s(t,x)^{d-1} (1-t)^{-d}
    \left(
        \frac{\frac{\sin \bigl(s(t,x)/(1-t)\bigr)}{s(t,x)/(1-t)} }
             {\sin r(x)}
    \right)^{d-1}.
\end{align*}

On the whole manifold $\mathcal{S}^{d}$, we have 
\begin{align*}
    J(t,x\mid x_{1})
    =
    \bigl(s(t,x)\bigr)_{+}^{\,d-1}\,(1-t)^{-d}
    \left(
        \frac{\frac{\sin \bigl(s(t,x)/(1-t)\bigr)}{s(t,x)/(1-t)} }
             {\sin r(x)}
    \right)^{d-1},
    \qquad
    (s)_{+}:=\max\{s,0\}.
\end{align*}
Since $s \mapsto (s)_{+}^{\,d-1}$ is $C^{m}$ as a one-dimensional function for every $m\le d-2$,
and since $s(t,x)$ is smooth in $(t,x)$, the composition
$\bigl(s(t,x)\bigr)_{+}^{\,d-1}$ is $C^{m}$ in $(t,x)$ for all $m\le d-2$.
It follows that $J$ is $C^{m}$, as a product of a smooth function and $\bigl(s(t,x)\bigr)_{+}^{\,d-1}$.
\end{proof}

\begin{remark}
    The fact that $v$ being $C^{2}$ in $(t,x)$ for $t\in[0,1)$ follows from 
    smoothness of $J$, and the following properties of $\Log$: (1) $\Log$ being uniformly bounded, and (2) singularity of derivatives of $\Log_{x}(x_{1})$ as $x_{1} \to -x$ is well controlled under polar coordinates.

    The existence and uniqueness of flow matching ODE follows from classical theory on time-dependent flow, 
    see for example \cite[Section 4.1]{marsden2002manifolds} and \cite[Theorem 9.48]{lee_introduction_2012}.
\end{remark}

\subsection{Auxiliary Lemmas} The following result studies the derivative of a function that only depends on the radial distance. In other words, if a function $\phi(x)$ on $M$ only depends on $x$ through $r(x) = d(x, x_{1})$ for some fixed $x_{1}$, then we can control its derivative through derivative of $r$. 

\begin{lemma}
    \label{Lemma_Derivative_Log_J} Fix $x_{1}\in \mathcal{S}^{d}$ and define $r(x):=d(x,x_{1})$. Assume $x\notin \Cut(x_{1})$,
    so that $r$ is smooth in a neighborhood of $x$ and $\|\grad_{x}r(x)\|=1$.
    Let $F\colon (0,\pi)\to\mathbb{R}$ be $C^{2}$ and set $\phi(x):=F(r(x))$.
    Then
    \begin{align*}
        \grad_{x}\phi(x) = F'(r(x))\,\grad_{x}r(x),
        \qquad
        \|\grad_{x}\phi(x)\| = |F'(r(x))|.
    \end{align*}
    Moreover,
    \begin{align*}
        \|\nabla \grad \phi(x)\|_{\op}
        \le |F'(r(x))|\,\|\nabla \grad r(x)\|_{\op}
        + |F''(r(x))|.
    \end{align*}
\end{lemma}

\begin{proof}
    We apply the chain rule for the Riemannian gradient. For any $u\in T_{x}S^{d}$, the differential satisfies
    \begin{align*}
        d\phi(x)[u]
        = d(F\circ r)(x)[u]
        = F'(r(x))\,dr(x)[u].
    \end{align*}
    By the definition of the Riemannian gradient, 
    \begin{align*}
        \langle \grad_{x}\phi, u\rangle
        = F'(r(x))\,\langle \grad_{x}r, u\rangle,
        \qquad \forall u\in T_xS^d,
    \end{align*}
    hence
    \begin{align*}
        \grad_{x}\phi(x)=F'(r(x))\,\grad_{x}r(x).
    \end{align*}

    Taking norms gives
    \begin{align*}
        \|\grad_{x}\phi(x)\| = |F'(r(x))|\,\|\grad_{x}r(x)\|.
    \end{align*}
    Since $x\notin \Cut(x_{1})$, the distance function satisfies $\|\grad_{x}r(x)\|=1$,
    and therefore
    \begin{align*}
        \|\grad_{x}\phi(x)\| = |F'(r(x))|.
    \end{align*}

    Next, for any $u\in T_{x}S^{d}$,
    \begin{align*}
        \nabla_{u}\grad \phi(x)
        &= \nabla_{u}\bigl(F'(r(x))\,\grad r(x)\bigr) \\
        &= F'(r(x))\,\nabla_{u}\grad r(x)
        + u \left(F'(r(x))\right)\grad r(x) \\
        &= F'(r(x))\,\nabla_{u}\grad r(x)
        + F''(r(x))\,u(r(x))\,\grad r(x) \\
        &= F'(r(x))\,\nabla_{u}\grad r(x)
        + F''(r(x))\,\langle \grad r(x),u\rangle\,\grad r(x).
    \end{align*}
    Hence, for $\|u\|=1$,
    \begin{align*}
        \|\nabla_{u}\grad \phi(x)\|
        &\le |F'(r(x))|\,\|\nabla_{u}\grad r(x)\|
        + |F''(r(x))|\,|\langle \grad r(x),u\rangle|\,\|\grad r(x)\| \\
        &\le |F'(r(x))|\,\|\nabla_{u}\grad r(x)\|
        + |F''(r(x))|,
    \end{align*}
    using $\|\grad r(x)\|=1$ and $|\langle \grad r(x),u\rangle|\le \|u\|=1$.
    Taking the supremum over $\|u\|=1$ yields
    \begin{align*}
        \|\nabla \grad \phi(x)\|_{\op}
        \le |F'(r(x))|\,\|\nabla \grad r(x)\|_{\op}
        + |F''(r(x))|.
    \end{align*}
\end{proof}

The following result controls the conditional expectation of some certain ``functions with singularity", which will be used to establish regularity. 

\begin{lemma}[Expectation of $\frac{1}{\sin^{a}u}$.]
    \label{Lemma_Integral_Sin}
    Consider $d \ge 3$ and $a = 1, 2$. Let $r=d(x,x_{1})$ and set $u=r/(1-t)\in(0,\pi)$. We have 
    \begin{align*}
        \mathbb{E}[ \frac{1}{\sin^{a} u} \mid X_{t}=x] \le 2\frac{M_{1}}{m_{1}},
    \end{align*}
    where $0 < m_{1} \le p_{1} \le M_{1}$.
\end{lemma}
\begin{proof}
    We have 
    \begin{align*}
        \mathbb{E}\bigl[\frac{1}{\sin^{a} u}\mid X_{t}=x\bigr]
        &=
        \int_{S^{d}} \frac{1}{\sin^{a} u}\, p_{t}(x_{1}\mid x)\, dV_{g}(x_{1}) \\
        &\le
        \frac{M_{1}}{m_{1}\Vol(S^{d})}
        \int_{S^{d}}
            \frac{1}{\sin^{a} u}\,
            J_{t}(x\mid x_{1})
        \, dV_{g}(x_{1}).
    \end{align*}
    Use $r = (1-t) u$ so that $\frac{1}{1-t}dr = du$,
    \begin{align*}
        &\int_{S^{d}}
            \frac{1}{\sin^{a} u}\,
            J_{t}(x\mid x_{1})
        \, dV_{g}(x_{1}) \\
        &\qquad= \int_{S^{d-1}} \int_{0}^{\pi}
            \frac{1}{\sin^{a} u}\
            J_{t}(x\mid x_{1})\, (\sin r)^{d-1} dr d\omega \\
        &\qquad = \Vol(S^{d-1})\int_{0}^{\pi}
            \frac{1}{\sin^{a} u}
            \sin(u)^{d-1}\,du \\
        &\qquad = \Vol(S^{d-1})\int_{0}^{\pi}
            \sin(u)^{d-1-a}\,du. 
    \end{align*}
    Therefore
    \begin{align*}
        \mathbb{E}\bigl[\frac{1}{\sin^{a} u}\mid X_{t}=x\bigr]
        &=
        \int_{S^{d}} \frac{1}{\sin^{a} u}\, p_{t}(x_{1}\mid x)\, dV_{g}(x_{1}) \\
        &\le
        \frac{M_{1}\Vol(S^{d-1})}{m_{1}\Vol(S^{d})}\int_{0}^{\pi}
            \sin(u)^{d-1-a}\,du. \\
        &= \frac{M_{1}\Vol(S^{d-1})}{m_{1}\Vol(S^{d})\Vol(S^{d-1-a})}\Vol(S^{d-1-a})\int_{0}^{\pi}
            \sin(u)^{d-1-a}\,du \\
        &= \frac{M_{1}\Vol(S^{d-1})}{m_{1}\Vol(S^{d})\Vol(S^{d-1-a})}
        \Vol(S^{d-a}).
    \end{align*}

    Recall
    \begin{align*}
        \Vol(S^{n})
        = \frac{2\pi^{(n+1)/2}}{\Gamma((n+1)/2)}.
    \end{align*}
    Hence
    \begin{align*}
        \frac{\Vol(S^{d-1})}{\Vol(S^{d})}
        &=
        \frac{ \frac{2\pi^{d/2}}{\Gamma(d/2)} }{ \frac{2\pi^{(d+1)/2}}{\Gamma((d+1)/2)} }
        =
        \pi^{-1/2}\frac{\Gamma((d+1)/2)}{\Gamma(d/2)}, \\
        \frac{\Vol(S^{d-a})}{\Vol(S^{d-1-a})}
        &=
        \frac{ \frac{2\pi^{(d-a+1)/2}}{\Gamma((d-a+1)/2)} }{ \frac{2\pi^{(d-a)/2}}{\Gamma((d-a)/2)} }
        =
        \pi^{1/2}\frac{\Gamma((d-a)/2)}{\Gamma((d-a+1)/2)}.
    \end{align*}
    By Kershaw's inequality, for $d \ge 3$, 
    \begin{align*}
        \frac{\Gamma((d+1)/2)}{\Gamma(d/2)} \le (d/2 - \frac{1}{2} + \sqrt{\frac{3}{4}})^{\frac{1}{2}},
    \end{align*}
    and 
    \begin{align*}
        \frac{\Gamma(d/2)}{\Gamma((d-1)/2)} \ge (\frac{d}{2}-1 + \frac{1}{2})^{\frac{1}{2}},
    \end{align*}
    \begin{align*}
        \frac{\Gamma((d-a+1)/2)}{\Gamma((d-a)/2)} \ge (\frac{d-a-1}{2} + \frac{1}{2})^{\frac{1}{2}}.
    \end{align*}
    Together, 
    \begin{align*}
        \frac{\Vol(S^{d-1})\Vol(S^{d-a})}{\Vol(S^{d})\Vol(S^{d-1-a})}
        &= \frac{\Gamma((d+1)/2)\Gamma((d-a)/2)}{\Gamma(d/2)\Gamma((d-a+1)/2)}
        \le (\frac{d/2 - \frac{1}{2} + \sqrt{\frac{3}{4}}}{\frac{d-a-1}{2} + \frac{1}{2}})^{\frac{1}{2}} \\
        &= (\frac{d -a + a - 1 + 2\sqrt{\frac{3}{4}}}{d-a})^{\frac{1}{2}}
        \le (1 + \frac{a - 1 + 2\sqrt{\frac{3}{4}}}{d-a})^{\frac{1}{2}}.
    \end{align*}
    For $a = 1, 2$, we have $(1 + \frac{a - 1 + 2\sqrt{\frac{3}{4}}}{d-a})^{\frac{1}{2}} \le 2$.
\end{proof}

The following results (Lemmas \ref{Lemma_Moment_Bound_Grad_Log_P}, \ref{Lemma_Bound_partial_t_log_p}, \ref{Lemma_Bound_Hessian_Log_P}, and \ref{Lemma_Bounds_Sphere}) provide bounds on the building blocks that appear in the derivative formulas for $v(t, x)$. For all the results below, we assume $0 < m_{1} \le p_{1}(x) \le M_{1}$. 
\begin{lemma}[Moment bounds for $\grad_{x}\log p_{t}(X_{1}\mid x)$.]
    \label{Lemma_Moment_Bound_Grad_Log_P}
    We have \begin{align*}
        \|\grad_{x}\log J_{t}(x\mid x_{1})\|
        \le \frac{2(d-1)}{(1-t)\sin u}.
    \end{align*}
    Furthermore, 
    \begin{align*}
        \mathbb{E}[\|\grad_{x}\log p_{t}(X_{1}\mid x)\|\mid X_{t}=x]
        &\le \frac{8(d-1)}{(1-t)}
        \frac{M_{1}}{m_{1}}, \\
        \mathbb{E}[\|\grad_{x}\log J_{t}(x\mid X_{1})\|^{2}\mid X_{t}=x] 
        &\le \frac{8(d-1)^{2}}{(1-t)^{2}} \frac{M_{1}}{m_{1}}, \\
        \mathbb{E}[\|\grad_{x}\log p_{t}(X_{1}\mid x)\|^{2}\mid X_{t}=x]
        &\le \frac{32(d-1)^{2}}{(1-t)^{2}}
        \frac{M_{1}}{m_{1}}.
    \end{align*}
\end{lemma}
\begin{proof}
    We first bound $\|\grad_{x}\log J_{t}(x\mid x_{1})\|$.
    Let $r=d(x,x_{1})$ and set $u=r/(1-t)\in(0,\pi)$. On the set $r<(1-t)\pi$,
    \begin{align*}
        \frac{\partial}{\partial r}\log J_{t}(r)
        = (d-1)\Bigl(\frac{1}{1-t}\cot\frac{r}{1-t} - \cot r\Bigr)
        = (d-1)\Bigl(\frac{1}{1-t}\cot u - \cot((1-t)u)\Bigr).
    \end{align*}
    By the Lemma \ref{Lemma_Derivative_Log_J},
    \begin{align*}
        \|\grad_{x}\log J_{t}(x\mid x_{1})\|
        = \Bigl|\frac{\partial}{\partial r}\log J_{t}(r)\Bigr|
        = (d-1)\Bigl|\frac{1}{1-t}\cot u - \cot((1-t)u)\Bigr|.
    \end{align*}
    For $u\in(0,\pi)$, we have $|\cot u|\le 1/\sin u$. 
    Since $\sin((1-t)u)\ge (1-t)\sin u$, we have 
    \begin{align*}
        |\cot((1-t)u)|
        \le \frac{1}{\sin((1-t)u)}
        \le \frac{1}{(1-t)\sin u}.
    \end{align*}
    Therefore for all $u\in(0,\pi)$, we can bound 
    \begin{align*}
        \Bigl|\frac{1}{1-t}\cot u - \cot((1-t)u)\Bigr|
        &\le \frac{1}{1-t}|\cot u| + |\cot((1-t)u)| 
        \le \frac{1}{(1-t)\sin u} + \frac{1}{(1-t)\sin u} \\
        &= \frac{2}{(1-t)\sin u}.
    \end{align*}
    So
    \begin{align*}
        \|\grad_{x}\log J_{t}(x\mid x_{1})\|
        \le \frac{2(d-1)}{(1-t)\sin u}.
    \end{align*}

    Now we bound $\mathbb{E}\|\grad_{x}\log p_{t}\|$ by $\mathbb{E}\|\grad_{x}\log J_{t}\|$.
    Let
    \begin{align*}
        Z_{t}(x) := \int_{S^{d}} p_{1}(z)J_{t}(x\mid z)\, dV_{g}(z),
    \end{align*}
    so
    \begin{align*}
        \log p_{t}(x_{1}\mid x) = \log p_{1}(x_{1}) + \log J_{t}(x\mid x_{1}) - \log Z_{t}(x).
    \end{align*}
    Since $p_{1}$ does not depend on $x$,
    \begin{align*}
        \grad_{x}\log p_{t}(x_{1}\mid x)
        = \grad_{x}\log J_{t}(x\mid x_{1}) - \grad_{x}\log Z_{t}(x).
    \end{align*}
    Also,
    \begin{align*}
        \grad_{x}Z_{t}(x)
        = \int_{S^{d}} p_{1}(z)\grad_{x}J_{t}(x\mid z)\, dV_{g}(z)
        = \int_{S^{d}} p_{1}(z)J_{t}(x\mid z)\grad_{x}\log J_{t}(x\mid z)\, dV_{g}(z),
    \end{align*}
    hence
    \begin{align*}
        \grad_{x}\log Z_{t}(x)
        = \int_{S^{d}} \grad_{x}\log J_{t}(x\mid z)\, p_{t}(z\mid x)\, dV_{g}(z)
        = \mathbb{E}\bigl[\grad_{x}\log J_{t}(x\mid X_{1})\mid X_{t}=x\bigr].
    \end{align*}
    Therefore
    \begin{align*}
        \grad_{x}\log p_{t}(x_{1}\mid x)
        =
        \grad_{x}\log J_{t}(x\mid x_{1})
        -
        \mathbb{E}\bigl[\grad_{x}\log J_{t}(x\mid X_{1})\mid X_{t}=x\bigr],
    \end{align*}
    and by triangle inequality,
    \begin{align*}
        \mathbb{E}\bigl[\|\grad_{x}\log p_{t}(X_{1}\mid x)\|\mid X_{t}=x\bigr]
        \le
        2\,\mathbb{E}\bigl[\|\grad_{x}\log J_{t}(x\mid X_{1})\|\mid X_{t}=x\bigr]
        \le \frac{8(d-1)}{(1-t)} \frac{M_{1}}{m_{1}},
    \end{align*}
    where we used Lemma \ref{Lemma_Integral_Sin}.

    For second moment, we have 
    \begin{align*}
        \mathbb{E}[\|\grad_{x}\log p_{t}(X_{1}\mid x)\|^{2}\mid X_{t}=x]
        \le 4\,\mathbb{E}[\|\grad_{x}\log J_{t}(x\mid X_{1})\|^{2}\mid X_{t}=x].
    \end{align*}
    Recall $\|\grad_{x}\log J_{t}(x\mid x_{1})\| \le \frac{2(d-1)}{(1-t)\sin u}$, using Lemma \ref{Lemma_Integral_Sin},
    \begin{align*}
        \mathbb{E}[\|\grad_{x}\log J_{t}(x\mid X_{1})\|^{2}\mid X_{t}=x]
        \le \frac{4(d-1)^{2}}{(1-t)^{2}}
        \mathbb{E}\Bigl[\frac{1}{\sin^{2}u}\mid X_{t}=x\Bigr]
        \le \frac{8(d-1)^{2}}{(1-t)^{2}}
        \frac{M_{1}}{m_{1}}.
    \end{align*}
    Hence
    \begin{align*}
        \mathbb{E}[\|\grad_{x}\log p_{t}(X_{1}\mid x)\|^{2}\mid X_{t}=x]
        \le \frac{32(d-1)^{2}}{(1-t)^{2}}
        \frac{M_{1}}{m_{1}}.
    \end{align*}
\end{proof}

\begin{lemma}[Moment bounds for $\partial_{t}\log p_{t}(X_{1}\mid x)$.]
    \label{Lemma_Bound_partial_t_log_p}
    We have 
    \begin{align*}
        \mathbb{E}\Bigl[\Bigl|\frac{\partial}{\partial t}\log p_{t}(X_{1}\mid x)\Bigr|\mid X_{t}=x\Bigr]
        &\le \frac{2}{1-t} + 4\pi \frac{d-1}{1-t}\,
            \frac{ M_{1}}{m_{1}}, \\
        \mathbb{E}\Bigl[\Bigl|\frac{\partial}{\partial t}\log p_{t}(X_{1}\mid x)\Bigr|^{2}\mid X_{t}=x\Bigr]
        &\le \frac{8}{(1-t)^{2}} + \frac{16\pi^{2}(d-1)^{2}}{(1-t)^{2}}\,
            \frac{M_{1}}{m_{1}}.
    \end{align*}
\end{lemma}
\begin{proof} Let
    \begin{align*}
        Z_{t}(x) = \int_{S^{d}} p_{1}(z)J_{t}(x\mid z)\, dV_{g}(z),
    \end{align*}
    so we have $\log p_{t}(x_{1}\mid x) = \log p_{1}(x_{1}) + \log J_{t}(x\mid x_{1}) - \log Z_{t}(x)$.
    Since $p_{1}$ does not depend on $t$,
    \begin{align*}
        \frac{\partial}{\partial t}\log p_{t}(x_{1}\mid x)
        = \frac{\partial}{\partial t}\log J_{t}(x\mid x_{1})
        - \frac{\partial}{\partial t}\log Z_{t}(x).
    \end{align*}
    Moreover, for $d\ge 2$ the function $\partial_{t}J_{t}(x\mid z)$ is integrable
    against $p_{1}(z)dV_{g}(z)$, and the map $t\mapsto Z_{t}(x)$ is $C^{1}$ with
    \begin{align*}
        \frac{\partial}{\partial t}Z_{t}(x)
        = \int_{S^{d}} p_{1}(z)\frac{\partial}{\partial t}J_{t}(x\mid z)\, dV_{g}(z)
        = \int_{S^{d}} p_{1}(z)J_{t}(x\mid z)\frac{\partial}{\partial t}\log J_{t}(x\mid z)\, dV_{g}(z).
    \end{align*}
    Dividing by $Z_{t}(x)$ yields
    \begin{align*}
        \frac{\partial}{\partial t}\log Z_{t}(x)
        = \int_{S^{d}} \frac{\partial}{\partial t}\log J_{t}(x\mid z)\, p_{t}(z\mid x)\, dV_{g}(z)
        = \mathbb{E}\Bigl[\frac{\partial}{\partial t}\log J_{t}(x\mid X_{1})\mid X_{t}=x\Bigr].
    \end{align*}
    Therefore
    \begin{align*}
        \frac{\partial}{\partial t}\log p_{t}(x_{1}\mid x)
        =
        \frac{\partial}{\partial t}\log J_{t}(x\mid x_{1})
        -
        \mathbb{E}\Bigl[\frac{\partial}{\partial t}\log J_{t}(x\mid X_{1})\mid X_{t}=x\Bigr],
    \end{align*}
    and by triangle inequality,
    \begin{align*}
        \mathbb{E}\Bigl[\Bigl|\frac{\partial}{\partial t}\log p_{t}(X_{1}\mid x)\Bigr|\mid X_{t}=x\Bigr]
        &\le 2\,
        \mathbb{E}\Bigl[\Bigl|\frac{\partial}{\partial t}\log J_{t}(x\mid X_{1})\Bigr|\mid X_{t}=x\Bigr], \\
        \mathbb{E}\Bigl[\Bigl|\frac{\partial}{\partial t}\log p_{t}(X_{1}\mid x)\Bigr|^{2}\mid X_{t}=x\Bigr]
        &\le 4\,
        \mathbb{E}\Bigl[\Bigl|\frac{\partial}{\partial t}\log J_{t}(x\mid X_{1})\Bigr|^{2}\mid X_{t}=x\Bigr].
    \end{align*}

    We compute $\partial_{t}\log J_{t}$ explicitly and bound its conditional expectation.
    For $r=d(x,x_{1})<(1-t)\pi$,
    \begin{align*}
        \log J_{t}(x\mid x_{1})
        = -\log(1-t) + (d-1)\Bigl(\log\sin(r/(1-t)) - \log\sin r\Bigr).
    \end{align*}
    Holding $r$ fixed and differentiating in $t$ gives
    \begin{align*}
        \frac{\partial}{\partial t}\log J_{t}(x\mid x_{1})
        = \frac{1}{1-t} + (d-1)\frac{r}{(1-t)^{2}}\cot(r/(1-t)).
    \end{align*}
    Write $u=r/(1-t)\in(0,\pi)$. Then $r=(1-t)u$ and
    \begin{align*}
        \Bigl|\frac{\partial}{\partial t}\log J_{t}(x\mid x_{1})\Bigr|
        &\le \frac{1}{1-t} + \frac{d-1}{1-t}\,|u\cot u|, \\
        \Bigl|\frac{\partial}{\partial t}\log J_{t}(x\mid x_{1})\Bigr|^{2}
        &\le 2\frac{1}{(1-t)^{2}} + 2\frac{(d-1)^{2}}{(1-t)^{2}}\,|u\cot u|^{2}.
    \end{align*}
    Hence
    \begin{align*}
        \mathbb{E}\Bigl[\Bigl|\frac{\partial}{\partial t}\log J_{t}(x\mid X_{1})\Bigr|\mid X_{t}=x\Bigr]
        &\le \frac{1}{1-t} + \pi \frac{d-1}{1-t}\,
            \frac{2 M_{1}}{m_{1}}, \\
        \mathbb{E}\Bigl[\Bigl|\frac{\partial}{\partial t}\log J_{t}(x\mid X_{1})\Bigr|^{2}\mid X_{t}=x\Bigr]
        &\le \frac{2}{(1-t)^{2}} + \frac{2\pi^{2}(d-1)^{2}}{(1-t)^{2}}\,
            \frac{2 M_{1}}{m_{1}}.
    \end{align*}
    where we used $|u\cot u| \le \frac{\pi}{\sin u}$.
    Consequently,
    \begin{align*}
        \mathbb{E}\Bigl[\Bigl|\frac{\partial}{\partial t}\log p_{t}(X_{1}\mid x)\Bigr|\mid X_{t}=x\Bigr]
        &\le \frac{2}{1-t} + 4\pi \frac{d-1}{1-t}\,
            \frac{ M_{1}}{m_{1}}, \\
        \mathbb{E}\Bigl[\Bigl|\frac{\partial}{\partial t}\log p_{t}(X_{1}\mid x)\Bigr|^{2}\mid X_{t}=x\Bigr]
        &\le \frac{8}{(1-t)^{2}} + \frac{16\pi^{2}(d-1)^{2}}{(1-t)^{2}}\,
            \frac{M_{1}}{m_{1}}.
    \end{align*}
\end{proof}

\begin{lemma}
    \label{Lemma_Bound_Hessian_Log_P}
    We have
    \begin{align*}
        \mathbb{E}\bigl[\|\nabla_x^2\log p_t(X_1\mid x)\|_{\op}\mid X_t=x\bigr]
        \le
        \frac{64(d-1)^2}{(1-t)^2}\frac{M_1}{m_1},
    \end{align*}
\end{lemma}
\begin{proof}
    Recall that 
    \begin{align*}
        \grad_{x}\log p_{t}(x_{1}\mid x)
        = \grad_{x}\log J_{t}(x\mid x_{1}) - \grad_{x}\log Z_{t}(x),
    \end{align*}
    Take derivative again, we obtain 
    \begin{align*}
        \nabla \grad_{x}\log p_{t}(x_{1}\mid x)
        = \nabla \grad_{x}\log J_{t}(x\mid x_{1}) - \nabla \grad_{x}\log Z_{t}(x),
    \end{align*}
    Recall 
    \begin{align*}
        \grad_{x}\log Z_{t}(x)
        = \mathbb{E}\bigl[\grad_{x}\log J_{t}(x\mid X_{1})\mid X_{t}=x\bigr].
    \end{align*}
    Hence 
    \begin{align*}
        &\nabla_{u} \grad_{x}\log Z_{t}(x)
        = \nabla_{u} \mathbb{E}\bigl[\grad_{x}\log J_{t}(x\mid X_{1})\mid X_{t}=x\bigr] \\
        =& \mathbb{E}\bigl[\nabla_{u} \grad_{x}\log J_{t}(x\mid X_{1})\mid X_{t}=x\bigr]
        + \mathbb{E}\bigl[\grad_{x}\log J_{t}(x\mid X_{1}) \nabla_{u} \log p_{t}(x_{1} \mid x) \mid X_{t}=x\bigr],
    \end{align*}
    and we can bound
    \begin{align*}
        &\|\nabla \grad_{x}\log Z_{t}(x)\|_{\op}  \\
        &\le \mathbb{E}\bigl[\|\nabla^{2}\log J_{t}(x\mid X_{1})\|_{\op}\mid X_{t}=x\bigr] \\
        & \qquad + \sqrt{\mathbb{E}\bigl[\|\grad_{x}\log J_{t}(x\mid X_{1})\|^{2} \mid X_{t}=x\bigr]
        \mathbb{E}\bigl[\| \nabla \log p_{t}(x_{1} \mid x)\|^{2} \mid X_{t}=x\bigr]} \\
        &\le \mathbb{E}\bigl[\|\nabla^{2}\log J_{t}(x\mid X_{1})\|_{\op}\mid X_{t}=x\bigr] 
        + \frac{16(d-1)^{2}M_{1}}{(1-t)^{2}m_{1}}.
    \end{align*}

    To control $\|\nabla^{2}\log J_{t}(x\mid X_{1})\|_{\op}$, we consider 
    \begin{align*}
        \frac{\partial^{2}}{\partial r^{2}} \log J_{t}(x\mid X_{1})
        &= \frac{\partial}{\partial r} (d-1)(\frac{1}{1-t}\cot \frac{r}{1-t} - \cot r)
        = (d-1)(- \frac{1}{(1-t)^{2}} \frac{1}{\sin^{2} \frac{r}{1-t}} + \frac{1}{\sin^{2} r}) \\
        &= (d-1) (\frac{1}{\sin^{2} (1-t)u} - \frac{1}{(1-t)^{2}} \frac{1}{\sin^{2}u}).
    \end{align*}

    Applying Lemma \ref{Lemma_Derivative_Log_J} with $\phi = \log J_{t}(r)$, we have 
    \begin{align*}
        \|\nabla \grad \log J_{t}(r) \|_{\op}
        &\le |\log J_{t}'(r(x))| \| \nabla \grad r(x)\|_{\op}
        + |\log J_{t}''(r(x)) | \\
        &\le |\log J_{t}'(r(x))| \cot r
        + |\log J_{t}''(r(x)) |,
    \end{align*}
    where $\| \nabla \grad r(x)\|_{\op} \le |\cot r|$ by \citet[Proposition 11.3]{lee2018introduction}

    Now using 
    \begin{align*}
        \frac{\partial}{\partial r} \log J_{t}(x\mid X_{1}) &= (d-1)(\frac{1}{1-t}\cot \frac{r}{1-t} - \cot r), \\
        \frac{\partial^{2}}{\partial r^{2}} \log J_{t}(x\mid X_{1}) &= (d-1) (\frac{1}{\sin^{2} (1-t)u} - \frac{1}{(1-t)^{2}} \frac{1}{\sin^{2}u}),
    \end{align*}
    we obtain 
    \begin{align*}
        \|\nabla^{2}\log J_{t}(x\mid X_{1})\|_{\op}
        \le \frac{4(d-1)}{(1-t)^{2}\sin^{2} u},
    \end{align*}
    hence
    \begin{align*}
        \mathbb{E}[\|\nabla^{2}\log J_{t}(x\mid X_{1})\|\mid X_{t}=x]
        \le \frac{16(d-1)}{(1-t)^{2}}\frac{M_{1}}{m_{1}}.
    \end{align*}

    Now we bound $\mathbb{E}[\|\nabla_x^2\log p_t(X_1\mid x)\|_{\op}\mid X_t=x]$.
    Recall 
    \begin{align*}
        \nabla \grad_{x}\log p_{t}(x_{1}\mid x)
        &= \nabla \grad_{x}\log J_{t}(x\mid x_{1}) - \nabla \grad_{x}\log Z_{t}(x), \\
        \|\nabla \grad_{x}\log Z_{t}(x)\|_{\op} 
        &\le \mathbb{E}\bigl[\|\nabla^{2}\log J_{t}(x\mid X_{1})\|_{\op}\mid X_{t}=x\bigr] 
        + \frac{16(d-1)^{2}M_{1}}{(1-t)^{2}m_{1}}.
    \end{align*}
    By triangle inequality, we obtain
    \begin{align*}
        \mathbb{E}\bigl[\|\nabla_x^2\log p_t(X_1\mid x)\|_{\op}\mid X_t=x\bigr]
        \le
        2\,\mathbb{E}\bigl[\|\nabla_x^2\log J_t(x\mid X_1)\|_{\op}\mid X_t=x\bigr]
        + \frac{16(d-1)^2}{(1-t)^2}\frac{M_1}{m_1}.
    \end{align*}
    Finally, using the bound
    \begin{align*}
        \mathbb{E}\bigl[\|\nabla_x^2\log J_t(x\mid X_1)\|_{\op}\mid X_t=x\bigr]
        \le \frac{16(d-1)}{(1-t)^2}\frac{M_1}{m_1},
    \end{align*}
    we conclude
    \begin{align*}
        \mathbb{E}\bigl[\|\nabla_x^2\log p_t(X_1\mid x)\|_{\op}\mid X_t=x\bigr]
        \le
        \frac{32(d-1)}{(1-t)^2}\frac{M_1}{m_1}
        + \frac{16(d-1)^2}{(1-t)^2}\frac{M_1}{m_1}
        <
        \frac{64(d-1)^2}{(1-t)^2}\frac{M_1}{m_1}.
    \end{align*}
\end{proof}

\begin{lemma} 
    \label{Lemma_Bounds_Sphere} We have
    \begin{align*}
        \mathbb{E}[\|\partial_{t}\grad_{x}\log p_{t}(X_{1}\mid x)\|\mid X_{t}=x]
        \le
        \frac{16\pi^{2}(d-1)^{2}}{(1-t)^{2}}\frac{M_{1}}{m_{1}}.
    \end{align*}
\end{lemma}
\begin{proof}
    We have 
    \begin{align*}
        \grad_{x}\log p_{t}(x_{1}\mid x) &= \grad_{x}\log J_{t}(x\mid x_{1})-\grad_{x}\log Z_{t}(x) \\
        &= \grad_{x}\log J_{t}(x\mid x_{1}) - \mathbb{E}\bigl[\grad_{x}\log J_{t}(x\mid X_{1})\mid X_{t}=x\bigr] ,
    \end{align*}
    so we have 
    \begin{align*}
        \|\partial_{t}\grad_{x}\log p_{t}(x_{1}\mid x)\|
        &\le
        \|\partial_{t}\grad_{x}\log J_{t}(x\mid x_{1})\| 
        + \mathbb{E}\bigl[\|\partial_{t}\grad_{x}\log J_{t}(x\mid X_{1})\| \mid X_{t}=x\bigr]  \\
        & \quad + \mathbb{E}[\|\grad_{x}\log J_{t}(x\mid X_{1})\|\,|\partial_{t}\log p_{t}(X_{1}\mid x)|\mid X_{t}=x].
    \end{align*}
    A direct differentiation at fixed $r$ gives
    \begin{align*}
        \|\partial_{t}\grad_{x}\log J_{t}(x\mid x_{1})\|
        &= \|\partial_{t} \frac{\partial }{\partial r} \log J_{t}(r) \grad r(x)\|
        = (d-1)\Bigl\|\partial_{t}\frac{1}{1-t}\cot u - \cot((1-t)u)\Bigr\| \\
        &= (d-1)\Bigl\|\frac{1}{(1-t)^{2}}\cot u - \frac{u}{\sin^{2}(1-t)u}\Bigr\| \\
        &\le \frac{d-1}{(1-t)^{2}} \frac{2\pi}{\sin^{2}u} .
    \end{align*}

    Taking expectation and using $\mathbb{E}[1/\sin^{a}u]\le 2M_{1}/m_{1}$ for $a=1,2$,
    \begin{align*}
        \mathbb{E}[\|\partial_{t}\grad_{x}\log J_{t}(x\mid X_{1})\|\mid X_{t}=x]
        \le \frac{4\pi(d-1)}{(1-t)^{2}} \frac{M_{1}}{m_{1}},
    \end{align*}

    Recall that 
    \begin{align*}
        \mathbb{E}\Bigl[\Bigl|\partial_{t} \log p_{t}(X_{1}\mid x)\Bigr|^{2}\mid X_{t}=x\Bigr]
        &\le 4 \mathbb{E}\Bigl[\Bigl|\partial_{t} \log J_{t}\Bigr|^{2}\mid X_{t}=x\Bigr] \\
        &\le 4 \mathbb{E}\Bigl[ \frac{(d-1)^{2} \pi^{2}}{(1-t)^{2}} \frac{1}{\sin^{2} u} \mid X_{t}=x\Bigr]
        \le \frac{8 \pi^{2}(d-1)^{2}}{(1-t)^{2}} \frac{M_{1}}{m_{1}}.
    \end{align*}

    Together with $\|\grad_{x}\log J_{t}(x\mid x_{1})\|\le \frac{2(d-1)}{(1-t)\sin u}$, by Cauchy-Schwarz,
    \begin{align*}
        &\mathbb{E}[\|\grad_{x}\log J_{t}(x\mid X_{1})\|\,|\partial_{t}\log p_{t}(X_{1}\mid x)|\mid X_{t}=x] \\
        &\le
        \Bigl(\mathbb{E}[\|\grad_{x}\log J_{t}(x\mid X_{1})\|^{2}\mid X_{t}=x]\Bigr)^{1/2}
        \Bigl(\mathbb{E}[|\partial_{t}\log p_{t}(X_{1}\mid x)|^{2}\mid X_{t}=x]\Bigr)^{1/2} \\
        &\le
        \sqrt{\frac{8\pi(d-1)^{2}}{(1-t)^{2}} \frac{M_{1}}{m_{1}}\frac{8 \pi^{2}(d-1)^{2}}{(1-t)^{2}} \frac{M_{1}}{m_{1}} }
        \le \frac{8\pi^{2}(d-1)^{2}}{(1-t)^{2}}\frac{M_{1}}{m_{1}}.
    \end{align*}
    Thus
    \begin{align*}
        \mathbb{E}[\|\partial_{t}\grad_{x}\log p_{t}(X_{1}\mid x)\|\mid X_{t}=x]
        \le
        \frac{16\pi^{2}(d-1)^{2}}{(1-t)^{2}}\frac{M_{1}}{m_{1}}.
    \end{align*}
\end{proof}

\subsection{Regularity for Flow Matching Vector Field}

In this section, we show that both the spatial derivative (Lemma \ref{Lemma_Nabla_v_Sphere}) and time derivative (Lemma \ref{Lemma_Partial_t_v_Sphere}) of $v(t, x)$ are uniformly bounded.

\begin{lemma}\label{Lemma_Nabla_v_Sphere}
    Assume $d\ge 3$ and $t\in(0,1)$. Let $p_{1}$ be a smooth density on $S^{d}$ such that
    \begin{align*}
        0<m_{1}\le p_{1}(z)\le M_{1}<\infty.
    \end{align*}
    For all unit tangent vector $w$, we have 
    \begin{align*}
        \|\nabla_{w} v(t,x)\| \le \frac{12\pi M_{1}(d-1)}{m_{1}(1-t)}.
    \end{align*}
    Hence $v(t, x)$ is $L$-Lipschitz with $L = \frac{12\pi M_{1}(d-1)}{m_{1}(1-t)}$.
\end{lemma}

\begin{proof}
    Fix $t<1$ and $x\in M$, and let $w\in T_{x}M$ be a unit tangent vector. 
We consider the covariant derivative of $v$ in the direction $w$:
\begin{align*}
    \nabla_{w} v(t,x)
    &= \frac{1}{1-t}\nabla_{w}\int_{M} \Log_{x}(x_{1})\, p_{t}(x_{1}|x)\, dV_{g}(x_{1}) \\
    &= \frac{1}{1-t}\int_{M} \nabla_{w}\big(\Log_{x}(x_{1})\big)\, p_{t}(x_{1}|x)\, dV_{g}(x_{1}) \\ & \qquad + \frac{1}{1-t}\int_{M} \Log_{x}(x_{1})\, \langle \grad_{x} p_{t}(x_{1}|x), w\rangle\, dV_{g}(x_{1}).
\end{align*}
    Define 
    \begin{align*}
        P(t,x)
        &:= \frac{1}{1-t}\int_{S^{d}} \Log_{x}(x_{1})\langle \grad_{x} p_{t}(x_{1}\mid x), w \rangle\, dV_{g}(x_{1}), \\
        G(t,x) &:=
        \frac{1}{1-t}\int_{S^{d}}
            \|\nabla_{w}\Log_{x}(x_{1})\|
            p_{t}(x_{1}\mid x)\, dV_{g}(x_{1}).
    \end{align*}

    We show that $\|P(t,x)\| \le \frac{8\pi M_{1}(d-1)}{m_{1}(1-t)}, \forall \|w\| = 1$. We first write $\|P(t,x)\|$ as a conditional moment of $\|\grad_{x}\log p_{t}\|$.
        Using $\grad_{x}p_{t} = p_{t}\grad_{x}\log p_{t}$ and $\|\Log_{x}(x_{1})\|=d(x,x_{1})$,
        \begin{align*}
            \|P(t,x)\|
            &\le \frac{1}{1-t}\int_{S^{d}} \|\Log_{x}(x_{1})\|\,\|\grad_{x}p_{t}(x_{1}\mid x)\|\, dV_{g}(x_{1}) \\
            &= \frac{1}{1-t}\int_{S^{d}} \|\Log_{x}(x_{1})\|\, p_{t}(x_{1}\mid x)\,\|\grad_{x}\log p_{t}(x_{1}\mid x)\|\, dV_{g}(x_{1}) \\
            &= \frac{1}{1-t}\mathbb{E}\bigl[\|\Log_{x}(X_{1})\|\,\|\grad_{x}\log p_{t}(X_{1}\mid x)\|\mid X_{t}=x\bigr].
        \end{align*}
        On the support of $p_{t}(\cdot\mid x)$ we have $d(x,X_{1})<(1-t)\pi$, hence
        \begin{align*}
            \|\Log_{x}(X_{1})\|\le (1-t)\pi,
        \end{align*}
        so
        \begin{align*}
            \|P(t,x)\|
            \le \pi\,\mathbb{E}\bigl[\|\grad_{x}\log p_{t}(X_{1}\mid x)\|\mid X_{t}=x\bigr].
        \end{align*}

        Using Lemma \ref{Lemma_Moment_Bound_Grad_Log_P},
        \begin{align*}
            \|P(t,x)\|
            \le \frac{8\pi M_{1}(d-1)}{m_{1}(1-t)}.
        \end{align*}

    Next, we show that $G(t,x) \le 4\pi\frac{M_{1}}{m_{1}(1-t)}$. 
        
        Notice that $\nabla_{w}\Log_{x}(x_{1})$ can be related to a Hessian.
        Denote $r=d(x,x_{1})\in(0,\pi)$, it is easy to show 
        (see, for example, \citet[Appendix B, Proof of Lemma 2]{alimisis2020continuous})
        \begin{align*}
            \|\nabla_{w}\Log_{x}(x_{1})\|
            \le 1+|r\cot(r)|.
        \end{align*}

        Now we bound the $\cot$ term.
        On the support of $p_{t}(\cdot\mid x)$ we have $r<(1-t)\pi$. Let $u=r/(1-t)\in(0,\pi)$, so $r=(1-t)u$.
        Using $|\cot r|\le 1/\sin r$ and the concavity bound $\sin((1-t)u)\ge (1-t)\sin u$,
        \begin{align*}
            |r\cot r|
            \le \frac{r}{\sin r}
            = \frac{(1-t)u}{\sin((1-t)u)}
            \le \frac{(1-t)u}{(1-t)\sin u}
            = \frac{u}{\sin u}
            \le \frac{\pi}{\sin u}.
        \end{align*}
        Therefore for $r<(1-t)\pi$,
        \begin{align*}
            \|\nabla_{w}\Log_{x}(x_{1})\|
            \le 1 + \frac{\pi}{\sin u} < \frac{2\pi}{\sin u}.
        \end{align*}

        Using the same technique as before,
        \begin{align*}
            \int_{S^{d}}\|\nabla_{w}\Log_{x}(x_{1})\|\,p_{t}(x_{1}\mid x)\,dV_{g}(x_{1})
            \le
            4\pi\frac{M_{1}}{m_{1}}.
        \end{align*}
        We conclude that 
        Multiplying by $(1-t)^{-1}$ yields
        \begin{align*}
            G(t,x)
            \le 4\pi\frac{M_{1}}{m_{1}(1-t)}.
        \end{align*}
    
\end{proof}

\begin{lemma}\label{Lemma_Partial_t_v_Sphere}
    Fix $d\ge 3$ and $t\in(0,1)$. Let $p_{1}$ be a smooth density on $S^{d}$
    satisfying
    \begin{align*}
        0<m_{1}\le p_{1}(z)\le M_{1}<\infty.
    \end{align*}
    For
    \begin{align*}
        v(t,x) = \frac{1}{1-t}\int_{S^{d}} \Log_{x}(x_{1})\, p_{t}(x_{1}\mid x)\, dV_{g}(x_{1}),
    \end{align*}
    we have for every $x\in S^{d}$,
    \begin{align*}
        \Bigl\|\frac{d}{dt}v(t,x)\Bigr\|
        \le
        \frac{3\pi}{1-t}
        + \frac{4\pi^{2}d}{1-t}\frac{M_{1}}{m_{1}} 
        \le \frac{8\pi^{2}d}{1-t}\frac{M_{1}}{m_{1}} .
    \end{align*}
\end{lemma}

\begin{proof} Observe that on the support of $p_{t}(\cdot\mid x)$ we have $d(x,x_{1})<(1-t)\pi$, hence
    \begin{align*}
        \|\Log_{x}(x_{1})\| = d(x,x_{1}) \le (1-t)\pi.
    \end{align*}
    Therefore
    \begin{align*}
        \|v(t,x)\|
        \le \frac{1}{1-t}\int_{S^{d}}\|\Log_{x}(x_{1})\|\,p_{t}(x_{1}\mid x)\, dV_{g}(x_{1})
        \le \frac{1}{1-t}(1-t)\pi
        = \pi.
    \end{align*}

    We differentiate $v(t,x)$ in $t$.
    Since $\Log_{x}(x_{1})$ does not depend on $t$,
    \begin{align*}
        \frac{d}{dt}v(t,x)
        &= \frac{d}{dt}\Bigl(\frac{1}{1-t}\Bigr)
        \int_{S^{d}} \Log_{x}(x_{1})\, p_{t}(x_{1}\mid x)\, dV_{g}(x_{1})
        + \frac{1}{1-t}\int_{S^{d}} \Log_{x}(x_{1})\, \frac{\partial}{\partial t}p_{t}(x_{1}\mid x)\, dV_{g}(x_{1}) \\
        &= \frac{1}{(1-t)^{2}}\int_{S^{d}} \Log_{x}(x_{1})\, p_{t}(x_{1}\mid x)\, dV_{g}(x_{1})
        + \frac{1}{1-t}\int_{S^{d}} \Log_{x}(x_{1})\, \frac{\partial}{\partial t}p_{t}(x_{1}\mid x)\, dV_{g}(x_{1}).
    \end{align*}
    Using
    \begin{align*}
        \int_{S^{d}} \Log_{x}(x_{1})\, p_{t}(x_{1}\mid x)\, dV_{g}(x_{1})
        = (1-t)\,v(t,x),
    \end{align*}
    we have
    \begin{align*}
        \frac{d}{dt}v(t,x)
        = \frac{1}{1-t}v(t,x)
        + \frac{1}{1-t}\int_{S^{d}} \Log_{x}(x_{1})\, \frac{\partial}{\partial t}p_{t}(x_{1}\mid x)\, dV_{g}(x_{1}).
    \end{align*}

    Since $\int_{S^{d}} p_{t}(x_{1}\mid x)\,dV_{g}(x_{1})=1$, we have
    \begin{align*}
        \int_{S^{d}} \frac{\partial}{\partial t}p_{t}(x_{1}\mid x)\, dV_{g}(x_{1}) = 0.
    \end{align*}
    Also, wherever $p_{t}(x_{1}\mid x)>0$,
    \begin{align*}
        \frac{\partial}{\partial t}p_{t}(x_{1}\mid x)
        = p_{t}(x_{1}\mid x)\,\frac{\partial}{\partial t}\log p_{t}(x_{1}\mid x).
    \end{align*}
    Hence
    \begin{align*}
        \Bigl\|\int_{S^{d}} \Log_{x}(x_{1})\, \frac{\partial}{\partial t}p_{t}(x_{1}\mid x)\, dV_{g}(x_{1})\Bigr\|
        &\le \int_{S^{d}} \|\Log_{x}(x_{1})\|\, p_{t}(x_{1}\mid x)\,
            \Bigl|\frac{\partial}{\partial t}\log p_{t}(x_{1}\mid x)\Bigr|\, dV_{g}(x_{1}) \\
        &\le (1-t)\pi\,
            \mathbb{E}\Bigl[\Bigl|\frac{\partial}{\partial t}\log p_{t}(X_{1}\mid x)\Bigr|\mid X_{t}=x\Bigr].
    \end{align*}
    Therefore
    \begin{align*}
        \Bigl\|\frac{d}{dt}v(t,x)\Bigr\|
        \le \frac{\pi}{1-t}
        + \pi\,
            \mathbb{E}\Bigl[\Bigl|\frac{\partial}{\partial t}\log p_{t}(X_{1}\mid x)\Bigr|\mid X_{t}=x\Bigr].
    \end{align*}
    Using Lemma \ref{Lemma_Bound_partial_t_log_p},
    \begin{align*}
        \Bigl\|\frac{d}{dt}v(t,x)\Bigr\|
        &\le \frac{\pi}{1-t}
        + \pi\,
            \mathbb{E}\Bigl[\Bigl|\frac{\partial}{\partial t}\log p_{t}(X_{1}\mid x)\Bigr|\mid X_{t}=x\Bigr] \\
        &\le \frac{\pi}{1-t}
        + 2\pi (\frac{1}{1-t} + \frac{d-1}{1-t}\frac{2\pi M_{1}}{m_{1}} )
        \le \frac{1}{1-t}(3\pi + 4\pi^{2}\frac{dM_{1}}{m_{1}}) .
    \end{align*}

\end{proof}

\subsection{Regularity for Divergence} 

In this section, we show that both the gradient (Lemma \ref{Lemma_Grad_Div_v_Sphere}) and the time derivative (Lemma \ref{Lemma_Partial_t_Div_v_Sphere}) of 
$\Div v(t,x)$ are uniformly bounded.

\begin{lemma}\label{Lemma_Grad_Div_v_Sphere}
    Assume $d\ge 3$ and $t\in(0,1)$. Let $p_{1}$ be a smooth density on $S^{d}$ such that
    \begin{align*}
        0<m_{1}\le p_{1}(z)\le M_{1}<\infty.
    \end{align*}
Then for every $x\in S^{d}$ we have 
    \begin{align*}
        \|\grad_{x}\Div v(t,x)\|
        \le \frac{128 \pi (d-1)^{2}}{(1-t)^{3}} \frac{M_{1}}{m_{1}}.
    \end{align*}
\end{lemma}

\begin{proof} Write
\begin{align*}
    v(t,x) = \frac{1}{1-t}\int_{S^{d}} \Log_{x}(x_{1})\,p_{t}(x_{1}\mid x)\,dV_{g}(x_{1}).
\end{align*}
Using $\Div(fW)=\langle\grad f,W\rangle + f\Div W$ with $f=p_{t}(\cdot\mid x)$ and $W=\Log_{x}(\cdot)$,
\begin{align*}
    \Div v(t,x)
    &=
    \frac{1}{1-t}\int_{S^{d}}
        \Div_{x}\Log_{x}(x_{1})\,p_{t}(x_{1}\mid x)\,dV_{g}(x_{1}) \\
    & \qquad + \frac{1}{1-t}\int_{S^{d}}
        \langle \grad_{x}p_{t}(x_{1}\mid x), \Log_{x}(x_{1})\rangle\,dV_{g}(x_{1}) \\
    &=
    \frac{1}{1-t}\mathbb{E}[\Div_{x}\Log_{x}(X_{1})\mid X_{t}=x] \\
    & \qquad + \frac{1}{1-t}\mathbb{E}[\langle \grad_{x}\log p_{t}(X_{1}\mid x), \Log_{x}(X_{1})\rangle\mid X_{t}=x].
\end{align*}
Differentiate this identity in $x$ along a unit vector $\xi\in T_{x}S^{d}$ and take norms.
We take gradient, and obtain (for a smooth integrand $F(x,x_{1})$),
\begin{align*}
    \grad_{x}\mathbb{E}[F(x,X_{1})\mid X_{t}=x]
    =
    \mathbb{E}[\grad_{x}F(x,X_{1})\mid X_{t}=x]
    + \mathbb{E}[F(x,X_{1})\grad_{x}\log p_{t}(X_{1}\mid x)\mid X_{t}=x].
\end{align*}
Applying this formula and using triangle inequality yields
\begin{align*}
    \|\grad_{x}\Div v(t,x)\|
    \le \frac{1}{1-t}\Bigl(T_{1}+T_{2}+T_{3}+T_{4}\Bigr),
\end{align*}
where
\begin{align*}
    T_{1}
    &= \mathbb{E}[\|\grad_{x}\Div_{x}\Log_{x}(X_{1})\|\mid X_{t}=x], \\
    T_{2}
    &= \mathbb{E}[|\Div_{x}\Log_{x}(X_{1})|\,\|\grad_{x}\log p_{t}(X_{1}\mid x)\|\mid X_{t}=x], \\
    T_{3}
    &= \mathbb{E}[\|\grad_{x}\langle \grad_{x}\log p_{t}(X_{1}\mid x),\Log_{x}(X_{1})\rangle\|\mid X_{t}=x], \\
    T_{4}
    &= \mathbb{E}[|\langle \grad_{x}\log p_{t}(X_{1}\mid x),\Log_{x}(X_{1})\rangle|\,\|\grad_{x}\log p_{t}(X_{1}\mid x)\|\mid X_{t}=x].
\end{align*}

Now we compute $\Div_{x}\Log_{x}(x_{1})$ and its $x$-gradient.
Let $r=d(x,x_{1})\in(0,\pi)$. Notice that 
\begin{align*}
    \Log_{x}(x_{1}) = -\grad_{x}\frac{1}{2}r^{2},
\end{align*}
hence
\begin{align*}
    \Div_{x}\Log_{x}(x_{1})
    = - \Div_{x} (r \grad r)
    = - r \Delta r - \langle \grad r, \grad r \rangle
    = - (d-1)r\cot r - 1,
\end{align*}
where we used $\Delta r = (d-1)\cot r$ (see for example \cite[Theorem 11.11]{lee2018introduction}) 
and $\|\grad r\|^{2}=1$.

Differentiate in $x$ using that $r$ is a radial function and $\|\grad r\|=1$:
\begin{align*}
    \grad_{x}\Div_{x}\Log_{x}(x_{1})
    = -(d-1)\bigl(\cot r - r\csc^{2}r\bigr)\grad_{x}r,
\end{align*}
which implies
\begin{align*}
    \|\grad_{x}\Div_{x}\Log_{x}(x_{1})\|
    = (d-1)\bigl|\cot r - r\csc^{2}r\bigr|
    = (d-1)\bigl|\cot r - r\frac{1}{\sin^{2} r}\bigr|.
\end{align*}

We bound $T_{1} = \mathbb{E}[\|\grad_{x}\Div_{x}\Log_{x}(X_{1})\|\mid X_{t}=x]$. Using $r=(1-t)u$ and $\sin((1-t)u)\ge (1-t)\sin u$ (concavity of $\sin$ on $[0,\pi]$),
    \begin{align*}
        |\cot r|
        \le \frac{1}{\sin r}
        \le \frac{1}{(1-t)\sin u},
        \qquad
        r\frac{1}{\sin^{2} r}
        \le \frac{(1-t)\pi}{\sin^{2}r}
        \le \frac{\pi}{(1-t)\sin^{2}u}.
    \end{align*}
    Hence
    \begin{align*}
        \|\grad_{x}\Div_{x}\Log_{x}(x_{1})\|
        \le (d-1)\Bigl(\frac{1}{(1-t)\sin u} + \frac{\pi}{(1-t)\sin^{2}u}\Bigr).
    \end{align*}
    Therefore
    \begin{align*}
        T_{1}
        \le (d-1)\frac{1}{1-t}\frac{M_{1}}{m_{1}}(2\pi + 2).
    \end{align*}

We bound $T_{2} = \mathbb{E}[|\Div_{x}\Log_{x}(X_{1})|\,\|\grad_{x}\log p_{t}(X_{1}\mid x)\|\mid X_{t}=x]$. Recall $|\Div_{x}\Log_{x}(x_{1})|\le 1+(d-1)r|\cot r|$, so we have
    \begin{align*}
        |\Div_{x}\Log_{x}(x_{1})|
        \le 1 + (d-1)\pi\frac{1}{\sin r}
        \le 1 + (d-1)\pi\frac{1}{(1-t)\sin u}.
    \end{align*}
    Thus by Cauchy-Schwarz,
    \begin{align*}
        T_{2}
        &\le
        \Bigl(\mathbb{E}[|\Div_{x}\Log_{x}(X_{1})|^{2}\mid X_{t}=x]\Bigr)^{1/2}
        \Bigl(\mathbb{E}[\|\grad_{x}\log p_{t}(X_{1}\mid x)\|^{2}\mid X_{t}=x]\Bigr)^{1/2}.
    \end{align*}
    Using $(a+b)^{2}\le 2a^{2}+2b^{2}$,
    \begin{align*}
        \mathbb{E}[|\Div_{x}\Log_{x}(X_{1})|^{2}\mid X_{t}=x]
        &\le 2 + 2(d-1)^{2}\pi^{2}\frac{1}{(1-t)^{2}}
            \mathbb{E}\Bigl[\frac{1}{\sin^{2}u}\mid X_{t}=x\Bigr] \\
        &\le 2 + 2(d-1)^{2}\pi^{2}\frac{1}{(1-t)^{2}}
            \frac{2M_{1}}{m_{1}}.
    \end{align*}
    Also recall 
    \begin{align*}
        \mathbb{E}[\|\grad_{x}\log p_{t}(X_{1}\mid x)\|^{2}\mid X_{t}=x]
        \le \frac{32(d-1)^{2}}{(1-t)^{2}}
        \frac{M_{1}}{m_{1}}.
    \end{align*}
    Hence we obtain
    \begin{align*}
        T_{2}
        &\le (2 + 2(d-1)^{2}\pi^{2}\frac{1}{(1-t)^{2}}
            \frac{2M_{1}}{m_{1}})^{\frac{1}{2}}
            (\frac{32(d-1)^{2}}{(1-t)^{2}}
        \frac{M_{1}}{m_{1}})^{\frac{1}{2}} \\
        &\le \frac{16 \pi M_{1}}{m_{1}}\frac{(d-1)^{2}}{(1-t)^{2}}  .
    \end{align*}

We bound $T_{3} = \mathbb{E}[\|\grad_{x}\langle \grad_{x}\log p_{t}(X_{1}\mid x),\Log_{x}(X_{1})\rangle\|\mid X_{t}=x]$.
We use the product rule: denote $g(x) = \langle g_{1}(x), g_{2}(x)\rangle$. 
    For $g_{1}, g_{2}$, we have (viewing as directional derivative along $u$ and use compatiblity) 
    \begin{align*}
        \langle \grad_{x} \langle g_{1}(x), g_{2}(x) \rangle, u \rangle
        = \nabla_{u} \langle g_{1}(x), g_{2}(x) \rangle
        = \langle \nabla_{u} g_{1}(x), g_{2}(x) \rangle
        + \langle g_{1}(x), \nabla_{u} g_{2}(x) \rangle.
    \end{align*}
    Hence we can write 
    \begin{align*}
        &\|\grad_{x}\langle \grad_{x}\log p_{t}(x_{1}\mid x),\Log_{x}(x_{1})\rangle \| \\
        =& \|\nabla \grad_{x}\log p_{t}(x_{1}\mid x)\|_{\op} \|\Log_{x}(x_{1})\| 
        + \|\nabla \Log_{x}(x_{1})\|_{\op} \|\grad_{x}\log p_{t}(x_{1}\mid x)\|,
    \end{align*}
    so
    \begin{align*}
        T_{3}
        &\le
        \mathbb{E}[\|\nabla \Log_{x}(X_{1})\|\,\|\grad_{x}\log p_{t}(X_{1}\mid x)\|\mid X_{t}=x] \\
        &\qquad + \mathbb{E}[\|\nabla^{2}\log p_{t}(X_{1}\mid x)\|\,\|\Log_{x}(X_{1})\|\mid X_{t}=x].
    \end{align*}
    On the support of $p_{t}(\cdot\mid x)$, $\|\Log_{x}(X_{1})\|\le (1-t)\pi$, hence
    \begin{align*}
        \mathbb{E}[\|\nabla^{2}\log p_{t}(X_{1}\mid x)\|\,\|\Log_{x}(X_{1})\|\mid X_{t}=x]
        \le (1-t)\pi\,\mathbb{E}[\|\nabla^{2}\log p_{t}(X_{1}\mid x)\|\mid X_{t}=x].
    \end{align*}
    By Lemma \ref{Lemma_Bound_Hessian_Log_P},
    \begin{align*}
        \mathbb{E}[\|\nabla^{2}\log p_{t}(X_{1}\mid x)\|\,\|\Log_{x}(X_{1})\|\mid X_{t}=x]
        \le \frac{64\pi(d-1)^{2}}{1-t}\frac{M_{1}}{m_{1}}.
    \end{align*}

    Also, on $S^{d}$,
    \begin{align*}
        \|\nabla \Log_{x}(x_{1})\|
        \le 1 + r|\cot r|
        < \pi\frac{2}{(1-t)\sin u},
    \end{align*}
    so by Cauchy-Schwarz inequality,
    \begin{align*}
        \mathbb{E}[\|\nabla \Log_{x}(X_{1})\|\,\|\grad_{x}\log p_{t}(X_{1}\mid x)\|\mid X_{t}=x]
        &\le \Bigl( \frac{32(d-1)^{2}}{(1-t)^{2}} \frac{M_{1}}{m_{1}} 
        \frac{8\pi^{2}}{(1-t)^{2}}\frac{M_{1}}{m_{1}} \Bigr)^{\frac{1}{2}} \\
        &\le \frac{16\pi M_{1}}{m_{1}} \frac{(d-1)}{(1-t)^{2}}.
    \end{align*}
    Hence we get 
    \begin{align*}
        T_{3} \le
        \frac{16\pi M_{1}}{m_{1}} \frac{(d-1)}{(1-t)^{2}}
        + \frac{64\pi(d-1)^{2}}{1-t}\frac{M_{1}}{m_{1}}
        \le \frac{64\pi M_{1}}{m_{1}} \frac{(d-1)^{2}}{(1-t)^{2}}.
    \end{align*}

We bound $T_{4} = \mathbb{E}[|\langle \grad_{x}\log p_{t}(X_{1}\mid x),\Log_{x}(X_{1})\rangle|\,\|\grad_{x}\log p_{t}(X_{1}\mid x)\|\mid X_{t}=x]$. We have 
    \begin{align*}
        T_{4}
        &\le
        \mathbb{E}[\|\Log_{x}(X_{1})\|\,\|\grad_{x}\log p_{t}(X_{1}\mid x)\|^{2}\mid X_{t}=x] \\
        &\le (1-t)\pi\,\mathbb{E}[\|\grad_{x}\log p_{t}(X_{1}\mid x)\|^{2}\mid X_{t}=x] \\
        &\le (1-t)\pi\cdot \frac{32(d-1)^{2}}{(1-t)^{2}}
        \frac{M_{1}}{m_{1}}
        = \frac{32\pi(d-1)^{2}}{1-t}\frac{M_{1}}{m_{1}}.
    \end{align*}

Together,
\begin{align*}
    &\|\grad_{x}\Div v(t,x)\| \\
    \le & \frac{1}{1-t}\Bigl(T_{1}+T_{2}+T_{3}+T_{4}\Bigr) \\
    \le& \frac{1}{(1-t)} \frac{M_{1}}{m_{1}} 
    \Bigl((d-1)\frac{1}{1-t}(2\pi + 2) + 16 \pi\frac{(d-1)^{2}}{(1-t)^{2}}  
    + 64\pi\frac{(d-1)^{2}}{(1-t)^{2}} + \frac{32\pi(d-1)^{2}}{1-t}\Bigr) \\
    \le& \frac{128 \pi (d-1)^{2}}{(1-t)^{3}} \frac{M_{1}}{m_{1}}.
\end{align*}

\end{proof}

\begin{lemma}\label{Lemma_Partial_t_Div_v_Sphere}
    Assume $d\ge 3$ and $t\in(0,1)$. Let $p_{1}$ be a smooth density on $S^{d}$ such that
    \begin{align*}
        0<m_{1}\le p_{1}(z)\le M_{1}<\infty.
    \end{align*}
    Then for every $x\in S^{d}$,
    \begin{align*}
        \Bigl|\frac{d}{dt}\Div v(t,x)\Bigr|
        \le
        \frac{128\pi^{2}(d-1)^{2}}{(1-t)^{3}}\frac{M_{1}}{m_{1}}.
    \end{align*}
\end{lemma}

\begin{proof} We start by computing the time derivative of divergence. Recall 
    \begin{align*}
        \Div v(t,x) =
        \frac{1}{1-t}\mathbb{E}[\Div_{x}\Log_{x}(X_{1})\mid X_{t}=x]
        + \frac{1}{1-t}\mathbb{E}[\langle \grad_{x}\log p_{t}(X_{1}\mid x), \Log_{x}(X_{1})\rangle\mid X_{t}=x].
    \end{align*}

    Let $F(t,x_{1})$ be smooth in $t$ and integrable under $p_{t}(\cdot\mid x)$. Then
    \begin{align*}
        \frac{d}{dt}\mathbb{E}[F(t,X_{1})\mid X_{t}=x]
        =
        \mathbb{E}[\partial_{t}F(t,X_{1})\mid X_{t}=x]
        +
        \mathbb{E}[F(t,X_{1})\,\partial_{t}\log p_{t}(X_{1}\mid x)\mid X_{t}=x].
    \end{align*}

    Applying this with $F(t,x_{1})=\Div_{x}\Log_{x}(x_{1})$ and $F(t,x_{1})=\langle \grad_{x}\log p_{t}(x_{1}\mid x), \Log_{x}(x_{1})\rangle$ respectively,
    \begin{align*}
        &\frac{d}{dt}\mathbb{E}[\Div_{x}\Log_{x}(X_{1})\mid X_{t}=x] 
        = 
        \mathbb{E}[\Div_{x}\Log_{x}(X_{1})\,\partial_{t}\log p_{t}(X_{1}\mid x)\mid X_{t}=x], \\
        &\frac{d}{dt}\mathbb{E}[\langle \grad_{x}\log p_{t}(X_{1}\mid x), \Log_{x}(X_{1})\rangle\mid X_{t}=x] 
        = 
        \mathbb{E}[\partial_{t}\langle \grad_{x}\log p_{t}(X_{1}\mid x), \Log_{x}(X_{1})\rangle\mid X_{t}=x] \\
        & \qquad + \mathbb{E}[ \langle \grad_{x}\log p_{t}(X_{1}\mid x), \Log_{x}(X_{1})\rangle\,\partial_{t}\log p_{t}(X_{1}\mid x)\mid X_{t}=x].
    \end{align*}

    Therefore, for fixed $x$,
    \begin{align*}
        \frac{d}{dt}\Div v(t,x)
        &=
        \frac{1}{(1-t)^{2}}
        \mathbb{E}\bigl[\Div_{x}\Log_{x}(X_{1}) + \langle \grad_{x}\log p_{t}(X_{1}\mid x), \Log_{x}(X_{1})\rangle\mid X_{t}=x\bigr] \\
        & \quad + \frac{1}{1-t}\mathbb{E}[\Div_{x}\Log_{x}(X_{1})\,\partial_{t}\log p_{t}(X_{1}\mid x)\mid X_{t}=x] \\
        &\quad+
        \frac{1}{1-t}\mathbb{E}[\partial_{t}\langle \grad_{x}\log p_{t}(X_{1}\mid x), \Log_{x}(X_{1})\rangle\mid X_{t}=x] \\
        &\quad+
        \frac{1}{1-t}\mathbb{E}[\langle \grad_{x}\log p_{t}(X_{1}\mid x), \Log_{x}(X_{1})\rangle\,\partial_{t}\log p_{t}(X_{1}\mid x)\mid X_{t}=x].
    \end{align*}

    Define 
    \begin{align*}
        T_{1} &= \frac{1}{(1-t)^{2}}\mathbb{E}\bigl[| \Div_{x}\Log_{x}(X_{1}) + \langle \grad_{x}\log p_{t}(X_{1}\mid x), \Log_{x}(X_{1})\rangle | \mid X_{t}=x\bigr], \\
        T_{2} &= \frac{1}{1-t}\mathbb{E}[|\Div_{x}\Log_{x}(X_{1})\,\partial_{t}\log p_{t}(X_{1}\mid x) |\mid X_{t}=x], \\
        T_{3} &= \frac{1}{1-t}\mathbb{E}[|\partial_{t}\langle \grad_{x}\log p_{t}(X_{1}\mid x), \Log_{x}(X_{1})\rangle | \mid X_{t}=x], \\
        T_{4} &= \frac{1}{1-t}\mathbb{E}[| \langle \grad_{x}\log p_{t}(X_{1}\mid x), \Log_{x}(X_{1})\rangle\,\partial_{t}\log p_{t}(X_{1}\mid x) | \mid X_{t}=x].
    \end{align*}

    For $T_{1}$, recall 
        \begin{align*}
            \frac{1}{(1-t)^{2}}\mathbb{E}\bigl[\Div_{x}\Log_{x}(X_{1}) \mid X_{t}=x\bigr]
            &\le \frac{1}{(1-t)^{2}}\mathbb{E}\bigl[1 + (d-1) \pi \frac{1}{(1-t)\sin u}  \mid X_{t}=x\bigr] \\
            &\le \pi \frac{(d-1)}{(1-t)^{3}} \frac{2M_{1}}{m_{1}},
        \end{align*} 
        and 
        \begin{align*}
            &\frac{1}{(1-t)^{2}}\mathbb{E}\bigl[\langle \grad_{x}\log p_{t}(X_{1}\mid x), \Log_{x}(X_{1})\rangle\mid X_{t}=x\bigr]  \\
            \le & \frac{1}{(1-t)^{2}}(1-t)\pi \mathbb{E}\bigl[\|\grad_{x}\log p_{t}(X_{1}\mid x)\| \mid X_{t}=x\bigr]  
            \le \frac{8\pi(d-1)}{(1-t)^{2}} \frac{M_{1}}{m_{1}}.
        \end{align*}
        Hence 
        \begin{align*}
            T_{1} \le \pi \frac{(d-1)}{(1-t)^{3}} \frac{2M_{1}}{m_{1}} + \frac{8\pi(d-1)}{(1-t)^{2}} \frac{M_{1}}{m_{1}} 
            \le \frac{16 \pi (d-1)}{(1-t)^{3}} \frac{M_{1}}{m_{1}} .
        \end{align*} 

    For $T_{2}$, recall 
        \begin{align*}
            \mathbb{E}[|\Div_{x}\Log_{x}(X_{1})|^{2}\mid X_{t}=x]
            \le 2 + 2(d-1)^{2}\pi^{2}\frac{1}{(1-t)^{2}}
                \frac{2M_{1}}{m_{1}} 
            \le \frac{(d-1)^{2}}{(1-t)^{2}} \frac{4\pi^{2}M_{1}}{m_{1}}.
        \end{align*}
        Hence we have 
        \begin{align*}
            T_{2} = &\frac{1}{1-t}\mathbb{E}[|\Div_{x}\Log_{x}(X_{1})\,\partial_{t}\log p_{t}(X_{1}\mid x) |\mid X_{t}=x] \\
            \le & \frac{1}{1-t} \Bigl(\mathbb{E}[|\Div_{x}\Log_{x}(X_{1})|^{2}\mid X_{t}=x]\mathbb{E}[|\partial_{t}\log p_{t}(X_{1}\mid x) |^{2}\mid X_{t}=x]\Bigr)^{\frac{1}{2}} \\
            \le & \frac{1}{1-t} \Bigl(\frac{(d-1)^{2}}{(1-t)^{2}} \frac{4\pi^{2}M_{1}}{m_{1}} 
            \frac{32\pi^{2}(d-1)^{2}}{(1-t)^{2}}\,
                \frac{M_{1}}{m_{1}}\Bigr)^{\frac{1}{2}} 
            \le \frac{16 \pi^{2}(d-1)^{2}}{(1-t)^{3}} \frac{M_{1}}{m_{1}},
        \end{align*}
        where by Lemma \ref{Lemma_Bound_partial_t_log_p},
        \begin{align*}
            \mathbb{E}\Bigl[\Bigl|\frac{\partial}{\partial t}\log p_{t}(X_{1}\mid x)\Bigr|^{2}\mid X_{t}=x\Bigr]
            \le \frac{8}{(1-t)^{2}} + \frac{16\pi^{2}(d-1)^{2}}{(1-t)^{2}}\,
                \frac{M_{1}}{m_{1}} 
            \le \frac{32\pi^{2}(d-1)^{2}}{(1-t)^{2}}\,
                \frac{M_{1}}{m_{1}}.
        \end{align*}

    For $T_{3}$, since $\Log_{x}(x_{1})$ does not depend on $t$,
        \begin{align*}
            &|\partial_{t}\langle \grad_{x}\log p_{t}(x_{1}\mid x), \Log_{x}(x_{1})\rangle| \\
            \le & \|\partial_{t}\grad_{x}\log p_{t}(x_{1}\mid x)\|\,\|\Log_{x}(x_{1})\|
            \le (1-t)\pi\,\|\partial_{t}\grad_{x}\log p_{t}(x_{1}\mid x)\|.
        \end{align*}
        Using Lemma \ref{Lemma_Bounds_Sphere}, we have 
        \begin{align*}
            \mathbb{E}[\|\partial_{t}\grad_{x}\log p_{t}(X_{1}\mid x)\|\mid X_{t}=x]
            \le
            \frac{16\pi^{2}(d-1)^{2}}{(1-t)^{2}}\frac{M_{1}}{m_{1}}.
        \end{align*}
        Hence 
        \begin{align*}
            \mathbb{E}[|\partial_{t}\langle \grad_{x}\log p_{t}(X_{1}\mid x), \Log_{x}(X_{1})\rangle|\mid X_{t}=x]
            &\le
            (1-t)\pi\cdot \frac{16\pi^{2}(d-1)^{2}}{(1-t)^{2}}\frac{M_{1}}{m_{1}} \\
            &=
            \frac{16\pi^{3}(d-1)^{2}}{(1-t)}\frac{M_{1}}{m_{1}}.
        \end{align*}

    Last, we bound $T_{4}$. Similarly, by Cauchy-Schwarz and the bound we did for $T_{3}$,
        \begin{align*}
            &\Bigl|\mathbb{E}[\langle \grad_{x}\log p_{t}(X_{1}\mid x), \Log_{x}(X_{1})\rangle\,\partial_{t}\log p_{t}(X_{1}\mid x)\mid X_{t}=x]\Bigr| \\
            &\le
            \Bigl(\mathbb{E}[|\langle \grad_{x}\log p_{t}(X_{1}\mid x), \Log_{x}(X_{1})\rangle|^{2}\mid X_{t}=x]\Bigr)^{1/2}
            \Bigl(\mathbb{E}[|\partial_{t}\log p_{t}(X_{1}\mid x)|^{2}\mid X_{t}=x]\Bigr)^{1/2}.
        \end{align*}
        We bound the first factor by
        \begin{align*}
            &\mathbb{E}[|\langle \grad_{x}\log p_{t}(X_{1}\mid x), \Log_{x}(X_{1})\rangle|^{2}\mid X_{t}=x] \\
            \le&
            (1-t)^{2}\pi^{2}\mathbb{E}[\|\grad_{x}\log p_{t}(X_{1}\mid x)\|^{2}\mid X_{t}=x]
            \le
            32\pi^{2}(d-1)^{2}\frac{M_{1}}{m_{1}}.
        \end{align*}

        Therefore
        \begin{align*}
            T_{4} = &\frac{1}{1-t}\Bigl|\mathbb{E}[\langle \grad_{x}\log p_{t}(X_{1}\mid x), \Log_{x}(X_{1})\rangle\,\partial_{t}\log p_{t}(X_{1}\mid x)\mid X_{t}=x]\Bigr| \\
            \le& 32\pi^{2}\frac{(d-1)^{2}}{(1-t)^{2}}\frac{M_{1}}{m_{1}}.
        \end{align*}

    Together, 
    \begin{align*}
        |\frac{d}{dt}\Div v(t,x)|
        \le & T_{1} + T_{2} + T_{3} + T_{4} \\
        \le & \frac{16 \pi (d-1)}{(1-t)^{3}} \frac{M_{1}}{m_{1}} + \frac{16 \pi^{2}(d-1)^{2}}{(1-t)^{3}} \frac{M_{1}}{m_{1}} 
        + \frac{16\pi^{3}(d-1)^{2}}{(1-t)^{2}}\frac{M_{1}}{m_{1}}
        + 32\pi^{2}\frac{(d-1)^{2}}{(1-t)^{2}}\frac{M_{1}}{m_{1}} \\
        \le & \frac{128 \pi^{2} (d-1)^{2}}{(1-t)^{3}} \frac{M_{1}}{m_{1}}.
    \end{align*}
\end{proof}

\subsection[Regularity of v(t, x) and log pt]{Regularity of $v(t, x)$ and $\log p_{t}$}

Finally, we bound $\|v(t,x)\|$ and $\mathbb{E}[\|\grad \log p_{t}(X_{t})\|^{2}]$ in Lemma \ref{Lemma_Bound_V_Norm} and \ref{Lemma_Bound_Score}.

\begin{lemma}\label{Lemma_Bound_V_Norm}
Assume $d\ge 2$ and $t\in(0,1)$. Let $p_{1}$ be a smooth density on $S^{d}$ such that
\begin{align*}
    0<m_{1}\le p_{1}(z)\le M_{1}<\infty.
\end{align*}
For every $x\in S^{d}$,
\begin{align*}
    \|v(t,x)\|\le \pi.
\end{align*}
\end{lemma}

\begin{proof}
Fix $x\in S^{d}$. By Jensen and the definition of $v$,
\begin{align*}
    \|v(t,x)\|
    &=
    \frac{1}{1-t}\Bigl\|\int_{S^{d}}\Log_{x}(x_{1})\,p_{t}(x_{1}\mid x)\,dV_{g}(x_{1})\Bigr\|
    \le
    \frac{1}{1-t}\int_{S^{d}}\|\Log_{x}(x_{1})\|\,p_{t}(x_{1}\mid x)\,dV_{g}(x_{1}).
\end{align*}
On the support of $p_{t}(\cdot\mid x)$ we have $d(x,x_{1})<(1-t)\pi$, hence
\begin{align*}
    \|\Log_{x}(x_{1})\|=d(x,x_{1})\le (1-t)\pi.
\end{align*}
Therefore
\begin{align*}
    \|v(t,x)\|
    &\le
    \frac{1}{1-t}\int_{S^{d}} (1-t)\pi\,p_{t}(x_{1}\mid x)\,dV_{g}(x_{1})
    =
    \pi\int_{S^{d}} p_{t}(x_{1}\mid x)\,dV_{g}(x_{1})
    =
    \pi,
\end{align*}
since $p_{t}(\cdot\mid x)$ is a probability density.
\end{proof}

\begin{lemma}\label{Lemma_Bound_Score}
    Assume $d\ge 3$ and $t\in(0,1)$. Let $p_{0}\equiv \Vol(S^{d})^{-1}$ be the uniform prior on $S^{d}$,
    and let $p_{1}$ be a smooth density on $S^{d}$ such that
    \begin{align*}
        0<m_{1}\le p_{1}(z)\le M_{1}<\infty.
    \end{align*}
    Let $X_{0}\sim p_{0}$ and $X_{1}\sim p_{1}$ be independent, and define the geodesic interpolation
    \begin{align*}
        X_{t} = \Exp_{X_{0}}(t\Log_{X_{0}}(X_{1})).
    \end{align*}
    Let $p_{t}$ denote the marginal density of $X_{t}$ with respect to $dV_{g}$.
    Then 
    \begin{align*}
        \mathbb{E}\bigl[\|\grad \log p_{t}(X_{t})\|^{2}\bigr]
        \le
        \frac{8(d-1)^{2}}{(1-t)^{2}}\frac{M_{1}}{m_{1}}.
    \end{align*}
\end{lemma}

\begin{proof} Using Lemma \ref{Lemma_interpolated_Density}, notice that 
    $\log p_{t}(x)=\log Z_{t}(x)-\log \Vol(S^{d})$, we have
    \begin{align*}
        \grad \log p_{t}(x)=\grad \log Z_{t}(x).
    \end{align*}
    Take gradient, we obtain 
    \begin{align*}
        \grad Z_{t}(x)
        =
        \int_{S^{d}} p_{1}(x_{1})\,\grad_{x}J_{t}(x\mid x_{1})\,
        \,dV_{g}(x_{1}).
    \end{align*}
    Using $\grad_{x}J_{t}=J_{t}\grad_{x}\log J_{t}$ and the definition of the conditional density
    \begin{align*}
        p_{t}(x_{1}\mid x)=\frac{p_{1}(x_{1})J_{t}(x\mid x_{1})}{Z_{t}(x)},
    \end{align*}
    we obtain
    \begin{align*}
        \grad \log Z_{t}(x)
        &=
        \frac{\grad Z_{t}(x)}{Z_{t}(x)}
        =
        \int_{S^{d}} \grad_{x}\log J_{t}(x\mid x_{1})\,p_{t}(x_{1}\mid x)\,dV_{g}(x_{1}) \\
        &=
        \mathbb{E}\bigl[\grad_{x}\log J_{t}(x\mid X_{1})\mid X_{t}=x\bigr].
    \end{align*}
    Now we have 
    \begin{align*}
        \grad \log p_{t}(X_{t})
        =
        \mathbb{E}[\grad_{x}\log J_{t}(X_{t}\mid X_{1})\mid X_{t}],
    \end{align*}
    which implies
    \begin{align*}
        \|\grad \log p_{t}(X_{t})\|^{2}
        \le
        \mathbb{E}[\|\grad_{x}\log J_{t}(X_{t}\mid X_{1})\|^{2}\mid X_{t}].
    \end{align*}
    Taking expectation again gives
    \begin{align*}
        \mathbb{E}[\|\grad \log p_{t}(X_{t})\|^{2}]
        \le
        \mathbb{E}\Bigl[\mathbb{E}[\|\grad_{x}\log J_{t}(X_{t}\mid X_{1})\|^{2}\mid X_{t}]\Bigr] 
        \le
        \frac{8(d-1)^{2}}{(1-t)^{2}}\frac{M_{1}}{m_{1}}.
    \end{align*}
    where the last inequality follows from Lemma \ref{Lemma_Moment_Bound_Grad_Log_P}. 

\end{proof}

\subsection{Finiteness of Score Regularity: Proof of Proposition \ref{Prop_Finite_Score_Compact}}


\begin{proof}[Proof of Proposition \ref{Prop_Finite_Score_Compact}]
    Notice that
\[
\mathbb{E}\bigl[\|\grad \log p_t(X_t)\|^2\bigr]
=
\int_M \|\grad \log p_t(x)\|^2\, p_t(x)\, dV_g(x)
=: I(p_t),
\]
i.e., it is exactly the (Riemannian) Fisher information of $p_t$.
Consider the continuity equation $\partial_t p_t + \Div(p_t\, v(t,\cdot)) = 0$. Let $s_t := \grad \log p_t$ denote the score. 

We first compute time derivative of $I$.
Differentiate and use the ordinary product rule:
\begin{align*}
\frac{d}{dt}I(p_t)
&=
\int_M \partial_t\big(\langle s_t,s_t\rangle\big)\,p_t\,dV_g
+\int_M \langle s_t,s_t\rangle\,\partial_t p_t\,dV_g \\
&=
2\int_M \big\langle \partial_t s_t,\,s_t\big\rangle\,p_t\,dV_g
+\int_M \|s_t\|^2\,\partial_t p_t\,dV_g .
\end{align*}
Observe that 
\begin{align*}
\partial_t \log p_t
&=
-\frac{1}{p_t}\Div(p_t v)
= -\Div v - \frac{1}{p_t}\langle \grad p_t, v\rangle \\
&=
-\Div v - \langle \grad\log p_t, v\rangle
= -\Div v - \langle s_t, v\rangle.
\end{align*}
Taking the gradient yields
\begin{align*}
\partial_t s_t
=
\grad(\partial_t\log p_t)
=
-\grad(\Div v)\;-\;\grad\big(\langle s_t, v\rangle\big).
\end{align*}

We compute the terms in the time derivative of $I(p_{t})$ separately. For the first term,
\begin{align*}
    2\int_M \big\langle \partial_t s_t,\,s_t\big\rangle\,p_t\,dV_g
    &= 2\int_M \big\langle -\grad(\Div v)\;-\;\grad\big(\langle s_t, v\rangle\big),\,s_t\big\rangle\,p_t\,dV_g \\
    &= - 2\int_M \big\langle \grad(\Div v),\,s_t\big\rangle\,p_t\,dV_g - 2\int_M \big\langle \;\grad\big(\langle s_t, v\rangle\big),\,s_t\big\rangle\,p_t\,dV_g \\
    &= - 2\int_M \big\langle \grad(\Div v),\,s_t\big\rangle\,p_t\,dV_g - 2\int_M 
    \langle \nabla_{s_{t}} s_{t}, v \rangle p_{t}
    + \langle \nabla_{s_{t}} v, s_{t} \rangle p_t\,dV_g.
\end{align*}
For the second term, 
\begin{align*}
\int_M \|s_t\|^2\,\partial_t p_t\,dV_g
&=
-\int_M \|s_t\|^2\,\Div(p_t v)\,dV_g =
\int_M \big\langle \grad\|s_t\|^2,\, v\big\rangle\,p_t\,dV_g \\
&=2\int_M \big\langle \nabla_{v} s_{t}, s_{t} \big\rangle\,p_t\,dV_g .
\end{align*}

Since $\Hess$ is symmetric, we have 
\begin{align*}
    \big\langle \nabla_{v} s_{t}, s_{t} \big\rangle 
    = \langle \nabla_{v} \grad \log p_{t}, s_{t} \rangle 
    = \bigl(\Hess \log p_{t}\bigr) (v, s_{t}) = \langle \nabla_{s_{t}} \grad \log p_{t}, v \rangle 
    = \big\langle \nabla_{s_{t}} s_{t}, v \big\rangle.
\end{align*}
Together,
\[
\frac{d}{dt}I(p_t)
=
-2\int_M \langle s_t, \nabla_{s_t} v\rangle\,p_t\,dV_g
-2\int_M \langle \grad(\Div v), s_t\rangle\,p_t\,dV_g.
\]

Now we derive an ODE that helps to bound $I$.
Using Cauchy--Schwarz and the pointwise bounds
$\|\nabla v(t,x)\|_{\op}\le L_t^{v,x}$ and $\|\grad \Div v(t,x)\|\le L_t^{\Div,x}$
(from Assumption~\ref{A_Regularity_V}), we obtain
\begin{align}
\frac{d}{dt} I(p_t)
&\le
2 L_t^{v,x}\, I(p_t)
\;+\;
2 L_t^{\Div,x}\, \int_M \|s_t(x)\|\, p_t(x)\, dV_g(x)
\nonumber\\
&\le
2 L_t^{v,x}\, I(p_t)
\;+\;
2 L_t^{\Div,x}\, \sqrt{I(p_t)}.
\label{eq:Iineq}
\end{align}
Let $y(t):=\sqrt{I(p_t)}$. Since $I'=2yy'$, \eqref{eq:Iineq} implies
\[
y'(t)\le L_t^{v,x}\, y(t) + L_t^{\Div,x}.
\]

It remains to solve the ODE and obtain a finite upper bound.
Define
\[
A(t):=\int_0^t L_s^{v,x}\,ds,
\qquad
\mu(t):=e^{-A(t)}.
\]
Then $\mu$ is absolutely continuous and
\[
\mu'(t)= -L_t^{v,x}\,e^{-A(t)}=-L_t^{v,x}\,\mu(t)
\quad\text{for a.e. }t.
\]
Multiply by $\mu(t)$:
\[
\mu(t)\,y'(t)\le \mu(t)\,L_t^{v,x}\,y(t)+\mu(t)\,L_t^{\Div,x}.
\]
Using the product rule and the identity for $\mu'(t)$,
\begin{align*}
\frac{d}{dt}\big(\mu(t)\,y(t)\big)
&=\mu'(t)\,y(t)+\mu(t)\,y'(t) \\
&=-L_t^{v,x}\,\mu(t)\,y(t)+\mu(t)\,y'(t) \\
&\le -L_t^{v,x}\,\mu(t)\,y(t)+\mu(t)\,L_t^{v,x}\,y(t)+\mu(t)\,L_t^{\Div,x}\\
&=\mu(t)\,L_t^{\Div,x}.
\end{align*}
Hence, for a.e. $t$,
\begin{align*}
    \frac{d}{dt}\big(\mu(t)\,y(t)\big)\le \mu(t)\,L_t^{\Div,x}.
\end{align*}
Integrating over $[0,t]$ yields
\[
\mu(t)\,y(t)-\mu(0)\,y(0)
\le
\int_0^t \mu(s)\,L_s^{\Div,x}\,ds.
\]
Since $\mu(0)=e^{-A(0)}=1$, we obtain
\[
\mu(t)\,y(t)
\le
y(0)+\int_0^t e^{-A(s)}\,L_s^{\Div,x}\,ds.
\]
Multiply both sides by $e^{A(t)}$:
\begin{align*}
y(t)
&\le
e^{A(t)}
\left(
y(0)+\int_0^t e^{-A(s)}\,L_s^{\Div,x}\,ds
\right)\\
&=
\exp\!\Bigl(\int_0^t L_s^{v,x}\,ds\Bigr)
\left(
y(0)+\int_0^t L_s^{\Div,x}\,
\exp\!\Bigl(-\int_0^s L_r^{v,x}\,dr\Bigr)\,ds
\right).
\end{align*}

Therefore, for all $t<1$,
\begin{align*}
    \sqrt{I(p_t)}
\le
\exp\!\Bigl(\int_0^t L_s^{v,x}\, ds\Bigr)
\left(
\sqrt{I(p_0)}
+
\int_0^t L_s^{\Div,x}\,
\exp\!\Bigl(-\int_0^s L_r^{v,x}\, dr\Bigr)\, ds
\right),
\end{align*}
and hence
\[
\mathbb{E}\bigl[\|\grad \log p_t(X_t)\|^2\bigr]=I(p_t)\le L_t^{\mathrm{score}}
\]
for some finite number $L_t^{\mathrm{score}}$.
\end{proof}

%% file: Appendix_SPD.tex
\section{SPD Manifold Regularity Results}\label{Sect_SPD}

In this section, we work on the SPD manifold with affine invariant metric, denoted as $\SPD(n)$. 

We briefly discuss the similarity and difference between the regularity analysis in this section and that of the hypersphere. 
\begin{itemize}
    \item On a hypersphere, there exist cut points, resulting indicator function in the conditional density. We provide the formula for the conditional density function, on a general Hadamard manifold, see Lemma \ref{Lemma_Density_Formula_Hadamard}. We see that with non-positive curvature, there would be no cut points. Consequently, there would be no indicator function in the expression of the conditional density, guaranteeing better smoothness.  
    \item The hypersphere is a compact manifold, but $\SPD(n)$ is non-compact. On a compact manifold, the vector field $v(t, x)$ itself is uniformly bounded. But on a non-compact manifold, $\|v(t, x)\|$ could possibly blow up. Therefore, on $\SPD(n)$, it's unlikely that regularity will hold pointwise, but we can still expect its norm to be bounded, in expectation. 
    \item Moreover, from a high-level idea, the procedure for providing upper bounds on derivatives of $v(t, x)$ remains the same. For example, in Lemma \ref{Lemma_Grad_Div_v_Sphere}, we expanded the gradient of $\Div v(t, x)$ as terms involving $\mathbb{E}[\|\grad_{x}\Div_{x}\Log_{x}(X_{1})\|\mid X_{t}=x]$, $\mathbb{E}[\|\grad_{x}\log p_{t}(X_{1}\mid x)\|^{2}\mid X_{t}=x]$, just to name a few. To establish regularity on $\SPD(n)$, we still have roughly the same expansion, involving the same collection of terms. The difference is that, instead of obtain uniform upper bounds (as in Lemma \ref{Lemma_Moment_Bound_Grad_Log_P}, \ref{Lemma_Bound_partial_t_log_p}, \ref{Lemma_Bound_Hessian_Log_P}, and \ref{Lemma_Bounds_Sphere}), our bounds for $\SPD(n)$ will depend on radial distance function $r(x) = d(x, x_{0})$ for some $x_{0}$. 
\end{itemize}

Throughout this section, we use the following notation: $\kappa:=\sqrt{-K_{\min}}>0$ and the model function is defined as 
\[
s_{K_{\min}}(r):=\frac{1}{\kappa}\sinh(\kappa r),
\qquad r\ge 0.
\]
Define the \emph{radial contraction map} and its inverse by
\begin{align*}
    \Phi_{t,x_1}(x)& :=\Exp_{x_1} \big((1-t)\Log_{x_1}(x)\big), \\
    \Psi_{t,x_1}(x) & :=\Exp_{x_1} \Big(\frac{1}{1-t}\Log_{x_1}(x)\Big).
\end{align*}
and note that under our geodesic interpolation, $x_t=\Phi_{t,x_1}(x_0)$, and $x_0 = \Psi_{t,x_1}(x)$.

We first provide the density formula. 
\begin{lemma}[Density of $p_t(x_1\mid x)$ on a Hadamard manifold]
\label{Lemma_Density_Formula_Hadamard}
Let $(M,g)$ be a complete, simply-connected, $d$-dimensional Riemannian manifold with non-positive curvature.
Let $p_{0}, p_{1}$ be the prior and target distributions, assuming independence. 
Then the conditional density of $X_1$ given $X_t=x$ is
\begin{align*}
p_t(x_1\mid x)
&=\frac{
    p_1(x_1)\,p_0\big(\Psi_{t,x_1}(x)\big)\,J_t(x\mid x_1)}{\displaystyle \int_M
    p_1(z)\,p_0\big(\Psi_{t,z}(x)\big)\,J_t(x\mid z)\,  dV_g(z)}, \\ & \quad \text{where} \quad J_t(x\mid x_1) = 
    (1-t)^{-d}\, \frac{\big|\det(d\Exp_{x_1})_{\frac{1}{1-t}\Log_{x_1}(x)}\big|}
         {\big|\det(d\Exp_{x_1})_{\Log_{x_1}(x)}\big|}.
\end{align*}
Furthermore, we have 
\begin{align*}
    p_{t}(x) = \int_M p_1(z)\,p_0\big(\Psi_{t,z}(x)\big)\,J_t(x\mid z)\,dV_g(z).
\end{align*}
\end{lemma}

\begin{proof}

    Notice that $\Phi_{t,x_1}$ is a global diffeomorphism, and $\Psi_{t,x_1}$ is its inverse.
    For fixed $x_1$ and $x$, the change-of-variables formula gives
    \begin{align*}
        dV_g(\Psi_{t,x_1}(x)) = \big|\det(d\Psi_{t,x_1})_x\big|\,dV_g(x),
    \end{align*}
    which implies
    \begin{align*}
        p(\Psi_{t,x_1}(x),x_1)\,dV_g(\Psi_{t,x_1}(x))\,dV_g(x_1)
        &=
        p_0\big(\Psi_{t,x_1}(x)\big)\,p_1(x_1)\,
        \big|\det(d\Psi_{t,x_1})_x\big|\,
        dV_g(x)\,dV_g(x_1).
    \end{align*}
    We define 
    \begin{align*}
        J_t(x\mid x_1) := \big|\det(d\Psi_{t,x_1})_x\big|.
    \end{align*}
    Equivalently, the joint measure of $(X_t,X_1)$ can be written as
    \begin{align*}
        p_{t}(x,x_1)\,dV_g(x)\,dV_g(x_1)
        &=
        p_0\big(\Psi_{t,x_1}(x)\big)\,p_1(x_1)\,J_t(x\mid x_1)\,dV_g(x)\,dV_g(x_1),
    \end{align*}

    Integrating the joint density over $x_1$ yields the marginal
    \[
    p_t(x)=\int_M p_1(z)\,p_0\big(\Psi_{t,z}(x)\big)\,J_t(x\mid z)\,dV_g(z).
    \]
    Therefore, by Bayes' rule,
    \begin{align*}
        p_t(x_1\mid x)=\frac{p_t(x,x_1)}{p_t(x)} = \frac{
        p_1(x_1)\,p_0\big(\Psi_{t,x_1}(x)\big)\,J_t(x\mid x_1)}{
        \displaystyle \int_M p_1(z)\,p_0\big(\Psi_{t,z}(x)\big)\,J_t(x\mid z)\,dV_g(z)}.
    \end{align*}

    It remains to compute $J_{t}(x \mid x_{1})$. Recall that $\Psi_{t,x_1}(x) =\Exp_{x_1} \Big(\frac{1}{1-t}\Log_{x_1}(x)\Big)$. 
    Using chain rule,
    \begin{align*}
        (d\Psi_{t,x_1})_x
        &=
        (d\Exp_{x_1})_v \circ \big((1-t)^{-1}Id\big)\circ (d\Log_{x_1})_x.
    \end{align*}
    Using $(d\Log_{x_1})_x = \big((d\Exp_{x_1})_{\Log_{x_1}(x)}\big)^{-1}$, we have 
    \begin{align*}
        (d\Psi_{t,x_1})_x
        &=
        (d\Exp_{x_1})_{\frac{1}{1-t}\Log_{x_1}(x)} \circ \big((1-t)^{-1}Id\big)\circ \big((d\Exp_{x_1})_{\Log_{x_1}(x)}\big)^{-1}.
    \end{align*}
    Taking determinants, 
    \begin{align*}
        J_t(x\mid x_1) = (d\Psi_{t,x_1})_x
        =
        (1-t)^{-d}\,
        \frac{\big|\det(d\Exp_{x_1})_{\frac{1}{1-t}\Log_{x_1}(x)}\big|}
            {\big|\det(d\Exp_{x_1})_{\Log_{x_1}(x)}\big|}.
    \end{align*}

\end{proof}

Now we prove Proposition \ref{Prop_Regularity_Explicit_SPD}, which verifies Assumption \ref{A_Regularity_Expectation_V}.

\begin{proof}[Proof of Proposition \ref{Prop_Regularity_Explicit_SPD}]
    The desired result is directly implied by Lemma \ref{Lemma_Bound_E_S_K} and Lemma \ref{Lemma_Regularity_SPD_Inter}:
    \begin{align*}
        \mathbb{E}[\|v(t, x)\|^{2}] \lesssim & \mathbb{E}[d(x_0, x_1)^{2}] \lesssim d ,\\
         \mathbb{E}[ \|\nabla v(t, x) \|] \lesssim & \frac{d}{1-t} L_{R} \mathbb{E}[d(x_0, x_1)^{2}]^{\frac{1}{2}} \mathbb{E}[ s_{K_{\min}}(d(x_0, x_1))^{6}]^{\frac{1}{2}} \\
         \lesssim & \frac{d^{2 + 6\lambda}}{1-t} L_{R} M_{\lambda_{1}}^{\frac{1}{2}} ,\\
        \mathbb{E}[|\frac{d}{dt}v(t,x)|] \lesssim & \frac{d}{1-t}L_{R} \mathbb{E}[d(x_0, x_1)^{2} s_{K_{\min}}(d(x_0, x_1))^{6}]^{\frac{1}{2}}\mathbb{E}[d(x_0, x_1)^{2}]^{\frac{1}{2}} \\
        \lesssim & \frac{d^{2 + 6 \lambda}}{1-t}L_{R} M_{\lambda_{1}}^{\frac{1}{2}} ,\\
        \mathbb{E}[\|\grad_{x}\Div v(t,x)\|] \lesssim & \mathbb{E}[d(x_0, x_1) s_{K_{\min}}(d(x_0, x_1))^{6} ]\frac{d^{2}}{(1-t)^{2}} L_{R}^{3} \lesssim \frac{d^{3+12\lambda}}{(1-t)^{2}} L_{R}^{3} ,\\
        \mathbb{E}[|\frac{d}{dt}\Div v(t,x)|] \lesssim & \mathbb{E}[d(x_0, x_1)^{2}]^{\frac{1}{2}} \frac{d^{2}}{(1-t)^{2}} L_{R}^{3} \mathbb{E}[d(x_1, x_0)^{2} s_{K_{\min}}(d(x_0, x_1))^{12}]^{\frac{1}{2}} \\
        \lesssim &  \frac{d^{3 + 12\lambda}}{(1-t)^{2}} L_{R}^{3} M_{\lambda_{1}}^{\frac{1}{2}},
    \end{align*}
    and 
    \begin{align*}
        \mathbb{E}[ \|\nabla v(t, x) \|^{2}] \lesssim & \frac{d^{2}}{(1-t)^{2}} L_{R}^{2} \mathbb{E}[d(x_0, x_1)^{4}]^{\frac{1}{2}} \mathbb{E}[ s_{K_{\min}}(d(x_0, x_1))^{12}]^{\frac{1}{2}} \lesssim \frac{d^{3+12\lambda}}{(1-t)^{2}} L_{R}^{2} M_{\lambda_{1}}^{\frac{1}{2}}, \\
        \mathbb{E}[|\frac{d}{dt}v(t,x)|^{2}] \lesssim & \frac{d^{2}}{(1-t)^{2}}L_{R}^{2} \mathbb{E}[d(x_0, x_1)^{4} s_{K_{\min}}(d(x_0, x_1))^{12}]^{\frac{1}{2}}\mathbb{E}[d(x_0, x_1)^{4}]^{\frac{1}{2}} \lesssim \frac{d^{3+12\lambda}}{(1-t)^{2}}L_{R}^{2} M_{\lambda_{1}}^{\frac{1}{2}}, \\
        \mathbb{E}[\|\grad_{x}\Div v(t,x)\|^{2}] \lesssim & \mathbb{E}[d(x_0, x_1)^{2} s_{K_{\min}}(d(x_0, x_1))^{12} ]\frac{d^{4}}{(1-t)^{4}} L_{R}^{6} \lesssim \frac{d^{5+24\lambda}}{(1-t)^{4}} L_{R}^{6} M_{\lambda_{1}}.
    \end{align*}
\end{proof}



\subsection{Proof of Proposition \ref{Prop_Score_Bound_Hadamard}}

We start to prove the score regularity result. In the following Lemma, we decompose $\mathbb E\big[\|\grad \log p_t(X_t)\|^2\big]$ as the sum of two terms. 

\begin{lemma}\label{Lemma_score_bound_hadamard}
For $X_t\sim p_t$,
\begin{align*}
\mathbb E\big[\|\grad \log p_t(X_t)\|^2\big]
\le
2\,\mathbb E\Big[
\big\|\grad_x \log p_0(\Psi_{t,X_1}(X_t))\big\|^2
\Big]
+
2\,\mathbb E\Big[
\big\|\grad_x \log J_t(X_t\mid X_1)\big\|^2
\Big].
\end{align*}
\end{lemma}

\begin{proof}
By Lemma~\ref{Lemma_Density_Formula_Hadamard},
\begin{align*}
    p_{t}(x) = \int_M p_1(z)\,p_0\big(\Psi_{t,z}(x)\big)\,J_t(x\mid z)\,dV_g(z).
\end{align*}
Taking the Riemannian gradient with respect to $x$ and differentiating under the integral sign,
\begin{align}\label{eq:grad_Zt_hadamard}
\grad p_t(x)
&=
\int_M p_1(z)\,\grad_x \Big(p_0(\Psi_{t,z}(x))\,J_t(x\mid z)\Big)\,dV_g(z).
\end{align}
Using the identity $\grad (\phi)=\phi\,\grad\log\phi$, we have 
\[
\grad_x \Big(p_0(\Psi_{t,z}(x))\,J_t(x\mid z)\Big)
=
p_0(\Psi_{t,z}(x))\,J_t(x\mid z)\,
\grad_x \log \Big(p_0(\Psi_{t,z}(x))\,J_t(x\mid z)\Big).
\]
Plugging into \eqref{eq:grad_Zt_hadamard} and dividing by $p_t(x)$ yields
\begin{align*}
\grad \log p_t(x)
&=
\frac{\grad p_t(x)}{p_t(x)}
=
\int_M
\grad_x \log \Big(p_0(\Psi_{t,z}(x))\,J_t(x\mid z)\Big)\,
\frac{p_1(z)\,p_0(\Psi_{t,z}(x))\,J_t(x\mid z)}{p_t(x)}\,dV_g(z).
\end{align*}
Observe that 
\begin{align*}
    \frac{p_1(z)\,p_0(\Psi_{t,z}(x))\,J_t(x\mid z)}{p_t(x)}
    = \frac{p_1(z)\,p_0(\Psi_{t,z}(x))\,J_t(x\mid z)}{\int_M p_1(z)\,p_0\big(\Psi_{t,z}(x)\big)\,J_t(x\mid z)\,dV_g(z)} = p_{t}(z \mid x).
\end{align*}
Hence we have 
\begin{align*}
    \grad \log p_t(x)=\grad \log p_t(x)
&=
\int_M
\grad_x \log \Big(p_0(\Psi_{t,z}(x))\,J_t(x\mid z)\Big)\,
p_t(z\mid x)\,dV_g(z) \\
&=
\mathbb E \Big[\grad_x \log \big(p_0(\Psi_{t,X_1}(x))J_t(x\mid X_1)\big)\mid X_t=x\Big].
\end{align*}
Next, apply Jensen's inequality,
\begin{align*}
    \mathbb{E}[\|\grad \log p_t(x)\|^{2}] 
    &= \mathbb{E}\Big[\|\mathbb E [\grad_x \log \big(p_0(\Psi_{t,X_1}(x))J_t(x\mid X_1)\big)\mid X_t=x]\|^{2}\Big] \\
    &\le \mathbb{E}\Big[\mathbb E [\|\grad_x \log \big(p_0(\Psi_{t,X_1}(x))J_t(x\mid X_1)\big) \|^{2}\mid X_t=x]\Big] \\
    &= \mathbb{E}\Big[\|\grad_x \log \big(p_0(\Psi_{t,X_1}(x))J_t(x\mid X_1)\big) \|^{2}\Big].
\end{align*}
The result follows from
\[
\grad_x \log \big(p_0(\Psi)\,J_t\big)=\grad_x\log p_0(\Psi)+\grad_x\log J_t,
\qquad
\|a+b\|^2\le 2\|a\|^2+2\|b\|^2.
\]
\end{proof}

Now we prove Proposition \ref{Prop_Score_Bound_Hadamard}.

\begin{proof}[Proof of Proposition \ref{Prop_Score_Bound_Hadamard}]
From Lemma \ref{Lemma_score_bound_hadamard}, we have 
\begin{align*}
    \mathbb E\big[\|\grad \log p_t(X_t)\|^2\big]
    \le
    2\,\mathbb E\Big[
    \big\|\grad_x \log p_0(\Psi_{t,X_1}(X_t))\big\|^2
    \Big]
    +
    2\,\mathbb E\Big[
    \big\|\grad_x \log J_t(X_t\mid X_1)\big\|^2
    \Big].
\end{align*}
We show the following in Appendix \ref{Subsec_Aux_Bounds_Derivative} (Lemma \ref{lem:prior_composed_bounds} and Lemma \ref{lem:Jt-pointwise-no-shorthand}):
\begin{align*}
    \|\grad_x \log J_t(x\mid x_1)\| \le & d\,\|(d\Log_{x_1})_x\| \Bigg(\frac{1}{1-t}\,
            \|\nabla(d\Exp_{x_1})_{\frac{1}{1-t}\Log_{x_1}(x)}\| +
            \|\nabla(d\Exp_{x_1})_{\Log_{x_1}(x)}\|
            \Bigg), \\
    \Big\|\grad_x \log p_0 (\Psi_{t,X_1}(X_t))\Big\|
        \le &
        2d d(x_0, z)
        \Big\|(d\Exp_{x_1})_{\tfrac{1}{1-t}\Log_{x_1}(x)}\Big\|_{\op}\;
        \frac{1}{1-t}\;
        \big\|(d\Log_{x_1})_x\big\|_{\op}.
\end{align*} 
We also prove the following in Lemma \ref{Lemma_Bound_Third_Deri_Exp} and Lemma \ref{lem:dLog-and-nabla-dLog}:
\begin{align*}
    \|(d\Log_{x_1})_x\| &\le 1, \\
    \Big\|(d\Exp_{x_1})_{\tfrac{1}{1-t}\Log_{x_1}(x)}\Big\|_{\op} &\le s_{K_{\min}}(\|\tfrac{1}{1-t}\Log_{x_1}(x)\|) / \|\tfrac{1}{1-t}\Log_{x_1}(x)\|
    = s_{K_{\min}}(d(x_0, x_1)) / d(x_0, x_1), \\
    \|\nabla(d\Exp_{x_1})_{\frac{1}{1-t}\Log_{x_1}(x)}\| &\le \frac{16}{3} s_{K_{\min}}(d(x_0, x_1)/2)^{2} L_{R} s_{K_{\min}}(d(x_0, x_1)).
\end{align*}
So together, 
\begin{align*}
    \|\grad_x \log J_t(x\mid x_1)\| \lesssim & \frac{d}{1-t} L_{R}  s_{K_{\min}}(d(x_0, x_1))^{3}  ,\\
    \Big\|\grad_x \log p_0 (\Psi_{t,X_1}(X_t))\Big\|
        \lesssim &
        \frac{d}{1-t} s_{K_{\min}}(d(x_0, x_1)) d(x_0, z).
\end{align*} 
Using Lemma \ref{Lemma_Bound_E_S_K}, we have 
\begin{align*}
    \mathbb E\big[\|\grad \log p_t(X_t)\|^2\big]
    \lesssim & \frac{d^{2}}{(1-t)^{2}} L_{R}^{2} \mathbb{E}[s_{K_{\min}}(d(x_0, x_1))^{6}] 
    \lesssim \frac{d^{2}}{(1-t)^{2}} L_{R}^{2} \mathbb{E}[s_{K_{\min}}(d(x_0, x_1))^{6}] \\
    \lesssim & \frac{d^{2 + 12\lambda}}{(1-t)^{2}} L_{R}^{2} M,
\end{align*}
where $\lambda = \max\{1, \kappa\}$.

\end{proof}

\subsection{Auxiliary Results for Regularity for Hadamard and SPD manifolds}

\subsubsection{Expectation Control}
Due to the non-compact nature of a Hadamard manifold as well as curvature distortion, bounding expectations of the form  
\begin{align*}
    \mathbb E [
d(X_0,X_1)^2\,
s_{K_{\min}} \big(d(X_0,X_1)\big)^a ], \quad a \in \mathbb{N},
\end{align*}
is a key step to establish regularity. We show that by choosing $X_{0}$ following a Riemannian Gaussian distribution, the above expectation can be controlled for a class of data distribution $X_{1}$ satisfying certain moment condition \eqref{Eq_Moment_Condition}.

\begin{lemma}\label{Lemma_Expectation_Regularity_Poly}
Let $(M,g)$ be a complete, simply connected $d$-dimensional Riemannian manifold (Hadamard) and assume
\[
K_{\min}\ \le\ \sec\ \le\ 0,
\qquad\text{where }K_{\min}<0\text{ is independent of }d.
\]
Fix a basepoint $z\in M$ and set $r(x):=d(x,z)$.
Define
\begin{align*}
    p_0(x)=\frac{1}{Z}e^{(-\beta\,r(x)^m)},\qquad 
    Z:=\int_M e^{-\beta\,r(x)^m}\,dV_g(x), \qquad \text{where} \qquad  \beta = d, m = 2.
\end{align*}
Let $\lambda\ge 0$ be independent of $d$.
Then for $b \in \{0, 1, 2, 3, 4\}$,
\[
\mathbb E_{p_0} \big[r(x)^b e^{\lambda r(x)}\big]
\ = \mathcal{O}(d^{2\lambda} (\log d) ^{b}).
\]
\end{lemma}

\begin{proof} We first write the expectation as an integral over tangent space.
Since $M$ is Hadamard, $\Exp_z:T_zM\to M$ is a global $C^\infty$ diffeomorphism.
We equip the vector space $T_zM$ with the inner product $g$, and denote by $dv$ the corresponding Lebesgue measure. For any measurable $F:M\to[0,\infty]$, the change-of-variables formula gives
\begin{align*}
\int_M F(x)\,dV_g(x)
=
\int_{T_zM} F(\Exp_z(v))\;\big|\det\big(d(\Exp_z)_v\big)\big|\;dv.
\end{align*}

We now apply this formula with $F_1(x)=e^{-d\,r(x)^2}$, $F_2(x)=r(x)^{b}e^{\lambda r(x)}e^{-d\,r(x)^2}$. Since $r(\Exp_z(v))=\|v\|$, we obtain
\begin{align}\label{eq:ratio_detdexp}
\mathbb E_{p_0}[r^be^{\lambda r}]
=
\frac{\int_{T_zM}\|v\|^{b}e^{\lambda\|v\|}e^{-d\|v\|^2}\,\big|\det(d(\Exp_z)_v)\big|\,dv}
{\int_{T_zM}e^{-d\|v\|^2}\,\big|\det(d(\Exp_z)_v)\big|\,dv}.
\end{align}
Recall that using comparison theorem (note that the $d-1$ comes from $\det$) we have 
\begin{equation}\label{eq:det_lower_upper}
1
\ \le\
\big|\det(d(\Exp_z)_{r\theta})\big|
\ \le\
\Big(\frac{s_{K_{\min}}(r)}{r}\Big)^{d-1},
\qquad \forall r > 0.
\end{equation}
Since $\frac{s_{K_{\min}}(r)}{r}=\frac{\sinh(\kappa r)}{\kappa r} \le \,e^{\kappa r}$, we have 
\begin{equation}\label{eq:det_upper_exp}
\big|\det(d(\Exp_z)_{r\theta})\big|
\ \le \,e^{\kappa(d-1)r}.
\end{equation}

We now integrate in $T_zM$ using Euclidean polar coordinates:
$v=r\theta$ with $r\in[0,\infty)$, $\theta\in\mathbb S^{d-1}$, and
\[
dv=r^{d-1}\,dr\,d\theta.
\]

We split at $R=2\log d$ and bound the expectation. Consider the decomposition below
\[
\mathbb E_{p_0}[r^b e^{\lambda r}]
=
\mathbb E_{p_0}[r^b e^{\lambda r}\mathbf 1_{\{r\le R\}}]
+
\mathbb E_{p_0}[r^b e^{\lambda r}\mathbf 1_{\{r>R\}}].
\]
We first consider the central part. On $\{r\le R\}$,
$r^b e^{\lambda r}\le R^b e^{\lambda R}$, hence
\begin{equation}\label{eq:center_part_final}
\mathbb E_{p_0}[r^b e^{\lambda r}\mathbf 1_{\{r\le R\}}]
\le
R^b e^{\lambda R}
=
(2\log d)^b\,d^{2\lambda}.
\end{equation}
We next consider the tail part. Write the tail contribution as $N_{\mathrm{tail}}/Z$.
Using change of variable formula and the upper bound \eqref{eq:det_upper_exp}, we obtain
\begin{align}
N_{\mathrm{tail}}
&:=
\int_{\{x:\,r(x)>R\}} r(x)^b e^{\lambda r(x)}e^{-dr(x)^2}\,dV_g(x)\notag\\
&=
\int_{\{v:\,\|v\|>R\}}\|v\|^b e^{\lambda\|v\|}e^{-d\|v\|^2}\big|\det(d(\Exp_z)_v)\big|\,dv \notag\\
&\le |\mathbb S^{d-1}|\int_{R}^{\infty}
r^{d-1+b}e^{(-d r^2+(\lambda+\kappa(d-1))r)}\,dr.
\label{eq:Ntail_radial}
\end{align}
Define
\[
\Psi_d(r):=-d r^2+(\lambda+\kappa(d-1))r+(d-1+b)\log r.
\]

We justify that this $\Psi_{d}(r)$ is dominated by the $-d r^2$ term. 
For $r\ge R=2\log d$, we have $(d-1+b)\log r \le \frac{d}{4}r^2, \forall d \ge 4$.
Also, for $r\ge R$ and $d$ large, $\kappa(d-1)r\le \frac{d}{4}r^2$ because this is equivalent to
$r\ge 4\kappa(1-1/d)$, which holds once $R\ge 4\kappa$ (i.e. $d\ge e^{2\kappa}$).
Finally, $\lambda r\le \frac{d}{8}r^2$ holds for all $r\ge R$ once $d\ge 8\lambda$.
Summing these inequalities yields, there exists some constant $d_{1}$ s.t. for all $d\ge d_1$ and all $r\ge R$,
\[
(\lambda+\kappa(d-1))r+(d-1+b)\log r \le \frac{d}{2}r^2,
\quad\text{hence}\quad
\Psi_d(r)\le -\frac{d}{2}r^2.
\]
Therefore, for $d\ge d_1$,
\begin{align*}
    \int_R^\infty e^{\Psi_d(r)}\,dr &\le \int_R^\infty e^{-(d/2)r^2}\,dr 
    = \int_R^\infty \frac{1}{rd} (-\frac{d}{dr} e^{-(d/2)r^2})\,dr 
    \le \frac{1}{Rd}\int_R^\infty  (-\frac{d}{dr} e^{-(d/2)r^2})\,dr \\
    &= \frac{1}{dR}\exp \Big(-\frac{d}{2}R^2\Big)
    =\frac{1}{dR}\exp \big(-2d(\log d)^2\big).
\end{align*}
Plugging into \eqref{eq:Ntail_radial} gives
\begin{equation}\label{eq:Ntail_final}
N_{\mathrm{tail}}
\le |\mathbb S^{d-1}|\cdot \frac{1}{dR}\exp \big(-2d(\log d)^2\big).
\end{equation}

For the denominator $Z$, recall that 
$\big|\det(d(\Exp_z)_v)\big|\ge 1$, so we get \begin{align*}
\int_{T_zM}e^{-d\|v\|^2}\big|\det(d(\Exp_z)_v)\big|\,dv
&\ge
\int_{T_zM}e^{-d\|v\|^2}\,dv
=
\Big(\frac{\pi}{d}\Big)^{d/2}.
\end{align*}
Hence, using \eqref{eq:Ntail_final},
\[
\mathbb E_{p_0}[r^b e^{\lambda r}\mathbf 1_{\{r>R\}}]
=\frac{N_{\mathrm{tail}}}{Z}
\le
\frac{|\mathbb S^{d-1}|}{dR}\exp \big(-2d(\log d)^2\big)\cdot
\Big(\frac{d}{\pi}\Big)^{d/2}
= \frac{|\mathbb S^{d-1}|}{2d\log d} d^{-2d \log d}\cdot
\Big(\frac{d}{\pi}\Big)^{d/2} \lesssim 1.
\]
Together, we conclude that for all $d\ge d_1$, combining \eqref{eq:center_part_final} with the tail bound gives
\[
\mathbb E_{p_0}[r^b e^{\lambda r}]
\le
(2\log d)^b d^{2\lambda}+1 = \mathcal{O}(d^{2\lambda} (\log d) ^{b}).
\]
\end{proof}

Now we present intermediate steps.
\begin{lemma}[Auxiliary bounds for expectation of $s_{K_{\min}}$]
\label{Lemma_Bound_E_S_K}
Let $M$ be a Hadamard manifold.
Fix $z\in M$. Choose prior as $X_0\sim p_0$ where
\[
p_0(x) \propto \exp \big(-d\,d(x,z)^2\big).
\]
Assume Assumption \ref{A_Curvature} holds. Let $\kappa = \sqrt{-K_{\min}}$ and $\lambda_{0} = a \max\{1, \kappa\}$, where $0 \le a \le a_{0}$ is some constant.
Assume 
\begin{align*}
    \max\big\{\mathbb E [d(X_1,z)^4 e^{\lambda_1 d(X_1,z)}], \mathbb E [e^{\lambda_1 d(X_1,z)}]\big\} \le M, \qquad \text{where} \qquad \lambda_{1} = a_{1} \max\{1, \kappa\}.
\end{align*}
Then we have, for $b \in \{0, 1, 2, 3, 4\}$, 
\begin{align*}
    \mathbb E \Big[
d(X_0,X_1)^b\,
s_{K_{\min}} \big(d(X_0,X_1)\big)^a \Big] &\lesssim M \mathbb E_{X_0\sim p_0} \Big[d(X_0,z)^b e^{\lambda_0 d(X_0,z)}\Big] \lesssim d^{2\lambda_{0}} (\log d)^{b} M, \\
\mathbb E \Big[
d(X_0,z)^b\,
s_{K_{\min}} \big(d(X_0,X_1)\big)^a \Big] &\lesssim M \mathbb E_{X_0\sim p_0} \Big[d(X_0,z)^b e^{\lambda_0 d(X_0,z)}\Big] \lesssim d^{2\lambda_{0}} (\log d)^{b} M. 
\end{align*}

\end{lemma}

\begin{proof} Observe that 
\begin{align*}
    s_{K_{\min}}(d(x_0, x_1))^{a}
    = \frac{\sinh^{a}(\kappa d (x_0, x_1) )}{\kappa^{a}}
    \le ( e^{\max\{1, \kappa\} d(x_0, x_1)} )^{a}.
\end{align*}
Hence 
\begin{align}
d(x_0,x_1)^b\,
s_{K_{\min}} \big(d(x_0,x_1)\big)^a 
\le d(x_0, x_1)^b ( e^{\max\{1, \kappa\} d(x_0, x_1)} )^{a}.
\end{align}
Taking expectation, it suffices to bound
$\mathbb E \big[d(x_0,x_1)^b e^{\lambda_0 d(x_0,x_1)}\big]$ with $\lambda_{0} = a \max\{1, \kappa\}$.

The triangle inequality gives
\begin{align*}
    d(x_0,x_1)\le d(x_0,z)+d(x_1,z).
\end{align*}
Therefore,
\begin{align*}
    d(x_0,x_1)^b e^{\lambda_0 d(x_0,x_1)}
\le &
\big(d(x_0,z)+d(x_1,z)\big)^b
\exp \big(\lambda_0(d(x_0,z)+d(x_1,z))\big) \\
\lesssim & (d(x_0, z)^b + d(x_1, z)^b ) \exp(\lambda_0 d(x_0, z) ) \exp(\lambda_0 d(x_1, z)).
\end{align*}
Take expectation and use independence of $X_0$ and $X_1$:
\begin{align}
\mathbb E \Big[d(X_0,X_1)^b e^{\lambda_0 d(X_0,X_1)}\Big]
&\le
2\,\mathbb E \Big[e^{\lambda_0 d(X_1,z)}\Big]\,
\mathbb E_{X_0\sim p_0} \Big[d(X_0,z)^b e^{\lambda_0 d(X_0,z)}\Big]
\nonumber\\
&\quad+
2\,\mathbb E \Big[d(X_1,z)^b e^{\lambda_0 d(X_1,z)}\Big]\,
\mathbb E_{X_0\sim p_0} \Big[e^{\lambda_0 d(X_0,z)}\Big].
\label{eq:separate_X0_X1}
\end{align}
By Assumption, both $\mathbb E \Big[e^{\lambda_0 d(X_1,z)}\Big]$ and $\mathbb E \Big[d(X_1,z)^b e^{\lambda_0 d(X_1,z)}\Big]$ are controlled by $M$.
It hence follows that 
\begin{align*}
    \mathbb E \Big[d(X_0,X_1)^b e^{\lambda_0 d(X_0,X_1)}\Big]
&\lesssim M \mathbb E_{X_0\sim p_0} \Big[d(X_0,z)^b e^{\lambda_0 d(X_0,z)}\Big] \lesssim d^{2\lambda_{0}} (\log d)^{b} M,
\end{align*}
where we applied Lemma \ref{Lemma_Expectation_Regularity_Poly}.

\end{proof}

\subsubsection{Regularity Control}

The following lemma summarizes the building blocks needed to establish regularity, serving as the same purpose as Lemma \ref{Lemma_Moment_Bound_Grad_Log_P}, \ref{Lemma_Bound_partial_t_log_p}, \ref{Lemma_Bound_Hessian_Log_P}, and \ref{Lemma_Bounds_Sphere}. 

\begin{lemma}\label{Lemma_Pointwise_Inter_SPD}
    On $\SPD(n)$, with prior distribution being $p_{0} \propto \exp (- d d(x_0, z)^{2} )$, assume Assumption \ref{A_Curvature}, we have the following bounds.    
    \begin{align*}
    \|\grad_x \log J_t(x\mid x_1)\| \lesssim & d\, \frac{1}{1-t} L_{R}  s_{K_{\min}}(d(x_0, x_1))^{3},  \\
        \Big\|\grad_x \log p_0 (\Psi_{t,x_1}(x))\Big\|
            \lesssim &
            d \frac{1}{1-t} s_{K_{\min}}(d(x_0, x_1)) d(x_0, z), \\
    |\partial_t \log J_t(x\mid x_1)| \lesssim & \frac{d}{1-t} d(x_0, x_1) s_{K_{\min}}(d(x_0, x_1))^{3} L_{R}, \\
    |\partial_t \log p_0(\Psi_{t,x_1}(x))| \lesssim & \frac{d}{1-t} d(x_0, z) s_{K_{\min}}(d(x_0, x_1)), \\
    \|\nabla_x^2 \log p_0(\Psi_{t,x_1}(x))\|_{\op} \lesssim & \frac{d}{(1-t)^{2}} s_{K_{\min}}(d(x_0, x_1))^{4} L_{R} d(x_0, z), \\
    \|\nabla_x^2 \log J_t(x\mid x_1)\|_{\op} \lesssim & \frac{d}{(1-t)^2}
            L_{R}^{3} s_{K_{\min}}(d(x_0, x_1))^{6}, \\
    \Big\|\partial_t \grad_x \log p_0(\Psi_{t,x_1}(x))\Big\| \lesssim & d(x_0, z) \frac{d}{(1-t)^{2}} d(x_0, x_1)
        L_{R} s_{K_{\min}}(d(x_0, x_1))^{3}, \\
        \|\partial_t \grad_x \log J_t(x\mid x_1)\| \lesssim &\frac{d}{(1-t)^2} L_{R}^{3} d(x_1, x_0) s_{K_{\min}}(d(x_0, x_1))^{6}, \\
        \|\nabla_x \Log_x(x_{1})\|_{\op}^{2} \lesssim & d(x_{0},x_{1})^{2}, \\
        |\Div_x\Log_x(x_{1})|^{2} \lesssim & d^{2} d(x_{0}, x_{1})^{2}, \\
        \|\grad \Div_x\Log_x(x_{1})\| \lesssim & d d(x_0, x_1)^{\frac{3}{2}}.
    \end{align*}
    Here we emphasize that $\Log_{x}(x_{1})$ is viewed as a vector field, for fixed $x_{1}$, so $\nabla \Log_{x}(x_{1})$ is covariant derivative of the vector field $\Log_{x}(x_{1})$. 
    Consequently, 
    \begin{align*}
        \mathbb{E}[\|\grad_x \log p_t(x_1\mid x)\|^{2} \mid X_{t} = x] \lesssim & \frac{d^{2}}{(1-t)^{2}} L_{R}^{2} \mathbb{E}[ s_{K_{\min}}(d(x_0, x_1))^{6}],  \\
        \mathbb{E}[\|\grad_x \log p_t(x_1\mid x)\|^{4} \mid X_{t} = x] \lesssim & \frac{d^{4}}{(1-t)^{4}} L_{R}^{4} \mathbb{E}[ s_{K_{\min}}(d(x_0, x_1))^{12}], \\
    \mathbb{E}[|\partial_t \log p_t(x_1\mid x)|^{2} \mid X_{t} = x] \lesssim & \frac{d^{2}}{(1-t)^{2}}L_{R}^{2} \mathbb{E}[d(x_0, x_1)^{2} s_{K_{\min}}(d(x_0, x_1))^{6}] , \\
    \mathbb{E}[|\partial_t \log p_t(x_1\mid x)|^{4} \mid X_{t} = x] \lesssim & \frac{d^{4}}{(1-t)^{4}}L_{R}^{4} \mathbb{E}[d(x_0, x_1)^{4} s_{K_{\min}}(d(x_0, x_1))^{12}] , \\
    \mathbb{E}[\|\nabla_x^2 \log p_t(x_1\mid x)\|_{\op}^{2} \mid X_{t} = x] \lesssim & \frac{d^{2}}{(1-t)^4}
            L_{R}^{6} \mathbb{E}[ s_{K_{\min}}(d(x_0, x_1))^{12}] \\ & + \frac{d^{4}}{(1-t)^{4}} L_{R}^{4} \mathbb{E}[ s_{K_{\min}}(d(x_0, x_1))^{6}]^{2}, \\
    \mathbb{E}[\|\nabla_x^2 \log p_t(x_1\mid x)\|_{\op}^{4} \mid X_{t} = x] \lesssim & \frac{d^{4}}{(1-t)^8}
            L_{R}^{12} \mathbb{E}[ s_{K_{\min}}(d(x_0, x_1))^{24}] \\ & + \frac{d^{8}}{(1-t)^{8}} L_{R}^{8} \mathbb{E}[ s_{K_{\min}}(d(x_0, x_1))^{12}]^{2}, \\
        \mathbb{E}[\|\partial_t \grad_x \log p_t(x_1\mid x)\|^{2} \mid X_{t} = x] \lesssim &\frac{d^{2}}{(1-t)^4} L_{R}^{6} \mathbb{E}[d(x_1, x_0)^{2} s_{K_{\min}}(d(x_0, x_1))^{12}] \\ & + \frac{d^{4}}{(1-t)^{4}} L_{R}^{4} \mathbb{E}[ s_{K_{\min}}(d(x_0, x_1))^{6}]\mathbb{E}[d(x_0, x_1)^{2} s_{K_{\min}}(d(x_0, x_1))^{6}].
    \end{align*}
\end{lemma}
\begin{proof}
    Throughout the proof, we use $r(x)$ to denote the radial distance function $d(x, z)$, where $z$ is the center of $p_{0} \propto \exp(- \beta d(x_0, z)^{2} )$.
    We show the following in Appendix \ref{Subsec_Aux_Bounds_Derivative} (Lemma \ref{lem:prior_composed_bounds} and Lemma \ref{lem:Jt-pointwise-no-shorthand}):
\begin{align*}
    \|\grad_x \log J_t(x\mid x_1)\| \le & d\,\|(d\Log_{x_1})_x\| \Bigg(\frac{1}{1-t}\,
            \|\nabla(d\Exp_{x_1})_{\frac{1}{1-t}\Log_{x_1}(x)}\| +
            \|\nabla(d\Exp_{x_1})_{\Log_{x_1}(x)}\|
            \Bigg), \\
    \Big\|\grad_x \log p_0 (\Psi_{t,X_1}(X_t))\Big\|
        \le &
        2d d_{g}(x_0, z)
        \Big\|(d\Exp_{x_1})_{\tfrac{1}{1-t}\Log_{x_1}(x)}\Big\|_{\op}\;
        \frac{1}{1-t}\;
        \big\|(d\Log_{x_1})_x\big\|_{\op}.
\end{align*} 
We also prove the following in (Lemma \ref{Lemma_Bound_Third_Deri_Exp} and Lemma \ref{lem:dLog-and-nabla-dLog}:
\begin{align*}
    \|(d\Log_{x_1})_x\| &\le 1, \\
    \big\|\big(\nabla(d\Log_x)\big)_y\big\|
     &\le
    \big\|\big(\nabla(d\Exp_x)\big)_{\Log_x(y)}\big\|, \\
    \Big\|(d\Exp_{x_1})_{\tfrac{1}{1-t}\Log_{x_1}(x)}\Big\|_{\op} &\le s_{K_{\min}}(\|\tfrac{1}{1-t}\Log_{x_1}(x)\|) / \|\tfrac{1}{1-t}\Log_{x_1}(x)\|
    \lesssim s_{K_{\min}}(d(x_0, x_1)), \\
    \|\nabla(d\Exp_{x_1})_{\frac{1}{1-t}\Log_{x_1}(x)}\| &\le \frac{16}{3} s_{K_{\min}}(d(x_0, x_1)/2)^{2} L_{R} s_{K_{\min}}(d(x_0, x_1)) \lesssim L_{R} s_{K_{\min}}(d(x_0, x_1))^{3}, \\
    \|\nabla^{2}(d\Exp_{x_1})_{\frac{1}{1-t}\Log_{x_1}(x)}\| &\lesssim L_{R}^{3} s_{K_{\min}}(d(x_0, x_1))^{5}.
\end{align*}

    We estimate the terms as follows. 
    For gradient and time derivative, we have 
    \begin{align*}
        \|\grad_x \log J_t(x\mid x_1)\| \lesssim & d\, \frac{1}{1-t} L_{R}  s_{K_{\min}}(d(x_0, x_1))^{3},  \\
        \Big\|\grad_x \log p_0 (\Psi_{t,X_1}(X_t))\Big\|
            \lesssim &
            d \frac{1}{1-t} s_{K_{\min}}(d(x_0, x_1)) d(x_0, z),
    \end{align*} 
    and 
    \begin{align*}
        |\partial_t \log J_t(x\mid x_1)| \le & \frac{d}{1-t}
            +\frac{d}{(1-t)^2}\,\|\Log_{x_1}(x)\|\; \|\nabla(d\Exp_{x_1})_{\frac{1}{1-t}\Log_{x_1}(x)}\| \\
            \lesssim & \frac{d}{1-t} d(x_0, x_1) s_{K_{\min}}(d(x_0, x_1))^{3} L_{R}, \\
        |\partial_t \log p_0(\Psi_{t,x_1}(x))|
        \le& 2 \beta r
        \Big\|(d\Exp_{x_1})_{\tfrac{1}{1-t}\Log_{x_1}(x)}\Big\|_{\op}\;
        \frac{1}{(1-t)^{2}}\;
        \big\|\Log_{x_1}(x)\big\| \\
        \lesssim & \frac{d}{1-t} d(x_0, z) s_{K_{\min}}(d(x_0, x_1)).
    \end{align*}
    For second covariant derivative, 
    \begin{align*}
        &\|\nabla_x^2 \log p_0(\Psi_{t,x_1}(x))\|_{\op} \\
        \le &
        \Big(
            \Big\|(d\Exp_{x_1})_{\tfrac{1}{1-t}\Log_{x_1}(x)}\Big\|_{\op}\;
            \frac{1}{1-t}\;
            \big\|(d\Log_{x_1})_x\big\|_{\op}
        \Big)^{2} \Big(
            2 \beta 
            +
            2 \beta r
            \|\nabla^{2} r\|_{\op}\Big|_{\Psi_{t,x_1}(x)}
        \Big)
        \\ &+
        \Big\|\nabla(d\Exp_{x_1})_{\tfrac{1}{1-t}\Log_{x_1}(x)}\Big\|\,
        \Big(\frac{1}{1-t}\big\|(d\Log_{x_1})_x\big\|_{\op}\Big)^{2}\,
        2 \beta r \\ &+
        \Big\|(d\Exp_{x_1})_{\tfrac{1}{1-t}\Log_{x_1}(x)}\Big\|_{\op}\,
        \frac{1}{1-t}\,
        \big\|\nabla(d\Log_{x_1})_x\big\|\,
        2 \beta r \\
        \lesssim & \frac{1}{(1-t)^{2}} s_{K_{\min}}(d(x_0, x_1))^{2} (d d(x_0, z) \kappa \coth(\kappa d(x_0, z) ) ) \\
        &+ \frac{d}{(1-t)^{2}} s_{K_{\min}}(d(x_0, x_1))^{3} L_{R} d(x_0, z) + \frac{d}{1-t}d(x_0, z) s_{K_{\min}}(d(x_0, x_1))^{4} L_{R} \\
        \lesssim & \frac{d}{(1-t)^{2}} s_{K_{\min}}(d(x_0, x_1))^{4} L_{R} d(x_0, z),
    \end{align*}
    where notice that $x\coth (x)$ is of order $x$.
    And we have 
    \begin{align*}
        &\|\nabla_x^2 \log J_t(x\mid x_1)\|_{\op} \\ \le&d\,\|(d\Log_{x_1})_x\|^2
            \Bigg[
            \frac{1}{(1-t)^2}\Big(
            \|\nabla(d\Exp_{x_1})_{\frac{1}{1-t}\Log_{x_1}(x)}\|^2  + 
            \|\nabla^2(d\Exp_{x_1})_{\frac{1}{1-t}\Log_{x_1}(x)}\|
            \Big)  \\
            &+\Big(
            \|\nabla(d\Exp_{x_1})_{\Log_{x_1}(x)}\|^2 + \|\nabla^2(d\Exp_{x_1})_{\Log_{x_1}(x)}\|\Big)\Bigg]  \\
            &+d\,\|\nabla(d\Log_{x_1})_x\|\Bigg(
            \frac{1}{1-t}\,
            \|\nabla(d\Exp_{x_1})_{\frac{1}{1-t}\Log_{x_1}(x)}\| +
            \|\nabla(d\Exp_{x_1})_{\Log_{x_1}(x)}\|
            \Bigg) \\
            \lesssim & 
            \frac{d}{(1-t)^2}
            (L_{R}^{2} s_{K_{\min}}(d(x_0, x_1))^{6} + L_{R}^{3} s_{K_{\min}}(d(x_0, x_1))^{5} )+ 
            \frac{d}{1-t} L_{R}^{2} s_{K_{\min}}(d(x_0, x_1))^{6} \\
            \lesssim & \frac{d}{(1-t)^2}
            L_{R}^{3} s_{K_{\min}}(d(x_0, x_1))^{6}.
    \end{align*}
    For time derivative of gradient, 
    \begin{align*}
        &\Big\|\partial_t \grad_x \log p_0(\Psi_{t,x_1}(x))\Big\| \\
    \le &
    \Big(
        \Big\|(d\Exp_{x_1})_{\tfrac{1}{1-t}\Log_{x_1}(x)}\Big\|_{\op}\;
        \frac{1}{1-t}\;
        \big\|(d\Log_{x_1})_x\big\|_{\op}
    \Big) \times
    \Big(
        2 \beta 
        +
        2 \beta r
        \|\nabla^{2} r\|_{\op}\Big|_{\Psi_{t,x_1}(x)}
    \Big)
    \\
    &\times
    \Big\|(d\Exp_{x_1})_{\tfrac{1}{1-t}\Log_{x_1}(x)}\Big\|_{\op}\;
    \frac{1}{(1-t)^{2}}\;
    \big\|\Log_{x_1}(x)\big\|
    \\
    &+
    \Bigg[
        \frac{1}{(1-t)^{2}}\Big\|(d\Exp_{x_1})_{\tfrac{1}{1-t}\Log_{x_1}(x)}\Big\|_{\op}\,
        \big\|(d\Log_{x_1})_x\big\|_{\op}
        \\
        &+
        \frac{1}{(1-t)^{3}}\, 
        \big\|\Log_{x_1}(x)\big\|\,
        \Big\|\nabla(d\Exp_{x_1})_{\tfrac{1}{1-t}\Log_{x_1}(x)}\Big\|\,
        \big\|(d\Log_{x_1})_x\big\|_{\op}
    \Bigg]
    2 \beta r \\
    \lesssim & 
    \Big(d + d d(x_0, z)
    \Big) s_{K_{\min}}(d(x_0, x_1))^{2}
    \frac{1}{(1-t)^{2}}\;
    \\
    &+ d d(x_0, z) 
    \Big(
        \frac{1}{(1-t)^{2}}s_{K_{\min}}(d(x_0, x_1)) + \frac{1}{(1-t)^{2}} d(x_0, x_1)
        L_{R} s_{K_{\min}}(d(x_0, x_1))^{3}
    \Big) \\
    \lesssim &d(x_0, z) \frac{d}{(1-t)^{2}} d(x_0, x_1)
        L_{R} s_{K_{\min}}(d(x_0, x_1))^{3}.
    \end{align*}
    And we have 
    \begin{align*}
        &\|\partial_t \grad_x \log J_t(x\mid x_1)\| \\ \le &
            \|(d\Log_{x_1})_x\|\Bigg[
            \frac{d}{(1-t)^2}
            \|\nabla(d\Exp_{x_1})_{\frac{1}{1-t}\Log_{x_1}(x)}\|  \\
            &+\frac{d}{(1-t)^3} \,\|\Log_{x_1}(x)\| \Big(
            \|\nabla(d\Exp_{x_1})_{\frac{1}{1-t}\Log_{x_1}(x)}\|^2 +
            \|\nabla^2(d\Exp_{x_1})_{\frac{1}{1-t}\Log_{x_1}(x)}\|
            \Big)
            \Bigg] \\
            \lesssim & 
            \frac{d}{(1-t)^2} L_{R} s_{K_{\min}}(d(x_0, x_1))^{3} +\frac{d}{(1-t)^2} d(x_1, x_0) \Big(
            L_{R}^{2} s_{K_{\min}}(d(x_0, x_1))^{6} + L_{R}^{3} s_{K_{\min}}(d(x_0, x_1))^{5}
            \Big) \\
            \lesssim & \frac{d}{(1-t)^2} L_{R}^{3} d(x_1, x_0) s_{K_{\min}}(d(x_0, x_1))^{6}.
    \end{align*}
    The three inequalities on $\Log$ follow from Appendix \ref{Subsec_Riem_Log}:
    \begin{align*}
    \|\nabla_x \Log_x(y)\|_{\op}^{2} &\le
    2 + 2\,d(x,y)^{2}\Big(\frac{s_{K_{\min}}'(d(x,y))}{s_{K_{\min}}(d(x,y))}\Big)^{2}, \\
    |\Div_x\Log_x(y)|^{2} &\le
    2 + 2(d-1)^{2}\,d(x,y)^{2}\Big(\frac{s_{K_{\min}}'(d(x,y))}{s_{K_{\min}}(d(x,y))}\Big)^{2}, \\
    \|\grad \Div_x\Log_x(y)\| &\le 
    \frac{\sqrt2 d}{2} (2\bigl(1+d(x,y)\,\frac{s_{K_{\min}}'(d(x,y))}{s_{K_{\min}}(d(x,y))}\bigr))^{\frac{3}{2}}.
\end{align*}

Now we prove the inequalities on derivatives of $\log p_{t}$. 
    We remark that due to Lemma \ref{Lemma_Bound_E_S_K}, comparing point-wise upper bounds for derivative of $\log J_{t}$ and that of $\log p_{0}$, since the $d(x_0, z)$ term doesn't increase the order of expectation compared with $d(x_0, x_1)$ term, we can treat $d(x_0, z)$ as $d(x_0, x_1)$ and see that derivatives of $\log J_{t}$ dominate over the same derivative of $\log p_{0}$. Thus in computing derivatives of $\log p_{t} = \log p_{0} (\Psi_{t,x_1}(x)) + \log J_{t}(x | x_{1}) + \text{const}$, it suffices to consider derivatives of $\log J_{t}(x | x_{1})$ terms.
    
    The first four inequalities are straightforward by applying the Cauchy-Schwarz inequality, and note that they are dominated by the $\log J_{t}$ terms. 
    Following Lemma \ref{Lemma_Bound_Hessian_Log_P}
    \begin{align*}
         \|\nabla_x^2 \log p_t(x_1\mid x)\|_{\op}^{2}
         \lesssim & \|\nabla^{2}\log J_{t}(x\mid X_{1})\|_{\op}^{2} + \mathbb{E}\bigl[\|\nabla^{2}\log J_{t}(x\mid X_{1})\|_{\op}^{2}\mid X_{t}=x\bigr] \\ &+ \mathbb{E}\bigl[\|\grad_{x}\log J_{t}(x\mid X_{1})\|^{2} \mid X_{t}=x\bigr]
        \mathbb{E}\bigl[\| \grad \log p_{t}(x_{1} \mid x)\|^{2} \mid X_{t}=x\bigr], \\
        \|\nabla_x^2 \log p_t(x_1\mid x)\|_{\op}^{4}
         \lesssim & \|\nabla^{2}\log J_{t}(x\mid X_{1})\|_{\op}^{4} \\ &+ \mathbb{E}\bigl[\|\grad_{x}\log J_{t}(x\mid X_{1})\|^{4} \mid X_{t}=x\bigr]
        \mathbb{E}\bigl[\| \grad \log p_{t}(x_{1} \mid x)\|^{4} \mid X_{t}=x\bigr], 
    \end{align*}
    and plug in the corresponding expression (using Cauch-Schwarz), we obtain the bounds for the Hessian of $\log p_t$, and following Lemma \ref{Lemma_Bounds_Sphere}, 
    \begin{align*}
        \|\partial_{t}\grad_{x}\log p_{t}(x_{1}\mid x)\|
        &\le
        \|\partial_{t}\grad_{x}\log J_{t}(x\mid x_{1})\| 
        + \mathbb{E}\bigl[\|\partial_{t}\grad_{x}\log J_{t}(x\mid X_{1})\| \mid X_{t}=x\bigr]  \\
        & \quad + \mathbb{E}[\|\grad_{x}\log J_{t}(x\mid X_{1})\|\,|\partial_{t}\log p_{t}(X_{1}\mid x)|\mid X_{t}=x],
    \end{align*}
    we plug in the corresponding expressions (using Cauch-Schwarz), we obtain the last inequality.
\end{proof}

We summarize the required bound in the following Lemma, which is similar to Lemma \ref{Lemma_Nabla_v_Sphere}, \ref{Lemma_Partial_t_v_Sphere}, \ref{Lemma_Grad_Div_v_Sphere} and \ref{Lemma_Partial_t_Div_v_Sphere}.
\begin{lemma}\label{Lemma_Regularity_SPD_Inter}
On $\SPD(n)$, with prior distribution being $p_{0} \propto \exp (- d d(x_0, z)^{2} )$, under the conditions in Assumption \ref{A_Curvature}, we have
    \begin{align*}
        \mathbb{E}[\|v(t, x)\|^{2}] \lesssim & \mathbb{E}[d(x_0, x_1)^{2}],\\
         \mathbb{E}[ \|\nabla v(t, x) \|] \lesssim & \frac{d}{1-t} L_{R} \mathbb{E}[d(x_0, x_1)^{2}]^{\frac{1}{2}} \mathbb{E}[ s_{K_{\min}}(d(x_0, x_1))^{6}]^{\frac{1}{2}}, \\
        \mathbb{E}[|\frac{d}{dt}v(t,x)|] \lesssim & \frac{d}{1-t}L_{R} \mathbb{E}[d(x_0, x_1)^{2} s_{K_{\min}}(d(x_0, x_1))^{6}]^{\frac{1}{2}}\mathbb{E}[d(x_0, x_1)^{2}]^{\frac{1}{2}}, \\
        \mathbb{E}[\|\grad_{x}\Div v(t,x)\|] \lesssim & \mathbb{E}[d(x_0, x_1) s_{K_{\min}}(d(x_0, x_1))^{6} ]\frac{d^{2}}{(1-t)^{2}} L_{R}^{3}, \\
        \mathbb{E}[|\frac{d}{dt}\Div v(t,x)|] \lesssim & \mathbb{E}[d(x_0, x_1)^{2}]^{\frac{1}{2}} \frac{d^{2}}{(1-t)^{2}} L_{R}^{3} \mathbb{E}[d(x_1, x_0)^{2} s_{K_{\min}}(d(x_0, x_1))^{12}]^{\frac{1}{2}}.
    \end{align*}
    Furthermore, 
    \begin{align*}
        \mathbb{E}[ \|\nabla v(t, x) \|^{2}] \lesssim & \frac{d^{2}}{(1-t)^{2}} L_{R}^{2} \mathbb{E}[d(x_0, x_1)^{4}]^{\frac{1}{2}} \mathbb{E}[ s_{K_{\min}}(d(x_0, x_1))^{12}]^{\frac{1}{2}}, \\
        \mathbb{E}[|\frac{d}{dt}v(t,x)|^{2}] \lesssim & \frac{d^{2}}{(1-t)^{2}}L_{R}^{2} \mathbb{E}[d(x_0, x_1)^{4} s_{K_{\min}}(d(x_0, x_1))^{12}]^{\frac{1}{2}}\mathbb{E}[d(x_0, x_1)^{4}]^{\frac{1}{2}}, \\
        \mathbb{E}[\|\grad_{x}\Div v(t,x)\|^{2}] \lesssim & \mathbb{E}[d(x_0, x_1)^{2} s_{K_{\min}}(d(x_0, x_1))^{12} ]\frac{d^{4}}{(1-t)^{4}} L_{R}^{6}.
    \end{align*}
\end{lemma}
\begin{proof}
Throughout the proof, we will use the bounds in Lemma \ref{Lemma_Pointwise_Inter_SPD}.We first control the vector field regularity. This is similar to Lemma \ref{Lemma_Nabla_v_Sphere} and Lemma \ref{Lemma_Partial_t_v_Sphere}. We have 
\begin{align*}
    \mathbb{E}[\|\nabla v(t, x) \|] \le & \frac{1}{1-t}\mathbb{E}\bigl[\|\Log_{x}(X_{1})\|\,\|\grad_{x}\log p_{t}(X_{1}\mid x)\|\mid X_{t}=x\bigr] \\
    &+ \frac{1}{1-t}\mathbb{E}[\|\nabla \Log_{x}(x_{1})\| \mid X_{t}=x] \\
    \lesssim & \frac{d}{1-t} L_{R} \mathbb{E}[d(x_0, x_1)^{2}]^{\frac{1}{2}} \mathbb{E}[ s_{K_{\min}}(d(x_0, x_1))^{6}]^{\frac{1}{2}},
\end{align*}
and 
\begin{align*}
        \mathbb{E}[|\frac{d}{dt}v(t,x)|]
        =& \frac{1}{(1-t)^{2}}\mathbb{E}[ \Log_{x}(x_{1}) \mid X_{t} = x]
        + \frac{1}{1-t}\mathbb{E}[ \Log_{x}(x_{1})\, \partial_t \log p_{t}(x_{1}\mid x) \mid X_{t} = x] \\
        \lesssim & 
        \frac{d}{1-t}L_{R} \mathbb{E}[d(x_0, x_1)^{2} s_{K_{\min}}(d(x_0, x_1))^{6}]^{\frac{1}{2}}\mathbb{E}[d(x_0, x_1)^{2}]^{\frac{1}{2}}.
    \end{align*}

    Now we control the divergence regularity.
    Similar to Lemma \ref{Lemma_Grad_Div_v_Sphere} and , we have 
    \begin{align*}
    \|\grad_{x}\Div v(t,x)\|
    \le \frac{1}{1-t}\Bigl(T_{1}+T_{2}+T_{3}+T_{4}\Bigr),
\end{align*}
where
\begin{align*}
    \mathbb{E}[T_{1}]
    =& \mathbb{E}[\|\grad_{x}\Div_{x}\Log_{x}(X_{1})\| ] \\
    \lesssim & d \mathbb{E}[d(x_0, x_1)^{\frac{3}{2}}], \\
    \mathbb{E}[T_{2}]
    =& \mathbb{E}[|\Div_{x}\Log_{x}(X_{1})|\,\|\grad_{x}\log p_{t}(X_{1}\mid x)\| ] \\
    \lesssim & \mathbb{E}[|\Div_{x}\Log_{x}(X_{1})|^{2} ]^{\frac{1}{2}}\mathbb{E}[\|\grad_{x}\log p_{t}(X_{1}\mid x)\|^{2} ]^{\frac{1}{2}} \\
    \lesssim & \frac{d^{2}}{1-t} L_{R} \mathbb{E}[d(x_0, x_1)^{2}]^{\frac{1}{2}} \mathbb{E}[s_{K_{\min}}(d(x_0, x_1))^{6}]^{\frac{1}{2}}, \\
    \mathbb{E}[T_{3}]
    =& \mathbb{E}[\|\grad_{x}\langle \grad_{x}\log p_{t}(X_{1}\mid x),\Log_{x}(X_{1})\rangle\| ] \\
    \le & \mathbb{E}[\|\nabla \Log_{x}(X_{1})\|\,\|\grad_{x}\log p_{t}(X_{1}\mid x)\| ] \\
        &\qquad + \mathbb{E}[\|\nabla^{2}\log p_{t}(X_{1}\mid x)\|\,\|\Log_{x}(X_{1})\| ] \\
    \lesssim & \Big(\mathbb{E}[d(x, x_1)^{2}] \frac{d^{2}}{(1-t)^{2}} L_{R}^{2} \mathbb{E}[s_{K_{\min}}(d(x_0, x_1))^{6}]\Big)^{\frac{1}{2}} \\ & + \Big(\frac{d^{2}}{(1-t)^4}
            L_{R}^{6} \mathbb{E}[ s_{K_{\min}}(d(x_0, x_1))^{12}] \mathbb{E}[d(x, x_1)^{2}]\Big)^{\frac{1}{2}} \\
    \lesssim & \frac{d}{1-t}
            L_{R}^{3} \mathbb{E}[ s_{K_{\min}}(d(x_0, x_1))^{12}]^{\frac{1}{2}} \mathbb{E}[d(x_0, x_1)^{2}]^{\frac{1}{2}}, \\
    \mathbb{E}[T_{4}]
    =& \mathbb{E}[|\langle \grad_{x}\log p_{t}(X_{1}\mid x),\Log_{x}(X_{1})\rangle|\,\|\grad_{x}\log p_{t}(X_{1}\mid x)\| ] \\
    \le & \mathbb{E}[\|\Log_{x}(X_{1})\|\,\|\grad_{x}\log p_{t}(X_{1}\mid x)\|^{2} ] \\
    \lesssim & \mathbb{E}[d(x, x_1) s_{K_{\min}}(d(x_0, x_1))^{6} ]\frac{d^{2}}{(1-t)^{2}} L_{R}^{2} \\
    \lesssim & \mathbb{E}[d(x_0, x_1) s_{K_{\min}}(d(x_0, x_1))^{6} ]\frac{d^{2}}{1-t} L_{R}^{2} .
\end{align*} 

Similar to Lemma \ref{Lemma_Partial_t_Div_v_Sphere},
\begin{align*}
        \frac{d}{dt}\Div v(t,x)
        &=
        \frac{1}{(1-t)^{2}}
        \mathbb{E}\bigl[\Div_{x}\Log_{x}(X_{1}) + \langle \grad_{x}\log p_{t}(X_{1}\mid x), \Log_{x}(X_{1})\rangle\mid X_{t}=x\bigr] \\
        & \quad + \frac{1}{1-t}\mathbb{E}[\Div_{x}\Log_{x}(X_{1})\,\partial_{t}\log p_{t}(X_{1}\mid x)\mid X_{t}=x] \\
        &\quad+
        \frac{1}{1-t}\mathbb{E}[\partial_{t}\langle \grad_{x}\log p_{t}(X_{1}\mid x), \Log_{x}(X_{1})\rangle\mid X_{t}=x] \\
        &\quad+
        \frac{1}{1-t}\mathbb{E}[\langle \grad_{x}\log p_{t}(X_{1}\mid x), \Log_{x}(X_{1})\rangle\,\partial_{t}\log p_{t}(X_{1}\mid x)\mid X_{t}=x].
    \end{align*}
We bound 
\begin{align*}
    \mathbb{E}[T_{1}] =& \frac{1}{(1-t)^{2}}
        \mathbb{E}\bigl[\Div_{x}\Log_{x}(X_{1}) + \langle \grad_{x}\log p_{t}(X_{1}\mid x), \Log_{x}(X_{1})\rangle \bigr] \\
        \lesssim & \frac{d}{(1-t)^{2}} \mathbb{E}[d(x, x_1)] + \frac{1}{(1-t)^{2}} \mathbb{E}[d(x, x_1)^{2}]^{\frac{1}{2}} \mathbb{E}[\frac{d^{2}}{(1-t)^{2}} L_{R}^{2}  s_{K_{\min}}(d(x_0, x_1))^{6}]^{\frac{1}{2}} \\
        \lesssim & \frac{d}{1-t} \mathbb{E}[d(x_0, x_1)] + \frac{d}{(1-t)^{2}}L_{R} \mathbb{E}[d(x_0, x_1)^{2}]^{\frac{1}{2}} \mathbb{E}[ s_{K_{\min}}(d(x_0, x_1))^{6}] ]^{\frac{1}{2}}, \\
    \mathbb{E}[T_{2}] =& \frac{1}{1-t}\mathbb{E}[\Div_{x}\Log_{x}(X_{1})\,\partial_{t}\log p_{t}(X_{1}\mid x) ] \\
        \lesssim & \frac{1}{1-t} d \mathbb{E}[d(x, x_1)^{2}]^{\frac{1}{2}}  \frac{d}{1-t}L_{R} \mathbb{E}[d(x_0, x_1)^{2} s_{K_{\min}}(d(x_0, x_1))^{6}]^{\frac{1}{2}} \\
        \lesssim &  \frac{d^{2}}{1-t}L_{R}  \mathbb{E}[d(x_0, x_1)^{2}]^{\frac{1}{2}}  \mathbb{E}[d(x_0, x_1)^{2} s_{K_{\min}}(d(x_0, x_1))^{6}]^{\frac{1}{2}}, \\
      \mathbb{E}[T_{3}] =& \frac{1}{1-t}\mathbb{E}[\partial_{t}\langle \grad_{x}\log p_{t}(X_{1}\mid x), \Log_{x}(X_{1})\rangle ] \\
      \lesssim & \mathbb{E}[d(x_0, x_1)^{2}]^{\frac{1}{2}} \frac{d}{(1-t)^{2}} L_{R}^{3} \mathbb{E}[d(x_1, x_0)^{2} s_{K_{\min}}(d(x_0, x_1))^{12}]^{\frac{1}{2}}, \\
      \mathbb{E}[T_{4}] =& \frac{1}{1-t}\mathbb{E}[\langle \grad_{x}\log p_{t}(X_{1}\mid x), \Log_{x}(X_{1})\rangle\,\partial_{t}\log p_{t}(X_{1}\mid x) ] \\
      \lesssim & \frac{1}{1-t} (\frac{d^{3}}{(1-t)^{3}} L_{R}^{3} \mathbb{E}[ s_{K_{\min}}(d(x_0, x_1))^{9}])^{\frac{1}{3}} 
      \mathbb{E}[d(x, x_1)^3]^{\frac{1}{3}}(\frac{d^{3}}{(1-t)^{3}}L_{R}^{3} \mathbb{E}[d(x_0, x_1)^{3} s_{K_{\min}}(d(x_0, x_1))^{9}])^{\frac{1}{3}} \\
      \lesssim &  \frac{d^{2}}{(1-t)^{2}}L_{R}^{2}\mathbb{E}[ s_{K_{\min}}(d(x_0, x_1))^{9}]^{\frac{1}{3}} 
      \mathbb{E}[d(x_0, x_1)^3]^{\frac{1}{3}} \mathbb{E}[d(x_0, x_1)^{3} s_{K_{\min}}(d(x_0, x_1))^{9}]^{\frac{1}{3}}.
\end{align*}
Higher order bounds follow exactly the same procedure. 

\end{proof}


\subsection{Auxiliary Bounds on Derivatives}
\label{Subsec_Aux_Bounds_Derivative}

Recall that we are using the following notation
\begin{align*}
    \Psi_{t,x_1}(x) = \Exp_{x_1} \Big(\tfrac{1}{1-t}\Log_{x_1}(x)\Big).
\end{align*}
\begin{lemma}[Controlling derivatives of $\log p_0$]
\label{lem:prior_composed_bounds}
Let $M$ be a Hadamard manifold. Fix $\beta>0$, $m = 2$, and $z\in M$, and define
\[
p_0(x) \propto \exp \big(-\beta\,d(x,z)^m\big).
\]
Fix $t\in[0,1)$, $x_1\in M$, and $x\in M$. Define $r:=d(\Psi_{t,x_1}(x),z)$, so that $r$ represent the radial distance between $x_0$ and $z$. Then the following bounds hold:

\begin{align*}
    \Big\|\grad_x \log p_0 \Psi_{t,x_1}(x)\Big\|
        \le &
        2 \beta r
        \Big\|(d\Exp_{x_1})_{\tfrac{1}{1-t}\Log_{x_1}(x)}\Big\|_{\op}\;
        \frac{1}{1-t}\;
        \big\|(d\Log_{x_1})_x\big\|_{\op}, \\
    |\partial_t \log p_0(\Psi_{t,x_1}(x))|
        \le& 2 \beta r
        \Big\|(d\Exp_{x_1})_{\tfrac{1}{1-t}\Log_{x_1}(x)}\Big\|_{\op}\;
        \frac{1}{(1-t)^{2}}\;
        \big\|\Log_{x_1}(x)\big\|, \\
    \|\nabla_x^2 \log p_0(\Psi_{t,x_1}(x))\|_{\op}
        \le &
        \Big(
            \Big\|(d\Exp_{x_1})_{\tfrac{1}{1-t}\Log_{x_1}(x)}\Big\|_{\op}\;
            \frac{1}{1-t}\;
            \big\|(d\Log_{x_1})_x\big\|_{\op}
        \Big)^{2} \\
        &\Big(
            2 \beta 
            +
            2 \beta r
            \|\nabla^{2} r\|_{\op}\Big|_{\Psi_{t,x_1}(x)}
        \Big)
        \\ &+
        \Big\|\nabla(d\Exp_{x_1})_{\tfrac{1}{1-t}\Log_{x_1}(x)}\Big\|\,
        \Big(\frac{1}{1-t}\big\|(d\Log_{x_1})_x\big\|_{\op}\Big)^{2}\,
        2 \beta r \\ &+
        \Big\|(d\Exp_{x_1})_{\tfrac{1}{1-t}\Log_{x_1}(x)}\Big\|_{\op}\,
        \frac{1}{1-t}\,
        \big\|\nabla(d\Log_{x_1})_x\big\|\,
        2 \beta r, \\
    \Big\|\partial_t \grad_x \log p_0(\Psi_{t,x_1}(x))\Big\|
    \le &
    \Big(
        \Big\|(d\Exp_{x_1})_{\tfrac{1}{1-t}\Log_{x_1}(x)}\Big\|_{\op}\;
        \frac{1}{1-t}\;
        \big\|(d\Log_{x_1})_x\big\|_{\op}
    \Big)
    \\
    &\times
    \Big(
        2 \beta 
        +
        2 \beta r
        \|\nabla^{2} r\|_{\op}\Big|_{\Psi_{t,x_1}(x)}
    \Big)
    \\
    &\times
    \Big\|(d\Exp_{x_1})_{\tfrac{1}{1-t}\Log_{x_1}(x)}\Big\|_{\op}\;
    \frac{1}{(1-t)^{2}}\;
    \big\|\Log_{x_1}(x)\big\|
    \\
    &+
    \Bigg[
        \frac{1}{(1-t)^{2}}\Big\|(d\Exp_{x_1})_{\tfrac{1}{1-t}\Log_{x_1}(x)}\Big\|_{\op}\,
        \big\|(d\Log_{x_1})_x\big\|_{\op}
        \\
        +
        \frac{1}{(1-t)^{3}}\, &
        \big\|\Log_{x_1}(x)\big\|\,
        \Big\|\nabla(d\Exp_{x_1})_{\tfrac{1}{1-t}\Log_{x_1}(x)}\Big\|\,
        \big\|(d\Log_{x_1})_x\big\|_{\op}
    \Bigg]
    2 \beta r.
\end{align*}
where 
\begin{align*}
    \|\nabla^{2} r\|_{\op}\Big|_{\Psi_{t,x_1}(x)} \le \frac{s_{K_{\min}}'(r)}{s_{K_{\min}}(r)} = \sqrt{-K_{\min}} \coth(r \sqrt{-K_{\min}}).    
\end{align*}
and note that $r \mapsto r \cosh (r)$ has no singularity at $r = 0$.
\end{lemma}
\begin{proof}
We start with deriving bounds for $\|\grad \log p_0(\Psi_{t,x_1}(x))\|$ and $\|\nabla^2\log p_0(\Psi_{t,x_1}(x))\|_{\op}$. Since $\log p_0(y)=-\beta\,d(y,z)^m+\mathrm{const}$, 
and $\|\grad d(\,\cdot\,,z)\|=1$, we have
\[
\|\grad \log p_0(\Psi_{t,x_1}(x))\|= 2\beta r.
\]
Moreover, notice that 
\[
\nabla^2(d(y,z)^m)=2d(y,z)\nabla^2d(y,z)+2\grad d(y,z) \otimes\grad d(y,z).
\]
The above identity, together with $\|\grad d(y,z)\otimes\grad d(y,z)\|_{\op}=1$, yields
\[
\|\nabla^2\log p_0(\Psi_{t,x_1}(x))\|_{\op}
\le
2\beta
+
2\beta\,r\,\|\nabla^{2} r\|_{\op}\Big|_{\Psi_{t,x_1}(x)}.
\]
\vspace{0.05in}

We next derive the first (spatial) derivative of $\Psi_{t,x_1}$. Using the chain rule,
\[
(d\Psi_{t,x_1})_x[u]
=
(d\Exp_{x_1})_{\tfrac{1}{1-t}\Log_{x_1}(x)}\Big[\tfrac{1}{1-t}\,(d\Log_{x_1})_x[u]\Big],
\qquad \forall\,u\in T_xM.
\]
Hence
\begin{align}\label{Eq_dF_op_bound}
    \|(d\Psi_{t,x_1})_x\|_{\op}
\le
\Big\|(d\Exp_{x_1})_{\tfrac{1}{1-t}\Log_{x_1}(x)}\Big\|_{\op}\,
\tfrac{1}{1-t}\,
\big\|(d\Log_{x_1})_x\big\|_{\op}.
\end{align}

For any $u\in T_xM$, the chain rule gives the identity
\begin{align*}
    \big\langle \grad_x \log p_0(\Psi_{t,x_1}(x)),\,u\big\rangle 
    = \big\langle \grad \log p_0(\Psi_{t,x_1}(x)),\, (d\Psi_{t,x_1})_x[u]\big\rangle.
\end{align*}
Therefore, by Cauchy-Schwarz with supremum over $\|u\|=1$, yields
\begin{align*}
    \big|\langle \grad_x \log p_0(\Psi_{t,x_1}(x)),u\rangle\big|
&\le
\|\grad \log p_0(\Psi_{t,x_1}(x))\|\,\|(d\Psi_{t,x_1})_x[u]\| \\
&\le
\|\grad \log p_0(\Psi_{t,x_1}(x))\|\,\|(d\Psi_{t,x_1})_x\|_{\op}.
\end{align*}
Substituting $\|\grad\log p_0(\Psi_{t,x_1}(x))\|=2\beta r$ 
and using \eqref{Eq_dF_op_bound}, we obtain 
\begin{align*}
    &\Big\|\grad_x \log p_0 \Big(\Exp_{x_1} \big(\tfrac{1}{1-t}\Log_{x_1}(x)\big)\Big)\Big\|\\
\le~~&
    2\beta r\,
    \Big\|(d\Exp_{x_1})_{\tfrac{1}{1-t}\Log_{x_1}(x)}\Big\|_{\op}\;
    \frac{1}{1-t}\;
    \big\|(d\Log_{x_1})_x\big\|_{\op}.
\end{align*}

Now we work on the first time derivative.
Recall $\Psi_{t,x_1}(x)=\Exp_{x_1} \Big(\tfrac{1}{1-t}\Log_{x_1}(x)\Big)$.
By chain rule, 
\begin{align*}
    \partial_t \log p_0(\Psi_{t,x_1}(x))
    =
    \big\langle \grad \log p_0(\Psi_{t,x_1}(x)),\,\partial_t \Psi_{t,x_1}(x)\big\rangle,
\end{align*}
where $\partial_t \Psi_{t,x_1}(x)=(d\Exp_{x_1})_{\tfrac{1}{1-t}\Log_{x_1}(x)}\Big[\tfrac{1}{(1-t)^2}\Log_{x_1}(x)\Big]$.
By Cauchy-Schwarz,
\begin{align*}
    |\partial_t \log p_0(\Psi_{t,x_1}(x))|
    &\le
    \|\grad \log p_0(\Psi_{t,x_1}(x))\|\,\|\partial_t \Psi_{t,x_1}(x)\| \\
    &\le
    \|\grad \log p_0(\Psi_{t,x_1}(x))\|\,
    \Big\|(d\Exp_{x_1})_{\tfrac{1}{1-t}\Log_{x_1}(x)}\Big\|_{\op}\,
    \tfrac{1}{(1-t)^2}\,
    \|\Log_{x_1}(x)\| \\
    &\le 2\beta r
    \Big\|(d\Exp_{x_1})_{\tfrac{1}{1-t}\Log_{x_1}(x)}\Big\|_{\op}\;
    \frac{1}{(1-t)^{2}}\;
    \big\|\Log_{x_1}(x)\big\|,
\end{align*}
where recall $\|\grad\log p_0(\Psi_{t,x_1}(x))\|= 2\beta r$.

Now we compute the Hessian.
For any $u,w\in T_{x}M$, using chain rule, we obtain 
\begin{align*}
    \nabla_x^2 \log p_0(\Psi_{t,x_1}(x))[u,w]
    =&
\nabla^2\log p_0(\Psi_{t,x_1}(x))\big[(\Psi_{t,x_1})_x[u],(\Psi_{t,x_1})_x[w]\big] \\ 
&+ \big\langle \grad\log p_0(\Psi_{t,x_1}(x)),\,\nabla_u\big((\Psi_{t,x_1})_x[w]\big)\big\rangle.
\end{align*}
By Cauchy-Schwarz and taking the supremum over $\|u\|=\|w\|=1$, we obtain
\begin{align*}
    &\|\nabla_x^2 \log p_0(\Psi_{t,x_1}(x))\|_{\op} = \sup_{\|u\|=\|w\|=1} \big|\nabla_x^2 \log p_0(\Psi_{t,x_1}(x))[u,w]\big| \\
    \le &
    \sup_{\|u\|=\|w\|=1}\|\nabla^2\log p_0(\Psi_{t,x_1}(x))\|_{\op}\,\|(\Psi_{t,x_1})_x[u]\|\,\|(\Psi_{t,x_1})_x[w]\| \\
    & \qquad +
    \|\grad\log p_0(\Psi_{t,x_1}(x))\|\,\big\|\nabla_u\big((\Psi_{t,x_1})_x[w]\big)\big\| \\
    \le &
    \|\nabla^2\log p_0(\Psi_{t,x_1}(x))\|_{\op}\,\|(\Psi_{t,x_1})_x\|_{\op}^2
    +
    \|\grad\log p_0(\Psi_{t,x_1}(x))\|\,\|\nabla(\Psi_{t,x_1})_x\|.
\end{align*}

It remains to bound $\|\nabla(\Psi_{t,x_1})_x\|$. Recall that 
\[
(\Psi_{t,x_1})_x[w]=(d\Exp_{x_1})_{\tfrac{1}{1-t}\Log_{x_1}(x)}\Big[\tfrac{1}{1-t}(d\Log_{x_1})_x[w]\Big].
\]
Recall the Leibniz rule with $v = \tfrac{1}{1-t}\Log_{x_1}(x)$ being a $x$-dependent tangent vector at $T_{x_{1}}M$.
\begin{align*}
\nabla_u\big((\Psi_{t,x_1})_x[w]\big)
&=
\nabla_u\Big( (d\Exp_{x_1})_{v}\Big[\tfrac{1}{1-t}(d\Log_{x_1})_x[w]\Big]\Big) \\
&=
\Big(\nabla_{\nabla_u v}(d\Exp_{x_1})\Big)_{v}
\Big[\tfrac{1}{1-t}(d\Log_{x_1})_x[w]\Big]
\;+\;
(d\Exp_{x_1})_{v}\Big[\tfrac{1}{1-t}\,\nabla_u\big((d\Log_{x_1})_x[w]\big)\Big].
\end{align*}
where note that $v$ depends on $x$, so using chain rule, 
\begin{align*}
    \nabla_{u}((d\Exp_{x_1})_{v})
    &= \nabla_{u}((d\Exp_{x_1})_{\tfrac{1}{1-t}\Log_{x_1}(x)}) \\
    &= \bigl(\nabla (d\Exp_{x_1})_{\tfrac{1}{1-t}\Log_{x_1}(x)}\bigr) (\nabla_u \tfrac{1}{1-t}\Log_{x_1}(x))\\
    &= \Big(\nabla_{\nabla_u \tfrac{1}{1-t}\Log_{x_1}(x)}(d\Exp_{x_1})\Big)_{\tfrac{1}{1-t}\Log_{x_1}(x)}.
\end{align*}
Notice that for $v = \tfrac{1}{1-t}\Log_{x_1}(x)$, its first derivative is a directional derivative. 
We have 
\begin{align*}
    \nabla_u v=\tfrac{1}{1-t}(d\Log_{x_1})_x[u],
\qquad
\text{and}\qquad
\nabla_u\big((d\Log_{x_1})_x[w]\big)=\nabla(d\Log_{x_1})_x[u,w].
\end{align*}
Taking norms and applying the definitions of the operator norms gives
\begin{align*}
\big\|\nabla_u\big((\Psi_{t,x_1})_x[w]\big)\big\|
&\le
\Big\|\nabla(d\Exp_{x_1})_{v}\Big\|\;
\big\|\nabla_u v\big\|\;
\Big\|\tfrac{1}{1-t}(d\Log_{x_1})_x[w]\Big\|
\; \\ & \qquad  +\;
\Big\|(d\Exp_{x_1})_{v}\Big\|_{\op}\;
\tfrac{1}{1-t}\;
\big\|\nabla(d\Log_{x_1})_x[u,w]\big\| \\
&=
\Big\|\nabla(d\Exp_{x_1})_{v}\Big\|\;
\Big\|\tfrac{1}{1-t}(d\Log_{x_1})_x[u]\Big\|\;
\Big\|\tfrac{1}{1-t}(d\Log_{x_1})_x[w]\Big\| \\
&\qquad +
\Big\|(d\Exp_{x_1})_{v}\Big\|_{\op}\;
\tfrac{1}{1-t}\;
\big\|\nabla(d\Log_{x_1})_x[u,w]\big\|.
\end{align*}
Taking supremum over $\|u\|=\|w\|=1$, we obtain
\begin{align*}
\big\|\nabla \big((d\Psi_{t,x_1})_x\big)\big\|
\le
\Big\|\nabla(d\Exp_{x_1})_{v}\Big\|\,
\Big(\tfrac{1}{1-t}\big\|(d\Log_{x_1})_x\big\|_{\op}\Big)^{2}
+
\Big\|(d\Exp_{x_1})_{v}\Big\|_{\op}\,
\tfrac{1}{1-t}\,
\big\|\nabla(d\Log_{x_1})_x\big\|.
\end{align*}

Finally using the fact that 
\begin{align*}
    \|\grad\log p_0(\Psi_{t,x_1}(x))\| &= 2\beta r, \\
\quad \|\nabla^2\log p_0(\Psi_{t,x_1}(x))\|_{\op}
&\le
2\beta + 2 \beta r\|\nabla^{2} r\|_{\op}\big|_{\Psi_{t,x_1}(x)},
\end{align*}
and using \eqref{Eq_dF_op_bound}, we have
\begin{align*}
    &\|\nabla_x^2 \log p_0(\Psi_{t,x_1}(x))\|_{\op}\\
    \le & 
    \|\nabla^2\log p_0(\Psi_{t,x_1}(x))\|_{\op}\,\|(\Psi_{t,x_1})_x\|_{\op}^2
    +
    \|\grad\log p_0(\Psi_{t,x_1}(x))\|\,\|\nabla(\Psi_{t,x_1})_x\| \\
\le\;
    &
    \Big(
        \Big\|(d\Exp_{x_1})_{\tfrac{1}{1-t}\Log_{x_1}(x)}\Big\|_{\op}\;
        \frac{1}{1-t}\;
        \big\|(d\Log_{x_1})_x\big\|_{\op}
    \Big)^{2}
    \Big(
        2 \beta 
        +
        2 \beta r
        \|\nabla^{2} r\|_{\op}\Big|_{\Psi_{t,x_1}(x)}
    \Big)
    \\ &+
    \Big\|\nabla(d\Exp_{x_1})_{\tfrac{1}{1-t}\Log_{x_1}(x)}\Big\|\,
    \Big(\frac{1}{1-t}\big\|(d\Log_{x_1})_x\big\|_{\op}\Big)^{2}\,
    2 \beta r \\ &+
    \Big\|(d\Exp_{x_1})_{\tfrac{1}{1-t}\Log_{x_1}(x)}\Big\|_{\op}\,
    \frac{1}{1-t}\,
    \big\|\nabla(d\Log_{x_1})_x\big\|\,
    2 \beta r.
\end{align*}

It remains to consider the mixed derivative term.
For any $u\in T_xM$, recall we have 
\begin{align*}
    \big\langle \grad_x \log p_0(\Psi_{t,x_1}(x)),u\big\rangle = \big\langle \grad\log p_0(\Psi_{t,x_1}(x)), (\Psi_{t,x_1})_x[u]\big\rangle.
\end{align*}
Differentiate both sides in $t$ (with $x_1,x,u$ fixed):
\begin{align*}
    \big\langle \partial_t\grad_x \log p_0(\Psi_{t,x_1}(x)),u\big\rangle
    =&
    \Big\langle \nabla^2\log p_0(\Psi_{t,x_1}(x))\big[\partial_t \Psi_{t,x_1}(x)\big],\, (\Psi_{t,x_1})_x[u]\Big\rangle \\
    &+
    \big\langle \grad\log p_0(\Psi_{t,x_1}(x)),\,\partial_t\big((\Psi_{t,x_1})_x[u]\big)\big\rangle.
\end{align*}
By Cauchy-Schwarz,
\begin{align*}
    \big|\langle \partial_t\grad_x \log p_0(\Psi_{t,x_1}(x)),u\rangle\big|
    &\le
    \|\nabla^2\log p_0(\Psi_{t,x_1}(x))\|_{\op}\,\|\partial_t\Psi_{t,x_1}(x)\|\,\|(\Psi_{t,x_1})_x\|_{\op}\,\|u\| \\
    & \quad +
    \|\grad\log p_0(\Psi_{t,x_1}(x))\|\,\Big\|\partial_t\big((\Psi_{t,x_1})_x[u]\big)\Big\|.
\end{align*}
We already have $\|\partial_t\Psi_{t,x_1}(x)\|$ and $\|(\Psi_{t,x_1})_x\|_{\op}$ from previous steps.

It remains to bound $\|\partial_t((\Psi_{t,x_1})_x[u])\|$.
Recall for any $u\in T_xM$,
\[
(\Psi_{t,x_1})_x[u]=(d\Exp_{x_1})_{v}\Big[\tfrac{1}{1-t}(d\Log_{x_1})_x[u]\Big],
\qquad
v=\tfrac{1}{1-t}\Log_{x_1}(x).
\]
Fix $x_1,x,u$ and differentiate with respect to $t$.
Since $x$ is fixed, the tensor $(d\Log_{x_1})_x[u]$ does not depend on $t$; only the scalar factor
$\tfrac{1}{1-t}$ and the point $v$ depend on $t$.
By the Leibniz rule (product rule for a $t$-dependent linear map applied to a $t$-dependent vector),
\begin{align*}
\partial_t\big((\Psi_{t,x_1})_x[u]\big)
&=
\partial_t\Big((d\Exp_{x_1})_{v}\Big)\Big[\tfrac{1}{1-t}(d\Log_{x_1})_x[u]\Big]
+
(d\Exp_{x_1})_{v}\Big[\partial_t\Big(\tfrac{1}{1-t}(d\Log_{x_1})_x[u]\Big)\Big].
\end{align*}
For the first term, $(d\Exp_{x_1})_{v}$ depends on $t$ only through $v(t)$, hence the chain rule gives
\begin{align*}
    \partial_t\Big((d\Exp_{x_1})_{v}\Big)
    = \Big( \nabla (d\Exp_{x_1})_{v}\Big) (\partial_{t} v)
    = \Big(\nabla_{\partial_t v}(d\Exp_{x_1})\Big)_{v}.
\end{align*}
For the second term, since $(d\Log_{x_1})_x[u]$ is $t$-independent,
\[
\partial_t\Big(\tfrac{1}{1-t}(d\Log_{x_1})_x[u]\Big)
=
\tfrac{1}{(1-t)^2}(d\Log_{x_1})_x[u].
\]
Combining these identities yields the formula
\begin{align*}
\partial_t\big((\Psi_{t,x_1})_x[u]\big)
&=
\Big(\nabla_{\partial_t v}(d\Exp_{x_1})\Big)_{v}
\Big[\tfrac{1}{1-t}(d\Log_{x_1})_x[u]\Big]
+
(d\Exp_{x_1})_{v}\Big[\tfrac{1}{(1-t)^2}(d\Log_{x_1})_x[u]\Big].
\end{align*}
Noting $\partial_t v=\tfrac{1}{(1-t)^2}\Log_{x_1}(x)$, 
we take norms and obtain (for $\|u\|=1$)
\begin{align*}
&\Big\|\partial_t\big((\Psi_{t,x_1})_x[u]\big)\Big\|\\
\le &
\Big\|\nabla(d\Exp_{x_1})_{v}\Big\|\;\|\partial_t v\|\;
\Big\|\tfrac{1}{1-t}(d\Log_{x_1})_x[u]\Big\| \\
& \qquad +
\Big\|(d\Exp_{x_1})_{v}\Big\|_{\op}\;
\tfrac{1}{(1-t)^2}\;
\big\|(d\Log_{x_1})_x\big\|_{\op}\,\|u\| \\
= &
\Big\|\nabla(d\Exp_{x_1})_{v}\Big\|\;
\tfrac{1}{(1-t)^2}\|\Log_{x_1}(x)\|\;
\tfrac{1}{1-t}\big\|(d\Log_{x_1})_x\big\|_{\op}\, \\
& \qquad +
\Big\|(d\Exp_{x_1})_{v}\Big\|_{\op}\;
\tfrac{1}{(1-t)^2}\;
\big\|(d\Log_{x_1})_x\big\|_{\op}\, \\
= & 
\tfrac{1}{(1-t)^3}\|\Log_{x_1}(x)\|\,
\Big\|\nabla(d\Exp_{x_1})_{v}\Big\|\,
\big\|(d\Log_{x_1})_x\big\|_{\op} \\
& \qquad + \tfrac{1}{(1-t)^2}\Big\|(d\Exp_{x_1})_{v}\Big\|_{\op}\,\big\|(d\Log_{x_1})_x\big\|_{\op}.
\end{align*}

Substituting the expressions for
$\|\grad\log p_0(\Psi_{t,x_1}(x))\|$ and $\|\nabla^2\log p_0(\Psi_{t,x_1}(x))\|_{\op}$, 
we get 
\begin{align*}
   & \Big\|\partial_t \grad_x \log p_0 \Big(\Exp_{x_1} \big(\tfrac{1}{1-t}\Log_{x_1}(x)\big)\Big)\Big\|\\
    \le\;&
    \Big(
        \Big\|(d\Exp_{x_1})_{\tfrac{1}{1-t}\Log_{x_1}(x)}\Big\|_{\op}\;
        \frac{1}{1-t}\;
        \big\|(d\Log_{x_1})_x\big\|_{\op}
    \Big)
    \\
    &\times
    \Big(
        2 \beta 
        +
        2 \beta r
        \|\nabla^{2} r\|_{\op}\Big|_{\Psi_{t,x_1}(x)}
    \Big)
    \\
    &\times
    \Big\|(d\Exp_{x_1})_{\tfrac{1}{1-t}\Log_{x_1}(x)}\Big\|_{\op}\;
    \frac{1}{(1-t)^{2}}\;
    \big\|\Log_{x_1}(x)\big\|
    \\
    &+
    \Bigg[
        \frac{1}{(1-t)^{2}}\Big\|(d\Exp_{x_1})_{\tfrac{1}{1-t}\Log_{x_1}(x)}\Big\|_{\op}\,
        \big\|(d\Log_{x_1})_x\big\|_{\op}
        \\
        &
        +
        \frac{1}{(1-t)^{3}}\,
        \big\|\Log_{x_1}(x)\big\|\,
        \Big\|\nabla(d\Exp_{x_1})_{\tfrac{1}{1-t}\Log_{x_1}(x)}\Big\|\,
        \big\|(d\Log_{x_1})_x\big\|_{\op}
    \Bigg]
    2 \beta r.
\end{align*}

Finally, we remark that when $K_{\min}<0$, the Hessian comparison theorem gives 
\begin{align*}
    \|\nabla^{2} r\|_{\op}\Big|_{\Psi_{t,x_1}(x)} \le \frac{s_{K_{\min}}'(r)}{s_{K_{\min}}(r)} = \sqrt{-K_{\min}} \coth(r \sqrt{-K_{\min}}).    
\end{align*}

\end{proof}

\begin{lemma}[Controlling derivatives of $J_t$-terms]\label{lem:Jt-pointwise-no-shorthand}
Assume $M$ is a Hadamard manifold. Fix $x_1\in M$, $x\in M$, and $t\in(0,1)$.
Recall
\begin{align*}
J_t(x\mid x_1)
&=
(1-t)^{-d}\,
\frac{\big|\det(d\Exp_{x_1})_{\frac{1}{1-t}\Log_{x_1}(x)}\big|}
     {\big|\det(d\Exp_{x_1})_{\Log_{x_1}(x)}\big|}, \\
\log J_t(x\mid x_1)
&= -d \log (1-t)
 + \log \det(d\Exp_{x_1})_{\frac{1}{1-t}\Log_{x_1}(x)}
 - \log \det(d\Exp_{x_1})_{\Log_{x_1}(x)}.
\end{align*}
Then the following pointwise bounds hold.
\begin{align*}
    \|\grad_x \log J_t(x\mid x_1)\| \le & d\,\|(d\Log_{x_1})_x\| \Bigg(\frac{1}{1-t}\,
            \|\nabla(d\Exp_{x_1})_{\frac{1}{1-t}\Log_{x_1}(x)}\| +
            \|\nabla(d\Exp_{x_1})_{\Log_{x_1}(x)}\|
            \Bigg), \\
    \|\nabla_x^2 \log J_t(x\mid x_1)\|_{\op} \le&d\,\|(d\Log_{x_1})_x\|^2
            \Bigg[
            \frac{1}{(1-t)^2}\Big(
            \|\nabla(d\Exp_{x_1})_{\frac{1}{1-t}\Log_{x_1}(x)}\|^2 \\ & + 
            \|\nabla^2(d\Exp_{x_1})_{\frac{1}{1-t}\Log_{x_1}(x)}\|
            \Big)  \\
            &+\Big(
            \|\nabla(d\Exp_{x_1})_{\Log_{x_1}(x)}\|^2 + \|\nabla^2(d\Exp_{x_1})_{\Log_{x_1}(x)}\|\Big)\Bigg]  \\
            &+d\,\|\nabla(d\Log_{x_1})_x\|\Bigg(
            \frac{1}{1-t}\,
            \|\nabla(d\Exp_{x_1})_{\frac{1}{1-t}\Log_{x_1}(x)}\| +
            \|\nabla(d\Exp_{x_1})_{\Log_{x_1}(x)}\|
            \Bigg), \\
    |\partial_t \log J_t(x\mid x_1)| \le & \frac{d}{1-t}
            +\frac{d}{(1-t)^2}\,\|\Log_{x_1}(x)\|\; \|\nabla(d\Exp_{x_1})_{\frac{1}{1-t}\Log_{x_1}(x)}\|, \\
    \|\partial_t \grad_x \log J_t(x\mid x_1)\| \le &
            \|(d\Log_{x_1})_x\|\Bigg[
            \frac{d}{(1-t)^2}
            \|\nabla(d\Exp_{x_1})_{\frac{1}{1-t}\Log_{x_1}(x)}\|  \\
            &+\frac{d}{(1-t)^3} \,\|\Log_{x_1}(x)\| \\ & \qquad \qquad \Big(
            \|\nabla(d\Exp_{x_1})_{\frac{1}{1-t}\Log_{x_1}(x)}\|^2 +
            \|\nabla^2(d\Exp_{x_1})_{\frac{1}{1-t}\Log_{x_1}(x)}\|
            \Big)
            \Bigg]. 
\end{align*}
\end{lemma}

\begin{proof}
We repeatedly use Lemma~\ref{lem:logdet-no-shorthand} with $A(\cdot)$ chosen to be
$\xi\mapsto (d\Exp_{x_1})_{\xi}$, and the chain rule through $\Log_{x_1}$.

We first estimate the gradient.
From
\[
\log J_t(x\mid x_1)=-d\log(1-t)
+\log\det(d\Exp_{x_1})_{\frac{1}{1-t}\Log_{x_1}(x)}
-\log\det(d\Exp_{x_1})_{\Log_{x_1}(x)},
\]
the constant $-d\log(1-t)$ has zero $x$-gradient, hence
\[
\grad_x \log J_t(x\mid x_1)
=
\grad_x \log\det(d\Exp_{x_1})_{\frac{1}{1-t}\Log_{x_1}(x)}
-\grad_x \log\det(d\Exp_{x_1})_{\Log_{x_1}(x)}.
\]
Consider the first term.
Differentiating the map $x\mapsto \frac{1}{1-t}\Log_{x_1}(x)$ yields a factor $\frac{1}{1-t}(d\Log_{x_1})_x$.
Thus,
\begin{align*}
&\big\|\grad_x \log\det(d\Exp_{x_1})_{\frac{1}{1-t}\Log_{x_1}(x)}\big\|\\
\le & \frac{1}{1-t}\,\|(d\Log_{x_1})_x\|\;
\Big\|\grad_{\xi}\log\det(d\Exp_{x_1})_{\xi}\Big\|\Bigg|_{\xi=\frac{1}{1-t}\Log_{x_1}(x)}.
\end{align*}
Now apply Lemma~\ref{lem:logdet-no-shorthand} (specifically, the gradient bound in \eqref{eq:grad-logdet-basic}) at
$A(\xi)=(d\Exp_{x_1})_{\xi}$ to get
\[
\Big\|\grad_{\xi}\log\det(d\Exp_{x_1})_{\xi}\Big\|
\le d\,\|(d\Exp_{x_1})_{\xi}^{-1}\|\;\|\nabla(d\Exp_{x_1})_{\xi}\|.
\]
Combining the above, we obtain 
\begin{align*}
&\big\|\grad_x \log\det(d\Exp_{x_1})_{\frac{1}{1-t}\Log_{x_1}(x)}\big\|\\
\le & \frac{d}{1-t}\,\|(d\Log_{x_1})_x\|\,
\|(d\Exp_{x_1})_{\frac{1}{1-t}\Log_{x_1}(x)}^{-1}\|\;
\|\nabla(d\Exp_{x_1})_{\frac{1}{1-t}\Log_{x_1}(x)}\|.
\end{align*}
The same argument (without the factor $\frac{1}{1-t}$) yields
\[
\big\|\grad_x \log\det(d\Exp_{x_1})_{\Log_{x_1}(x)}\big\|
\le
d\,\|(d\Log_{x_1})_x\|\,
\|(d\Exp_{x_1})_{\Log_{x_1}(x)}^{-1}\|\;
\|\nabla(d\Exp_{x_1})_{\Log_{x_1}(x)}\|.
\]
Finally, apply the triangle inequality, 
we get \begin{align*}
&\|\grad_x \log J_t(x\mid x_1)\|
\le
d\,\|(d\Log_{x_1})_x\| \Bigg(
\frac{1}{1-t}\,
\|(d\Exp_{x_1})_{\frac{1}{1-t}\Log_{x_1}(x)}^{-1}\|\;
\|\nabla(d\Exp_{x_1})_{\frac{1}{1-t}\Log_{x_1}(x)}\|  \\
&\hspace{5.65cm}
+\|(d\Exp_{x_1})_{\Log_{x_1}(x)}^{-1}\|\;
\|\nabla(d\Exp_{x_1})_{\Log_{x_1}(x)}\|
\Bigg). 
\end{align*}

Now we estimate the Hessian.
Write $f(x):=\log\det(d\Exp_{x_1})_{\frac{1}{1-t}\Log_{x_1}(x)}$.
Taking second derivative (with product rule) we have 
\begin{align*}
\|\nabla_x^2 f(x)\|_{\op}
\le&
\Big\|\nabla_{\xi}^2\log\det(d\Exp_{x_1})_{\xi}\Big\|_{\op}\Bigg|_{\xi=\frac{1}{1-t}\Log_{x_1}(x)}\cdot
\Big\|\frac{1}{1-t}(d\Log_{x_1})_x\Big\|^2 \\
&+
\Big\|\grad_{\xi}\log\det(d\Exp_{x_1})_{\xi}\Big\|\Bigg|_{\xi=\frac{1}{1-t}\Log_{x_1}(x)}\cdot
\Big\|\nabla\Big(\frac{1}{1-t}(d\Log_{x_1})_x\Big)\Big\|.
\end{align*}
Thus we have (here $\nabla_{\xi}^{2}$ represent taking derivative w.r.t. $\xi$)
\begin{align*}
&\|\nabla_x^2 \log J_t(x\mid x_1)\|_{\op}\\
\le&
\|(d\Log_{x_1})_x\|^2
\Bigg(
\frac{1}{(1-t)^2}\,
\Big\|\nabla_{\xi}^2 \log\det(d\Exp_{x_1})_{\xi}\Big\|_{\op}\Bigg|_{\xi=\frac{1}{1-t}\Log_{x_1}(x)}
+
\Big\|\nabla_{\xi}^2 \log\det(d\Exp_{x_1})_{\xi}\Big\|_{\op}\Bigg|_{\xi=\Log_{x_1}(x)}
\Bigg)  \\
&
+\|\nabla(d\Log_{x_1})_x\|\,
\Bigg(
\frac{1}{1-t}\,
\Big\|\grad_{\xi}\log\det(d\Exp_{x_1})_{\xi}\Big\|\Bigg|_{\xi=\frac{1}{1-t}\Log_{x_1}(x)}
+
\Big\|\grad_{\xi}\log\det(d\Exp_{x_1})_{\xi}\Big\|\Bigg|_{\xi=\Log_{x_1}(x)}
\Bigg)  \\
\le&
d\,\|(d\Log_{x_1})_x\|^2
\Bigg[
\frac{1}{(1-t)^2}\Big(
\|(d\Exp_{x_1})_{\frac{1}{1-t}\Log_{x_1}(x)}^{-1}\|^2\;
\|\nabla(d\Exp_{x_1})_{\frac{1}{1-t}\Log_{x_1}(x)}\|^2 \\
&+
\|(d\Exp_{x_1})_{\frac{1}{1-t}\Log_{x_1}(x)}^{-1}\|\;
\|\nabla^2(d\Exp_{x_1})_{\frac{1}{1-t}\Log_{x_1}(x)}\|
\Big)  \\
&
+\Big(
\|(d\Exp_{x_1})_{\Log_{x_1}(x)}^{-1}\|^2\;
\|\nabla(d\Exp_{x_1})_{\Log_{x_1}(x)}\|^2
+
\|(d\Exp_{x_1})_{\Log_{x_1}(x)}^{-1}\|\;
\|\nabla^2(d\Exp_{x_1})_{\Log_{x_1}(x)}\|
\Big)\Bigg]  \\
&\quad
+d\,\|\nabla(d\Log_{x_1})_x\|
\Bigg(
\frac{1}{1-t}\,
\|(d\Exp_{x_1})_{\frac{1}{1-t}\Log_{x_1}(x)}^{-1}\|\;
\|\nabla(d\Exp_{x_1})_{\frac{1}{1-t}\Log_{x_1}(x)}\| \\
&+
\|(d\Exp_{x_1})_{\Log_{x_1}(x)}^{-1}\|\;
\|\nabla(d\Exp_{x_1})_{\Log_{x_1}(x)}\|
\Bigg),
\end{align*}
where in the first inequality we used triangle inequality, 
and in the second inequality we applied Lemma \ref{lem:logdet-no-shorthand}.

For time derivative, since $x$ is fixed and only $\frac{1}{1-t}$ depends on $t$,
\[
\partial_t \log J_t(x\mid x_1)
=
\frac{d}{1-t}
+\partial_t\Big(\log\det(d\Exp_{x_1})_{\frac{1}{1-t}\Log_{x_1}(x)}\Big).
\]
By the chain rule,
\[
\partial_t\Big(\log\det(d\Exp_{x_1})_{\frac{1}{1-t}\Log_{x_1}(x)}\Big)
=
D\log\det(d\Exp_{x_1})_{\xi}\Big|_{\xi=\frac{1}{1-t}\Log_{x_1}(x)}
\Big[\partial_t\Big(\frac{1}{1-t}\Log_{x_1}(x)\Big)\Big].
\]
But $\partial_t\big(\frac{1}{1-t}\Log_{x_1}(x)\big)=\frac{1}{(1-t)^2}\Log_{x_1}(x)$, hence
\[
\Big|\partial_t\Big(\log\det(d\Exp_{x_1})_{\frac{1}{1-t}\Log_{x_1}(x)}\Big)\Big|
\le
\frac{1}{(1-t)^2}\,\|\Log_{x_1}(x)\|\;
\Big\|\grad_{\xi}\log\det(d\Exp_{x_1})_{\xi}\Big\|\Bigg|_{\xi=\frac{1}{1-t}\Log_{x_1}(x)}.
\]
Apply Lemma~\ref{lem:logdet-no-shorthand} (gradient bound \eqref{eq:grad-logdet-basic}) at 
$\xi=\frac{1}{1-t}\Log_{x_1}(x)$, we obtain 
\begin{align*}
|\partial_t \log J_t(x\mid x_1)|
&\le \frac{d}{1-t}
+\frac{d}{(1-t)^2}\,\|\Log_{x_1}(x)\|\; \|(d\Exp_{x_1})_{\frac{1}{1-t}\Log_{x_1}(x)}^{-1}\|\;
\|\nabla(d\Exp_{x_1})_{\frac{1}{1-t}\Log_{x_1}(x)}\|. 
\end{align*}

For mixed derivative, recall
\[
\grad_x \log J_t(x\mid x_1)
=
\grad_x \log\det(d\Exp_{x_1})_{\frac{1}{1-t}\Log_{x_1}(x)}
-\grad_x \log\det(d\Exp_{x_1})_{\Log_{x_1}(x)}.
\]
Only the first term depends on $t$.
Differentiate the first term: there are two contributions,
one from differentiating the prefactor $\frac{1}{1-t}$ and one from differentiating
$\grad_{\xi}\log\det(d\Exp_{x_1})_{\xi}$ at $\xi=\frac{1}{1-t}\Log_{x_1}(x)$.
This yields
\begin{align*}
&\|\partial_t \grad_x \log\det(d\Exp_{x_1})_{\frac{1}{1-t}\Log_{x_1}(x)}\|\\ 
\le &
\|(d\Log_{x_1})_x\|\Bigg[
\frac{1}{(1-t)^2}\,
\Big\|\grad_{\xi}\log\det(d\Exp_{x_1})_{\xi}\Big\|\Bigg|_{\xi=\frac{1}{1-t}\Log_{x_1}(x)} \\
&+\frac{1}{1-t}\,
\Big\|\partial_t\Big(\grad_{\xi}\log\det(d\Exp_{x_1})_{\xi}\Big)\Big\|\Bigg|_{\xi=\frac{1}{1-t}\Log_{x_1}(x)}
\Bigg].
\end{align*}

Next, note that 
\begin{align*}
\partial_t(\grad_{\xi}\log\det(d\Exp_{x_1})_{\xi})=
\nabla_{\xi}^2\log\det(d\Exp_{x_1})_{\xi}\,[\partial_t\xi]
\end{align*}
with
$\partial_t\xi=\frac{1}{(1-t)^2}\Log_{x_1}(x)$. Hence
\begin{align*}
    \Big\|\partial_t\Big(\grad_{\xi}\log\det(d\Exp_{x_1})_{\xi}\Big)\Big\|
\le
\frac{1}{(1-t)^2}\,\|\Log_{x_1}(x)\|\;
\Big\|\nabla_{\xi}^2\log\det(d\Exp_{x_1})_{\xi}\Big\|_{\op}.
\end{align*}
By Lemma \ref{lem:logdet-no-shorthand}, 
we obtain  
\begin{align*}
&\|\partial_t \grad_x \log J_t(x\mid x_1)\|\\
\le &
\|(d\Log_{x_1})_x\|\Bigg[
\frac{1}{(1-t)^2}\,
d\,\|(d\Exp_{x_1})_{\frac{1}{1-t}\Log_{x_1}(x)}^{-1}\|\;
\|\nabla(d\Exp_{x_1})_{\frac{1}{1-t}\Log_{x_1}(x)}\|  \\
&
+\frac{1}{(1-t)^3}\,\|\Log_{x_1}(x)\|\,
d\Big(
\|(d\Exp_{x_1})_{\frac{1}{1-t}\Log_{x_1}(x)}^{-1}\|^2\;
\|\nabla(d\Exp_{x_1})_{\frac{1}{1-t}\Log_{x_1}(x)}\|^2 \\
&+
\|(d\Exp_{x_1})_{\frac{1}{1-t}\Log_{x_1}(x)}^{-1}\|\;
\|\nabla^2(d\Exp_{x_1})_{\frac{1}{1-t}\Log_{x_1}(x)}\|
\Big)
\Bigg]. 
\end{align*}

Finally, by \citet[Theorem 3.12]{lezcano2020curvature} with $r = \frac{\|\Log_{x_1}(x)\|}{1-t}$, 
    and noting that on $\SPD(n)$ sectional curvature is upper bounded by zero, we have 
    \begin{align*}
        \|(d\Exp_{x_1})_{\frac{1}{1-t}\Log_{x_1}(x)}^{-1}\| \le 1.
    \end{align*}
Hence we get 
\begin{align*}
    \|\grad_x \log J_t(x\mid x_1)\| \le & d\,\|(d\Log_{x_1})_x\| \Bigg(\frac{1}{1-t}\,
            \|\nabla(d\Exp_{x_1})_{\frac{1}{1-t}\Log_{x_1}(x)}\| +
            \|\nabla(d\Exp_{x_1})_{\Log_{x_1}(x)}\|
            \Bigg), \\
    \|\nabla_x^2 \log J_t(x\mid x_1)\|_{\op} \le&d\,\|(d\Log_{x_1})_x\|^2 \\ & \qquad 
            \Bigg[
            \frac{1}{(1-t)^2}\Big(
            \|\nabla(d\Exp_{x_1})_{\frac{1}{1-t}\Log_{x_1}(x)}\|^2 +
            \|\nabla^2(d\Exp_{x_1})_{\frac{1}{1-t}\Log_{x_1}(x)}\|
            \Big)  \\
            & \qquad +\Big(
            \|\nabla(d\Exp_{x_1})_{\Log_{x_1}(x)}\|^2 + \|\nabla^2(d\Exp_{x_1})_{\Log_{x_1}(x)}\|\Big)\Bigg]  \\
            &+d\,\|\nabla(d\Log_{x_1})_x\|\Bigg(
            \frac{1}{1-t}\,
            \|\nabla(d\Exp_{x_1})_{\frac{1}{1-t}\Log_{x_1}(x)}\| +
            \|\nabla(d\Exp_{x_1})_{\Log_{x_1}(x)}\|
            \Bigg), \\
    |\partial_t \log J_t(x\mid x_1)| \le & \frac{d}{1-t}
            +\frac{d}{(1-t)^2}\,\|\Log_{x_1}(x)\|\; \|\nabla(d\Exp_{x_1})_{\frac{1}{1-t}\Log_{x_1}(x)}\|, \\
    \|\partial_t \grad_x \log J_t(x\mid x_1)\| \le &
            \|(d\Log_{x_1})_x\|\Bigg[
            \frac{d}{(1-t)^2}
            \|\nabla(d\Exp_{x_1})_{\frac{1}{1-t}\Log_{x_1}(x)}\|  \\
            &+\frac{d}{(1-t)^3} \,\|\Log_{x_1}(x)\| \\ & \qquad \qquad \Big(
            \|\nabla(d\Exp_{x_1})_{\frac{1}{1-t}\Log_{x_1}(x)}\|^2 +
            \|\nabla^2(d\Exp_{x_1})_{\frac{1}{1-t}\Log_{x_1}(x)}\|
            \Big)
            \Bigg], 
\end{align*}
completing the proof.
\end{proof}

\subsection[Bounding Third Derivative of Exp Map]{Bounding $\bigl\| \bigl(\nabla^2 d\Exp_p\bigr)_{rv} \bigr\|_{\op}$}
\label{Subsec_Third_Derivative_Exp}

In this section, we bound the third derivative of $\Exp$. We first recall some notation in \cite{lezcano2020curvature}, which allows us to write the third derivative as some ODE. Then we apply some comparison theorem to derive the upper bound. 
We remark that on a symmetric space, the (covariant) derivative of the curvature tensor is identically zero, see for example \cite[Theorem 10.19]{lee2018introduction}. The techniques in this section works for non-symmetric spaces, but as we are working on $\SPD(n)$, we assume $M$ is a symmetric space, which would simplify our computation. 

Recall that $\nabla$ denotes the (Levi--Civita) connection on $TM$ over $M$.
Let $\nabla^{\text{flat}}$ be the flat connection on the vector space $T_pM$.
Set $\Exp_p:U\subset T_pM\to M$ on a normal neighborhood of $p$.

We first introduce pullback connection.
Instead of differentiating along a vector field on $M$, we can also differentiate along a vector field $X$ on $T_{p}M$. 
Define the corresponding connection as 
\begin{align*}
\bigl(\nabla^{\text{pullback}}_{X}Y\bigr)(u)
=
\nabla_{\,(d\Exp_p)_u(X(u))}\,Y,
\end{align*}
In other words, given $u \in T_{p}M$, and we have $q := \Exp_{p}(u)$. 
Then 
differentiate along $X$ in the tangent space via the pullback connection, evaluated at $u \in T_{p}M$, 
is equivalent to 
differentiate along $d\Exp_p(X)$ (which can be viewed as a vector field on $M$) evaluated at $q$.

Recall the pullback bundle $\Exp_p^{*}(TM)$ is a vector bundle over the base $T_pM$.
Concretely, 
\begin{align*}
    \Exp_p^{*}(TM) 
    := \{(u,v): u\in T_pM, v \in TM \text{ with base point } \Exp_p(u)\}.
\end{align*} 

For example, given an element of the tensor product space $(\alpha, (u, v)) \in T_{p}^{*}M \otimes \Exp_p^{*}(TM)$, we have 
$(\alpha, (u, v))(w) = \alpha(w) v$, where $\alpha \in T_{p}^{*}M$ is a covector, so that $\alpha(w) \in \mathbb{R}$.

Let $\{e_{i}\}$ be an orthonormal basis for $T_{p}M$, 
and $\{e^{i}\}$ be the corresponding basis in the cotangent space $T_{p}^{*}M$.
By definition, 
\begin{align*}
  \big(\sum_{i}e^{i} \otimes (u, (d\Exp_{p})_{u}[e_{i}]) \big) (w) 
  := \sum_{i}e^{i}(w) (d\Exp_{p})_{u}[e_{i}],
\end{align*}
and by linearity
\begin{align*}
  \sum_{i}e^{i}(w) (d\Exp_{p})_{u}[e_{i}] = (d\Exp_{p})_{u}[\sum_{i}e^{i}(w) e_{i}]
  = (d\Exp_{p})_{u}(w).
\end{align*}
Thus we see that $d\Exp_{p}$ can be viewed as an element of the tensor product space:
\begin{align*}
  (d\Exp_{p})_{u} \sum_{i}e^{i} \otimes (u, (d\Exp_{p})_{u}[e_{i}]) \in T_{p}^{*}M \otimes \Exp_p^{*}(TM).
\end{align*}

We equip $T_p^*M$ with the connection induced by $\nabla^{\text{flat}}$ and $\Exp_p^*(TM)$
with $\nabla^{\text{pullback}}$. These induce a connection on the tensor product bundle, denoted by
$\nabla^{\text{induced}}$, characterized by the Leibniz rule.
Now consider $\overline W_1, \overline W_2$ being vector fields in the tangent space $T_{p}M$, 
and $d\Exp_p(\overline W_2)$ can be viewed as a section, i.e., 
$d\Exp_p(\overline W_2): u \mapsto (d\Exp_p)_{u}(\overline W_2(u))$.
We differentiate 
\begin{align*}
 &\nabla^{\text{pullback}}_{\overline W_1} (d\Exp_p(\overline W_2)) 
= \nabla^{\text{pullback}}_{\overline W_1} (\sum_{i}e^{i}(\overline{W_2})(d\Exp_{p})[e_{i}]) \\
=& \sum_{i} \overline{W}_1 e^{i}(\overline{W_2}) (d\Exp_{p})[e_{i}] + \sum_{i}e^{i}(\overline{W_2}) \nabla^{\text{pullback}}_{\overline W_1} (d\Exp_{p})[e_{i}].
\end{align*}
where we emphasize that $(d\Exp_{p})[e_{i}]$ is a function of $u$, similar as before: $(d\Exp_{p})[e_{i}]: u \mapsto (d\Exp_{p})_{u}[e_{i}]$.

By linearity of $(d\Exp_{p})$, and definition of flat connection, we have
\begin{align*}
  \sum_{i} \overline{W}_1 e^{i}(\overline{W_2}) (d\Exp_{p})[e_{i}]
  = (d\Exp_{p})[\sum_{i} \overline{W}_1 e^{i}(\overline{W_2})e_{i}]
  = (d\Exp_{p})[\nabla^{\text{flat}}_{\overline W_1} \overline{W}_2].
\end{align*}
On the other hand, since $e^{i}$ are constants, we know $\nabla^{\text{flat}}_{\overline W_1} e^{i} = 0$, so that 
\begin{align*}
  e^{i} \otimes \nabla^{\text{pullback}}_{\overline W_1} (d\Exp_{p})[e_{i}] 
  =& e^{i} \otimes \nabla^{\text{pullback}}_{\overline W_1} (d\Exp_{p})[e_{i}] + \nabla^{\text{flat}}_{\overline W_1} e^{i} \otimes (d\Exp_{p})[e_{i}] \\
  =& \nabla^{\text{induced}}_{\overline W_1} (e^{i} \otimes (d\Exp_{p})[e_{i}]),
\end{align*}
hence 
\begin{align*}
&\sum_{i} e^i(\overline W_2) \otimes \nabla^{\text{pullback}}_{\overline W_1}((d\Exp_p)[e_i])
=\Big(\sum_{i} e^i \otimes \nabla^{\text{pullback}}_{\overline W_1}((d\Exp_p)[e_i])\Big) (\overline W_2)\\
=& \Big(\sum_{i} \nabla^{\text{induced}}_{\overline W_1}(e^i\otimes (d\Exp_p)[e_i])\Big) (\overline W_2)
= \sum_{i} \Big(\nabla^{\text{induced}}_{\overline W_1}( e^i\otimes (d\Exp_p)[e_i])\Big) (\overline W_2)\\
=& (\nabla^{\text{induced}}_{\overline W_1}(\sum_{i} e^i\otimes (d\Exp_p)[e_i])) (\overline W_2)
=(\nabla^{\text{induced}}_{\overline W_1} d\Exp_p)(\overline W_2).
\end{align*}
Substitute into the previous expression, we obtain 
\begin{align*}
 &\nabla^{\text{pullback}}_{\overline W_1} (d\Exp_p(\overline W_2)) \\
=& \sum_{i} \overline{W}_1 e^{i}(\overline{W_2}) (d\Exp_{p})[e_{i}] + \sum_{i}e^{i}(\overline{W_2}) \nabla^{\text{pullback}}_{\overline W_1} (d\Exp_{p})[e_{i}] \\
=& (d\Exp_{p})[\nabla^{\text{flat}}_{\overline W_1} \overline{W}_2]
+ \nabla^{\text{induced}}_{\overline W_1} (d\Exp_{p}) (\overline{W}_2)
\end{align*}

Thus, for vector fields $\overline W_1,\overline W_2$ on $U\subset T_pM$ we have
\begin{equation}\label{eq:Leibniz-dexp-pullback}
\nabla^{\text{pullback}}_{\overline W_1} \bigl(d\Exp_p(\overline W_2)\bigr)
=
\bigl(\nabla^{\text{induced}}_{\overline W_1}d\Exp_p\bigr)(\overline W_2)
+
d\Exp_p \bigl(\nabla^{\text{flat}}_{\overline W_1}\overline W_2\bigr).
\end{equation}

Define $c(t, s_{1}, s_{2}, s_{3}) = \Exp_{p}(t(v + s_{1} w_{1} + s_{2} w_{2} + s_{3} w_{3})) $, and $\gamma(t) = \Exp_{p}(tv) $. We briefly recall what \cite{lezcano2020curvature} did, to control the second derivative of $\Exp$ map.
Define $J_{1}(t) = (d\Exp_p)_{tv}(t w_1)$ to be the Jacobi field along $\gamma$ with initial condition $w_{1}$.
We can define the extension of $J_1$ in the $w_2$-direction by $\widetilde J_1(t,s):=(d\Exp_p)_{t(v+s w_2)}(t w_1)$, where note that $\widetilde J_1(t,0)=J_1(t)$.
For fixed $t$, the tangent space vector field $u\mapsto tw_1$ is constant on $T_pM$, hence
\begin{equation}\label{eq:flat-zero-tw}
\nabla^{\text{flat}}_{tw_2}(tw_1)=0.
\end{equation}
Applying \eqref{eq:Leibniz-dexp-pullback} with $\overline W_1\equiv tw_2$ and $\overline W_2\equiv tw_1$, and evaluated at $t(v+s w_2)$, we have 
\[
\bigl(\nabla^{\text{pullback}}_{tw_2}\widetilde J_1(t,s)\bigr)(t(v+s w_2))
=
\bigl(\nabla^{\text{induced}}_{tw_2}d\Exp_p\bigr)_{t(v+s w_2)}(tw_1)
+
(d\Exp_p)_{t(v+s w_2)}\bigl(\nabla^{\text{flat}}_{tw_2}(tw_1)\bigr).
\]
By \eqref{eq:flat-zero-tw} the last term vanishes. 
Using the fact $\bigl(\nabla^{\text{pullback}}_{X}Y\bigr)(u)=\nabla_{\,(d\Exp_p)_u(X(u))} Y$ with $X=tw_2$, $Y = \widetilde J_1$, $u = tv$, we have (restricted to $s = 0$)
\begin{align*}
    \bigl(\nabla^{\text{pullback}}_{tw_2}\widetilde J_1\bigr)(tv)=\nabla_{\,(d\Exp_p)_{tv}(tw_2)} \widetilde J_1 .
\end{align*}
Together, with the previously defined notation $J_2 = (d\Exp_{p})_{tv}(tw_{2})$, we have 
\begin{align*}
    \nabla_{J_2}J_1 = \nabla_{J_2}\widetilde J_1\big|_{s=0} 
    = \nabla^{\text{pullback}}_{tw_2}\widetilde J_1\big|_{s=0} =
    \bigl(\nabla^{\text{induced}} d\Exp_p\bigr)_{tv}(tw_1,tw_2).
\end{align*}
We therefore define the second-order variation field along $\gamma$ by
\[
K_{12}(t):=\bigl(\nabla^{\text{induced}} d\Exp_p\bigr)_{tv}(tw_1,tw_2)=\nabla_{J_2}J_1.
\]
Then \cite[Proposition~4.1]{lezcano2020curvature} shows that $K_{12}$ satisfies
\begin{equation}\label{eq:K-ode}
\ddot K_{12} + R(K_{12},\dot\gamma)\dot\gamma + Y_{12}=0,
\qquad K_{12}(0)=0,\ \dot K_{12}(0)=0,
\end{equation}
with $Y_{12}$ given explicitly therein. 
Using the same technique, we can define 
\begin{align*}
    K_{13}(t):=\bigl(\nabla^{\text{induced}} d\Exp_p\bigr)_{tv}(tw_1,tw_3)=\nabla_{J_3}J_1, 
    K_{23}(t):=\bigl(\nabla^{\text{induced}} d\Exp_p\bigr)_{tv}(tw_2,tw_3)=\nabla_{J_3}J_2,
\end{align*}
satisfying the corresponding ODE as \eqref{eq:K-ode}.

Now we study the third derivative $L_{123}=\bigl((\nabla^{\text{induced}})^2 d\Exp_p\bigr)(tw_1,tw_2,tw_3)$.
Define $J_{i}(t)$ to be the Jacobi field along $\gamma$ with initial condition $w_{i}$ for $i = 2, 3$, in the same way as we did for $J_{1}$.
Extend $K_{12}(t)=\bigl(\nabla^{\text{induced}} d\Exp_p\bigr)_{tv}(tw_1,tw_2)=\nabla_{J_2}J_1$ off $\gamma$ by
\begin{align*}
    \widetilde K_{12}(t,s_3) :=
    \bigl(\nabla^{\text{induced}} d\Exp_p\bigr)_{t(v+s_3 w_3)}(tw_1,tw_2),
\end{align*}
so that $\widetilde K_{12}(t,0)=K_{12}(t)$.
We then define the third-order variation field along $\gamma$ by
\begin{equation}\label{eq:def-L123}
L_{123}(t)
:=
\nabla_{J_3}K_{12}(t)
:=
\nabla_{J_3}\widetilde K_{12}(t,s_3)\Big|_{s_3=0}.
\end{equation}
Apply $\nabla^{\text{pullback}}_{\overline W_3}$ to the first-order Leibniz rule
\eqref{eq:Leibniz-dexp-pullback}. For vector fields $\overline W_1,\overline W_2,\overline W_3$ on $T_pM$ we obtain,
at any $u\in T_pM$,
\begin{align}\label{eq:Leibniz-second-order}
\bigl(\nabla^{\text{pullback}}_{\overline W_3}\nabla^{\text{pullback}}_{\overline W_1}(d\Exp_p(\overline W_2))\bigr)(u)
&=
\bigl((\nabla^{\text{induced}}_{\overline W_3}\nabla^{\text{induced}}_{\overline W_1} d\Exp_p)_u\bigr)(\overline W_2(u))
\nonumber\\
&\quad
+\bigl(\nabla^{\text{induced}}_{\overline W_1} d\Exp_p\bigr)_u \bigl((\nabla^{\text{flat}}_{\overline W_3}\overline W_2)(u)\bigr)
\nonumber\\
&\quad
+\bigl(\nabla^{\text{induced}}_{\overline W_3} d\Exp_p\bigr)_u \bigl((\nabla^{\text{flat}}_{\overline W_1}\overline W_2)(u)\bigr)
\nonumber\\
&\quad
+(d\Exp_p)_u \bigl((\nabla^{\text{flat}}_{\overline W_3}\nabla^{\text{flat}}_{\overline W_1}\overline W_2)(u)\bigr).
\end{align}
Now specialize \eqref{eq:Leibniz-second-order} to the constant vector fields
$\overline W_1\equiv tw_2$, $\overline W_2\equiv tw_1$, $\overline W_3\equiv tw_3$.
Since $\nabla^{\text{flat}}_{tw_i}(tw_j)=0$ for all $i,j$, the last three terms in
\eqref{eq:Leibniz-second-order} vanish. Evaluating at $u=tv$ yields
\begin{align*}
    L_{123}(t) = \bigl(\nabla^{\text{pullback}}_{tw_3}\widetilde K_{12}\bigr)(tv) 
    = \bigl((\nabla^{\text{induced}}_{tw_3}\nabla^{\text{induced}}_{tw_2} d\Exp_p)_{tv}\bigr)(tw_1),
\end{align*}
where note that $(d\Exp_p)_{tv}(tw_3)=J_3(t)$.

By definition of the second covariant derivative tensor $(\nabla^{\text{induced}})^2 d\Exp_p$, this is precisely
\begin{equation}\label{eq:L123-is-nabla2dexp}
L_{123}(t)
=
\bigl((\nabla^{\text{induced}})^2 d\Exp_p\bigr)_{tv}(tw_1,tw_2,tw_3).
\end{equation}

Throughout this section, we consider covariant derivative along $J_{i}$, for example $\nabla_{J_{3}} J_{2}$, as abbreviation of $\nabla_{J_{3}} \tilde J_{2}$. Here, the $\tilde{\cdot}$ represent the corresponding extension along suitable direction, as constructed above. 

\begin{lemma}
    For $L_{123}$ defined above, we have 
    \begin{align*}
        \ddot L_{123} + R(L_{123},\dot\gamma)\dot\gamma +Y_{123}=0,
        \qquad L_{123}(0)=0,\ \dot L_{123}(0)=0.
    \end{align*}
    When $M$ is a symmetric space, 
    \begin{align*}
        Y_{123} &= R(\dot J_3,\dot\gamma)K_{12}
        + 2R(J_3,\dot\gamma)\dot K_{12} 
        + R(K_{12},\dot J_3)\dot\gamma
        + R(K_{12},\dot\gamma)\dot J_3 \\
        &
    +2R(\nabla_{J_3}J_1,\dot\gamma)\dot J_2
    +2R(J_1,\dot J_3)\dot J_2
    +2R(J_1,\dot\gamma)\nabla_{J_3}\dot J_2 \\
        &
    +2R(\nabla_{J_3}J_2,\dot\gamma)\dot J_1
    +2R(J_2,\dot J_3)\dot J_1
    +2R(J_2,\dot\gamma)\nabla_{J_3}\dot J_1 .
    \end{align*}
\end{lemma}

\begin{proof}
    Now recall the following ODE is staisfied by $K$ \cite[Proposition 4.1.]{lezcano2020curvature}.
    \begin{align*}
        \ddot K_{12} + R(K_{12},\dot\gamma)\dot\gamma + Y_{12}=0.
    \end{align*} 
    By differentiating the ODE (note that the vector field to which we take covariant derivative, is viewed as the corresponding extension, i.e., $\tilde K_{12}$ for $K_{12}$, $\grad r$ for $\dot \gamma$ and $\nabla_{\grad r} \nabla_{\grad r} \tilde K_{12}$ for $\ddot K_{12}$), we obtain 
    \begin{align*}
        \nabla_{J_3} \ddot K_{12} + \nabla_{J_3} R(K_{12},\dot\gamma)\dot\gamma + \nabla_{J_3} Y_{12}=0,
    \end{align*}

    We first compute $\nabla_{J_3}\ddot K_{12}$.
    Recall $\dot X:=\nabla_{\dot\gamma}X$ and $\ddot X:=\nabla_{\dot\gamma}\nabla_{\dot\gamma}X$.
    Let $L_{123}:=\nabla_{J_3}K_{12}$.
    We compute $\nabla_{J_3}\ddot K_{12}$ step by step.

    Recall by definition of curvature tensor
    \begin{align*}
        \nabla_{J_3}\nabla_{\dot\gamma}X = \nabla_{\dot\gamma}\nabla_{J_3}X 
        +\nabla_{[J_3,\dot\gamma]}X +R(J_3,\dot\gamma)X,
    \end{align*}
    and the fact that $[J_3,\dot\gamma]=0$ (for the same reason as in \citet[Proposition 4.1]{lezcano2020curvature}),
    we have
    \begin{align*}
    \nabla_{J_3}\ddot K_{12}
        &=\nabla_{J_3}\nabla_{\dot\gamma}\dot K_{12} 
        =\nabla_{\dot\gamma}\nabla_{J_3}\dot K_{12} + R(J_3,\dot\gamma)\dot K_{12} \\
        &= \nabla_{\dot\gamma}\nabla_{J_3}\nabla_{\dot\gamma}K_{12} + R(J_3,\dot\gamma)\dot K_{12} \\
        &= \nabla_{\dot\gamma}(\nabla_{\dot\gamma}\nabla_{J_3} K_{12} + R(J_3,\dot\gamma)K_{12}) + R(J_3,\dot\gamma)\dot K_{12} 
        = \nabla_{\dot\gamma}\nabla_{\dot\gamma} L_{123} + \nabla_{\dot\gamma}(R(J_3,\dot\gamma)K_{12}) + R(J_3,\dot\gamma)\dot K_{12}.
    \end{align*}

    Notice that (by Leibniz rule)
    \begin{align*}
        \nabla_{\dot\gamma}(R(J_3,\dot\gamma)K_{12})
        &= (\nabla_{\dot\gamma} R) (J_3,\dot\gamma)K_{12}
        + R(\nabla_{\dot\gamma}J_3,\dot\gamma)K_{12}
        + R(J_3,\nabla_{\dot\gamma}\dot\gamma)K_{12}
        + R(J_3,\dot\gamma)\nabla_{\dot\gamma}K_{12} \\
        &= (\nabla_{\dot\gamma} R) (J_3,\dot\gamma)K_{12}
        + R(\dot J_3,\dot\gamma)K_{12}
        + R(J_3,\dot\gamma)\dot K_{12}.
    \end{align*}
    Hence we have 
    \begin{align*}
        \nabla_{J_3}\ddot K_{12}
        = \ddot L_{123} 
        + (\nabla_{\dot\gamma} R) (J_3,\dot\gamma)K_{12}
        + R(\dot J_3,\dot\gamma)K_{12}
        + 2R(J_3,\dot\gamma)\dot K_{12}.
    \end{align*}

    Now we compute $\nabla_{J_3}\bigl(R(K_{12},\dot\gamma)\dot\gamma\bigr)$. Note that $\nabla_{J_3} \dot\gamma = \nabla_{\dot\gamma} J_{3} = \dot J_{3}$, by Leibniz rule,
    \begin{align*}
        \nabla_{J_3}\bigl(R(K_{12},\dot\gamma)\dot\gamma\bigr)
        &= (\nabla_{J_3}R)(K_{12},\dot\gamma)\dot\gamma
        + R(\nabla_{J_3}K_{12},\dot\gamma)\dot\gamma
        + R(K_{12},\nabla_{J_3} \dot\gamma)\dot\gamma
        + R(K_{12},\dot\gamma)\nabla_{J_3} \dot\gamma \\
        &= (\nabla_{J_3}R)(K_{12},\dot\gamma)\dot\gamma
        + R(L_{123},\dot\gamma)\dot\gamma
        + R(K_{12},\dot J_3)\dot\gamma
        + R(K_{12},\dot\gamma)\dot J_3.
    \end{align*}

    Recall
    \begin{align*}
        Y :=2R(J_1,\dot\gamma)\dot J_2 +2R(J_2,\dot\gamma)\dot J_1 +(\nabla_{\dot\gamma}R)(J_2,\dot\gamma)J_1
        +(\nabla_{J_2}R)(J_1,\dot\gamma)\dot\gamma.
    \end{align*}
    Also recall $[J_3,\dot\gamma]=0$ and $\nabla_{J_3}\dot\gamma=\nabla_{\dot\gamma}J_3=\dot J_3$.
    Then, by the Leibniz rule for tensor fields, we have
    \begin{align*}
    \nabla_{J_3}Y
    &=2\,\nabla_{J_3}\big(R(J_1,\dot\gamma)\dot J_2\big)
    +2\,\nabla_{J_3}\big(R(J_2,\dot\gamma)\dot J_1\big) \\
    &\quad +\nabla_{J_3}\Big((\nabla_{\dot\gamma}R)(J_2,\dot\gamma)J_1\Big)
    +\nabla_{J_3}\Big((\nabla_{J_2}R)(J_1,\dot\gamma)\dot\gamma\Big),
    \end{align*}
    where each term expands as follows.
    For the first term, 
    \begin{align*}
    \nabla_{J_3}\big(R(J_1,\dot\gamma)\dot J_2\big)
    &=(\nabla_{J_3}R)(J_1,\dot\gamma)\dot J_2
    +R(\nabla_{J_3}J_1,\dot\gamma)\dot J_2
    +R(J_1,\nabla_{J_3}\dot\gamma)\dot J_2
    +R(J_1,\dot\gamma)\nabla_{J_3}\dot J_2 \\
    &=(\nabla_{J_3}R)(J_1,\dot\gamma)\dot J_2
    +R(\nabla_{J_3}J_1,\dot\gamma)\dot J_2
    +R(J_1,\dot J_3)\dot J_2
    +R(J_1,\dot\gamma)\nabla_{J_3}\dot J_2.
    \end{align*}
    For the second term, 
    \begin{align*}
    \nabla_{J_3}\big(R(J_2,\dot\gamma)\dot J_1\big)
    &=(\nabla_{J_3}R)(J_2,\dot\gamma)\dot J_1
    +R(\nabla_{J_3}J_2,\dot\gamma)\dot J_1
    +R(J_2,\nabla_{J_3}\dot\gamma)\dot J_1
    +R(J_2,\dot\gamma)\nabla_{J_3}\dot J_1 \\
    &=(\nabla_{J_3}R)(J_2,\dot\gamma)\dot J_1
    +R(\nabla_{J_3}J_2,\dot\gamma)\dot J_1
    +R(J_2,\dot J_3)\dot J_1
    +R(J_2,\dot\gamma)\nabla_{J_3}\dot J_1.
    \end{align*}
    For the third term, 
    \begin{align*}
    \nabla_{J_3}\Big((\nabla_{\dot\gamma}R)(J_2,\dot\gamma)J_1\Big)
    &=(\nabla_{J_3}\nabla_{\dot\gamma}R)(J_2,\dot\gamma)J_1
    +(\nabla_{\dot\gamma}R)(\nabla_{J_3}J_2,\dot\gamma)J_1 \\ & \qquad 
    +(\nabla_{\dot\gamma}R)(J_2,\nabla_{J_3}\dot\gamma)J_1
    +(\nabla_{\dot\gamma}R)(J_2,\dot\gamma)\nabla_{J_3}J_1 \\
    &=(\nabla_{J_3}\nabla_{\dot\gamma}R)(J_2,\dot\gamma)J_1
    +(\nabla_{\dot\gamma}R)(\nabla_{J_3}J_2,\dot\gamma)J_1 \\ & \qquad +(\nabla_{\dot\gamma}R)(J_2,\dot J_3)J_1
    +(\nabla_{\dot\gamma}R)(J_2,\dot\gamma)\nabla_{J_3}J_1.
    \end{align*}
    For the last term, 
    \begin{align*}
    \nabla_{J_3}\Big((\nabla_{J_2}R)(J_1,\dot\gamma)\dot\gamma\Big)
    &=(\nabla_{J_3}\nabla_{J_2}R)(J_1,\dot\gamma)\dot\gamma
    +(\nabla_{J_2}R)(\nabla_{J_3}J_1,\dot\gamma)\dot\gamma \\
    & \qquad +(\nabla_{J_2}R)(J_1,\nabla_{J_3}\dot\gamma)\dot\gamma
    +(\nabla_{J_2}R)(J_1,\dot\gamma)\nabla_{J_3}\dot\gamma \\
    &=(\nabla_{J_3}\nabla_{J_2}R)(J_1,\dot\gamma)\dot\gamma
    +(\nabla_{J_2}R)(\nabla_{J_3}J_1,\dot\gamma)\dot\gamma \\
    & \qquad +(\nabla_{J_2}R)(J_1,\dot J_3)\dot\gamma
    +(\nabla_{J_2}R)(J_1,\dot\gamma)\dot J_3.
    \end{align*}

    Plug in the expressions, we have 
    \begin{align*}
        \ddot L_{123} + R(L_{123},\dot\gamma)\dot\gamma +Y_{123}=0,
    \end{align*}
    where 
    \begin{align*}
        Y_{123} &= (\nabla_{\dot\gamma} R) (J_3,\dot\gamma)K_{12} + (\nabla_{J_3}R)(K_{12},\dot\gamma)\dot\gamma
        + R(\dot J_3,\dot\gamma)K_{12}
        + 2R(J_3,\dot\gamma)\dot K_{12} \\
        & + R(K_{12},\dot J_3)\dot\gamma
        + R(K_{12},\dot\gamma)\dot J_3 \\
        & + 2(\nabla_{J_3}R)(J_1,\dot\gamma)\dot J_2
    +2R(\nabla_{J_3}J_1,\dot\gamma)\dot J_2
    +2R(J_1,\dot J_3)\dot J_2
    +2R(J_1,\dot\gamma)\nabla_{J_3}\dot J_2 \\
        & + 2(\nabla_{J_3}R)(J_2,\dot\gamma)\dot J_1
    +2R(\nabla_{J_3}J_2,\dot\gamma)\dot J_1
    +2R(J_2,\dot J_3)\dot J_1
    +2R(J_2,\dot\gamma)\nabla_{J_3}\dot J_1 \\
        &+ (\nabla_{J_3}\nabla_{\dot\gamma}R)(J_2,\dot\gamma)J_1
    +(\nabla_{\dot\gamma}R)(\nabla_{J_3}J_2,\dot\gamma)J_1
    +(\nabla_{\dot\gamma}R)(J_2,\dot J_3)J_1
    +(\nabla_{\dot\gamma}R)(J_2,\dot\gamma)\nabla_{J_3}J_1 \\
    &+ (\nabla_{J_3}\nabla_{J_2}R)(J_1,\dot\gamma)\dot\gamma
    +(\nabla_{J_2}R)(\nabla_{J_3}J_1,\dot\gamma)\dot\gamma
    +(\nabla_{J_2}R)(J_1,\dot J_3)\dot\gamma
    +(\nabla_{J_2}R)(J_1,\dot\gamma)\dot J_3.
    \end{align*}

    On a symmetric space, 
    \begin{align*}
        Y_{123} &= R(\dot J_3,\dot\gamma)K_{12}
        + 2R(J_3,\dot\gamma)\dot K_{12} 
        + R(K_{12},\dot J_3)\dot\gamma
        + R(K_{12},\dot\gamma)\dot J_3 \\
        &
    +2R(\nabla_{J_3}J_1,\dot\gamma)\dot J_2
    +2R(J_1,\dot J_3)\dot J_2
    +2R(J_1,\dot\gamma)\nabla_{J_3}\dot J_2 \\
        &
    +2R(\nabla_{J_3}J_2,\dot\gamma)\dot J_1
    +2R(J_2,\dot J_3)\dot J_1
    +2R(J_2,\dot\gamma)\nabla_{J_3}\dot J_1 .
    \end{align*}
\end{proof}

Now we can apply comparison theory \cite[Proposition 4.9]{lezcano2020curvature} to control $\|L\|$.
\begin{lemma}\label{Lemma_Bound_Third_Deri_Exp}
Let $(M,g)$ be a Riemannian \emph{symmetric space}, i.e. $\nabla R\equiv 0$.
Fix a unit-speed geodesic $\gamma:[0,r]\to M$ with $\|\dot\gamma\|\equiv 1$.
Under Assumption \ref{A_Curvature}, we have 
    \begin{align*}
        \|L(t)\| \le \frac{7L_{R}^{2}}{3 }s_{K_{\min}}(s)^{5}
        + 16L_{R}^{3} t^{2} s_{K_{\min}}(t)^{5}
        + 20L_{R}^{2} t s_{K_{\min}}(t)^{4}.
    \end{align*}
    Consequently, 
    \begin{align*}
        &\bigl\| \bigl((\nabla)^2 d\Exp_p\bigr)_{rv} \bigr\|_{\op}
        = \sup_{\|w_{1}\|, \|w_{2}\|, \|w_{3}\| \le 1}\bigl\| \bigl((\nabla^{\text{induced}})^2 d\Exp_p\bigr)_{rv}(w_1,w_2,w_3) \bigr\|  \\
        = & \sup_{\|w_{1}\|, \|w_{2}\|, \|w_{3}\| \le 1} 
        \frac{1}{r^{3}} \|L_{123}(r)\| 
        \le 
        \frac{7L_{R}^{2}}{3 r^{3} }s_{K_{\min}}(r)^{5}
        + 16L_{R}^{3}\frac{1}{r} s_{K_{\min}}(r)^{5}
        + 20L_{R}^{2} \frac{1}{r^{2}} s_{K_{\min}}(r)^{4}.
    \end{align*}
\end{lemma}

\begin{proof}
    Recall 
    \begin{align*}
        Y_{123}
        &= R(\dot J_3,\dot\gamma)K_{12}
        + 2R(J_3,\dot\gamma)\dot K_{12}
        + R(K_{12},\dot J_3)\dot\gamma
        + R(K_{12},\dot\gamma)\dot J_3 \\
        &\quad
        +2R(\nabla_{J_3}J_1,\dot\gamma)\dot J_2
        +2R(J_1,\dot J_3)\dot J_2
        +2R(J_1,\dot\gamma)\nabla_{J_3}\dot J_2 \\
        &\quad
        +2R(\nabla_{J_3}J_2,\dot\gamma)\dot J_1
        +2R(J_2,\dot J_3)\dot J_1
        +2R(J_2,\dot\gamma)\nabla_{J_3}\dot J_1.
    \end{align*}
    By definition of $K_{13},K_{23}$, we have $\nabla_{J_3}J_1=K_{13}$ and $\nabla_{J_3}J_2=K_{23}$.
    Next we express $\nabla_{J_3}\dot J_i$ in terms of $\dot K_{i3}$ and curvature.
    By definition of $R$ and $[J_3,\dot\gamma]=0$,
    \begin{align*}
    \nabla_{J_3}\dot J_i
    =\nabla_{J_3}\nabla_{\dot\gamma}J_i
    &=\nabla_{\dot\gamma}\nabla_{J_3}J_i + R(J_3,\dot\gamma)J_i \\
    &=\dot K_{i3} + R(J_3,\dot\gamma)J_i.
    \end{align*}
    Therefore
    \begin{align*}
    2R(J_1,\dot\gamma)\nabla_{J_3}\dot J_2
    &=2R(J_1,\dot\gamma)\dot K_{23}
    + 2R(J_1,\dot\gamma)\big(R(J_3,\dot\gamma)J_2\big),\\
    2R(J_2,\dot\gamma)\nabla_{J_3}\dot J_1
    &=2R(J_2,\dot\gamma)\dot K_{13}
    + 2R(J_2,\dot\gamma)\big(R(J_3,\dot\gamma)J_1\big).
    \end{align*}
    Substituting these and using $R(x, y)z \le L_{R}\|x\| \|y\| \|z\|$, we have 
    \begin{align*}
    \|Y_{123}\|
    &\le 
    3 L_{R} \|K_{12}\| \|\dot J_3\| + 2 L_{R} \|J_3\| \|\dot K_{12}\| 
    + 2 L_{R} \|K_{13}\| \|\dot J_2\| + 2 L_{R} \|K_{23}\| \|\dot J_1\| \\
    &\quad
    + 2 L_{R} \|J_1\| \|\dot K_{23}\| + 2 L_{R} \|J_2\| \|\dot K_{13}\| 
    + 2 L_{R} \|J_1\| \|\dot J_2\| \|\dot J_3\| + 2 L_{R} \|J_2\| \|\dot J_1\| \|\dot J_3\| \\
    &\quad
    + 4 L_{R}^2 \|J_1\| \|J_2\| \|J_3\| \\
    &\le 7 L_{R} \|K\| \|\dot J\| 
    + 6 L_{R} \|J\| \|\dot K \|  
    + 4 L_{R} \|J\| \|\dot J\|^{2} + 4 L_{R}^2 \|J\|^{3} .
    \end{align*} 

    Using Lemma \ref{Lemma_Bound_Kdot}, \cite[Theorem 3.12, Theorem 4.11]{lezcano2020curvature} we have 
    \begin{align*}
        \|K(t)\| &\le \frac{16}{3}s_{K_{\min}}(t/2)^{2} L_{R}s_{K_{\min}}(t), \\
        \|\dot K(t)\|
        &\le L_{R}^{2}\frac{16}{3} t s_{K_{\min}}(t)^{3}
        + 2L_{R}s_{K_{\min}}(t)^{2}, \\
        \|J(t)\| &\le s_{K_{\min}}(t), \\
        \|\dot J(t) \| &\le s_{K_{\min}}'(t), 
    \end{align*}
    hence we can bound 
    \begin{align*}
        \|Y_{123}\|
        \le & 7 L_{R} \frac{16}{3 }s_{K_{\min}}(t/2)^{2} L_{R}s_{K_{\min}}(t) s_{K_{\min}}'(t)
        + 6 L_{R} s_{K_{\min}}(t) (L_{R}^{2}\frac{16}{3 } t s_{K_{\min}}(t)^{3} 
        + 2L_{R}s_{K_{\min}}(t)^{2})  \\
        &+ 4 L_{R} s_{K_{\min}}(t) (s_{K_{\min}}'(t))^{2} + 4 L_{R}^2 (s_{K_{\min}}(t))^{3} \\
        \le & \frac{112L_{R}^{2}}{3 }
        s_{K_{\min}}(t/2)^{2} s_{K_{\min}}(t) s_{K_{\min}}'(t)
        + 32L_{R}^{3}ts_{K_{\min}}(t)^{4}
        + 20L_{R}^{2}s_{K_{\min}}(t)^{3} \\
        \le & \frac{28L_{R}^{2}}{3 }
        s_{K_{\min}}(t)^{3} s_{K_{\min}}'(t)
        + 32L_{R}^{3}ts_{K_{\min}}(t)^{4}
        + 20L_{R}^{2}s_{K_{\min}}(t)^{3} .
    \end{align*}

    We apply \cite[Proposition 4.9]{lezcano2020curvature}.
    Consider ODE 
    \begin{align*}
        \rho''(t) + K_{\min} \rho(t) = \eta(t).
    \end{align*} 
    Define $y(t) = s_{K_{\min}}(t)$. 
    Note that $y$ solves $y''(t) + K_{\min} y(t) = 0$. 
    Then the function $\rho(t) = \int_{0}^{t} y(t-s) \eta(s) ds$ satisfies $\rho''(t) + K_{\min} \rho(t) = \eta(t)$.
    We apply this to each part of the bound on $Y_{123}$:
    \begin{align*}
        \int_{0}^{t} s_{K_{\min}}(t-s) \frac{28L_{R}^{2}}{3 } s_{K_{\min}}(s)^{3} s_{K_{\min}}'(s) ds 
        &\le \frac{7L_{R}^{2}}{3 }s_{K_{\min}}(t) 
        \int_{0}^{t} 4s_{K_{\min}}(s)^{3} s_{K_{\min}}'(s) ds 
        = \frac{7L_{R}^{2}}{3 }s_{K_{\min}}(s)^{5}, \\
        \int_{0}^{t} s_{K_{\min}}(t-s) 32L_{R}^{3}s_{K_{\min}}(s)^{4} s ds 
        &\le 32L_{R}^{3} s_{K_{\min}}(t)
        \int_{0}^{t} s_{K_{\min}}(s)^{4} s ds \\
        &\le 32L_{R}^{3} s_{K_{\min}}(t)^{5} \int_{0}^{t} s ds
        = 16L_{R}^{3} t^{2} s_{K_{\min}}(t)^{5}, \\
        \int_{0}^{t} s_{K_{\min}}(t-s) 20L_{R}^{2} s_{K_{\min}}(s)^{3} ds 
        &\le 20L_{R}^{2} t s_{K_{\min}}(t)^{4}.
    \end{align*}
    Hence we can use the following $\tilde{\rho}$ as upper bound for $\rho$:
    \begin{align*}
        \rho(t) \le \tilde{\rho}(t) := \frac{7L_{R}^{2}}{3}s_{K_{\min}}(s)^{5}
        + 16L_{R}^{3} t^{2} s_{K_{\min}}(t)^{5}
        + 20L_{R}^{2} t s_{K_{\min}}(t)^{4}.
    \end{align*}
    In other words, 
    \begin{align*}
        \|L(t)\| \le \frac{7L_{R}^{2}}{3 }s_{K_{\min}}(s)^{5}
        + 16L_{R}^{3} t^{2} s_{K_{\min}}(t)^{5}
        + 20L_{R}^{2} t s_{K_{\min}}(t)^{4}.
    \end{align*}
\end{proof}

\begin{lemma}\label{Lemma_Bound_Kdot}
    When $M$ is a symmetric space, we have 
    \begin{align*}
        \|\dot K(t)\|
        \le L_{R}^{2}\frac{16}{3} t s_{K_{\min}}(t)^{3}
        + 2L_{R}s_{K_{\min}}(t)^{2}.
    \end{align*}
\end{lemma}

\begin{proof}

    Since $K$ solves $\ddot K + R(K,\dot\gamma)\dot\gamma + Y = 0$, we have 
    \begin{align*}
        \ddot K = -R(K,\dot\gamma)\dot\gamma - Y.
    \end{align*}
    Integrate over $[0, t]$, we obtain 
    \begin{align*}
        \dot K(t) - \dot K(0) = \int_{0}^{t} \ddot K(s) ds = \int_{0}^{t} -R(K,\dot\gamma)\dot\gamma (s) - Y(s) ds.
    \end{align*}
    Using $\|\dot\gamma\|=1$ and $\dot K(0) = 0$, 
    we obtain 
    \begin{align*}
        \|\dot K(t)\| = \|\int_{0}^{t} \ddot K(s) ds\|  
        \le \int_{0}^{t} \|R(K,\dot\gamma)\dot\gamma (s)\| + \|Y(s)\| ds
        \le \int_{0}^{t} L_{R}\|K(s)\| + \|Y(s)\| ds.
    \end{align*}

    It suffices to obtain bound for $\|Y\|$ and $\|K\|$.
    Recall
    \begin{align*}
        Y :=2R(J_1,\dot\gamma)\dot J_2 +2R(J_2,\dot\gamma)\dot J_1 +(\nabla_{\dot\gamma}R)(J_2,\dot\gamma)J_1
        +(\nabla_{J_2}R)(J_1,\dot\gamma)\dot\gamma,
    \end{align*}
    and that for a symmetric space, we have 
    Recall
    \begin{align*}
        Y =2R(J_1,\dot\gamma)\dot J_2 +2R(J_2,\dot\gamma)\dot J_1.
    \end{align*}
    Hence 
    \begin{align*}
        \|Y(s)\| \le 2L_{R} (\|J_1\| \|\dot J_2 \| + \|\dot J_1\| \|J_2\|)
        \le 4 L_{R} s_{K_{\min}}'(s)s_{K_{\min}}(s)
        = 2L_{R}s_{K_{\min}}(2s),
    \end{align*}
    where by \citet[Theorem 3.12]{lezcano2020curvature} we have for unit tangent vector $v$, 
    $$\|d (\Exp_{p})_{rv} (rw)\| \le \max\{1, \frac{s_{K_{\min}}(r)}{r}\}r\|w\|,$$ 
    so that $\|J(s)\| \le s_{K_{\min}}(s)$.
    Also, \citet[Theorem 3.11]{lezcano2020curvature} gives $$\|\Hess r \| \le \frac{s_{K_{\min}}'(r)}{s_{K_{\min}}(r)},$$ 
    so that 
    $\|\dot J(s) \| \le \|J\| \|\Hess r\| \le \frac{s_{K_{\min}}'(s)}{s_{K_{\min}}(s)}s_{K_{\min}}(s) \le s_{K_{\min}}'(s)$.\\

    Furthermore, \citet[Theorem 4.11]{lezcano2020curvature} implies (with $w$ replaced by $rw$)
    \begin{align*}
        K(s) \le \frac{16}{3}s_{K_{\min}}(s/2)^{2} L_{R}s_{K_{\min}}(s).
    \end{align*}
    Hence, we have 
    \begin{align*}
        \|\dot K(t)\|
        \le \int_{0}^{t} L_{R}(\frac{16}{3}s_{K_{\min}}(s/2)^{2} L_{R}s_{K_{\min}}(s)) + 2(L_{R}s_{K_{\min}}(2s)) ds.
    \end{align*}
    Note that in our case, $s_{K_{\min}}(t) = \frac{\sinh (\sqrt{|K_{\min}|} t)}{\sqrt{|K_{\min}|}}$, 
    so we have $s_{K_{\min}}(s/2) \le s_{K_{\min}}(s)$.\\

    For the first term, we have 
    \begin{align*}
        \int_{0}^{t} L_{R}(\frac{16}{3}s_{K_{\min}}(s/2)^{2} L_{R}s_{K_{\min}}(s)) ds 
        &\le L_{R}^{2}\frac{16}{3 } 
        \int_{0}^{t} s_{K_{\min}}(s)^{3}ds
        \le L_{R}^{2}\frac{16}{3} \int_{0}^{t} s_{K_{\min}}(t)^{3}ds \\ & = L_{R}^{2}\frac{16}{3 } t s_{K_{\min}}(t)^{3}.
    \end{align*}
    For the second term, 
    \begin{align*}
        \int_{0}^{t} 2(L_{R}s_{K_{\min}}(2s)) ds
        \le 2L_{R}s_{K_{\min}}(t)^{2}.
    \end{align*}
    Put together, we obtain
    \begin{align*}
        \|\dot K(t)\|
        \le L_{R}^{2}\frac{16}{3} t s_{K_{\min}}(t)^{3}
        + 2L_{R}s_{K_{\min}}(t)^{2}.
    \end{align*}

\end{proof}

\subsection{Bounds involving Riemannian Log}
\label{Subsec_Riem_Log}

Here, we obtain bounds on $\|(d\Log_x)_y\|$, $\big\|\big(\nabla(d\Log_x)\big)_y\big\|$, $ \|\nabla_x \Log_x(y)\|_{\op}$, $\|\nabla_x \Log_x(y)\|_{\op}^{2}$, $|\Div_x\Log_x(y)|$, and $|\Div_x\Log_x(y)|^{2}$. We also obtain a bound $ \|\grad \Div_x\Log_x(y)\|$ when the manifold is $\SPD(n)$. In this section, for different differential operators $d$ and $\nabla$, note that 
$\nabla \Log_{x}(x_{1})$ is covariant derivative (viewing $\Log_{x}(x_{1})$ as a vector field, with $x_{1}$ being fixed. On the other hand, $d \Log_{x_{1}}$ is considering fixed base point. 

Before presenting our results, we also recall some facts. Let $M$ be a Hadamard manifold of dimension $d$. Fix $x\in M$ and let $y\in M$, $y\neq x$.
Then:
\begin{enumerate}
\item[(i)] $\|\Log_x(y)\| = d(x,y)$.
\item[(ii)] $\grad_x\Big(\frac12 d(x,y)^2\Big) = -\Log_x(y)$.
\item[(iii)] $\Div_x \Log_x(y) = -\Delta_x\Big(\frac12 d(x,y)^2\Big)$.
\end{enumerate}

\begin{lemma}[Bounds for $d\Log_x$ and $\nabla(d\Log_x)$]\label{lem:dLog-and-nabla-dLog}
Let $M$ be a Hadamard manifold of dimension $d$. Fix $x\in M$. 
Then, we have
\begin{align*}
    \|(d\Log_x)_y\| &\le 1, \\
    \big\|\big(\nabla(d\Log_x)\big)_y\big\|
     &\le
    \big\|\big(\nabla(d\Exp_x)\big)_{\Log_x(y)}\big\|.
\end{align*}
\end{lemma}

\begin{proof} On a Hadamard manifold, $\Exp_x:T_xM\to M$ is a global diffeomorphism, hence $\Log_x=\Exp_x^{-1}$ is smooth on $M$.
Let $u=\Log_x(y)$. Then $\Exp_x(u)=y$ and the inverse function theorem gives
\[
(d\Log_x)_y = \big((d\Exp_x)_u\big)^{-1} = \big((d\Exp_x)_{\Log_x(y)}\big)^{-1}.
\]
Take norm on both sides and recall $\big\|\big((d\Exp_x)_{\Log_x(y)}\big)^{-1}\big\| \le 1$, we get the first bound. 

Now we differentiate the identity
\[
(d\Exp_x)_{\Log_x(\cdot)}\circ (d\Log_x)_{\cdot} = Id_{T_{\cdot}M}
\]
covariantly at $y$ in the direction $w_1\in T_yM$, and then apply the resulting operator to $w_2\in T_yM$.
Using Leibniz rule and that $\nabla(Id)=0$, we get
\[
\big(\nabla(d\Exp_x)\big)_{\Log_x(y)}\Big[(d\Log_x)_y w_1,\,(d\Log_x)_y w_2\Big]
+
(d\Exp_x)_{\Log_x(y)}\Big(\big(\nabla(d\Log_x)\big)_y[w_1,w_2]\Big)
=0.
\]
Re-arrange the terms, we get 
\begin{align*}
    &\big(\nabla(d\Log_x)\big)_y[w_1,w_2] \\
=&
-\big((d\Exp_x)_{\Log_x(y)}\big)^{-1}\Big(
\big(\nabla(d\Exp_x)\big)_{\Log_x(y)}\Big[
\big((d\Exp_x)_{\Log_x(y)}\big)^{-1}w_1,\,
\big((d\Exp_x)_{\Log_x(y)}\big)^{-1}w_2
\Big]\Big).
\end{align*}
Take norms, and recalling that $\|\big((d\Exp_x)_{\Log_x(y)}\big)^{-1}\| \le 1$,
we obtain 
\begin{align*}
    \big\|\big(\nabla(d\Log_x)\big)_y\big\|
\le
\big\|\big(\nabla(d\Exp_x)\big)_{\Log_x(y)}\big\|.
\end{align*}

\end{proof}

\begin{lemma}[Bounds for $\|\nabla_x \Log_x(y)\|$]\label{lem:nabla-Log-via-Hess-r}
Let $M$ be a Hadamard manifold of dimension $d$ with sectional curvature bounded below by $K_{\min}\le 0$.
Fix $x\in M$ and $y\in M$, $y\neq x$. Then
\begin{align*}
    \|\nabla_x \Log_x(y)\|_{\op} &\le
    1 + d(x,y)\,\frac{s_{K_{\min}}'(d(x,y))}{s_{K_{\min}}(d(x,y))}, \\
    \|\nabla_x \Log_x(y)\|_{\op}^{2} &\le
    2 + 2\,d(x,y)^{2}\Big(\frac{s_{K_{\min}}'(d(x,y))}{s_{K_{\min}}(d(x,y))}\Big)^{2}.
\end{align*}
\end{lemma}

\begin{proof}
Let $r(\cdot):=d(\cdot,y)$, so $r(x)=d(x,y)$.
Since $r\,\grad_x r = -\Log_x(y)$ (since $\grad_x(\frac12 r^2)=r\grad_x r$),
we have
\[
\Log_x(y) = -\,r(x)\,\grad_x r(x).
\]
Differentiate covariantly in $x$, we obtain
\[
\nabla_x \Log_x(y)
=
-\nabla_x \big(r(x)\,\grad_x r(x)\big)
=
-\big(\nabla_x r(x)\big)\otimes \grad_x r(x) \;-\; r(x)\,\nabla_x\grad_x r(x).
\]
Taking operator norms and using $\|\grad_x r(x)\|=1$ gives
\[
\|\nabla_x \Log_x(y)\|_{\op}
\le
\|\grad_x r(x)\|^{2} + r(x)\,\|\nabla \grad_{x} r(x)\|_{\op}
=
1 + d(x,y)\,\|\nabla \grad_{x} r(x)\|_{\op}.
\]
By Hessian comparison under $\Sec\ge K_{\min}$, we have
\begin{align*}
    \|\nabla \grad_{x} r(x)\|_{\op} \le \frac{s_{K_{\min}}'(d(x,y))}{s_{K_{\min}}(d(x,y))}.
\end{align*}
Combining the above estimates yields the first inequality. 
The squared bound follows from $(a+b)^2\le 2a^2+2b^2$ with $a=1$ and
$b=d(x,y)\, \frac{s_{K_{\min}}'(d(x,y))}{s_{K_{\min}}(d(x,y))}$.
\end{proof}

\begin{lemma}[Bounds for $\Div_x\Log_x(y)$ and $|\Div_x\Log_x(y)|^2$]\label{lem:Div-Log-bound}
Let $M$ be a Hadamard manifold of dimension $d$ with sectional curvature bounded below by $K_{\min}\le 0$.
Fix $x\in M$ and $y\in M$, $y\neq x$. Then
\begin{align*}
    |\Div_x\Log_x(y)| &\le
    1 + (d-1)\,d(x,y)\,\frac{s_{K_{\min}}'(d(x,y))}{s_{K_{\min}}(d(x,y))}, \\
    |\Div_x\Log_x(y)|^{2} &\le
    2 + 2(d-1)^{2}\,d(x,y)^{2}\Big(\frac{s_{K_{\min}}'(d(x,y))}{s_{K_{\min}}(d(x,y))}\Big)^{2}.
\end{align*}
\end{lemma}

\begin{proof}
Let $r(\cdot):=d(\cdot,y)$. 
Then we have $\Div_x\Log_x(y)=-\Delta_x\Big(\frac12 r(x)^2\Big)$. Using the product rule for Laplacian, we have 
$\Delta(\frac12 r^2)=\langle \grad r,\grad r\rangle + r\Delta r = 1 + r \Delta r$.
For the upper bound, Laplacian comparison theorem under $\Sec\ge K_{\min}$ gives
\begin{align*}
    \Delta r \le (d-1)\,\frac{s_{K_{\min}}'(d(x,y))}{s_{K_{\min}}(d(x,y))}.
\end{align*}
Substitute into $|\Div_x\Log_x(y)| = \Delta_x(\frac12 r^2)$
we obtain
\[
|\Div_x\Log_x(y)| \le 1 + (d-1)\,d(x,y)\,\frac{s_{K_{\min}}'(d(x,y))}{s_{K_{\min}}(d(x,y))}.
\]
Finally, using $(a+b)^2\le 2a^2+2b^2$ with
$a=1$ and $b=(d-1)\,d(x,y)\,\frac{s_{K_{\min}}'(d(x,y))}{s_{K_{\min}}(d(x,y))}$ gives the desired results.
\end{proof}

Moreover, on $\SPD(n)$, we have the following results. 

\begin{lemma}[$\|\grad\Div\Log\|$ bound on $\SPD(n)$]
\label{Lemma_gradDivLog_SPD_AIRM}
Let $\SPD(n)$ be the manifold of real symmetric positive-definite matrices endowed with the
affine-invariant Riemannian metric
\[
\langle U,V\rangle_P := \Tr \big(P^{-1}UP^{-1}V\big),
\qquad
U,V\in T_P\SPD(n)=Sym(n).
\]
Fix $y\in \SPD(n)$ and define
\[
r(x):=d(x,y),\qquad f(x):=d(x,y)^2=r(x)^2,
\qquad d:=\dim(\SPD(n))=\frac{n(n+1)}{2}.
\]
Then we have 
\begin{align*}
    \|\grad\Div\Log\|\le \frac{\sqrt2 d}{2} \left(2\Big(1+r\,\frac{s_{K_{\min}}'(r)}{s_{K_{\min}}(r)}\Big)\right)^{\frac{3}{2}}.
\end{align*}
\end{lemma}

\begin{proof} We first show that $\SPD(n)$ is a totally geodesic submanifold of $PD(n,\mathbb{C})$ under the affine-invariant metric. 
    Let $\PD(n,\mathbb{C})$ denote the manifold of complex Hermitian positive-definite matrices endowed with the
    same affine-invariant metric
    $\langle U,V\rangle_P=\Tr(P^{-1}UP^{-1}V)$.
    Define a smooth isometry by
    \begin{align*}
        \sigma:\PD(n,\mathbb{C})\to \PD(n,\mathbb{C}),\qquad \sigma(P):=\overline{P}.
    \end{align*}
    Notice that $\sigma$ is an isometry because for any $P\in \PD(n,\mathbb{C})$ and $U,V\in Herm(n)$,
    \[
    \langle d\sigma_P[U],d\sigma_P[V]\rangle_{\sigma(P)}
    =
    \Tr \big(\overline{P}^{-1}\,\overline{U}\,\overline{P}^{-1}\,\overline{V}\big)
    =
    \overline{\Tr(P^{-1}UP^{-1}V)}
    =
    \Tr(P^{-1}UP^{-1}V)
    =
    \langle U,V\rangle_P,
    \]
    where we used that $\Tr(P^{-1}UP^{-1}V)\in\mathbb{R}$ for Hermitian $P,U,V$.
    Thus we see that the fixed-point set of $\sigma$ is exactly $\SPD(n)$. According to \citet[Theorem 5.1]{kobayashi1972transformation} we have that the fixed-point set of an isometry is a (embedded) totally geodesic submanifold. 
    Hence $\SPD(n)$ is totally geodesic in $\PD(n,\mathbb{C})$.

    Let $F(P):=d_{\PD(n,\mathbb{C})}(P,y)^2$ be the squared distance to $y\in\SPD(n)\subset PD(n,\mathbb{C})$,
    considered as a function on $PD(n,\mathbb{C})$.
    By \citet[Theorem 1.4]{hirai2023interior}, $F$ is $2$-self-concordant on $\PD(n,\mathbb{C})$: we have that for all $P\in \PD(n,\mathbb{C})$ and $u,v,w\in T_P \PD(n,\mathbb{C})$,
    \[
    |(\nabla^3 F)_P(u,v,w)|
    \le
    \sqrt2\,
    \sqrt{(\nabla^2 F)_P(u,u)}\,
    \sqrt{(\nabla^2 F)_P(v,v)}\,
    \sqrt{(\nabla^2 F)_P(w,w)}.
    \]
    Now restrict $F$ to the totally geodesic submanifold $\SPD(n)$. By definition, for a totally geodesic submanifold, the second fundamental form vanishes, and Gauss formula \cite[Theorem 8.2]{lee2018introduction} implies that the connection $\nabla$ on $\SPD(n)$ is induced by that of $\PD(n, \mathbb{C})$.
    Moreover, the geodesic distance function on $\SPD(n)$ is induced by that on $\PD(n, \mathbb{C})$, so we know self-concordance of distance squared also holds on $\SPD(n)$.

Now we make use of the self-concordance property to bound $\|\grad\Div\Log\|$. Let $f(x)=d_{\SPD}(x,y)^2$ and note that on the Hadamard manifold $\SPD(n)$,
    \[
    \|\grad_x\Div_x\Log_x(y)\|=\tfrac12\|\grad_x\Delta_x f(x)\|.
    \]
    Let $\{e_i\}_{i=1}^d$ be an orthonormal basis of $T_x\SPD(n)$. Such an orthonormal basis can be extended to a neighborhood of $x$, and we obtain an orthonormal frame: $\nabla e_{i}|_{x} = 0, \forall i$ \cite[Exercise 5-21]{lee2018introduction}. Thus we can differentiate $\Delta f$ as follows. 
    For any unit $w\in T_x\SPD(n)$,
    \[
    \langle \grad\Delta f,w\rangle
    =
    \sum_{i=1}^d (\nabla^3 f)(w,e_i,e_i) + 2\nabla^{2}f(\nabla_{w} e_{i}, e_{i})
    =
    \sum_{i=1}^d (\nabla^3 f)(w,e_i,e_i).
    \]
    By self-concordance, we have 
    $|(\nabla^3 f)(w,e_i,e_i)| \le \sqrt2\,\sqrt{(\nabla^2 f)(w,w)}\,(\nabla^2 f)(e_i,e_i)$.
    Therefore
    \begin{align*}
        \|\grad\Div\Log\|\le \frac{\sqrt2 d}{2}\,\sqrt{\|\nabla^2 f\|_{\op}^{3}}.
    \end{align*}
Writing $f=r^2$, we have
    \[
    \nabla^2 f = 2\,dr\otimes dr + 2r\,\nabla^2 r,
    \qquad
    \Delta f = 2 + 2r\,\Delta r,
    \]
    using $\|\grad r\|=1$.
    Under $\Sec\ge K_{\min}$, 
    Hessian comparison gives $\|\nabla^2 r\|_{\op}\le \frac{s_{K_{\min}}'(r)}{s_{K_{\min}}(r)}$, hence
    \[
    \|\nabla^2 f\|_{\op}
    \le
    2 + 2r\,\|\nabla^2 r\|_{\op}
    \le
    2\Big(1+r\,\frac{s_{K_{\min}}'(r)}{s_{K_{\min}}(r)}\Bigr).
    \]
    Substituting this in the previous estimate yields
    \begin{align*}
        \|\grad\Div\Log\|\le \frac{\sqrt2 d}{2} \left(2\Big(1+r\,\frac{s_{K_{\min}}'(r)}{s_{K_{\min}}(r)}\Big)\right)^{1.5}.
    \end{align*}
\end{proof}

\subsection{Auxiliary results}

Given a matrix $A(x)$, the directional derivative in $v$ can be computed through
\begin{align*}
    D\log \det (A(x))[v] = \tr (A^{-1} D A(x)[v]).
\end{align*}
Hence 
\begin{align*}
    \|\nabla \log \det (A(x)) \|_{\op} 
    = \sup_{\|v\| = 1} |\tr (A^{-1} D A(x)[v])|
    \le d\|A^{-1}\| \times \| DA(x) \|_{\op}.
\end{align*}
Furthermore, using the product rule, 
and noting that $D A^{-1}(x) = - A^{-1}(x) DA(x) A^{-1}(x)$, we obtain 
\begin{align*}
    &D^{2}\log \det (A(x))[v, v] \\
    =& \tr (D A^{-1}(x)[v] D A(x)[v])
    + \tr (A^{-1}(x) D^{2} A(x)[v, v]) \\
    = & - \tr (A^{-1}(x) DA(x)[v] A^{-1}(x) D A(x)[v])
    + \tr (A^{-1}(x) D^{2} A(x)[v, v]).
\end{align*}
Hence, we have 
\begin{align*}
    \| \nabla^{2} \log \det (A(x)) \|_{op} 
    \le d\| A^{-1}(x)\|^{2} \times \| \nabla A(x)\|^{2}
    + d \| A^{-1}(x) \| \times \| \nabla^{2} A(x) \|.
\end{align*}

This is summarized as the following result, and is used in Lemma \ref{lem:Jt-pointwise-no-shorthand}. 

\begin{lemma}[Derivatives of $\log\det$]\label{lem:logdet-no-shorthand}
Let $A(x)$ be a smooth family of invertible linear maps on a $d$-dimensional inner product space.
Then for any unit vector $v$,
\[
D\log \det (A(x))[v] = \tr \big(A(x)^{-1} D A(x)[v]\big).
\]
Consequently,
\begin{align}
\|\nabla \log \det (A(x)) \|
&= \sup_{\|v\| = 1} \big|\tr \big(A(x)^{-1} D A(x)[v]\big)\big|
\le d\,\|A(x)^{-1}\| \,\|\nabla A(x)\|, \label{eq:grad-logdet-basic}\\
\|\nabla^{2} \log \det (A(x)) \|_{\op}
&\le d\,\| A(x)^{-1}\|^{2} \,\| \nabla A(x)\|^{2}
    + d \,\| A(x)^{-1} \| \,\| \nabla^{2} A(x) \|. \label{eq:hess-logdet-basic}
\end{align}
\end{lemma}


%% file: main.bbl
\begin{thebibliography}{44}
\providecommand{\natexlab}[1]{#1}
\providecommand{\url}[1]{\texttt{#1}}
\expandafter\ifx\csname urlstyle\endcsname\relax
  \providecommand{\doi}[1]{doi: #1}\else
  \providecommand{\doi}{doi: \begingroup \urlstyle{rm}\Url}\fi

\bibitem[Alimisis et~al.(2020)Alimisis, Orvieto, B{\'e}cigneul, and
  Lucchi]{alimisis2020continuous}
F.~Alimisis, A.~Orvieto, G.~B{\'e}cigneul, and A.~Lucchi.
\newblock {A continuous-time perspective for modeling acceleration in
  Riemannian optimization}.
\newblock In \emph{International Conference on Artificial Intelligence and
  Statistics}, pages 1297--1307. PMLR, 2020.

\bibitem[Bansal et~al.(2024)Bansal, Roy, Sarkar, and
  Rinaldo]{bansal2024wasserstein}
V.~Bansal, S.~Roy, P.~Sarkar, and A.~Rinaldo.
\newblock {On the Wasserstein Convergence and Straightness of Rectified Flow}.
\newblock \emph{arXiv preprint arXiv:2410.14949}, 2024.

\bibitem[Benton et~al.(2024)Benton, Deligiannidis, and Doucet]{benton2024error}
J.~Benton, G.~Deligiannidis, and A.~Doucet.
\newblock Error bounds for flow matching methods.
\newblock \emph{Transactions on Machine Learning Research}, 2024.

\bibitem[Cheeger et~al.(1975)Cheeger, Ebin, and Ebin]{cheeger1975comparison}
J.~Cheeger, D.~G. Ebin, and D.~G. Ebin.
\newblock \emph{{Comparison theorems in Riemannian geometry}}, volume~9.
\newblock North-Holland publishing company Amsterdam, 1975.

\bibitem[Chen and Lipman(2024)]{chen2023flow}
R.~T.~Q. Chen and Y.~Lipman.
\newblock Flow matching on general geometries.
\newblock In \emph{The Twelfth International Conference on Learning
  Representations}, 2024.
\newblock URL \url{https://openreview.net/forum?id=g7ohDlTITL}.

\bibitem[Cheng et~al.(2025)Cheng, Lo, Lee, Miret, and
  Aspuru-Guzik]{cheng2025stiefel}
A.~H. Cheng, A.~Lo, K.~L.~K. Lee, S.~Miret, and A.~Aspuru-Guzik.
\newblock {Stiefel Flow Matching for Moment-Constrained Structure Elucidation}.
\newblock In \emph{The Thirteenth International Conference on Learning
  Representations}, 2025.
\newblock URL \url{https://openreview.net/forum?id=84WmbzikPP}.

\bibitem[Cheng et~al.(2022)Cheng, Zhang, and Sra]{cheng2022efficient}
X.~Cheng, J.~Zhang, and S.~Sra.
\newblock {Efficient sampling on Riemannian manifolds via Langevin MCMC}.
\newblock \emph{Advances in Neural Information Processing Systems},
  35:\penalty0 5995--6006, 2022.

\bibitem[Collas et~al.(2025)Collas, Ju, Salvy, and
  Thirion]{collas2025riemannian}
A.~Collas, C.~Ju, N.~Salvy, and B.~Thirion.
\newblock Riemannian flow matching for brain connectivity matrices via pullback
  geometry.
\newblock In \emph{The Thirty-ninth Annual Conference on Neural Information
  Processing Systems}, 2025.
\newblock URL \url{https://openreview.net/forum?id=NY3LzmUXl7}.

\bibitem[Criscitiello and Boumal(2023)]{criscitiello2023accelerated}
C.~Criscitiello and N.~Boumal.
\newblock {An accelerated first-order method for non-convex optimization on
  manifolds}.
\newblock \emph{Foundations of Computational Mathematics}, 23\penalty0
  (4):\penalty0 1433--1509, 2023.

\bibitem[De~Bortoli et~al.(2022)De~Bortoli, Mathieu, Hutchinson, Thornton, Teh,
  and Doucet]{de2022riemannian}
V.~De~Bortoli, E.~Mathieu, M.~Hutchinson, J.~Thornton, Y.~W. Teh, and
  A.~Doucet.
\newblock {Riemannian score-based generative modelling}.
\newblock \emph{Advances in neural information processing systems},
  35:\penalty0 2406--2422, 2022.

\bibitem[Gatmiry and Vempala(2022)]{gatmiry2022convergence}
K.~Gatmiry and S.~S. Vempala.
\newblock {Convergence of the Riemannian Langevin algorithm}.
\newblock \emph{arXiv preprint arXiv:2204.10818}, 2022.

\bibitem[Guan et~al.(2025)Guan, Balasubramanian, and Ma]{guan2025riemannian}
Y.~Guan, K.~Balasubramanian, and S.~Ma.
\newblock {Riemannian Proximal Sampler for High-accuracy Sampling on
  Manifolds}.
\newblock In \emph{The Thirty-ninth Annual Conference on Neural Information
  Processing Systems}, 2025.
\newblock URL \url{https://openreview.net/forum?id=KxhCJc8BOg}.

\bibitem[Guan et~al.(2026)Guan, Balasubramanian, and Ma]{guan2025mirror}
Y.~Guan, K.~Balasubramanian, and S.~Ma.
\newblock Mirror flow matching with heavy-tailed priors for generative modeling
  on convex domains.
\newblock In \emph{The Fourteenth International Conference on Learning
  Representations}, 2026.
\newblock URL \url{https://openreview.net/forum?id=dZKl7uc0XQ}.

\bibitem[Hirai et~al.(2023)Hirai, Nieuwboer, and Walter]{hirai2023interior}
H.~Hirai, H.~Nieuwboer, and M.~Walter.
\newblock {Interior-point methods on manifolds: theory and applications}.
\newblock In \emph{2023 IEEE 64th Annual Symposium on Foundations of Computer
  Science (FOCS)}, pages 2021--2030. IEEE, 2023.

\bibitem[Huang et~al.(2022)Huang, Aghajohari, Bose, Panangaden, and
  Courville]{huang2022riemannian}
C.-W. Huang, M.~Aghajohari, J.~Bose, P.~Panangaden, and A.~C. Courville.
\newblock {Riemannian diffusion models}.
\newblock \emph{Advances in Neural Information Processing Systems},
  35:\penalty0 2750--2761, 2022.

\bibitem[Huang et~al.(2025)Huang, Huang, and Lin]{huang2025convergence}
D.~Z. Huang, J.~Huang, and Z.~Lin.
\newblock {Convergence analysis of probability flow ODE for score-based
  generative models}.
\newblock \emph{IEEE Transactions on Information Theory}, 2025.

\bibitem[Kobayashi(1972)]{kobayashi1972transformation}
S.~Kobayashi.
\newblock \emph{{Transformation groups in differential geometry}}.
\newblock Springer, 1972.

\bibitem[Kong and Tao(2024)]{kong2024convergence}
L.~Kong and M.~Tao.
\newblock {Convergence of kinetic Langevin Monte Carlo on Lie groups}.
\newblock In \emph{The Thirty Seventh Annual Conference on Learning Theory},
  pages 3011--3063. PMLR, 2024.

\bibitem[Lang(2012)]{lang2012differential}
S.~Lang.
\newblock \emph{{Differential and Riemannian manifolds}}, volume 160.
\newblock Springer Science \& Business Media, 2012.

\bibitem[Lee(2012)]{lee_introduction_2012}
J.~M. Lee.
\newblock \emph{Introduction to {Smooth} {Manifolds}}.
\newblock Springer, New York, 2012.
\newblock ISBN 978-1-4419-9982-5.

\bibitem[Lee(2018)]{lee2018introduction}
J.~M. Lee.
\newblock \emph{{Introduction to Riemannian manifolds}}, volume~2.
\newblock Springer, 2018.

\bibitem[Lezcano-Casado(2020)]{lezcano2020curvature}
M.~Lezcano-Casado.
\newblock {Curvature-dependant global convergence rates for optimization on
  manifolds of bounded geometry}.
\newblock \emph{arXiv preprint arXiv:2008.02517}, 2020.

\bibitem[Li et~al.(2024{\natexlab{a}})Li, Wei, Chi, and Chen]{li2024sharp}
G.~Li, Y.~Wei, Y.~Chi, and Y.~Chen.
\newblock A sharp convergence theory for the probability flow odes of diffusion
  models.
\newblock \emph{arXiv preprint arXiv:2408.02320}, 2024{\natexlab{a}}.

\bibitem[Li and Erdogdu(2023)]{li2023riemannian}
M.~Li and M.~A. Erdogdu.
\newblock {Riemannian Langevin algorithm for solving semidefinite programs}.
\newblock \emph{Bernoulli}, 29\penalty0 (4):\penalty0 3093--3113, 2023.

\bibitem[Li et~al.(2025)Li, Di, and Gu]{li2024unified}
R.~Li, Q.~Di, and Q.~Gu.
\newblock Unified convergence analysis for score-based diffusion models with
  deterministic samplers.
\newblock In \emph{The Thirteenth International Conference on Learning
  Representations}, 2025.
\newblock URL \url{https://openreview.net/forum?id=HrdVqFSn1e}.

\bibitem[Li et~al.(2024{\natexlab{b}})Li, Yu, He, Shen, Li, Sun, and
  Lin]{li2024spd}
Y.~Li, Z.~Yu, G.~He, Y.~Shen, K.~Li, X.~Sun, and S.~Lin.
\newblock {SPD-DDPM: Denoising diffusion probabilistic models in the symmetric
  positive definite space}.
\newblock In \emph{Proceedings of the AAAI conference on artificial
  intelligence}, volume~38, pages 13709--13717, 2024{\natexlab{b}}.

\bibitem[Liu et~al.(2025)Liu, Hu, Chen, and Huang]{liu2025finiteode}
Y.~Liu, R.~Hu, Y.~Chen, and L.~Huang.
\newblock {Finite-Time Convergence Analysis of ODE-based Generative Models for
  Stochastic Interpolants}.
\newblock \emph{arXiv preprint arXiv:2508.07333}, 2025.

\bibitem[Lou et~al.(2023)Lou, Xu, Farris, and Ermon]{lou2023scaling}
A.~Lou, M.~Xu, A.~Farris, and S.~Ermon.
\newblock {Scaling Riemannian diffusion models}.
\newblock \emph{Advances in Neural Information Processing Systems},
  36:\penalty0 80291--80305, 2023.

\bibitem[Luo et~al.(2025)Luo, Wang, Wang, Shao, Lv, Wang, Wang, and
  Ma]{luo2025crystalflow}
X.~Luo, Z.~Wang, Q.~Wang, X.~Shao, J.~Lv, L.~Wang, Y.~Wang, and Y.~Ma.
\newblock Crystalflow: a flow-based generative model for crystalline materials.
\newblock \emph{Nature Communications}, 16\penalty0 (1):\penalty0 9267, 2025.

\bibitem[Marsden et~al.(2002)Marsden, Ratiu, and Abraham]{marsden2002manifolds}
J.~E. Marsden, T.~Ratiu, and R.~Abraham.
\newblock \emph{{Manifolds, Tensor Analysis, and Applications}}.
\newblock Applied Mathematical Sciences. Springer, 3rd edition, 2002.
\newblock ISBN 0-201-10168-S.

\bibitem[Mathieu and Nickel(2020)]{mathieu2020riemannian}
E.~Mathieu and M.~Nickel.
\newblock {Riemannian continuous normalizing flows}.
\newblock \emph{Advances in neural information processing systems},
  33:\penalty0 2503--2515, 2020.

\bibitem[Mena et~al.(2025)Mena, Kuchibhotla, and
  Wasserman]{mena2025statistical}
G.~Mena, A.~K. Kuchibhotla, and L.~Wasserman.
\newblock Statistical properties of rectified flow.
\newblock \emph{arXiv preprint arXiv:2511.03193}, 2025.

\bibitem[Miller et~al.(2024)Miller, Chen, Sriram, and Wood]{miller2024flowmm}
B.~K. Miller, R.~T.~Q. Chen, A.~Sriram, and B.~M. Wood.
\newblock {F}low{MM}: Generating materials with {R}iemannian flow matching.
\newblock In \emph{Proceedings of the 41st International Conference on Machine
  Learning}, volume 235, pages 35664--35686, 2024.
\newblock URL \url{https://proceedings.mlr.press/v235/miller24a.html}.

\bibitem[Roy et~al.(2026)Roy, Rinaldo, and Sarkar]{roy2026low}
S.~Roy, A.~Rinaldo, and P.~Sarkar.
\newblock {Low-Dimensional Adaptation of Rectified Flow: A New Perspective
  through the Lens of Diffusion and Stochastic Localization}.
\newblock \emph{arXiv preprint arXiv:2601.15500}, 2026.

\bibitem[Song et~al.(2021)Song, Sohl-Dickstein, Kingma, Kumar, Ermon, and
  Poole]{song2020score}
Y.~Song, J.~Sohl-Dickstein, D.~P. Kingma, A.~Kumar, S.~Ermon, and B.~Poole.
\newblock Score-based generative modeling through stochastic differential
  equations.
\newblock In \emph{International Conference on Learning Representations}, 2021.
\newblock URL \url{https://openreview.net/forum?id=PxTIG12RRHS}.

\bibitem[Sriram et~al.(2024)Sriram, Miller, Chen, and Wood]{sriram2024flowllm}
A.~Sriram, B.~K. Miller, R.~T. Chen, and B.~M. Wood.
\newblock {FlowLLM: Flow matching for material generation with large language
  models as base distributions}.
\newblock \emph{Advances in Neural Information Processing Systems},
  37:\penalty0 46025--46046, 2024.

\bibitem[Su et~al.(2025)Su, Hu, Pi, and Liu]{su2025flow}
M.~Su, J.~Y.-C. Hu, S.~Pi, and H.~Liu.
\newblock On flow matching kl divergence.
\newblock \emph{arXiv preprint arXiv:2511.05480}, 2025.

\bibitem[Villani(2008)]{villani2008optimal}
C.~Villani.
\newblock \emph{{Optimal transport: old and new}}, volume 338.
\newblock Springer, 2008.

\bibitem[Wan et~al.(2025)Wan, Wang, Mishne, and Wang]{wanelucidating}
Z.~Wan, Q.~Wang, G.~Mishne, and Y.~Wang.
\newblock {Elucidating Flow Matching {ODE} Dynamics via Data Geometry and
  Denoisers}.
\newblock In \emph{Forty-second International Conference on Machine Learning},
  2025.
\newblock URL \url{https://openreview.net/forum?id=f5czhqYK3H}.

\bibitem[Wu et~al.(2025)Wu, Chen, Zhou, Meng, Zhu, and Ma]{wu2025riemannian}
J.~Wu, B.~Chen, Y.~Zhou, Q.~Meng, R.~Zhu, and Z.-M. Ma.
\newblock {Riemannian Neural Geodesic Interpolant}.
\newblock \emph{arXiv preprint arXiv:2504.15736}, 2025.

\bibitem[Xu et~al.(2026)Xu, Zhang, Nakahira, Qu, and Chi]{xu2026polynomial}
X.~Xu, Z.~Zhang, Y.~Nakahira, G.~Qu, and Y.~Chi.
\newblock {Polynomial Convergence of Riemannian Diffusion Models}.
\newblock \emph{arXiv preprint arXiv:2601.02499}, 2026.

\bibitem[Yarotsky(2017)]{yarotsky2017error}
D.~Yarotsky.
\newblock {Error bounds for approximations with deep ReLU networks}.
\newblock \emph{Neural networks}, 94:\penalty0 103--114, 2017.

\bibitem[Yue et~al.(2025)Yue, Wang, and Xu]{yue2025reqflow}
A.~Yue, Z.~Wang, and H.~Xu.
\newblock Re{QF}low: Rectified quaternion flow for efficient and high-quality
  protein backbone generation.
\newblock In \emph{Forty-second International Conference on Machine Learning},
  2025.
\newblock URL \url{https://openreview.net/forum?id=f375uEmYDf}.

\bibitem[Zhou and Liu(2025)]{zhouerror}
Z.~Zhou and W.~Liu.
\newblock {An Error Analysis of Flow Matching for Deep Generative Modeling}.
\newblock In \emph{Forty-second International Conference on Machine Learning},
  2025.
\newblock URL \url{https://openreview.net/forum?id=vES22INUKm}.

\end{thebibliography}
